\documentclass[twoside]{article}

\usepackage[accepted]{aistats2020}

\usepackage{ascmac}  
\usepackage[ruled]{algorithm2e}
\usepackage{amsmath}
\usepackage{amssymb}
\usepackage{amsthm}
\usepackage{bm}
\usepackage{booktabs}
\usepackage[font=footnotesize,labelfont=bf]{caption}
\usepackage{chngcntr}
\usepackage{color}
\usepackage{dsfont}
\usepackage{eucal}
\usepackage[T1]{fontenc}
\usepackage{hyperref}
\usepackage{mathtools}
\usepackage{mathrsfs}
\usepackage{multirow}
\usepackage{pifont}
\usepackage{rotating}
\usepackage{subcaption}
\usepackage{tabularx}
\usepackage{tikz}
\usepackage{wrapfig}

\hypersetup{
  colorlinks=true,
  linkcolor=red,
  citecolor=blue
}

\newtheorem{theorem}{Theorem}
\newtheorem*{theorem*}{Theorem}
\newtheorem*{proof*}{proof}
\newtheorem{definition}[theorem]{Definition}
\newtheorem*{definition*}{Definition}

\newtheorem*{assumption*}{Assumption}
\newtheorem{lemma}[theorem]{Lemma}
\newtheorem*{lemma*}{Lemma}
\newtheorem{proposition}[theorem]{Proposition}
\newtheorem*{proposition*}{Proposition}

\newtheorem*{corollary*}{Corollary}

\DontPrintSemicolon
\SetArgSty{textnormal}
\SetKwInOut{Input}{Input}
\SetKwInOut{Output}{Output}
\SetKwComment{Comment}{$\triangleright$\ }{}
\SetKwRepeat{Do}{do}{while}

\renewcommand{\tilde}{\widetilde}
\renewcommand{\hat}{\widehat}
\def\E{\mathbb{E}}
\def\P{\mathbb{P}}
\def\R{\mathbb{R}}

\def\Z{\mathbb{Z}}

\def\calA{\mathcal{A}}
\def\calB{\mathcal{B}}
\def\calC{\mathcal{C}}
\def\calE{\mathcal{E}}
\def\calF{\mathcal{F}}
\def\calG{\mathcal{G}}

\def\calN{\mathcal{N}}
\def\calR{\mathcal{R}}
\def\calS{\mathcal{S}}
\def\calU{\mathcal{U}}
\def\calV{\mathcal{V}}

\def\calX{\mathcal{X}}
\def\calY{\mathcal{Y}}
\def\calZ{\mathcal{Z}}
\newcommand{\cmark}{\ding{51}}
\newcommand{\xmark}{\ding{55}}
\newcommand{\one}{\mbox{(i)}}
\newcommand{\two}{\mbox{(i\hspace{-.1em}i)}}
\newcommand{\three}{\mbox{(i\hspace{-.1em}i\hspace{-.1em}i)}}

\newcommand{\sign}{\mathrm{sgn}}

\DeclareMathOperator*{\argmax}{arg\,max}

\newcommand{\ind}[1]{\mathds{1}_{\{#1\}}}

\renewcommand{\d}{\mathrm{d}}
\newcommand{\dx}{\d{x}}

\def\tp{\mathsf{TP}}
\def\fn{\mathsf{FN}}
\def\fp{\mathsf{FP}}
\def\tn{\mathsf{TN}}

\usepackage[round]{natbib}

\begin{document}

\twocolumn[

\runningtitle{Calibrated Surrogate Maximization of Linear-fractional Utility in Binary Classification}

\aistatstitle{Calibrated Surrogate Maximization of Linear-fractional Utility \\ in Binary Classification}

\aistatsauthor{
  Han Bao \And
  Masashi Sugiyama
}

\aistatsaddress{
  The University of Tokyo \\ RIKEN AIP \\
  \href{mailto:tsutsumi@ms.k.u-tokyo.ac.jp}{\texttt{tsutsumi@ms.k.u-tokyo.ac.jp}} \And
  RIKEN AIP \\ The University of Tokyo \\
  \href{mailto:sugi@k.u-tokyo.ac.jp}{\texttt{sugi@k.u-tokyo.ac.jp}}
}
]

\begin{abstract}
  Complex classification performance metrics such as the F${}_\beta$-measure and Jaccard index are often used,
  in order to handle class-imbalanced cases such as information retrieval and image segmentation.
  These performance metrics are not decomposable, that is, they cannot be expressed in a per-example manner,
  which hinders a straightforward application of M-estimation widely used in supervised learning.
  In this paper, we consider \emph{linear-fractional metrics}, which are a family of classification performance metrics that encompasses many standard ones such as the F${}_\beta$-measure and Jaccard index,
  and propose methods to directly maximize performances under those metrics.
  A clue to tackle their direct optimization is a \emph{calibrated surrogate utility},
  which is a tractable lower bound of the true utility function representing a given metric.
  We characterize sufficient conditions which make the surrogate maximization coincide with the maximization of the true utility.
  Simulation results on benchmark datasets validate the effectiveness of our calibrated surrogate maximization
  especially if the sample sizes are extremely small.
\end{abstract}

\section{Introduction}
\label{sec:introduction}

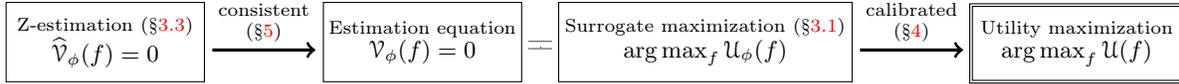
\begin{figure*}[t]
    \centering
    \scriptsize
    \begin{tikzpicture}
        \draw (0pt, 0pt) rectangle (75pt, 30pt)
            node [midway, align=center] {
                Z-estimation (\S\ref{sec:gradient-direction}) \\
                \footnotesize
                $\hat\calV_\phi(f) = 0$
            };
        \draw (120pt, 0pt) rectangle (195pt, 30pt)
            node [midway, align=center] {
                Estimation equation \\
                \footnotesize
                $\calV_\phi(f) = 0$
            };
        \draw (209pt, 0pt) rectangle (320pt, 30pt)
            node [midway, align=center] {
                Surrogate maximization (\S\ref{sec:surrogate-utility}) \\
                \footnotesize
                $\argmax_f \calU_\phi(f)$
            };
        \draw [double] (365pt, 0pt) rectangle (445pt, 29pt)
            node [midway, align=center] {
                Utility maximization \\
                \footnotesize
                $\argmax_f \calU(f)$
            };

        \draw [->, very thick] (78pt, 12pt) -- (117pt, 12pt)
            node [above, midway, align=center] {
                consistent \\ (\S\ref{sec:consistency-analysis})
            };

        \draw (202pt, 13pt) node {\Large =} -- (202pt, 13pt);

        \draw [->, very thick] (323pt, 12pt) -- (362pt, 12pt)
            node [above, midway, align=center] {
                calibrated \\ (\S\ref{sec:calibration})
            };
    \end{tikzpicture}
    \caption{
        Overview of this work.
        Intuitively, we can obtain the utility maximizer by solving $\hat\calV_\phi(f) = 0$.
    }
    \label{fig:overview}
\end{figure*}

Binary classification, one of the main focuses in machine learning, is a problem to predict binary responses for input covariates.
Classifiers are usually evaluated by the \emph{classification accuracy},
which is the expected proportion of correct predictions.
Since the accuracy cannot evaluate classifiers appropriately
under class imbalance~\citep{Menon:2013} or in the presence of label noises~\citep{Menon:2015},
alternative performance metrics have been employed
such as the F${}_\beta$-measure~\citep{Rijsbergen:1974,Joachims:2005,Nan:2012,Koyejo:2014}, Jaccard index~\citep{Koyejo:2014,Berman:2018}, and balanced error rate (BER)~\citep{Brodersen:2010,Menon:2013,Menon:2015,Charoenphakdee:2019}.
Once a performance metric is given, it is a natural strategy to optimize the performance of classifiers directly under the given performance metric.
However, the alternative performance metrics have difficulty in direct optimization in general,
because they are non-decomposable, for which per-example loss decomposition is unavailable.
In other words, the M-estimation procedure~\citep{vdG:2000} cannot be applied,
which makes the optimization of non-decomposable metrics hard.

One of the earliest works tackling the non-traditional metrics~\citep{Koyejo:2014} generalized performance metrics into the linear-fractional metrics,
which are the linear-fractional form of entries in the confusion matrix,
and encompasses the BER, F${}_\beta$-measure, Jaccard index, Gower-Legendre index~\citep{Gower:1986,Natarajan:2016}, and weighted accuracy~\citep{Koyejo:2014}.
\citet{Koyejo:2014} formulated the optimization problem in two ways.
One is a plug-in rule~\citep{Koyejo:2014,Narasimhan:2014,Yan:2018} to estimate the class-posterior probability and its optimal threshold,
and the other is an iterative weighted empirical risk minimization approach~\citep{Koyejo:2014,Parambath:2014} to find a better cost with which the minimizer of the cost-sensitive risk~\citep{Scott:2012} achieves higher utilities.
Although they provide statistically consistent esitmators,
the former suffers from high sample complexity due to the class-posterior probability estimation,
while the latter is computationally demanding because of iterative classifier training.

Our goal is to seek for computationally more efficient procedures to directly optimize the linear-fractional metrics,
without sacrificing the statistical consistency.
Specifically, we provide a novel calibrated surrogate utility which is a tractable lower bound
of the true utility representing the metric of our interest.
The surrogate maximization is formulated as the combination of concave and quasi-concave programs,
which can be optimized efficiently.
Then, we derive sufficient conditions on the surrogate calibration,
under which the surrogate maximization implies the maximization of the true utility.
In addition, we show the consistency of the empirical estimation procedure based on the theory of Z-estimation~\citep{vdV:2000}.
An overview of our proposed method is illustrated in Fig.~\ref{fig:overview}.

\textbf{Contributions: }
\one~\emph{Surrogate calibration} (Sec.~\ref{sec:calibration}):
We propose a tractable lower bound of the linear-fractional metrics with calibration conditions,
which guarantee that the surrogate maximization implies the maximization of the true utility.
This approach is model-agnostic differently from many previous approaches~\citep{Koyejo:2014,Narasimhan:2014,Narasimhan:2015,Yan:2018}.
\two~\emph{Efficient gradient-based optimization} (Secs.~\ref{sec:hybrid-optimization} and~\ref{sec:gradient-direction}):
The surrogate utility has affinity with gradient-based optimization
because of its non-vanishing gradient and an unbiased estimator of the gradient direction.
Even though the linear-fractional objective does not admit concavity in general,
our proposed algorithm is a two-step approach combining concave and quasi-concave programs
and hence computationally efficient.
\three~\emph{Consistency analysis} (Sec.~\ref{sec:consistency-analysis}):
The estimator obtained via the surrogate maximization with a finite sample is shown to be consistent
to the maximizer of the expected utility.
\section{Preliminaries}
\label{sec:background}

Throughout this work, we focus on binary classification.
Let $[n] \doteq \{1, \dots, n\}$.
Let $\ind{A} \doteq 1$ if the predicate $A$ holds and $0$ otherwise.
Let $\calX \subset \R^d$ be a feature space and $\calY = \{\pm 1\}$ be the label space.
We assume that a sample $\calS \doteq \{(x_i, y_i)\}_{i=1}^n \subset \calX \times \calY$ independently follows the joint distribution $\P$ with a density $p$.
We often split $\calS$ into two independent samples $\calS_0 = \{(x_i, y_i)\}_{i=1}^{m}$ and $\calS_1 = \{(x_i, y_i)\}_{i=m+1}^n$.
Usually, $m = \lfloor \frac{n}{2} \rfloor$.
For a function $h: \calX \times \calY \to \R$,
we write $\E[h(X,Y)] = \int_{\calX \times \calY} h(X, Y) \d\P$.
An expectation with respect to $X$ is written as $\E_X[h(X)] \doteq \int_\calX h(X) \d\P_\calX$ for a function $h: \calX \to \R$,
where $\P_\calX$ denotes the $\calX$-marginal distribution.
A classifier is given as a function $f: \calX \to \R$,
where $\sign(f(\cdot))$ determines predictions.
Here we adopt the convention $\sign(0) = -1$.
Let $\calF \subset \R^\calX$ be a hypothesis set of classifiers.
Let $\pi \doteq \P(Y = +1)$ and $\eta(X) \doteq \P(Y = +1 | X)$ be the class-prior/-posterior probabilities of $Y = +1$, respectively.
The 0/1-loss is denoted as $\ell(t) \doteq \ind{t \leq 0}$,
while $\phi: \R \to \R_{\geq 0}$ denotes a surrogate loss.
The norm $\|\cdot\|$ without a subscript is the $\mathbb{L}^2$-norm.
For a set $\calA \subset \calF$, denote $\calA^c$ as the complementary set of $\calA$, namely, $\calA^c \doteq \calF \setminus \calA$.

The following four quantities are focal targets in binary classification:
the true positives ($\tp$), false negatives ($\fn$), false positives ($\fp$), and true negatives ($\tn$).
\begin{definition}[Confusion matrix]
    Given a classifier $f \in \calF$ and a distribution $\P$,
    its confusion matrix is defined as
    $C(f, \P) \doteq [\tp, \fn; \fp, \tn]$,
    where
    \begin{align*}
        \tp(f, \P) &\doteq \P(Y = +1, \sign(f(X)) = +1), \\
        \fn(f, \P) &\doteq \P(Y = +1, \sign(f(X)) = -1), \\
        \fp(f, \P) &\doteq \P(Y = -1, \sign(f(X)) = +1), \\
        \tn(f, \P) &\doteq \P(Y = -1, \sign(f(X)) = -1).
    \end{align*}
\end{definition}
$\fn$ and $\tp$ as well as $\tn$ and $\fp$ can be transformed to each other:
$\fn(f, \P) = \pi - \tp(f, \P)$ and $\tn(f, \P) = (1 - \pi) - \fp(f, \P)$.
They can be expressed with $\ell$ and $\eta$, such as $\tp(f, \P) = \E_X[\ell(-f(X))\eta(X)]$.
The goal of binary classification is to obtain a classifier that ``maximizes'' $\tp$ and $\tn$,
while keeping $\fp$ and $\fn$ as ``low'' as possible.
Classifiers are evaluated by performance metrics that trade off those four quantities.
Performance metrics need to be chosen based on user's preference on the confusion matrix~\citep{Sokolova:2009,Menon:2015}.
In this work, we focus on the following family of utilities
representing the linear-fractional metrics.
\begin{definition}[Linear-fractional utility\footnote{
    As mentioned by \citet{Dembczynski:2017}, there is a dichotomy in the definition of performance metrics: population utility (PU) and expected test utility (ETU).
    We adopt PU, which is defined as the linear-fractional transform of the population confusion matrix in this context.
    This is convenient to avoid estimating $\eta$ compared to ETU.
}]
    \label{def:metric}
    A \emph{linear-fractional utility} $\calU: \calF \to [0, 1]$ is defined as
    \begin{align}
        \calU(f) \doteq
        \frac{\E_X[W_{0}(f(X), \eta(X))]}{\E_X[W_{1}(f(X), \eta(X))]}
        ,
        \label{eq:generalized-metric}
    \end{align}
    where $W_{0}, W_{1}: \R \times [0, 1] \to \R$ are class-conditional score functions given as
    \begin{align*}
        W_{k}(\xi, q) \doteq a_{k,+1}\ell(-\xi)q + a_{k,-1}\ell(-\xi)(1 - q) + b_k,
    \end{align*}
    and $a_{0,+1} > 0, a_{0,-1} \leq 0, b_0 \in \R, a_{1,+1} \geq 0, a_{1,-1} \geq 0, b_1 \in \R$ are constants
    such that $0 \leq \calU(f) \leq 1$ ($\forall f$).
\end{definition}
The class-conditional score functions correspond to a linear-transformation of $\tp$ and $\fp$:
$\E_X[W_{k}(f(X), \eta(X))] = a_{k,+1}\tp(f, \P) + a_{k,-1}\fp(f, \P) + b_k$.
Examples of $\calU$ are shown in Tab.~\ref{tab:generalized-performance-metrics}.
\begin{table*}[t]
    \centering
    \caption{
        Examples of the linear-fractional performance metrics.
        $\beta > 0$ is a trade-off parameter for the F${}_\beta$-measure,
        while $\alpha \geq 0$ is for the Gower-Legendre index.
    }
    \label{tab:generalized-performance-metrics}
    \begin{tabular}{|c||c|c|c|} \hline
        \renewcommand{\arraystretch}{1.3}
        Metric & \shortstack{F$_\beta$-measure \\ \citep{Rijsbergen:1974}} & \shortstack{Jaccard index \\ \citep{Jaccard:1901}} & \shortstack{Gower-Legendre index \\ \citep{Gower:1986}}
        \\ \hline \hline
        \renewcommand{\arraystretch}{1.2}
        Definition &
        $\frac{(1 + \beta^2)\tp}{(1 + \beta^2)\tp + \beta^2\fn + \fp}$ &
        $\frac{\tp}{\tp + \fn + \fp}$ &
        $\frac{\tp + \tn}{\tp + \alpha(\fp + \fn) + \tn}$
        \\ \hline
        $(a_{0,+1}, a_{0,-1})$ &
        $(1 + \beta^2, 0)$ &
        $(1, 0)$ &
        $(1, -1)$
        \\
        $b_0$ &
        $0$ &
        $0$ &
        $1 - \pi$
        \\
        $(a_{1,+1}, a_{1,-1})$ &
        $(1, 1)$ &
        $(0, 1)$ &
        $(1 - \alpha, \alpha - 1)$
        \\
        $b_1$ &
        $\beta^2\pi$ &
        $\pi$ &
        $1 + (\alpha - 1)\pi$
        \\ \hline
    \end{tabular}
\end{table*}
Given a utility function $\calU$, our goal is to obtain a classifier $f^\dagger$ that maximizes $\calU$.
\begin{align}
    f^\dagger = \argmax_{f \in \calF} \calU(f).
    \label{eq:utility-maximization}
\end{align}

\textbf{Traditional Supervised Classification:}
Here, we briefly review a traditional procedure for supervised classification~\citep{Vapnik:1998}.
Our aim is to obtain a classifier with high accuracy,
which corresponds to minimizing the classification risk $\calR(f) \doteq \E[\ell(Yf(X))]$.
Since optimizing the 0/1-loss $\ell$ is a computationally infeasible problem~\citep{Ben-David:2003,Feldman:2012},
it is a common practice to instead minimize a surrogate risk $\calR_\phi(f) \doteq \E[\phi(Yf(X))]$,
where $\phi: \R \to \R_{\geq 0}$ is a surrogate loss.
If $\phi$ is a classification-calibrated loss~\citep{Bartlett:2006}, it is known that minimizing $\calR_\phi$ corresponds to minimizing $\calR$.
Eventually, what we actually minimize is the empirical (surrogate) risk $\hat{\calR}_\phi(f) \doteq \frac{1}{n}\sum_{i=1}^n\phi(y_if(x_i))$.
The empirical risk $\hat{\calR}_\phi(f)$ is an unbiased estimator of the true risk $\calR_\phi(f)$ for a fixed $f \in \calF$,
and the uniform law of large numbers guarantees that $\hat{\calR}_\phi(f)$ converges to $\calR_\phi(f)$ for any $f \in \calF$ in probability~\citep{Vapnik:1998,vdG:2000,Mohri:2012}.
This strategy to minimize $\hat{\calR}_\phi$ is called empirical risk minimization (ERM).

The traditional ERM is devoted to maximizing the accuracy,
which is not necessarily suitable when another metric is used for evaluation.
Our aim is to give an alternative procedure to maximize $\calU$ directly as in Eq.~\eqref{eq:utility-maximization}.
In the next section, we introduce a tractable counterpart of the true utility $\calU$
because $\calU$ contains the 0/1-loss $\ell$ and is intractable as $\calR_\phi$ above.
\section{Surrogate Utility and Optimization}
\label{sec:method}

The true utility in Eq.~\eqref{eq:generalized-metric} consists of the 0/1-loss $\ell$,
which is difficult to optimize.
In this section, we introduce a surrogate utility in order to make the optimization problem in Eq.~\eqref{eq:utility-maximization} easier.

\subsection{Lower Bounding True Utility}
\label{sec:surrogate-utility}

Assume that we are given a surrogate loss $\phi: \R \to \R_{\geq 0}$.
We hope that the surrogate utility should lower-bound the true utility $\calU$,
and that $\tp$ / $\fp$ should become larger / smaller as a result of optimization, respectively.
We realize these ideas by constructing surrogate class-conditional score functions $W_{0,\phi}$ and $W_{{}1,\phi}$ as follows:
\begin{align}
    \begin{aligned}
        &W_{0,\phi}(\xi, q) \\ &\;\; \doteq a_{0,+1}(1 - \phi(\xi))q + a_{0,-1}\phi(-\xi)(1 - q) + b_0, \\
        &W_{1,\phi}(\xi, q) \\ &\;\; \doteq a_{1,+1}(1 + \phi(\xi))q + a_{1,-1}\phi(-\xi)(1 - q) + b_1.
    \end{aligned}
    \label{eq:surrogate-scores}
\end{align}
We often abbreviate $\E_X[W_{k,\phi}(f(X), \eta(X))]$ as $\E[W_{{}k,\phi}]$ if it is clear from the context.
Given the surrogate class-conditional scores, define the surrogate utility as follows.
\begin{align}
    \calU_\phi(f) \doteq \frac{\E_X[W_{0,\phi}(f(X), \eta(X))]}{\E_X[W_{{}1,\phi}(f(X), \eta(X))]}
    = \frac{\E[W_{0,\phi}]}{\E[W_{1,\phi}]}.
    \label{eq:surrogate-utility}
\end{align}

To construct $\calU_\phi$, the 0/1-losses appearing in the true utility $\calU$ are replaced with the surrogate loss $\phi$.
The surrogate class-conditional scores in~\eqref{eq:surrogate-scores} are designed
so that the surrogate utility in~\eqref{eq:surrogate-utility} bounds $\calU$ from below.
\begin{lemma}
    \label{lem:utility-lower-bound}
    For all $f$ and a surrogate loss $\phi: \R \to \R_{\geq 0}$ such that $\phi(t) \geq \ell(t)$ for all $t \in \R$,
    $\calU_\phi(f) \leq \calU(f)$.
\end{lemma}
\begin{proof}
    Fix $\xi \in \R$ and $q \in [0, 1]$.
    Since $\ell(-\xi) = 1 - \ell(\xi)$,
    $a_{0,+1}\ell(-\xi) = a_{0,+1}(1 - \ell(\xi)) \geq a_{0,+1}(1 - \phi(\xi))$ ($\because$ $a_{0,+1} \geq 0$).
    Together with $a_{0,-1}\ell(-\xi) \geq a_{0,-1}\phi(-\xi)$ ($\because$ $a_{0,-1} \leq 0$), we confirm $W_0(\xi,q) \geq W_{0,\phi}(\xi,q)$.
    It can be confirmed that $W_1(\xi,q) \leq W_{1,\phi}(\xi,q)$ as well.
    Hence, $\calU(f) \geq \calU_\phi(f)$ is easy to see.
\end{proof}
Due to this property, maximizing $\calU_\phi$ is at least maximizing a lower bound of $\calU$.
We will discuss the goodness of this lower bound in Sec.~\ref{sec:calibration},
but we can immediately see $\calU_\phi(f) (\leq \calU(f)) \leq 1$ for any $f$.
In the rest of this paper, we assume that $\calU_\phi$ is Fr{\'e}chet differentiable.

\subsection{Tractability of Surrogate Utility}
\label{sec:hybrid-optimization}

The surrogate utility $\calU_\phi$ comes to have a non-vanishing gradient by using a smooth $\phi$,
and is guaranteed to be a lower bound of $\calU$.
In this subsection, we discuss how it can be maximized efficiently.

Let us consider an empirical estimator of $\calU_\phi$:
\begin{align}
    \hat\calU_\phi(f) = \frac{\frac{1}{m}\sum_{i=1}^m \tilde{W}_{0,\phi}(f(x_i), y_i)}{\frac{1}{n-m}\sum_{i=m+1}^n \tilde{W}_{1,\phi}(f(x_i), y_i)},
    \label{eq:empirical-utility}
\end{align}
where
\begin{align*}
    \tilde{W}_{0,\phi}(\xi, y)
    &\doteq \begin{cases}
        a_{0,+1}(1 - \phi(\xi)) + b_0 & \text{if $y = +1$}, \\
        a_{0,-1}\phi(-\xi) + b_0 & \text{if $y = -1$},
    \end{cases}
    \\
    \tilde{W}_{1,\phi}(\xi, y)
    &\doteq \begin{cases}
        a_{1,+1}(1 + \phi(\xi)) + b_1 & \text{if $y = +1$}, \\
        a_{1,-1}\phi(-\xi) + b_1 & \text{if $y = -1$}.
    \end{cases}
\end{align*}
A global maximizer of $\hat\calU_\phi$ could be efficiently obtained if $\hat\calU_\phi$ were concave.
However, this is hard to achieve in our case regardless of the choice of $\phi$ due to its fractional form.
Nonetheless, we may formulate our optimization problem as a \emph{quasi-concave} program under a certain condition.
First, we introduce the notion of quasi-concavity.
\begin{definition}[Quasi-concavity~\citep{Boyd:2004}]
    A function $h: A \to \R$ is said to be quasi-concave if the super-level set $\{x \in A \mid h(x) \geq \alpha\}$ is a convex set for $\forall \alpha \in \R$.
\end{definition}
A quasi-concave function is a generalization of a concave function and has the unimodality though it is not necessarily concave,
which ensures the uniqueness of the solution.
Next, we have the following result, whose proof is given in App.~\ref{sec:supp:quasi-concavity}.
Let $\hat\calU_\phi^\mathrm{n}(f) \doteq \frac{1}{m}\sum_{i=1}^m \tilde{W}_{0,\phi}(f(x_i), y_i)$ be the numerator of $\hat\calU_\phi$.
\begin{lemma}
    \label{lem:quasi-concavity}
    Let $\bar{\mathcal{F}} \doteq \{f \mid \hat\calU_\phi^\mathrm{n}(f) \geq 0\} \subset \calF$.
    If $\phi$ is convex,
    $\hat\calU_\phi$ in Eq.~\eqref{eq:empirical-utility} is quasi-concave over $\bar{\calF}$ and
    $\hat\calU_\phi^\mathrm{n}$ is concave over $\calF$.
\end{lemma}
From Lemma~\ref{lem:quasi-concavity}, we observe the following two important facts.
First, in the range of $f \not\in \bar{\calF}$, our objective $\hat\calU_\phi$ is generally neither concave nor quasi-concave,
but its numerator $\hat\calU_\phi^\mathrm{n}$ is concave.
Second, $\hat\calU_\phi$ is quasi-concave over $\bar{\calF}$.
These observations motivate us to employ Algorithm~\ref{alg:hybrid},
which first increases the numerator $\hat\calU_\phi^\mathrm{n}$ only to make it positive and then maximizes the fractional form $\hat\calU_\phi$.
Since the former is a concave program and the latter is a quasi-concave program within $\bar{\calF}$,
the entire optimization can be performed computationally efficiently.
For quasi-concave optimization, normalized gradient ascent (NGA)~\citep{Hazan:2015} is applied,
which is guaranteed to find a global solution of quasi-concave objectives.
The behavior of Algorithm~\ref{alg:hybrid} is illustrated in Figure~\ref{fig:hybrid}.

\begin{algorithm}[t]
    \caption{Hybrid Optimization Algorithm}
    \label{alg:hybrid}
    \Input{$\phi$ convex loss, $\theta$ initial classifier parameter}
    \While{$\hat\calU_\phi^\mathrm{n}(f_\theta) \leq 0$}{
        $g^\mathrm{n} \longleftarrow \nabla_{\theta} \hat\calU_\phi^\mathrm{n}(f_\theta)$\;
        $\theta \longleftarrow \texttt{gradient\_based\_update}(\theta, g^\mathrm{n})$\;
    }
    \While{stopping criterion is not satisfied}{
        $g \longleftarrow \nabla_{\theta} \hat\calU_\phi(f_\theta)$, $\hat g = g / \|g\|$\;
        $\theta \longleftarrow \texttt{gradient\_based\_update}(\theta, \hat g)$\;
    }
    \Output{maximizer $f_\theta$}
\end{algorithm}

\begin{figure}[t]
    \centering
    \begin{tikzpicture}[scale=1.4]
        \clip (-2,-1.6) rectangle (2,2.5);
        \draw (-2,1.0) -- (2,1.0);
        \draw (-2,-1.0) -- (2,-1.0);
        \draw [line width=0.4mm,blue] plot [smooth] coordinates {
            (-2,0.2) (-1.5,0.8) (-0.9,0.5) (-0.5,0.97)
            (-0.2,1.1) (0.2,2.0) (0.6,1.1)
            (1,0.96) (1.5,0.2) (2,0.1)
        };
        \draw [line width=0.4mm,blue] plot [smooth] coordinates {
            (-2,-2.0) (-0.5,-1.0) (0.5,-0.5) (1,-1.0) (1.5,-2.0)
        };
        \node [draw] at (-1.7,2.2) {$\hat\calU_\phi$};
        \node [draw] at (-1.7,-0.5) {$\hat\calU_\phi^\mathrm{n}$};
        \draw [dashed] (-0.5,-2) -- (-0.5,3);
        \draw [dashed] (1.0,-2) -- (1.0,3);
        \draw [<->, line width=0.30mm] (-0.48,2.15) -- node[above] {$\bar{\calF}$} (0.98,2.15);
        \draw [->, line width=0.30mm] (-2,1.65) -- node[above] {$\bar{\calF}^c$} (-0.52,1.65);
        \draw [<-, line width=0.30mm] (1.02,1.65) -- node[above] {$\bar{\calF}^c$} (2,1.65);
        \fill [red] (-1.27,-1.5) circle (1.3pt);
        \fill [red] (-0.5,-1.0) circle (1.3pt);
        \fill [red] (-0.5,1.0) circle (1.3pt);
        \fill [red] (0.2,2.0) circle (1.3pt);
        \draw [->,red,line width=0.27mm] plot [smooth] coordinates {
            (-1.37,-1.4) (-0.9,-1.08) (-0.6,-0.9)
        };
        \draw [->,red,line width=0.27mm] plot [smooth] coordinates {
            (-0.4,-0.85) (-0.4,0.9)
        };
        \draw [->,red,line width=0.27mm] plot [smooth] coordinates {
            (-0.45,1.1) (-0.3,1.2) (-0.05,1.85) (0.1,2.1)
        };
        \draw [red] (-1.25,-1.13) circle [radius=0.1] node [red] {1};
        \draw [red] (-0.22,0) circle [radius=0.1] node [red] {2};
        \draw [red] (-0.28,1.7) circle [radius=0.1] node [red] {3};
    \end{tikzpicture}
    \caption{
        Illustration of our hybrid optimization approach in Algorithm~\ref{alg:hybrid}.
        {\normalsize \textcircled{\small{1}}} maximize the numerator $\hat\calU_\phi^\mathrm{n}$ (concave),
        {\normalsize \textcircled{\small{2}}} once $\hat\calU_\phi^\mathrm{n}(f) \geq 0$, optimize the fraction $\hat\calU_\phi$,
        {\normalsize \textcircled{\small{3}}} maximize the fraction $\hat\calU_\phi$ (quasi-concave only in $\bar\calF$).
    }
    \label{fig:hybrid}
\end{figure}
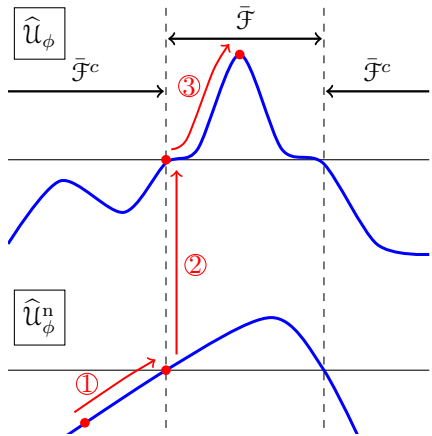

\subsection{Gradient Direction Estimator}
\label{sec:gradient-direction}

The empirical estimator $\hat\calU_\phi$ in Eq.~\eqref{eq:empirical-utility} is generally biased due to its fractional form.
Nevertheless, its gradient $\nabla_f\hat\calU_\phi$ is unbiased to the expected gradient $\nabla_f\calU_\phi$ up to a positive scalar multiple.
Hence, we may safely use $\nabla_f\hat\calU_\phi$ as the update direction in NGA.

Below, we state this idea formally.
Under the interchangeability of the expectation and derivative,
the gradient of the expected utility $\calU_\phi$ is expressed as
\begin{align*}
    &
    \nabla_f\calU_\phi(f) \\
    &=
    \underbrace{ \frac{1}{(\E[W_{1,\phi}])^2} }_\text{positive scalar}
    \underbrace{
        \E[W_{1,\phi}] \E[\nabla W_{0,\phi}]
        - \E[W_{0,\phi}] \E[\nabla W_{1,\phi}]
    }_\text{gradient direction ($\doteq \calV_\phi(f)$)}
    \\
    &= c \calV_\phi(f), \qquad \text{where} \; c = (\E[W_{1,\phi}])^{-2} > 0
    ,
\end{align*}
from which we can see that its gradient direction is parallel to $\calV_\phi$.
$\calV_\phi$ can be unbiasedly estimated.

\begin{lemma}
    \label{lem:unbiased-gradient}
    Denote $\tilde{W}_{0,\phi}(f(x_i), y_i) = \tilde{W}_{0,\phi}(z_i)$ for simplicity.
    Define
    \begin{align}
        \hat\calV_\phi(f)
        \doteq
        \frac{1}{m(n - m)} \sum_{i=1}^m & \sum_{j=m+1}^{n} \left\{
            \tilde{W}_{1,\phi}(z_j) \nabla_f\tilde{W}_{0,\phi}(z_i)
        \right. \nonumber \\ & \quad \left.
            -
            \tilde{W}_{0,\phi}(z_i) \nabla_f\tilde{W}_{1,\phi}(z_j)
        \right\}
        \label{eq:gradient-estimator}
        .
    \end{align}
    Then, we have $\calV_\phi(f) = \E_\calS[\hat\calV_\phi(f)]$.
\end{lemma}

\begin{algorithm}[t]
    \caption{Normalized Gradient Ascent}
    \label{alg:u-ga}
    \Input{$\theta$ initial classifier parameter, $\gamma$ learning rate}
    \While{stopping criterion is not satisfied}{
        $g \longleftarrow \hat\calV_\phi(f_\theta)$, $\hat g = g / \|g\|$\;
        $\theta \longleftarrow \theta + \gamma \hat g$\;
    }
    \Output{learned classifier parameter $\theta$}
\end{algorithm}

Lemma~\ref{lem:unbiased-gradient} can be confirmed by simple algebra,
noting that two samples $\calS_0$ and $\calS_1$ are independent and identically drawn from $\P$.
Since solving $\nabla\hat\calU_\phi(f) = 0$ is identical to solving $\hat\calV_\phi(f) = 0$,
gradient updates using $\nabla\hat\calU_\phi$ is aligned to the maximization of $\calU_\phi$.
Hence, optimization procedures that only need gradients such as gradient ascent and quasi-Newton methods~\citep{Boyd:2004}, e.g., the Broyden-Fletcher-Goldfarb-Shanno (BFGS) algorithm~\citep{Fletcher:2013} can be applied to maximize $\calU_\phi$.
Note that Algorithm~\ref{alg:u-ga} can be regarded as an extension of the traditional gradient ascent using $\hat\calV_\phi$.
We plug either Algorithm~\ref{alg:u-ga} or BFGS using the normalized gradient to the second half of Algorithm~\ref{alg:hybrid}.
\section{Calibration Analysis: Bridging Surrogate Utility and True Utility}
\label{sec:calibration}

\begin{figure}[t]
    \centering
    \includegraphics[width=\columnwidth]{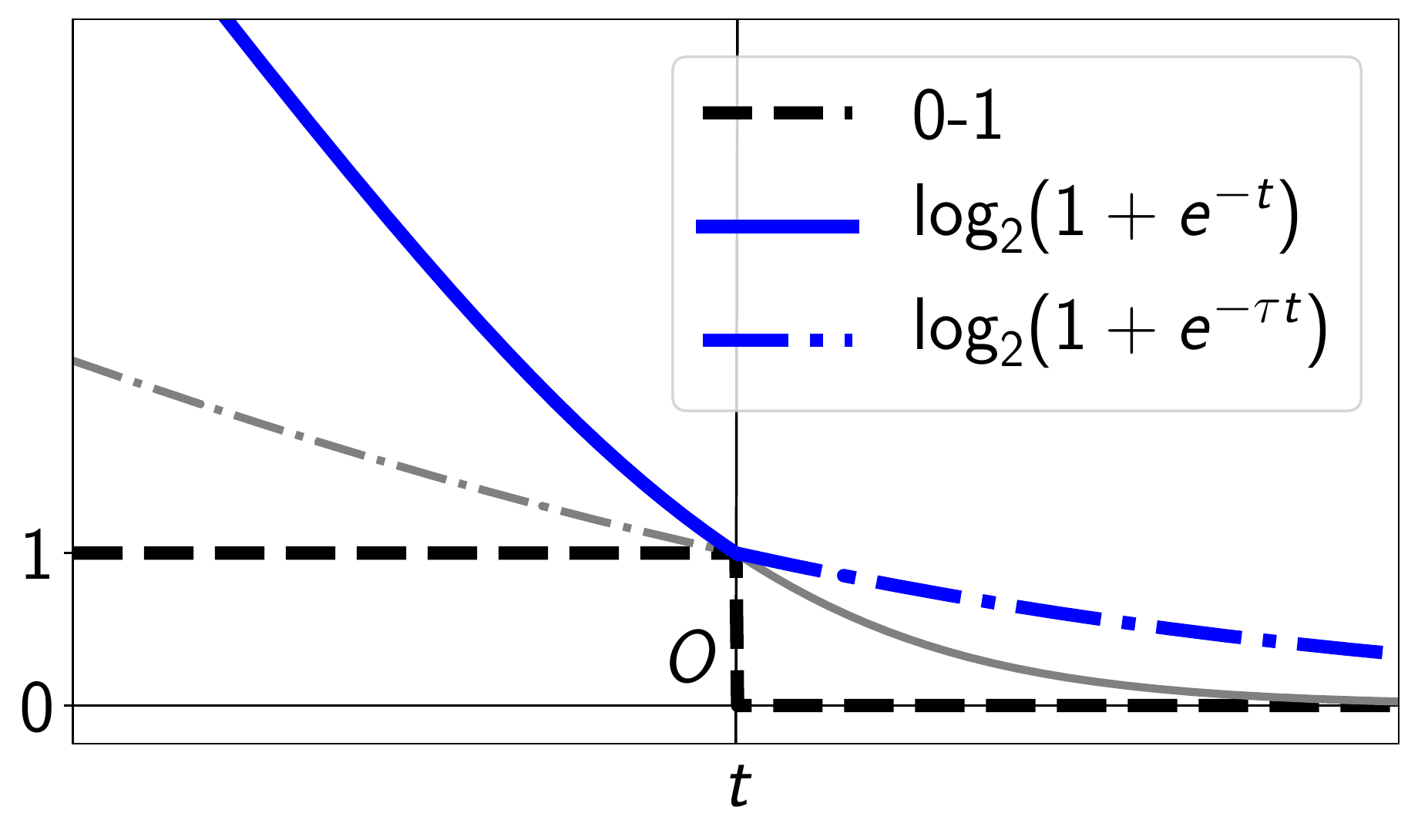}
    \caption{
        An example of $\tau$-discrepant loss with $\tau > 0$:
        $\phi(t) = \log_2(1 + e^{-t})$ for $t \leq 0$ and $\phi(t) = \log_2(1 + e^{-\tau t})$ for $t > 0$.
    }
    \label{fig:fisher-consistent-surrogate-loss}
\end{figure}

In Sec.~\ref{sec:method}, we formulated the tractable surrogate utility.
Given the surrogate utility $\calU_\phi$,
a natural question arises in the same way as the classification calibration in binary classification~\citep{Zhang:2004a,Bartlett:2006}:
\emph{Does maximizing the surrogate utility $\calU_\phi$ imply maximizing the true utility $\calU$?}
In this section, we study sufficient conditions on the surrogate loss $\phi$
in order to connect the maximization of $\calU_\phi$ and the maximization of $\calU$.
All proofs in this section are deferred to App.~\ref{sec:supp:calibration}.

First, we define the notion of \emph{$\calU$-calibration}.
\begin{definition}[$\calU$-calibration]
    \label{def:consistency}
    The surrogate utility $\calU_\phi$ is said to be $\calU$-calibrated
    if for any sequence of measurable functions $\{f_l\}_{l \geq 1}$ and any distribution $\P$,
    it holds that $\calU_\phi(f_l) \to \calU_\phi^* \Longrightarrow \calU(f_l) \to \calU^\dagger$ when $l \to \infty$,
    where $\calU_\phi^* \doteq \sup_f\calU_\phi(f)$ and $\calU^\dagger \doteq \sup_f\calU(f)$ are the suprema taken over all measurable functions.
\end{definition}
This definition is motivated by calibration in other learning problems
such as binary classification~\citep[Theorem 3]{Bartlett:2006}, multi-class classification~\citep[Theorem 3]{Zhang:2004b},
structured prediction~\citep[Theorem 2]{Osokin:2017}, and AUC optimization~\citep[Definition 1]{Gao:2015}.
If a surrogate utility is $\calU$-calibrated, we may safely optimize the surrogate utility instead of the true utility $\calU$.
Note that $\calU$-calibration is a concept to reduce the surrogate maximization to the maximization of $\calU$
\emph{within all measurable functions}.
The approximation error of $\calU_\phi$ is not the target of our analysis as in \citet{Bartlett:2006}.

Next, we give a property of loss functions that is needed to guarantee $\calU$-calibration.
\begin{definition}[$\tau$-discrepant loss]
    \label{def:discrepant-loss}
    For a fixed $\tau > 0$,
    a convex loss function $\phi: \R \to \R_{\geq 0}$ is said to be \emph{$\tau$-discrepant}
    if $\phi$ satisfies $\lim_{t \searrow 0}\phi'(t) \geq \tau \lim_{t \nearrow 0}\phi'(t)$.
\end{definition}
Intuitively, $\tau$-discrepancy means that the gradient of $\phi$ around the origin
is steeper in the negative domain than the positive domain (see Figure~\ref{fig:fisher-consistent-surrogate-loss}).
The value $\tau$ controls \emph{steepness} of the $\tp$ / $\fp$ surrogates appearing in the surrogate utility $\calU_\phi$.
Note that $\phi(\xi)$ and $\phi(-\xi)$ appearing in Eqs.~\eqref{eq:surrogate-scores} and~\eqref{eq:surrogate-utility} correspond to $\tp$ and $\fp$, respectively, by their constructions.

Below, we see calibration properties for specific linear-fractional metrics,
the F${}_\beta$-measure and Jaccard index.
Note that those calibration analyses can be extended to general linear-fractional utilities,
which is deferred to App.~\ref{sec:supp:general-calibration-analysis}.

\textbf{F${}_\beta$-measure:}
The F${}_\beta$-measure has been widely used especially in the field of information retrieval where relevant items are rare~\citep{Manning:2008}.
Since it is defined as the weighted harmonic mean of the precision and recall (see Tab.~\ref{tab:generalized-performance-metrics}),
its optimization is difficult in general.
Although much previous work has tried to directly optimize it in the context of the class-posterior probability estimation~\citep{Nan:2012,Koyejo:2014,Yan:2018}
or the iterative cost-sensitive learning~\citep{Koyejo:2014,Parambath:2014},
we show that there exists a calibrated surrogate utility that can be used in the direct optimization as well.

For the F${}_\beta$-measure $\frac{(1 + \beta^2)\tp}{(1 + \beta^2)\tp + \beta^2\fn + \fp} = \frac{(1 + \beta^2)\tp}{\tp + \fp + \beta^2\pi}$,
define the true utility $\calU^{\mathsf{F}_\beta}$ and the surrogate utility $\calU_\phi^{\mathsf{F}_\beta}$ as
\begin{align*}
    \calU^{\mathsf{F}_\beta}(f) &= \frac{
        \E_X\left[(1 + \beta^2)\ell(-f)\eta\right]
    }{
        \E_X\left[\ell(-f)\eta + \ell(-f)(1 - \eta) + \beta^2\pi\right]
    },
    \\
    \calU_\phi^{\mathsf{F}_\beta}(f) &= \frac{
        \E_X\left[(1 + \beta^2)(1 - \phi(f))\eta\right]
    }{
        \E_X\left[(1 + \phi(f))\eta + \phi(-f)(1 - \eta) + \beta^2\pi\right]
    }.
\end{align*}
As for $\calU_\phi^{\mathsf{F}_\beta}$, we have the following F${}_\beta$-calibration guarantee.
Denote $(\calU_\phi^{\mathsf{F}_\beta})^* \doteq \sup_f \calU_\phi^{\mathsf{F}_\beta}(f)$.
\begin{theorem}[F${}_\beta$-calibration]
    \label{thm:f-measure-calibration}
    Assume that a surrogate loss $\phi: \R \to \R_{\geq 0}$ is differentiable almost everywhere, convex, and non-increasing,
    and that $(\calU_\phi^{\mathsf{F}_\beta})^* \geq \frac{(1 + \beta^2)\tau}{\beta^2 - \tau}$ and $\phi$ is $\tau$-discrepant for some constant $\tau \in (0, \beta^2)$.\footnote{Note that $(\calU_\phi^{\mathsf{F}_\beta})^*$ is non-negative and therefore such $\tau$ always exists. The non-negativity is discussed in App.~\ref{sec:supp:non-negativity-of-surrogate-utilities}.}
    Then, $\calU_\phi^{\mathsf{F}_\beta}$ is F${}_\beta$-calibrated.
\end{theorem}

An example of the $\tau$-discrepant surrogate loss is shown in Figure~\ref{fig:fisher-consistent-surrogate-loss}.
Here $\tau$ is a discrepancy hyperparameter.
From the fact $(\calU_\phi^{\mathsf{F}_\beta})^* \leq 1$, $\tau$ ranges over $(0, \tfrac{\beta^2}{2 + \beta^2}]$.
We may determine $\tau$ by cross-validation, or fix it at $\tau = \tfrac{\beta^2}{2 + \beta^2}$ by assuming $(\calU_\phi^{\mathsf{F}_\beta})^* \approx 1$.

\textbf{Jaccard Index:}
The Jaccard index, also referred to as the \emph{intersection over union (IoU)}, is a metric of similarity between two sets:
For two sets $A$ and $B$, it is defined as $\frac{|A \cap B|}{|A \cup B|} \in [0, 1]$~\citep{Jaccard:1901}.
The Jaccard index between the sets of examples predicted as positives and labeled as positives
becomes $\frac{\tp}{\tp + \fn + \fp}$,
as is shown in Tab.~\ref{tab:generalized-performance-metrics}.
This measure is not only used for measuring the performance of binary classification~\citep{Koyejo:2014,Narasimhan:2015},
but also for semantic segmentation~\citep{Everingham:2010,Csurka:2013,Ahmed:2015,Berman:2018}.

For the Jaccard index $\frac{\tp}{\tp + \fn + \fp} = \frac{\tp}{\fp + \pi}$, define the true utility $\calU^{\mathsf{Jac}}$ and the surrogate utility $\calU_\phi^{\mathsf{Jac}}$ as
\begin{align*}
    \calU^{\mathsf{Jac}}(f) = \frac{
        \E_X[\ell(-f)\eta]
    }{
        \E_X[\ell(-f)(1 - \eta) + \pi]
    },
    \\
    \calU_\phi^{\mathsf{Jac}}(f) = \frac{
        \E_X[(1 - \phi(f))\eta]
    }{
        \E_X[\phi(-f)(1 - \eta) + \pi]
    }.
\end{align*}
As for $\calU_\phi^{\mathsf{Jac}}$, we have the following Jaccard-calibration.
Denote $(\calU_\phi^{\mathsf{Jac}})^* \doteq \sup_f \calU_\phi^{\mathsf{Jac}}(f)$.
\begin{theorem}[Jaccard-calibration]
    \label{thm:jaccard-calibration}
    Assume that a surrogate loss $\phi: \R \to \R_{\geq 0}$ is differentiable almost everywhere, convex, and non-increasing,
    and that $(\calU_\phi^{\mathsf{Jac}})^* \geq \tau$ and $\phi$ is $\tau$-discrepant for some constant $\tau \in (0, 1)$.
    Then, $\calU_\phi^{\mathsf{Jac}}$ is Jaccard-calibrated.
\end{theorem}

Theorem~\ref{thm:jaccard-calibration} also relies on the $\tau$-discrepancy as in Theorem~\ref{thm:f-measure-calibration}.
Thus, the loss shown in Figure~\ref{fig:fisher-consistent-surrogate-loss} can also be used in the Jaccard case with a certain range of $\tau$.
In the same manner as the F${}_\beta$-measure, a hyperparameter $\tau$ ranges over $(0, 1)$,
which we may either determine by cross-validation or fix to a certain value.

\textbf{Remark:}
The $\tau$-discrepancy is a technical assumption making stationary points of $\calU_\phi$ lie in the Bayes optimal set of $\calU$.
This is a mere sufficient condition for $\calU$-calibration,
while the classification-calibration~\citep{Bartlett:2006} is the necessary and sufficient condition for the accuracy.\footnote{We give the surrogate calibration conditions for the accuracy in App.~\ref{sec:supp:accuracy-calibration}.}
It is left as an open problem to seek for necessary conditions.
\section{Consistency Analysis: Bridging Finite Sample and Asymptotics}
\label{sec:consistency-analysis}

In this section, we analyze statistical properties of the estimator $\hat\calV_\phi$ in Eq.~\eqref{eq:gradient-estimator}.
To make our analysis simple, the linear-in-input model $f_\theta(x) = \theta^\top x$ is considered throughout this section,
where $\theta \in \Theta$ is a classifier parameter and $\Theta \subset \R^d$ is a compact parameter space.
The maximization procedure introduced above can be naturally seen as \emph{Z-estimation}~\citep{vdV:2000},
which is an estimation procedure to solve an estimation equation.
In our case, the maximization of $\calU_\phi$ is reduced to a Z-estimation problem to solve the system $\hat\calV_\phi(f) = 0$.
The first lemma shows that the derivative estimator $\hat\calV_\phi$ admits the uniform convergence.
Its proof is deferred to App.~\ref{sec:supp:proof-uniform-convergence}.
\begin{lemma}[Uniform convergence]
    \label{lem:uniform-convergence}
    For simplicity, assume that $m = n / 2$.
    For $k = 0, 1$, let $c_k \doteq \sup_{\xi \in \R, y \in \calY}|W_{k,\phi}(\xi, y)| < +\infty$.
    Assume that $W_{k}(\cdot, y)$ for $y \in \calY$ are $\rho_k$-Lipschitz continuous for some $0 < \rho_k < \infty$,
    and that $\|x\| < c_\calX$ ($\forall x \in \calX$) and $\|\theta\| < c_\Theta$ ($\forall \theta \in \Theta$) for some $0 < c_\calX, c_\Theta < \infty$.
    Then,
    \begin{align}
        \sup_{\theta \in \Theta} \left\|\hat\calV_\phi(f_\theta) - \calV_\phi(f_\theta)\right\| = \mathcal{O}_p(n^{-\frac{1}{2}}),
        \label{eq:uniform-convergence}
    \end{align}
    where $\mathcal{O}_p$ denotes the order in probability.
\end{lemma}

The Lipschitz continuity and smoothness assumptions in Lemma~\ref{lem:uniform-convergence} can be satisfied if the surrogate loss $\phi$ satisfies a certain Lipschitzness and smoothness.
Note that Lemma~\ref{lem:uniform-convergence} still holds for $\tau$-discrepant surrogates
since we allow surrogates to have different smoothness parameters for both positive and negative domains.
Lemma~\ref{lem:uniform-convergence} is the basis for showing the consistency.
Let $\theta^* \doteq \argmax_{\theta \in \Theta} \calU_\phi(f_\theta)$ and
$\hat\theta_n = \argmax_{\theta \in \Theta} \hat\calU_\phi(f_\theta)$.
Under the identifiability described below, $f_{\theta^*}$ and $f_{\hat\theta_n}$ are roots of $\calV_\phi$ and $\hat\calV_\phi$, respectively.
Then, we can show the consistency of $\hat\theta_n$.

\begin{theorem}[Consistency]
    \label{thm:consistency}
    Assume that $\theta^*$ is \emph{identifiable}, that is,
    $\inf\{ \|\calV_\phi(f_\theta)\| \mid \|\theta - \theta^*\| \geq \epsilon \} > \|\calV_\phi(f_{\theta^*})\| = 0$ for all $\epsilon > 0$,
    and that Eq.~\eqref{eq:uniform-convergence} holds for $\hat\calV_\phi$.
    Then, $\hat\theta_n \overset{p}{\to} \theta^*$.
\end{theorem}

Theorem~\ref{thm:consistency} is an immediate result of \citet[Theorem~5.9]{vdV:2000},
given the uniform convergence (Lemma~\ref{lem:uniform-convergence}) and the identifiability assumption.
Note that the identifiability assumes that $\calV_\phi$ has a unique zero $f_{\theta^*}$,
which is also usual in the M-estimation: The global optimizer is identifiable.
Since Algorithm~\ref{alg:hybrid} is a combination of concave and quasi-concave programs,
the identifiability would be reasonable to assume.
\section{Related Work}
\label{sec:related}

In this section, we summarize the existing lines of research on the optimization of generalized performance metrics,
which elucidates advantages of our approach.

\emph{\one~Surrogate optimization:}
One of the earliest attempts to optimize non-decomposable performance metrics dates back to \citet{Joachims:2005},
formulating the structured SVM as a surrogate objective.
However, \citet{Dembczynski:2013} showed that this surrogate is inconsistent,
which means that the surrogate maximization does not necessarily imply the maximization of the true metric.
\citet{Kar:2014} showed the sublinear regret for the structural surrogate by \citet{Joachims:2005} in online setting.
Later, \citet{Yu:2015}, \citet{Eban:2017}, and \citet{Berman:2018} have tried different surrogates,
but their calibration has not been studied yet.

\emph{\two~Plug-in rule:}
Instead of the surrogate optimization, \citet{Dembczynski:2013} mentioned that a plug-in rule is consistent,
where $\eta$ and a threshold parameter are estimated independently.
We can estimate $\eta$ by minimizing strictly proper losses~\citep{Reid:2009}.
The plug-in rule has been investigated in many settings~\citep{Nan:2012,Dembczynski:2013,Koyejo:2014,Narasimhan:2014,Busa-Fekete:2015,Yan:2018}.
However, one of the weaknesses of the plug-in rule is that it requires an accurate estimate of $\eta$,
which is less sample-efficient than the usual classification with convex surrogates~\citep{Bousquet:2004,Tsybakov:2008}.
Moreover, estimation of the threshold parameter heavily relies on an estimate of $\eta$.

\emph{\three~Cost-sensitive risk minimization:}
On the other hand, \citet{Parambath:2014} is a pioneering work to focus on the \emph{pseudo-linearity} of the metrics,
which reduces their maximization to an alternative optimization with respect to a classifier and the sublevel.
This can be formulated as an iterative cost-sensitive risk minimization~\citep{Koyejo:2014,Narasimhan:2015,Narasimhan:2016,Sanyal:2018}.
Though these methods are blessed with the consistency,
they need to train classifiers many times, which may lead to high computational costs,
especially for complex hypothesis sets.

\textbf{Remark:}
Our proposed methods can be considered to belong to the family \one,
while one of the crucial differences is the fact that we have calibration guarantee.
We do not need to estimate the class-posterior probability as in \two, or train classifiers many times as in \three.
This comparison is summarized in Tab.~\ref{tab:comparison}.

\begin{table}[t]
    \small
    \caption{Comparison of related work.}
    \begin{tabular}{c||c|c|c} \hline
        Method & Consistency & \shortstack{Avoids \\ to estimate $\eta$} & \shortstack{Efficient \\ optimization} \\ \hline \hline
        \textbf{ours} & \cmark & \cmark & \cmark \\ \hline
        \one & \xmark & \cmark & \cmark \\
        \two & \cmark & \xmark & \cmark \\
        \three & \cmark & \cmark & \xmark \\ \hline
    \end{tabular}
    \label{tab:comparison}
\end{table}
\section{Experiments}
\label{sec:experiments}

In this section, we investigate empirical performances of the surrogate optimizations (Algorithm~\ref{alg:hybrid} with NGA and normalized BFGS).
Details of datasets, baselines, and full experimental results are shown in App.~\ref{sec:supp:experiment}.

\begin{table*}[t]
    \centering
    \caption{
        Benchmark results:
        50 trials are conducted for each pair of a method and dataset.
        Standard errors (multiplied by $10^4$) are shown in parentheses.
        Bold-faces indicate outperforming methods, chosen by one-sided t-test with the significant level 5\%.
        We emphasize that the same number of gradient updates are executed for both U-GD and U-BFGS.
    }
    \label{tab:benchmark-result}
    \scalebox{0.9}{
    \begin{tabular}{cccccc} \hline
        (F${}_1$-measure) & \multicolumn{2}{c}{Proposed} & \multicolumn{3}{c}{Baselines} \\ \cmidrule(lr){2-3} \cmidrule(lr){4-6}
        Dataset & U-GD & U-BFGS & ERM & W-ERM & Plug-in
        \\ \hline
        adult & 0.617 (101)  & 0.660 (11)  & 0.639 (51)  & 0.676 (18)  & \textbf{0.681 (9)} \\
        breast-cancer & \textbf{0.963 (31)} & \textbf{0.960 (32)} & 0.950 (37)  & 0.948 (44)  & 0.953 (40)  \\
        diabetes & \textbf{0.834 (32)} & \textbf{0.828 (31)} & 0.817 (50)  & 0.821 (40)  & 0.820 (42)  \\
        sonar & \textbf{0.735 (95)} & \textbf{0.740 (91)} & 0.706 (121)  & 0.655 (189)  & \textbf{0.721 (113)} \\
        \hline
    \end{tabular}
    }
    \scalebox{0.9}{
    \begin{tabular}{cccccc} \hline
        (Jaccard index) & \multicolumn{2}{c}{Proposed} & \multicolumn{3}{c}{Baselines} \\ \cmidrule(lr){2-3} \cmidrule(lr){4-6}
        Dataset & U-GD & U-BFGS & ERM & W-ERM & Plug-in
        \\ \hline
        adult & 0.499 (44)  & 0.498 (11)  & 0.471 (51)  & 0.510 (20)  & \textbf{0.516 (10)} \\
        breast-cancer & \textbf{0.921 (54)} & \textbf{0.918 (55)} & 0.905 (66)  & 0.903 (78)  & \textbf{0.913 (69)} \\
        diabetes & \textbf{0.714 (44)} & 0.702 (50)  & 0.692 (70)  & 0.698 (56)  & 0.695 (60)  \\
        sonar & \textbf{0.600 (125)} & \textbf{0.600 (111)} & 0.552 (147)  & 0.495 (202)  & \textbf{0.572 (134)} \\
        \hline
    \end{tabular}
    }
\end{table*}

\textbf{Implementation Details of Proposed Methods:}
The linear-in-input model $f_\theta(x) = \theta^\top x$ was used for the hypothesis set $\calF$.
As the initializer of $\theta$, the ERM minimizer trained by SVM was used.
For both NGA and BFGS, gradient updates were iterated 300 times.
NGA and normalized BFGS are referred to as U-GD and U-BFGS below, respectively.
The surrogate loss shown in Fig.~\ref{fig:fisher-consistent-surrogate-loss} was used:
$\phi(m) = \log_2(1+e^{-m})$ when $m \leq 0$ and $\phi(m) = \log_2(1+e^{-\tau m})$ when $m > 0$,
where $\tau$ was set to $0.33$ in the F${}_1$-measure case and $0.75$ in the Jaccard index case.\footnote{
    The discrepancy parameter $\tau$ should be chosen within $(0, \frac{1}{3})$ and $(0, 1)$ for the F${}_1$-measure and Jaccard index, respectively.
    Here, we fix them to the slightly small values than the upper limits of their ranges.
    In App.~\ref{sec:supp:tau-sensitivity}, we study the relationship between performance sensitivity on $\tau$.
}
The training set was divided into 4 to 1 and the latter set was used for validation.
We used a common learning rate in Algorithm~\ref{alg:hybrid}, which was chosen from $\{10^1, 10^{-1}, 10^{-3}, 10^{-5}\}$ by cross validation.

\begin{figure}[t]
    \centering
    \begin{minipage}{0.49\columnwidth}
        \includegraphics[width=\columnwidth]{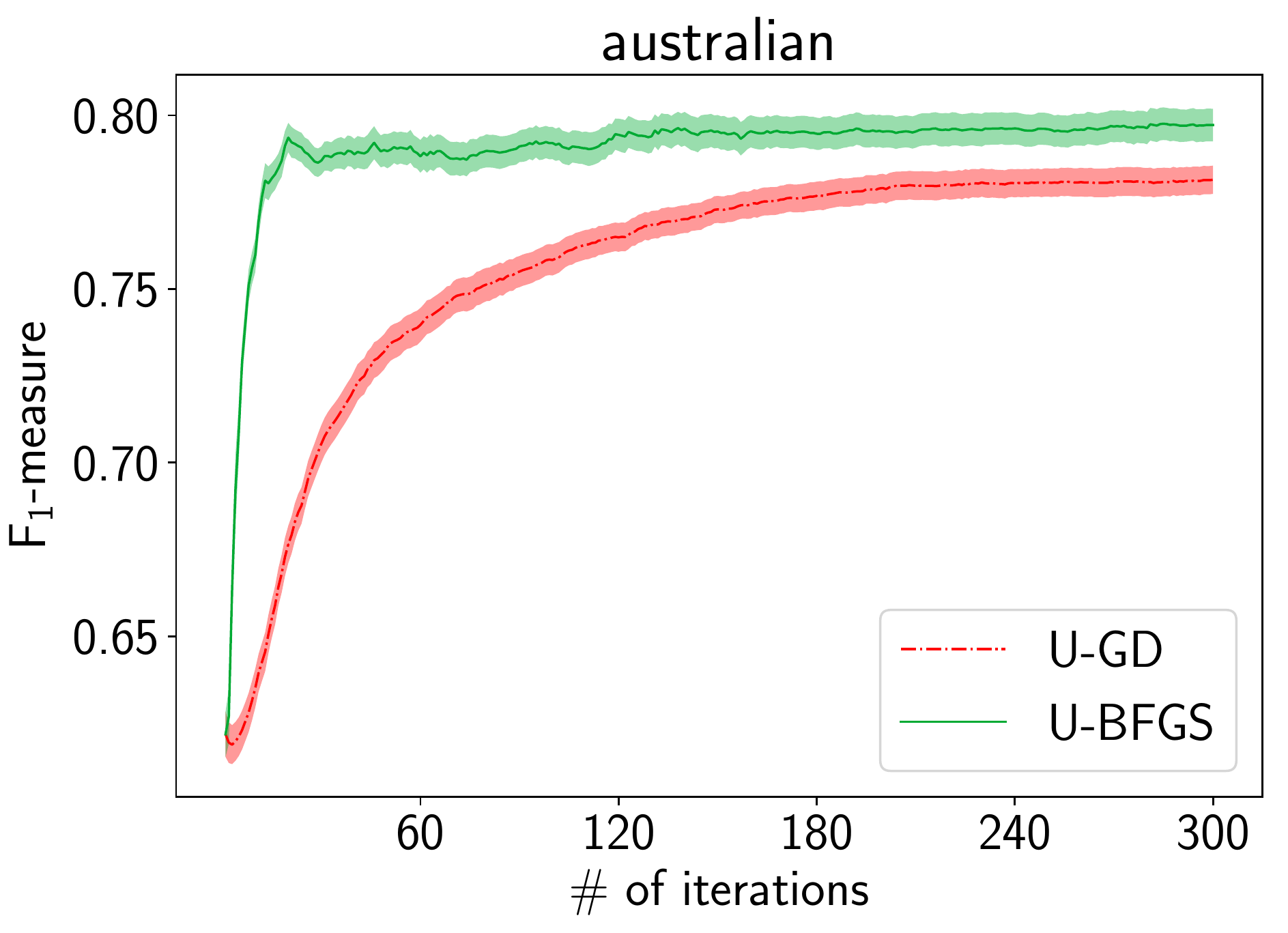}
    \end{minipage}
    \begin{minipage}{0.49\columnwidth}
        \includegraphics[width=\columnwidth]{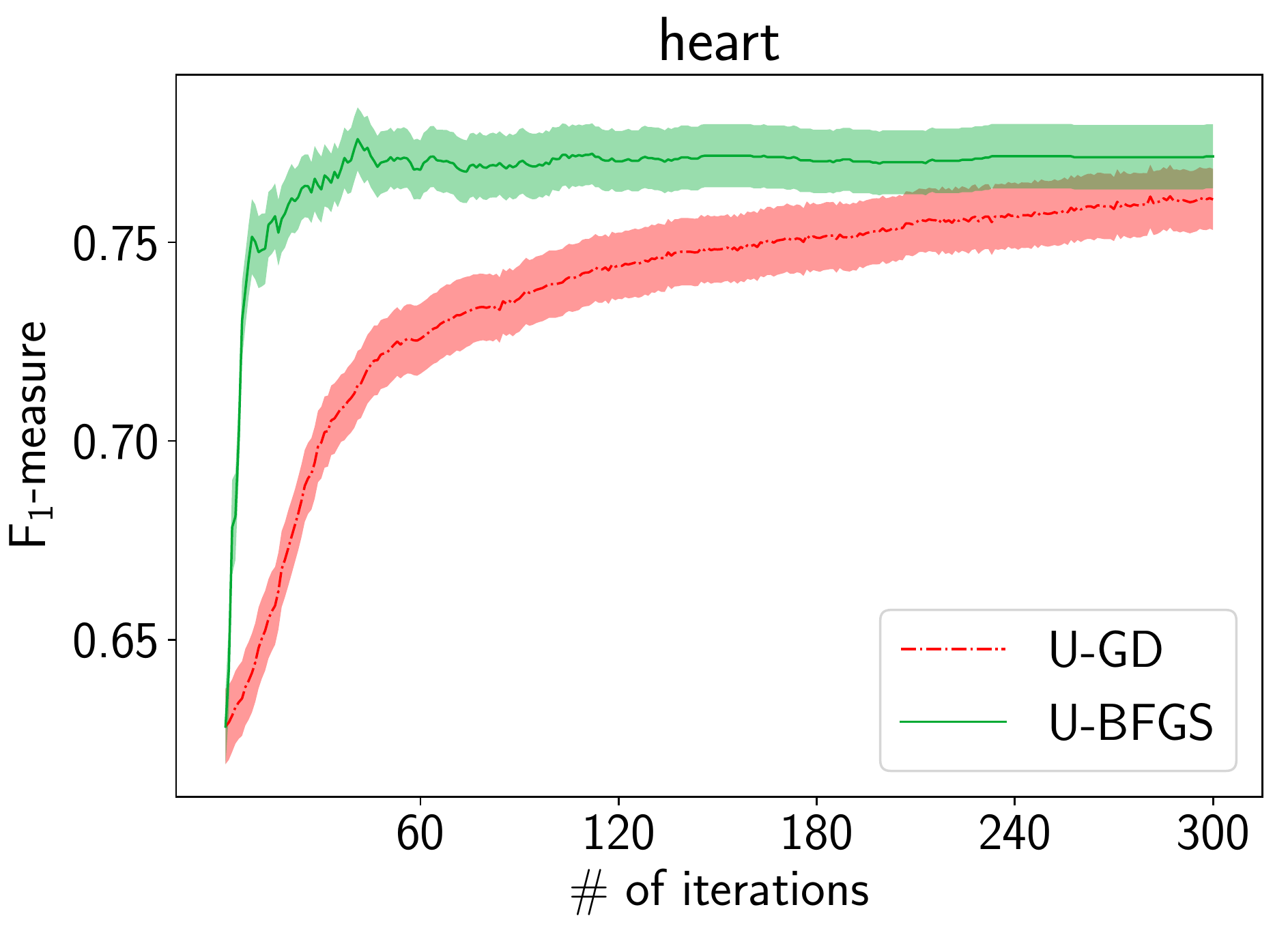}
    \end{minipage}
    \caption{
        Convergence comparison of the F${}_1$-measure (left two figures) and Jaccard index (right two figures).
        Standard errors of 50 trials are shown as shaded areas.
    }
    \label{fig:itertest}
\end{figure}

\begin{figure}[t]
    \centering
    \begin{minipage}{0.49\columnwidth}
        \includegraphics[width=\columnwidth]{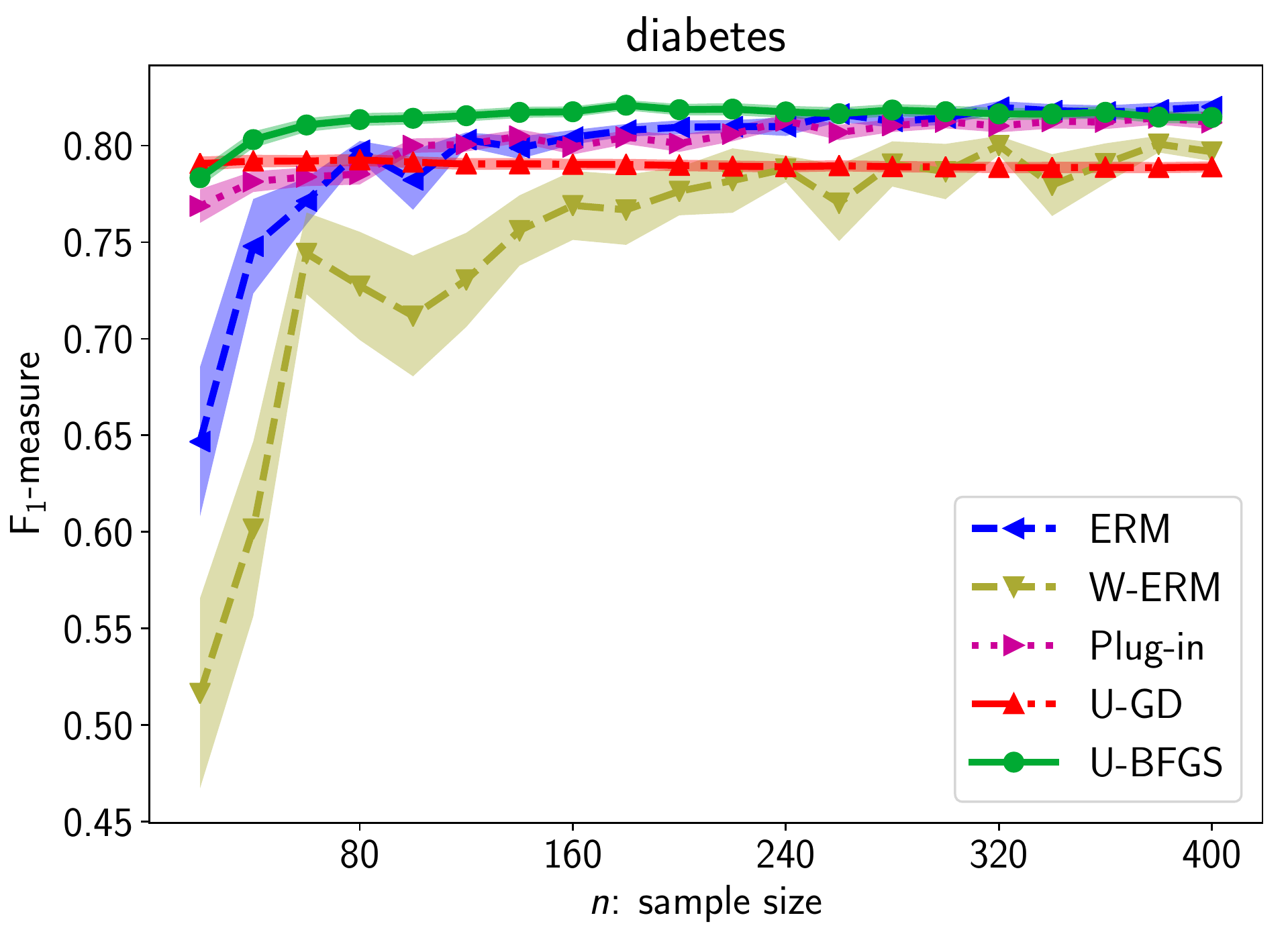}
    \end{minipage}
    \begin{minipage}{0.49\columnwidth}
        \includegraphics[width=\columnwidth]{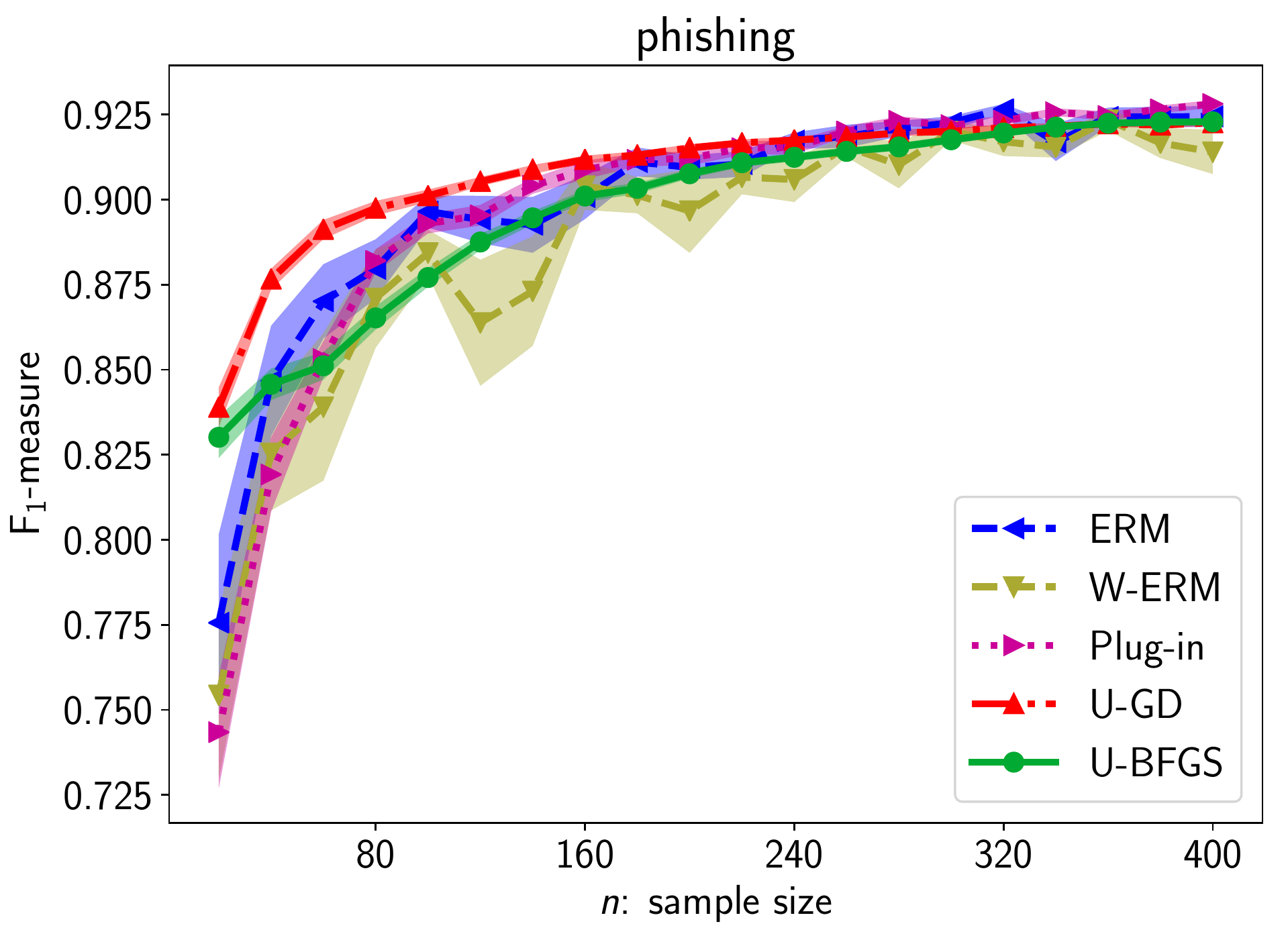}
    \end{minipage}
    \caption{
        The relationship of the test F${}_1$-measure (left two figures) / Jaccard index (right two figures) and sample size (horizontal axes).
        Standard errors of 50 trials are shown as shaded areas.
    }
    \label{fig:sample-complexity}
\end{figure}

\textbf{Convergence Comparison:}
We compare convergence behaviors of U-GD and U-BFGS.
In this experiment, we ran them 300 iterations from randomly initialized parameters drawn from $\calN(0_d, I_d)$.
The results are summarized in Fig.~\ref{fig:itertest}.
As we expected, U-BFGS converges much faster than U-GD in most of the cases, up to 30 iterations.
Note that U-BFGS and U-GD are in the trade-off relationship in that the former converges within fewer steps while the latter can update the solution faster in each step.

\textbf{Performance Comparison in Benchmark:}
We compared the proposed methods with baselines.
The results of the F${}_1$-measure and Jaccard index are summarized in Tab.~\ref{tab:benchmark-result}, respectively,
from which we can see the better or at least comparable performances of the proposed methods.

\textbf{Sample Complexity:}
We empirically study the relationship between the performance and the sample size.
We randomly subsample each original dataset to reduce the sample sizes to $\{20, 40, \dots, 400\}$,
and train all methods on the reduced samples.
The experimental results are shown in Fig.~\ref{fig:sample-complexity}.
Overall, U-GD and U-BFGS outperform,
which is especially significant when the sample sizes are quite small.
It is worth noting that U-GD works even better than U-BFGS in some cases,
though U-GD does not behave significantly better in Tab.~\ref{tab:benchmark-result}.
This can happen because the Hessian approximation in BFGS might not work well when the sample sizes are extremely small.

\section{Conclusion}
\label{sec:conclusion}

In this work, we gave a new insight into the calibrated surrogate for the linear-fractional metrics.
Sufficient conditions for the surrogate calibration were stated,
which is the first calibration result for the linear-fractional metrics to the best of our knowledge.
The surrogate maximization can be performed by the combination of concave and quasi-concave programs, and its performance is validated via simulations.

\section*{Acknowledgement}
We would like to thank Nontawat Charoenphakdee, Junya Honda, and Akiko Takeda for fruitful discussions.
HB was supported by JST, ACT-I, Japan, Grant Number JPMJPR18UI.
MS was supported by JST CREST Grant Number JPMJCR18A2.

\bibliography{bao}
\bibliographystyle{apalike}

\clearpage

\newpage
\appendix
\onecolumn

\begin{center}
    \fontsize{14.4pt}{16pt}\selectfont
    \textbf{
        Supplementary Material for ``Calibrated Surrogate Maximization of Linear-fractional Utility in Binary Classification''
    }
\end{center}

\section{Calibration Analysis and Deferred Proofs from Section~\ref{sec:calibration}}
\label{sec:supp:calibration}

In this section, we analyze calibration of the surrogate utility.
Before proceeding, we need to describe Bayes optimal classifier for a given metric.

\begin{definition}
    Given a linear-fractional utility $\calU$, \emph{Bayes optimal} set $\calB \subset \R^\calX$ is a set of functions that achieve the supremum of $\calU$,
    that is, $\calB \doteq \{f \mid \calU(f) = \calU^\dagger = \sup_{f'}\calU(f') \}$.
\end{definition}

Classifiers in $\calB$ are referred to as \emph{Bayes optimal classifiers}.
Note that they are not necessarily unique.
In this work, we assume that $\calB \ne \emptyset$.
First, we characterize Bayes optimal set $\calB$.

\def\da{\Delta a}

\begin{proposition}
    \label{prop:bayes-optimal-set}
    Given a linear-fractional utility $\calU$ in Eq.~\eqref{eq:generalized-metric},
    the Bayes optimal set $\calB$ for $\calU$ is
    \begin{align*}
        \calB = \{f \mid f(x)\{(\da_0 - \da_1\calU(f))\eta(x) - (a_{1,-1}\calU(f) - a_{0,-1})\} > 0 \; \forall x \in \calX\},
    \end{align*}
    where $\da_0 \doteq a_{0,+1} - a_{0,-1}$ and $\da_1 \doteq a_{1,+1} - a_{1,-1}$.
\end{proposition}

\begin{proof}
    The maximization problem in Eq.~\eqref{eq:utility-maximization} can be restated as follows.
    \begin{align*}
        \max_{\lambda \in \Lambda} \bar\calU(\lambda);
        \quad \bar\calU(\lambda) \doteq \frac{
            \E_X[a_{0,+1}\lambda(X)\eta(X) + a_{0,-1}\lambda(X)(1-\eta(X)) + b_0]
        }{
            \E_X[a_{1,+1}\lambda(X)\eta(X) + a_{1,-1}\lambda(X)(1-\eta(X)) + b_1]
        },
    \end{align*}
    where $\Lambda \doteq \{x \mapsto \ell(-f(x)) \mid f \in \calF\} \subset \R^\calX$.
    First, the Fr{\'e}chet derivative of $\bar\calU$ evaluated at $x$ is obtained as follows.
    \begin{align*}
        [\nabla_\lambda\bar\calU(\lambda)]_x
        &= \frac{(\da_0\eta(x) + a_{0,-1}) \E[W_1] - (\da_1\eta(x) + a_{1,-1})\E[W_0]}{\E[W_1]^2} p(x)
        \\
        &= \frac{p(x)}{\E[W_{1}]}\left\{ \left(\da_0 - \da_1\frac{\E[W_{0}]}{\E[W_{1}]}\right)\eta(x) - \left(a_{1,-1}\frac{\E[W_{0}]}{\E[W_{1}]} - a_{0,-1}\right) \right\}
        \\
        &= \frac{p(x)}{\E[W_{1}]}\left\{ (\da_0 - \da_1\bar\calU(\lambda))\eta(x) - (a_{1,-1}\bar\calU(\lambda) - a_{0,-1}) \right\}
    \end{align*}
    Let $f^\dagger \in \calF$ be a function that maximizes $\calU$,
    and $\lambda^\dagger \doteq \ell(-f^\dagger(\cdot))$.
    Then, $\lambda^\dagger$ maximizes $\bar\calU$, and it satisfies \citep[lemma 12]{Koyejo:2014}
    \begin{align*}
        \int_\calX [\nabla_\lambda\bar\calU(\lambda^\dagger)]_x\lambda^\dagger(x)dx
        \geq \int_\calX [\nabla_\lambda\bar\calU(\lambda^\dagger)]_x\lambda(x)dx
        \quad \forall \lambda \in \Lambda.
    \end{align*}
    Thus, the necessary condition for local optimality is that $\sign(\lambda^\dagger(x)) = \sign([\nabla_\lambda\bar\calU(\lambda^\dagger)]_x)$ for all $x \in \calX$.\footnote{
        This can be confirmed in a similar manner to the proof of \citet[Theorem 3.1]{Yan:2018}.
    }
    Since $\sign(\lambda^\dagger(x)) = \sign(\ell(-f^\dagger(x))) = \sign(f^\dagger(x))$,
    the above condition is $\sign(f^\dagger(x)) = \sign([\nabla_\lambda\bar\calU(\lambda^\dagger)]_x)$ for all $x \in \calX$,
    which is equivalent to the condition $f^\dagger(x)\{ (\da_0 - \da_1\calU(f^\dagger))\eta(x) - (a_{1,-1}\calU(f^\dagger) - a_{0,-1}) \} > 0$ for all $x \in \calX$.
    This concludes the proof.
    Note that $p(x) / \E[W_{1}]$ is a positive value, and $\bar\calU(\lambda^\dagger) = \calU(f^\dagger)$.
\end{proof}

You may confirm that Proposition~\ref{prop:bayes-optimal-set} is consistent with Bayes optimal classifier in the classical case, accuracy~\citep{Bartlett:2006}:
a Bayes optimal classifier $f^\dagger$ should satisfy $f^\dagger(x)(2\eta(x) - 1) > 0$ for all $x \in \calX$,
since $a_{0,+1} = 1$, $a_{0,-1} = -1$, $a_{1,+1} = a_{1,-1} = b_0 = b_1 = 0$.

Next, we state a proposition which gives a direction to prove the surrogate calibration of a surrogate utility.
This proposition follows a latter half of \citet[Theorem 2]{Gao:2015}.
\begin{proposition}
    \label{prop:calibration}
    Fix a true utility $\calU$, a surrogate utility $\calU_\phi$,
    and let $\calB$ a Bayes optimal set corresponding to the utility $\calU$.
    Assume that
    \begin{align}
        \sup_{f \not\in \calB}\calU_\phi(f) < \sup_f\calU_\phi(f)
        \label{eq:utility-calibration}
        .
    \end{align}
    Then, the surrogate utility $\calU_\phi$ is $\calU$-calibrated.
\end{proposition}

\begin{proof}
    Remind that $\calU_\phi \doteq \sup_f\calU_\phi(f)$ and let
    \begin{align*}
        \delta \doteq \calU_\phi^* - \sup_{f \not\in \calB}\calU_\phi(f) > 0,
    \end{align*}
    and $\{f_l\}_{l \geq 1}$ be any sequence such that $\calU_\phi(f_l) \overset{l \to \infty}{\longrightarrow} \calU_\phi^*$.
    Then, for any $\varepsilon > 0$, there exists $l_0 \in \Z$ such that $\calU_\phi^* - \calU_\phi(f_l) < \varepsilon$ for $l \geq l_0$.
    Here we set $\varepsilon = \frac{\delta}{2}$: $\calU_\phi^* - \calU_\phi(f_l) < \frac{\delta}{2}$ for $l \geq l_0$.
    If we assume that $f_l \not\in \calB$, this contradicts with the following facts:
    for a function $f \not\in \calB$,
    \begin{align*}
        \calU_\phi^* - \calU_\phi(f)
        = \underbrace{\calU_\phi^* - \sup_{f' \not\in \calB}\calU_\phi(f')}_{= \delta} + \underbrace{\sup_{f' \not\in \calB}\calU_\phi(f') - \calU_\phi(f)}_{\geq 0}
        \geq \delta.
    \end{align*}
    Thus, it holds that $f_l \in \calB$ for $l \geq l_0$, that is, $\calU(f_l) = \calU^\dagger$,
    which indicates $\calU$-calibration.
\end{proof}

Thus, the proof of $\calU$-calibration of $\calU_\phi$ is reduced to show the condition~\eqref{eq:utility-calibration}.
Below, we show the surrogate calibration for the F${}_\beta$-measure and Jaccard index utilizing Propositions~\ref{prop:bayes-optimal-set} and~\ref{prop:calibration}.
The proofs are based on the above propositions, \citet[Lemma 6]{Gao:2015} and \citet[Theorem 11]{Charoenphakdee:2019}.

Throughout the proofs, we assume that for the critical set $\calC^\dagger \doteq \{x \mid (\da_0 - \da_1\calU(f^\dagger)) \eta(x) - (a_{1,-1}\calU(f^\dagger) - a_{0,-1}) = 0\}$,
$\P(\calC^\dagger) = 0$, where $f^\dagger$ is the classifier attaining the supremum of $\calU$.
For example, this holds for any \emph{$\eta$-continuous} distribution~\citep[Assumption~2]{Yan:2018}.

\subsection{Proof of Theorem~\ref{thm:f-measure-calibration}}

\def\UFb{\calU^{\mathsf{F}_\beta}}
\def\LFbnum{W_{0,\phi}^{\mathsf{F}_\beta}}
\def\LFbden{W_{1,\phi}^{\mathsf{F}_\beta}}
\def\BFb{\calB^{\mathsf{F}_\beta}}
\def\df{\delta\!f}

\begin{figure}[t]
    \centering
    \begin{minipage}{0.45\columnwidth}
        \centering
        \includegraphics[width=\columnwidth]{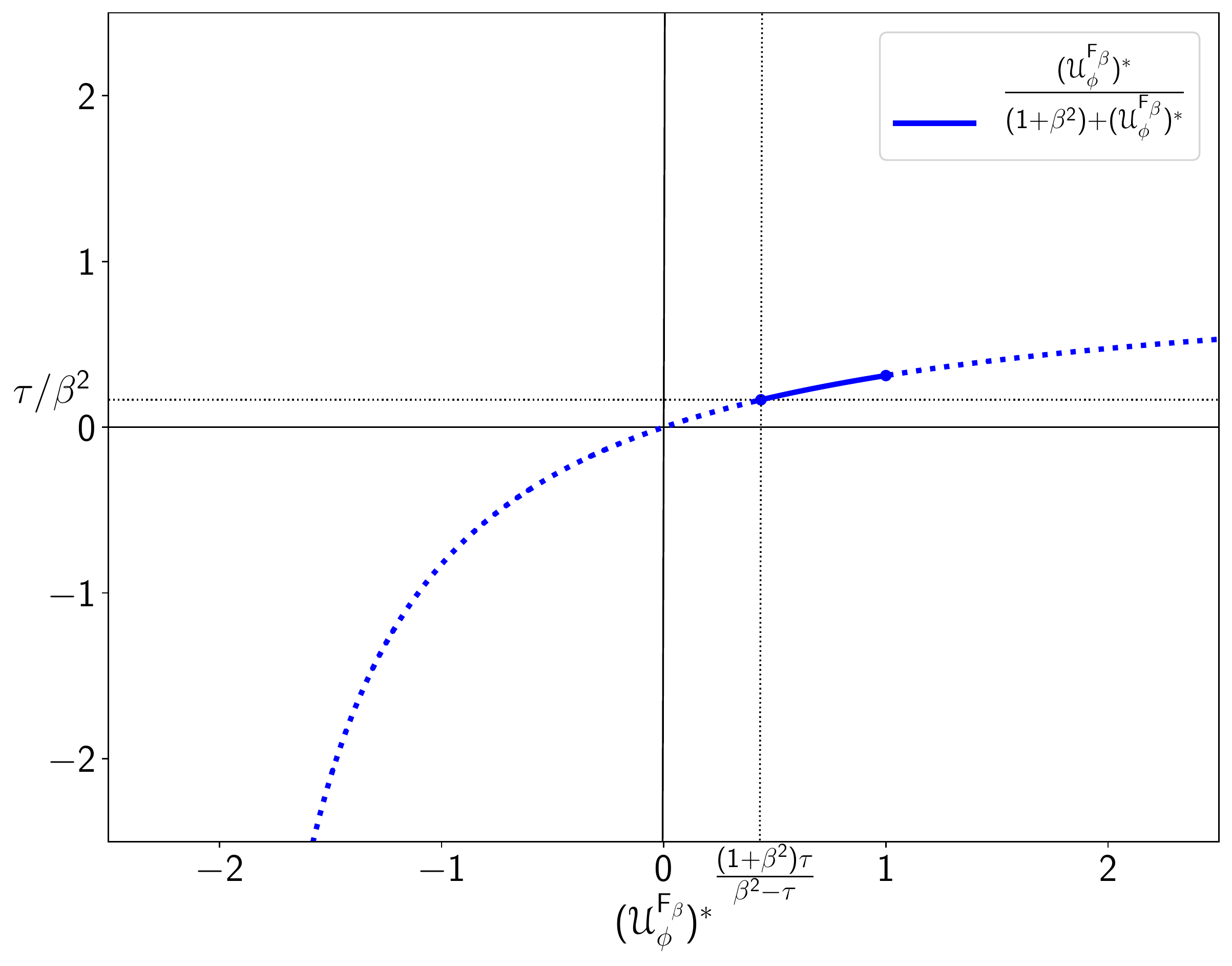}
        \caption{\protect\raggedright The range of $\frac{(\UFb_\phi)^*}{(1 + \beta^2) + (\UFb_\phi)^*}$ in $\frac{(1 + \beta^2)\tau}{\beta^2 - \tau} \leq (\UFb_\phi)^* \leq 1$.}
        \label{fig:supp:f1-calibration-1}
    \end{minipage}
    \begin{minipage}{0.45\columnwidth}
        \centering
        \includegraphics[width=\columnwidth]{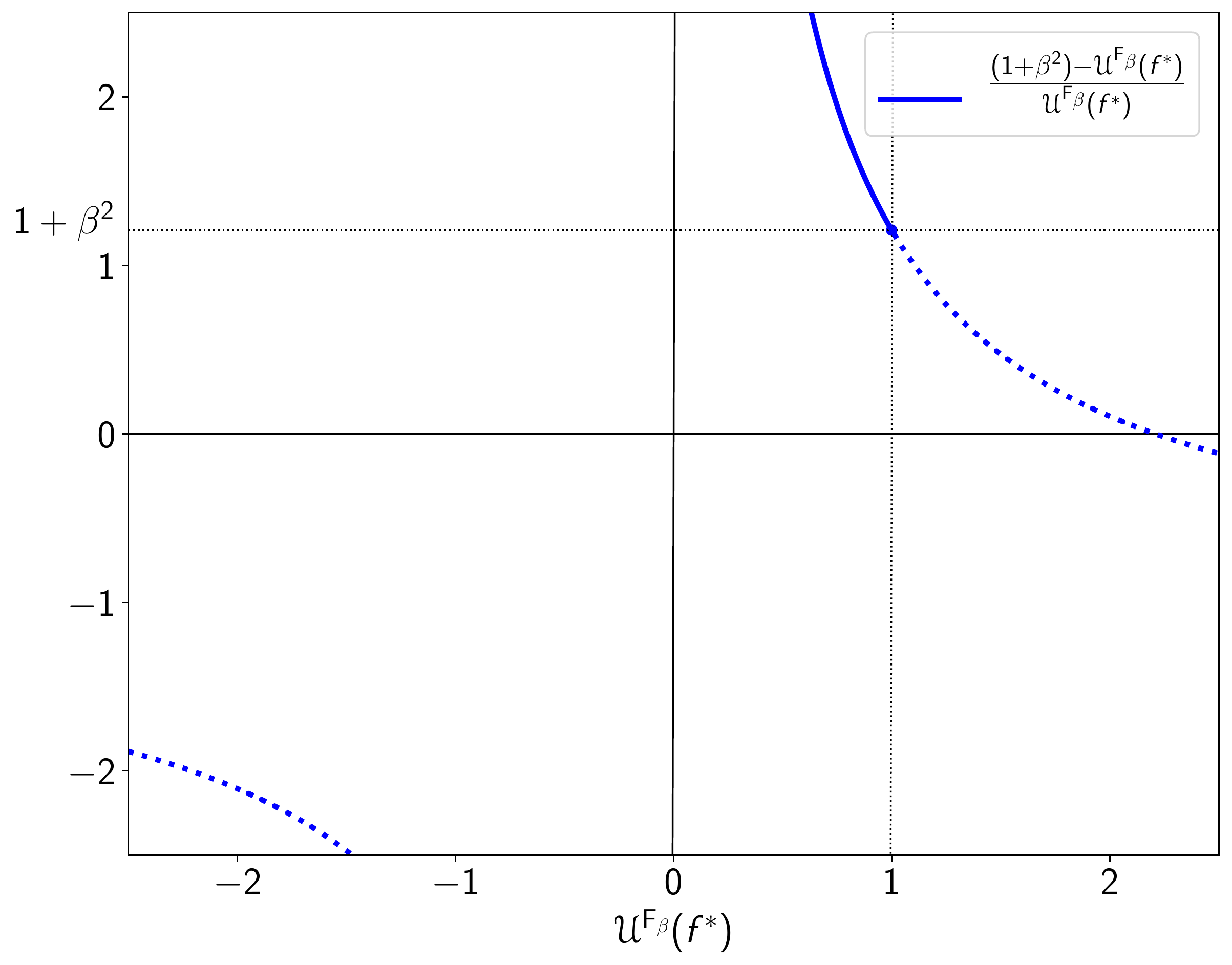}
        \caption{The range of $\frac{(1 + \beta^2) - \UFb(f^*)}{\UFb(f^*)}$ in $0 < \UFb(f^*) \leq 1$.}
        \label{fig:supp:f1-calibration-2}
    \end{minipage}
    \begin{minipage}{0.45\columnwidth}
        \centering
        \includegraphics[width=\columnwidth]{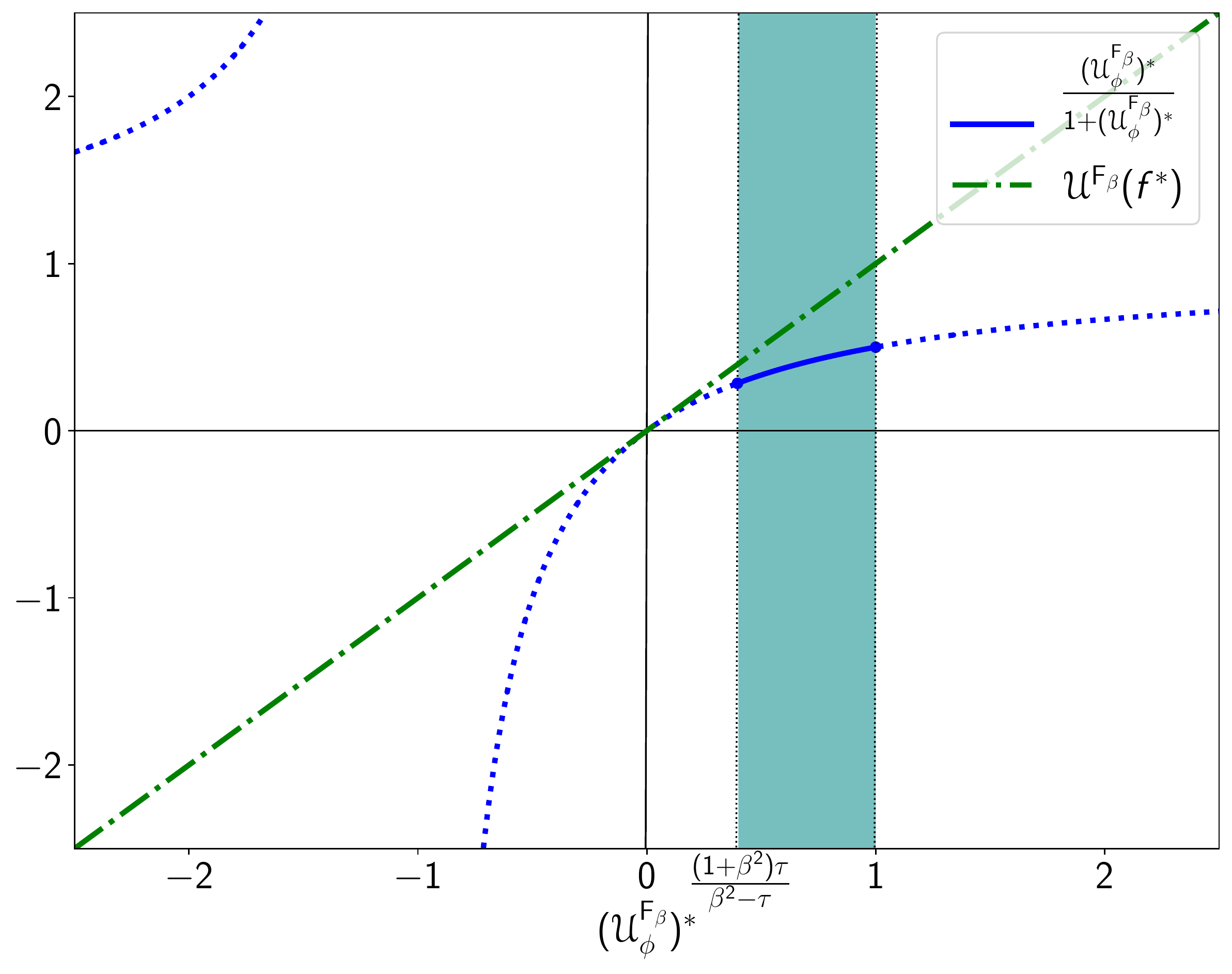}
        \caption{If $\frac{(1 + \beta^2)\tau}{\beta^2 - \tau} \leq (\UFb_\phi)^* \leq 1$, then $\frac{(\UFb_\phi)^*}{1 + (\UFb_\phi)^2} \leq \UFb(f^*)$.}
        \label{fig:supp:f1-calibration-3}
    \end{minipage}
\end{figure}

\begin{proof}[\unskip\nopunct]
    As a surrogate utility of the F${}_\beta$-measure following Eq.~\eqref{eq:surrogate-utility}, we have
    \begin{align*}
        \UFb_\phi(f) &= \frac{
            \int_\calX\{(1 + \beta^2)(1 - \phi(f(x)))\eta(x)\}p(x)dx
        }{
            \int_\calX\{(1 + \phi(f(x)))\eta(x) + \phi(-f(x))(1 - \eta(x)) + \beta^2\pi\}p(x)dx
        }
        \\
        &\doteq \frac{\E_X[\LFbnum(f(X), \eta(X))]}{\E_X[\LFbden(f(X), \eta(X))]}
        ,
    \end{align*}
    where
    \begin{align*}
        \LFbnum(\xi, q) &\doteq (1 + \beta^2)(1 - \phi(\xi))q, \\
        \LFbden(\xi, q) &\doteq (1 + \phi(\xi))q + \phi(-\xi)(1 - q) + \beta^2 \pi.
    \end{align*}
    From Proposition~\ref{prop:bayes-optimal-set}, the Bayes optimal set $\BFb$ for the F${}_\beta$-measure is
    \begin{align*}
        \BFb \doteq \{ f \mid f(x)((1 + \beta^2)\eta(x) - \UFb(f)) > 0 \quad \forall x \in \calX \}.
    \end{align*}
    We will show F${}_\beta$-calibration by utilizing Proposition~\ref{prop:calibration},
    which casts our proof target into showing Eq.~\eqref{eq:utility-calibration}.
    We prove it by contradiction.
    Assume that
    \begin{align*}
        \sup_{f \not\in \BFb}\UFb_\phi(f) = \sup_f\UFb_\phi(f).
    \end{align*}
    This implies that there exists an optimal function $f^* \not\in \BFb$ that achieves $\UFb_\phi(f^*) = \sup_f\UFb_\phi(f) \doteq (\UFb_\phi)^*$,
    that is, $\UFb_\phi(f^*) = (\UFb_\phi)^*$ and $f^*(\bar{x})((1 + \beta^2)\eta(\bar{x}) - \UFb(f^*)) \leq 0$ for some $\bar{x} \in \calX$.

    Let us describe the \emph{stationary condition} of $f^*$.
    We introduce a function $\df$:
    \begin{align*}
        \df(x) \doteq \begin{cases}
            1 & \text{if $x = \bar{x}$}, \\
            0 & \text{if $x \ne \bar{x}$}.
        \end{cases}
    \end{align*}
    Let $G(\gamma) \doteq \UFb_\phi(f^* + \gamma\df)$.
    Since $\UFb_\phi$ is G{\^a}teaux differentiable\footnote{Fr{\'e}chet differentiability implies G{\^a}teaux differentiability.}
    and its G{\^a}teaux derivative at $f^*$ must be zero in any direction,
    we claim that $G'(0) = 0$, where $G'(0)$ corresponds to G{\^a}teaux derivative of $\UFb_\phi$ at $f^*$ in the direction of $\df$.
    Here, $G'(0)$ is computed as
    \begin{align*}
        G'(0)
        &= \frac{1}{\E[\LFbden(f^*)]^2} \left\{
            \E[\LFbden(f^*)] \int_\calX\{-(1 + \beta^2)\phi'(f^*(x))\df(x)\eta(x)\}p(x)dx \right.\\
            &\qquad \left. -
            \E[\LFbnum(f^*)] \int_\calX\{\phi'(f^*(x))\df(x)\eta(x) - \phi'(-f^*(x))\df(x)(1 - \eta(x))\}p(x)dx
        \right\}
        \\
        &= \frac{1}{\E[\LFbden(f^*)]} \left\{ \int_\calX\{-(1 + \beta^2)\phi'(f^*(x))\df(x)\eta(x)\}p(x)dx \right. \\
            &\qquad \left. -
            (\UFb_\phi)^* \int_\calX\{\phi'(f^*(x))\df(x)\eta(x) - \phi'(-f^*(x))\df(x)(1 - \eta(x))\}p(x)dx
        \right\}
        \\
        &= \frac{\{ -(1 + \beta^2)\phi'(f^*(\bar{x}))\eta(\bar{x}) - \phi'(f^*(\bar{x}))(\UFb_\phi)^*\eta(\bar{x}) + \phi'(-f^*(\bar{x}))(\UFb_\phi)^*(1 - \eta(\bar{x})) \} p(\bar{x})}{\E[\LFbden(f^*)]}
        ,
    \end{align*}
    where $\E[\LFbnum(f^*)] = \E_X[\LFbnum(f^*(X), \eta(X))]$ and $\E[\LFbden(f^*)] = \E_X[\LFbden(f^*(X), \eta(X))]$.
    Thus, the stationary condition is
    \begin{align}
        -(1 + \beta^2)\phi'(f^*(\bar{x}))\eta(\bar{x}) - \phi'(f^*(\bar{x}))(\UFb_\phi)^*\eta(\bar{x}) + \phi'(-f^*(\bar{x}))(\UFb_\phi)^*(1 - \eta(\bar{x}))
        &= 0 \nonumber
        \\
        \left\{-(1 + \beta^2)\phi'(f^*(\bar{x})) - \phi'(f^*(\bar{x}))(\UFb_\phi)^* - \phi'(-f^*(\bar{x}))(\UFb_\phi)^*\right\}\eta(\bar{x}) + \phi'(-f^*(\bar{x}))(\UFb_\phi)^*
        &= 0
        .
        \label{eq:supp:fbeta-stationary-condition}
    \end{align}
    Since $\phi'(\pm f^*(\bar{x})) < 0$, we have $-(1 + \beta^2)\phi'(f^*(\bar{x})) - \phi'(f^*(\bar{x}))(\UFb_\phi)^* - \phi'(-f^*(\bar{x}))(\UFb_\phi)^* > 0$.
    Thus, the condition~\eqref{eq:supp:fbeta-stationary-condition} becomes
    \begin{align}
        \eta(\bar{x}) = \frac{
            \phi'(-f^*(\bar{x}))(\UFb_\phi)^*
        }{
            (1 + \beta^2)\phi'(f^*(\bar{x})) + (\phi'(f^*(\bar{x})) + \phi'(-f^*(\bar{x})))(\UFb_\phi)^*
        }
        .
        \label{eq:supp:fbeta-stationary-eta}
    \end{align}

    From now on, we divide the cases to take care of the Bayes optimal condition $f^*(\bar{x})((1 + \beta^2)\eta(\bar{x}) - \UFb(f^*)) \geq 0$.

    \begin{enumerate}
        \renewcommand{\labelenumi}{\arabic{enumi}) }

        \item \underline{If $f^*(\bar{x}) > 0$ and $\eta(\bar{x}) < \frac{1}{1 + \beta^2}\UFb(f^*)$: }
        We show
        \begin{align}
            \frac{
                \phi'(-f^*(\bar{x}))(\UFb_\phi)^*
            }{
                (1 + \beta^2)\phi'(f^*(\bar{x})) + (\phi'(f^*(\bar{x})) + \phi'(-f^*(\bar{x})))(\UFb_\phi)^*
            } \geq \frac{\UFb(f^*)}{1 + \beta^2}
            .
            \label{eq:supp:f-case1-contradiction}
        \end{align}
        Take the difference of the left-hand side and the right-hand side:
        \begin{align*}
            &\frac{
                \phi'(-f^*(\bar{x}))(\UFb_\phi)^*
            }{
                (1 + \beta^2)\phi'(f^*(\bar{x})) + (\phi'(f^*(\bar{x})) + \phi'(-f^*(\bar{x})))(\UFb_\phi)^*
            } - \frac{\UFb(f^*)}{1 + \beta^2}
            \\
            &\quad= \frac{
                (1 + \beta^2)\phi'(-f^*(\bar{x}))(\UFb_\phi)^* - (1 + \beta^2)\phi'(f^*(\bar{x}))\UFb(f^*) -
                (\phi'(f^*(\bar{x})) + \phi'(-f^*(\bar{x}))(\UFb_\phi)^*\UFb(f^*)
            }{
                (1 + \beta^2)((1 + \beta^2)\phi'(f^*(\bar{x})) + (\phi'(f^*(\bar{x})) + \phi'(-f^*(\bar{x})))(\UFb_\phi)^*)
            },
        \end{align*}
        where the denominator is always negative,
        which reduces to show the numerator is always negative, too:
        \begin{align*}
            (1 + \beta^2)&\phi'(-f^*(\bar{x}))(\UFb_\phi)^* - (1 + \beta^2)\phi'(f^*(\bar{x}))\UFb(f^*) -
            (\phi'(f^*(\bar{x})) + \phi'(-f^*(\bar{x}))(\UFb_\phi)^*\UFb(f^*)
            \\
            &= \UFb(f^*)((1 + \beta^2) + (\UFb_\phi)^*)\left(
                \underbrace{\frac{(\UFb_\phi)^*}{(1 + \beta^2) + (\UFb_\phi)^*}}_{\geq \tau/\beta^2}
                \underbrace{\frac{(1 + \beta^2) - \UFb(f^*)}{\UFb(f^*)}}_{\geq \beta^2}
                \phi'(-f^*(\bar{x}))
                - \phi'(f^*(\bar{x}))
            \right)
            \\
            &\leq \UFb(f^*)((1 + \beta^2) + (\UFb_\phi)^*)\left(
                \tau \phi'(-f^*(\bar{x})) - \phi'(f^*(\bar{x}))
            \right)
            \\
            &\leq 0,
        \end{align*}
        where the first inequality holds because $\frac{(\UFb_\phi)^*}{(1 + \beta^2) + (\UFb_\phi)^*} \geq \frac{\tau}{\beta^2}$ when $\frac{(1 + \beta^2)\tau}{\beta^2 - \tau} \leq (\UFb_\phi)^* \leq 1$ (see Figure~\ref{fig:supp:f1-calibration-1})
        and $\frac{(1 + \beta^2) - \UFb(f^*)}{\UFb(f^*)} \geq \beta^2$ when $0 \leq \UFb(f^*) \leq 1$ (see Figure~\ref{fig:supp:f1-calibration-2}).
        Note that $\phi'(-f^*(\bar{x})) < 0$.
        The second inequality holds because of the assumption that $\lim_{m \searrow 0}\phi'(m) \geq \tau \lim_{m \nearrow 0}\phi'(m)$ and $\phi$ is convex,
        which implies $\tau \phi'(-m) - \phi'(m) \leq 0$ for every $m > 0$.

        Thus, the inequality~\eqref{eq:supp:f-case1-contradiction} holds,
        which implies the following contradiction.
        \begin{align*}
            \eta(\bar{x})
            =
            \frac{
                \phi'(-f^*(\bar{x}))(\UFb_\phi)^*
            }{
                (1 + \beta^2)\phi'(f^*(\bar{x})) + (\phi'(f^*(\bar{x})) + \phi'(-f^*(\bar{x})))(\UFb_\phi)^*
            } \geq \frac{\UFb(f^*)}{1 + \beta^2}
            > \eta(\bar{x})
            .
        \end{align*}

        \item \underline{If $f^*(\bar{x}) \leq 0$ and $\eta(\bar{x}) > \frac{1}{1 + \beta^2}\UFb(f^*)$: }
        As well as the previous case, we begin from the stationary condition~\eqref{eq:supp:fbeta-stationary-eta}.
        If $\phi'(-f^*(\bar{x})) < 0$,
        \begin{align*}
            \eta(\bar{x})
            &= \frac{
                \phi'(-f^*(\bar{x}))(\UFb_\phi)^*
            }{
                (1 + \beta^2)\phi'(f^*(\bar{x})) + (\phi'(f^*(\bar{x})) + \phi'(-f^*(\bar{x})))(\UFb_\phi)^*
            }
            \\
            &= \frac{(\UFb_\phi)^*}{
                (1 + \beta^2)\frac{\phi'(f^*(\bar{x}))}{\phi'(-f^*(\bar{x}))} + \left(\frac{\phi'(f^*(\bar{x}))}{\phi'(-f^*(\bar{x}))} + 1\right)(\UFb_\phi)^*
            }
            \\
            &\leq \frac{1}{1 + \beta^2} \frac{(\UFb_\phi)^*}{1 + (\UFb_\phi)^*}
            \\
            &\leq \frac{1}{1 + \beta^2}\UFb(f^*)
            \\
            &< \eta(\bar{x}),
            \qquad \text{(contradiction)}
        \end{align*}
        where the first inequality holds because $\frac{\phi'(-m)}{\phi'(m)} \geq 1$ for every $m \geq 0$ and $f^*(\bar{x}) \leq 0$,
        and the second inequality holds because $\UFb_\phi(f) \leq \UFb(f)$ ($\forall f$) implies $\frac{(\UFb_\phi)^*}{1 + (\UFb_\phi)^*} \leq \UFb(f^*)$ when $\frac{(1 + \beta^2)\tau}{\beta^2 - \tau} \leq (\UFb_\phi)^* \leq 1$ (see Figure~\ref{fig:supp:f1-calibration-3}).

        If $\phi'(-f^*(\bar{x})) = 0$, it is easy to see the contradiction.
    \end{enumerate}
    Combining the above cases, it follows that
    \begin{align*}
        \sup_{f \not\in \BFb}\UFb_\phi(f) < \sup_f\UFb_\phi(f).
    \end{align*}
    Eventually, we claim that $\UFb_\phi$ is F${}_\beta$-calibrated by using Proposition~\ref{prop:calibration}.
\end{proof}

\subsection{Proof of Theorem~\ref{thm:jaccard-calibration}}

\def\UJac{\calU^{\mathsf{Jac}}}
\def\BJac{\calB^{\mathsf{Jac}}}

\begin{figure}[t]
    \centering
    \includegraphics[width=0.45\columnwidth]{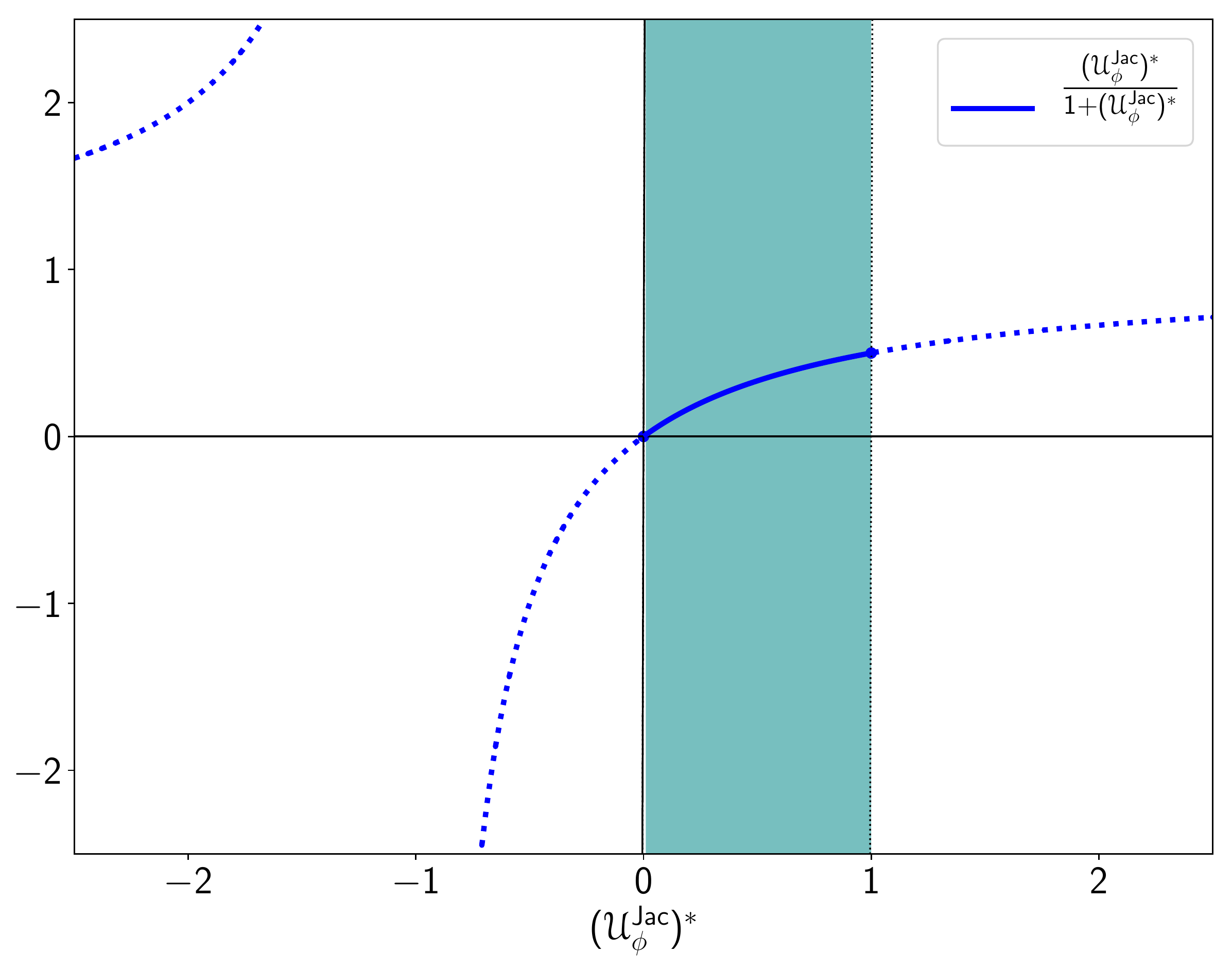}
    \caption{$\frac{(\UJac_\phi)^*}{1 + (\UJac_\phi)^*}$ is monotonically increasing if $0 \leq (\UJac_\phi)^* \leq 1$.}
    \label{fig:supp:jaccard-calibration}
\end{figure}

\begin{proof}[\unskip\nopunct]
    As a surrogate utility of the Jaccard index following Eq.~\eqref{eq:surrogate-utility},
    we have
    \begin{align*}
        \UJac_\phi(f)
        = \frac{
            \int_\calX (1 - \phi(f(x)))\eta(x)p(x)dx
        }{
            \int_\calX \{\phi(-f(x))(1 - \eta(x)) + \pi\}p(x)dx
        }
        ,
    \end{align*}
    and we have the Bayes optimal set $\BJac$ for the Jaccard index such as
    \begin{align*}
        \BJac \doteq \left\{ f \mid f(x)\{(1 + \UJac(f))\eta(x) - \UJac(f)\} > 0 \quad \forall x \in \calX \right\},
    \end{align*}
    utilizing Proposition~\ref{prop:bayes-optimal-set}.
    We follow the same proof technique, proof by contradiction, as we use in the proof of Theorem~\ref{thm:f-measure-calibration}.
    Assume that
    \begin{align*}
        \sup_{f \not\in \BJac}\UJac_\phi(f) = \sup_f\UJac_\phi(f),
    \end{align*}
    which implies that there exits an optimal function $f^* \not\in \BJac$ that achieves $\UJac_\phi(f^*) = \sup_f\UJac_\phi(f) \doteq (\UJac_\phi)^*$,
    that is, $\UJac_\phi(f^*) = (\UJac_\phi)^*$ and $f^*(\bar{x})\{(1 + \UJac(f^*))\eta(\bar{x}) - \UJac(f^*)\} \leq 0$ for some $\bar{x} \in \calX$.

    The stationary condition of $\UJac_\phi$ around $f^*$ can be stated as well as Eq.~\eqref{eq:supp:fbeta-stationary-eta} in Theorem~\ref{thm:f-measure-calibration}:
    \begin{align}
        \eta(\bar{x}) = \frac{\phi'(-f^*(\bar{x}))(\UJac_\phi)^*}{\phi'(f^*(\bar{x})) + \phi'(-f^*(\bar{x}))(\UJac_\phi)^*}
        \label{eq:supp:jaccard-stationary-eta}
        .
    \end{align}

    \begin{enumerate}
        \renewcommand{\labelenumi}{\arabic{enumi}) }

        \item \underline{If $f^*(\bar{x}) > 0$ and $\eta(\bar{x}) < \frac{\UJac(f^*)}{1 + \UJac(f^*)}$:}
        We show
        \begin{align*}
            \frac{\phi'(-f^*(\bar{x}))(\UJac_\phi)^*}{\phi'(f^*(\bar{x})) + \phi'(-f^*(\bar{x}))(\UJac_\phi)^*} \geq \frac{\UJac(f^*)}{1 + \UJac(f^*)}.
        \end{align*}
        First, take the difference of the left-hand side and the right-hand side.
        \begin{align*}
            &\frac{\phi'(-f^*(\bar{x}))(\UJac_\phi)^*}{\phi'(f^*(\bar{x})) + \phi'(-f^*(\bar{x}))(\UJac_\phi)^*} - \frac{\UJac(f^*)}{1 + \UJac(f^*)}
            \\
            &=\; \frac{
                \phi'(-f^*(\bar{x}))(\UJac_\phi)^* - \phi'(f^*(\bar{x}))\UJac(f^*)
            }{
                (\phi'(f^*(\bar{x})) + \phi'(-f^*(\bar{x}))(\UJac_\phi)^*)(1 + \UJac(f^*))
            }
            ,
        \end{align*}
        where the denominator is always negative,
        which reduces to show the numerator is always negative, too.
        If $\phi'(-f^*(\bar{x})) < 0$,
        \begin{align*}
            \phi'&(-f^*(\bar{x}))(\UJac_\phi)^* - \phi'(f^*(\bar{x}))\UJac(f^*)
            \\
            &= \phi'(-f^*(\bar{x}))\left((\UJac_\phi)^* - \frac{\phi'(f^*(\bar{x}))}{\phi'(-f^*(\bar{x})) }\UJac(f^*)\right)
            \\
            &\leq \phi'(-f^*(\bar{x}))\left((\UJac_\phi)^* - \frac{\phi'(f^*(\bar{x}))}{\phi'(-f^*(\bar{x}))}\right)
                && \text{($\because$ $\UJac(f^*) \leq 1$)}
            \\
            &\leq \phi'(-f^*(\bar{x}))((\UJac_\phi)^* - \tau)
            \\
            &\leq 0,
                && \text{($\because$ $(\UJac_\phi)^* \geq \tau$)}
        \end{align*}
        where the second inequality holds because of the assumption that $\lim_{m \searrow 0}\phi'(m) \geq \tau \lim_{m \nearrow 0}\phi'(m)$ for every $m > 0$, and $\phi$ is convex.
        Thus, we admit the contradiction.
        \begin{align*}
            \eta(\bar{x}) =
            \frac{\phi'(-f^*(\bar{x}))(\UJac_\phi)^*}{\phi'(f^*(\bar{x})) + \phi'(-f^*(\bar{x}))(\UJac_\phi)^*} > \frac{\UJac(f^*)}{1 + \UJac(f^*)}
            > \eta(\bar{x})
            .
        \end{align*}

        If $\phi'(-f^*(\bar{x})) = 0$, then $\phi'(f^*(\bar{x})) = 0$ from the assumption $\lim_{m \searrow 0}\phi'(m) \geq \tau \lim_{m \nearrow 0}\phi'(m)$,
        which immediately results in the contradiction.

        \item \underline{If $f^*(\bar{x}) \leq 0$ and $\eta(\bar{x}) > \frac{\UJac(f^*)}{1 + \UJac(f^*)}$:}
        We begin from the stationary condition in Eq.~\eqref{eq:supp:jaccard-stationary-eta}.
        If $\phi'(-f^*(\bar{x})) < 0$,
        \begin{align*}
            \eta(\bar{x}) &=
            \frac{\phi'(-f^*(\bar{x}))(\UJac_\phi)^*}{\phi'(f^*(\bar{x})) + \phi'(-f^*(\bar{x}))(\UJac_\phi)^*}
            \\
            &= \frac{(\UJac_\phi)^*}{\frac{\phi'(f^*(\bar{x}))}{\phi'(-f^*(\bar{x}))} + (\UJac_\phi)^*}
            \\
            &\leq \frac{(\UJac_\phi)^*}{1 + (\UJac_\phi)^*}
                && \text{$\left(\because \; \frac{\phi'(f^*(\bar{x}))}{\phi'(-f^*(\bar{x}))} \geq 1 \quad \forall f^*(\bar{x}) \leq 0\right)$}
            \\
            &\leq \frac{\UJac(f^*)}{1 + \UJac(f^*)}
            \\
            &< \eta(\bar{x}), && \text{(contradiction)}
        \end{align*}
        where the second inequality follows because $\UJac_\phi(f) \leq \UJac(f)$ ($\forall f$)
        and a function $x \mapsto \frac{x}{1 + x}$ ($0 \leq x \leq 1$) is monotonically increasing (see Figure~\ref{fig:supp:jaccard-calibration}).

        It is easy to see contradiction in case of $\phi'(-f^*(\bar{x})) = 0$.

    \end{enumerate}

    Combining the above cases, it follows that
    \begin{align*}
        \sup_{f \not\in \BJac}\UJac_\phi(f) < \sup_f\UJac(f).
    \end{align*}
    Eventually, we claim that $\UJac_\phi$ is Jaccard-calibrated by using Proposition~\ref{prop:calibration}.
\end{proof}

\subsection{Analysis of Accuracy-Calibration}
\label{sec:supp:accuracy-calibration}

\def\UAcc{\calU^{\mathsf{Acc}}}
\def\BAcc{\calB^{\mathsf{Acc}}}

In this subsection, we show accuracy-calibration conditions in the same manners as the F${}_\beta$-measure and Jaccard index,
and confirm that the $\tau$-discrepancy is not necessary in this case.
As the true and a surrogate utility of the accuracy following Eq.~\eqref{eq:surrogate-utility},
define
\begin{align*}
    \UAcc(f) &= \int_\calX \left\{ \ell(-f(x))\eta(x) - \ell(-f(x))(1 - \eta(x)) + (1 - \pi) \right\} p(x) \dx,
    \\
    \UAcc_\phi(f) &= \int_\calX \left\{ (1 - \phi(f(x)))\eta(x) - \phi(-f(x))(1 - \eta(x)) + (1 - \pi) \right\} p(x) \dx.
\end{align*}

\begin{proposition}[Accuracy-calibration]
    \label{prop:accuracy-calibration}
    Assume that a surrogate loss $\phi: \R \to \R_{\geq 0}$ is convex, differentiable almost everywhere,
    and $\phi'(0) < 0$.
    Then, $\UAcc_\phi$ is accuracy-calibrated.
\end{proposition}

\begin{proof}
    We have the Bayes optimal set $\BAcc$ for the accuracy such as
    \begin{align*}
        \BAcc \doteq \{ f \mid f(x)(2\eta(x) - 1) > 0 \quad \forall x \in \calX \}
    \end{align*}
    utilizing Proposition~\ref{prop:bayes-optimal-set}.
    In the same manner as the proofs of Theorems~\ref{thm:f-measure-calibration} and~\ref{thm:jaccard-calibration},
    assume that
    \begin{align*}
        \sup_{f \not\in \BAcc} \UAcc_\phi(f) = \sup_f \UAcc_\phi(f),
    \end{align*}
    and we prove by contradiction.
    The above assumption implies that there exists an optimal function $f^* \not\in \BAcc$
    such that $\UAcc_\phi(f^*) = \sup_f\UAcc_\phi(f) \doteq (\UAcc_\phi)^*$,
    that is, $\UAcc_\phi(f^*) = (\UAcc_\phi)^*$ and $f^*(\bar{x})(2\eta(\bar{x}) - 1) \leq 0$ for some $\bar{x} \in \calX$.

    The stationary condition of $\UAcc_\phi$ around $f^*$ can be stated in the same way as Eq.~\eqref{eq:supp:fbeta-stationary-condition}:
    \begin{align}
        (\phi'(f^*(\bar{x})) + \phi'(-f^*(\bar{x})))\eta(\bar{x}) - \phi'(-f^*(\bar{x})) = 0.
        \label{eq:supp:accuracy-stationary-condition}
    \end{align}
    We divide the cases based on the sign of $f^*(\bar{x})$.

    \begin{enumerate}
        \renewcommand{\labelenumi}{\arabic{enumi}) }

        \item \underline{$f^*(\bar{x}) > 0$ and $\eta(\bar{x}) < \frac{1}{2}$:}
        Since $\phi'(-f^*(\bar{x})) < \phi'(f^*(\bar{x})) < 0$ because of the convexity of $\phi$,
        \begin{align*}
            \eta(\bar{x}) = \frac{1}{1 + \frac{\phi'(f^*(\bar{x}))}{\phi'(-f^*(\bar{x}))}} > \frac{1}{2},
        \end{align*}
        which contradicts with $\eta(\bar{x}) < \frac{1}{2}$.
        Note that $\frac{\phi'(f^*(\bar{x}))}{\phi'(-f^*(\bar{x}))} \in (0, 1)$ when $\phi'(-f^*(\bar{x})) < \phi'(f^*(\bar{x})) < 0$.

        \item \underline{$f^*(\bar{x}) < 0$ and $\eta(\bar{x}) > \frac{1}{2}$:}
        Since $\phi'(f^*(\bar{x})) < \phi'(-f^*(\bar{x})) < 0$ because of the convexity of $\phi$,
        \begin{align*}
            \eta(\bar{x}) = \frac{1}{1 - \frac{\phi'(f^*(\bar{x}))}{\phi'(-f^*(\bar{x}))}} < \frac{1}{2},
        \end{align*}
        which contradicts with $\eta(\bar{x}) > \frac{1}{2}$.
        Note that $\frac{\phi'(f^*(\bar{x}))}{\phi'(-f^*(\bar{x}))} > 1$ when $\phi'(f^*(\bar{x})) < \phi'(-f^*(\bar{x})) < 0$.

        \item \underline{$f^*(\bar{x}) = 0$:}
        Since $\phi'(f^*(\bar{x})) = \phi'(-f^*(\bar{x})) = \phi'(0) < 0$,
        the stationary condition~\eqref{eq:supp:accuracy-stationary-condition} reduces to $\phi'(-f^*(\bar{x})) = 0$,
        which contradicts with $\phi'(-f^*(\bar{x})) = \phi'(0) < 0$.
    \end{enumerate}

    Thus, it follows that $\sup_{f \not\in \UAcc}\UAcc_\phi(f) < \sup_f \UAcc_\phi(f)$.
    Eventually, we claim that $\UAcc_\phi$ is accuracy-calibrated by using Proposition~\ref{prop:calibration}.
\end{proof}

As we can see from Proposition~\ref{prop:accuracy-calibration},
our surrogate calibration analysis can also be applied to the classification accuracy.
In addition, the $\tau$-discrepancy condition disappears from assumptions in the accuracy case,
which recovers the conditions Theorem~6 in \citet{Bartlett:2006}.
Even so, our analysis still remains to be sufficient conditions.
Further analysis towards the necessary and sufficient conditions in the general calibration analysis is left as an future work.

\subsection{Calibration Analysis of General Linear-fractional Metrics}
\label{sec:supp:general-calibration-analysis}

So far, we analyze the surrogate calibration for the F${}_\beta$-measure in Theorem~\ref{thm:f-measure-calibration},
and Jaccard index in Theorem~\ref{thm:jaccard-calibration}.
In addition, we take a look at how our analysis goes for the classification accuracy in Theorem~\ref{prop:accuracy-calibration}.
Now, we move on to the generalized result of the surrogate calibration which encompasses the entire linear-fractional metrics.
Let us consider the maximization of the true utility $\calU$ in Eq.~\eqref{eq:generalized-metric},
and the maximization of the corresponding surrogate utility $\calU_\phi$ in Eq.~\eqref{eq:surrogate-utility}.

\begin{theorem}[$\calU$-calibration in general case]
    \label{thm:general-calibration}
    Let $f^*$ be a measurable function that achieves $\calU_\phi(f^*) = \sup_f \calU_\phi(f) \doteq \calU_\phi^*$.
    Assume that a surrogate loss $\phi: \R \to \R_{\geq 0}$ is convex, non-increasing, and differentiable almost everywhere.
    On the true utility, we assume the following conditions.
    \begin{enumerate}
        \renewcommand{\labelenumi}{(\arabic{enumi}) }
        \item $\da_0 > 0$.
        \item $\da_1 \leq 0$.
        \item $a_{0,-1} \ne 0$ or $a_{1,-1} \ne 0$.
        \item $a_{1,-1} + a_{0,-1} \ne 0$.
        \item $a_{0,+1} a_{1,-1} + a_{0,-1} a_{1,+1} > 0$.
        \item If $a_{1,-1} > 0$, then $\calU(f^*) > -\frac{a_{0,-1}}{a_{1,-1}}$.
    \end{enumerate}
    Moreover, assume that there exists $\tau \in (0, 1)$ such that $\tau$ satisfies the following conditions.
    \begin{enumerate}
        \renewcommand{\labelenumi}{(\alph{enumi}) }
        \item $\phi$ is $\tau$-discrepant.
        \item $\calU_\phi^*$ satisfies
        \begin{align*}
            \tau \leq \frac{a_{0,+1} - a_{1,+1}}{a_{1,-1} + a_{0,-1}} \cdot \frac{-a_{0,-1} + a_{1,-1}\calU_\phi^*}{a_{0,+1} + a_{1,+1}\calU_\phi^*}
            .
        \end{align*}
        \label{enum:supp:tau:surrogate}
        \item $\calU_\phi^*$ and $\calU(f^*)$ satisfy
        \begin{align*}
            \tau \leq \frac{a_{0,-1} - a_{1,-1}\calU(f^*)}{a_{1,+1}\calU(f^*) - a_{0,+1}} \cdot \frac{a_{0,+1} + a_{1,+1}\calU_\phi^*}{-a_{0,-1} + a_{1,-1}\calU_\phi^*}
            .
        \end{align*}
        \label{enum:supp:tau:true}
    \end{enumerate}
    Then, the surrogate utility $\calU_\phi$ is $\calU$-calibrated.
\end{theorem}

The conditions~(1), (2), (3), (4), and (5) exclude \emph{pathological} true utilities which cannot be handled by the Bayes optimal analysis.
For instance, the Bayes optimal rule would be a classifier that always outputs positive values without the conditions~(1) and (2);
on the other hand, the Bayes optimal rule would be a classifier that always outputs negative values without the condition~(3).
The conditions~(6), (a), (b), and (c) force the surrogate utility $\calU_\phi$ to be calibrated to $\calU$.

Below, we give the proof of Theorem~\ref{thm:general-calibration}.

\begin{proof}[Proof of Theorem~\ref{thm:general-calibration}]
    We focus on the following surrogate utility $\calU_\phi$ as in Eq.~\eqref{eq:surrogate-utility}:
    \begin{align*}
        \calU_\phi(f)
        &= \frac{
            \int_\calX \left\{ a_{0,+1}(1 - \phi(-f(x)))\eta(x) + a_{0,-1}\phi(f(x))(1 - \eta(x)) + b_0 \right\} p(x)\d x
        }{
            \int_\calX \left\{ a_{1,+1}(1 + \phi(-f(x)))\eta(x) + a_{1,-1}\phi(f(x))(1 - \eta(x)) + b_1\right\} p(x) \d x
        }
        \\
        &= \frac{\E_X[W_{0,\phi}(f(X), \eta(X))]}{\E_X[W_{1,\phi}(f(X), \eta(X))]},
    \end{align*}
    where
    \begin{align*}
        W_{0,\phi}(\xi, q) &\doteq a_{0,+1}(1 - \phi(-\xi))q + a_{0,-1}\phi(\xi)(1 - q) + b_0,
        \\
        W_{1,\phi}(\xi, q) &\doteq a_{1,+1}(1 + \phi(-\xi))q + a_{1,-1}\phi(\xi)(1 - q) + b_1.
    \end{align*}
    Proposition~\ref{prop:bayes-optimal-set} tells us that the Bayes optimal set $\calB$ for the utility $\calU$ is
    \begin{align*}
        \calB = \{ f \mid f(x)\{(\da_0 - \da_1\calU(f))\eta(x) - (a_{1,-1}\calU(f) - a_{0,-1})\} > 0 \; \forall x \in \calX \}.
    \end{align*}
    We prove $\calU$-calibration by contradiction.
    Assume that
    \begin{align*}
        \sup_{f \not\in \calB} \calU_\phi(f) = \sup_f \calU_\phi(f).
    \end{align*}
    This implies that there exists $\bar{x} \in \calX$ such that $f^*(\bar{x})\{(\da_0 - \da_1\calU(f^*))\eta(\bar{x}) - (a_{1,-1}\calU(f^*) - a_{0,-1})\} \geq 0$.

    Let us describe the \emph{stationary condition} of $\calU_\phi$ at $f^*$ in the same manner as the proof of Theorem~\ref{thm:f-measure-calibration}.
    We introduce a function $\df$:
    \begin{align*}
        \df(x) \doteq \begin{cases}
            1 & \text{if $x = \bar{x}$}, \\
            0 & \text{if $x \ne \bar{x}$}.
        \end{cases}
    \end{align*}
    Let $G(\gamma) \doteq \calU_\phi(f^* + \gamma \df)$,
    then the stationary condition is $G'(0) = 0$.
    Here, $G'(0)$ is computed as
    \begin{align*}
        G'(0) &=
        \frac{1}{\E[W_{1,\phi}(f^*)]^2} \\ & \cdot \left\{
            \E[W_{1,\phi}(f^*)] \int_\calX (-a_{0,+1}\phi'(f^*(x))\eta(x) - a_{0,-1}\phi'(-f^*(x))(1 - \eta(x))) \df(x) p(x) \d x
            \right. \\ & \qquad \left.
            - \E[W_{0,\phi}(f^*)] \int_\calX (a_{1,+1}\phi'(f^*(x))\eta(x) - a_{1,-1}\phi'(-f^*(x))(1 - \eta(x))) \df(x) p(x) \d x
        \right\}
        \\
        &= \frac{1}{\E[W_{1,\phi}(f^*)]} \left\{
            \int_\calX (-a_{0,+1}\phi'(f^*(x))\eta(x) - a_{0,-1}\phi'(-f^*(x))(1 - \eta(x))) \df(x) p(x) \d x
            \right. \\ & \hspace{60pt} \left.
            - \calU_\phi^* \int_\calX (a_{1,+1}\phi'(f^*(x))\eta(x) - a_{1,-1}\phi'(-f^*(x))(1 - \eta(x))) \df(x) p(x) \d x
        \right\}
        \\
        &= \frac{1}{\E[W_{1,\phi}(f^*)]} \left\{
            (-a_{0,+1}\phi'(f^*(\bar{x}))\eta(\bar{x}) - a_{0,-1}\phi'(-f^*(\bar{x}))(1 - \eta(\bar{x})))
            \right. \\ & \hspace{75pt} \left. -
            \calU_\phi^* (a_{1,+1}\phi'(f^*(\bar{x}))\eta(\bar{x}) - a_{1,-1}\phi'(-f^*(\bar{x}))(1 - \eta(\bar{x})))
        \right\}
        ,
    \end{align*}
    where $\E[W_{0,\phi}(f^*)] = \E_X[W_{0,\phi}(f^*(X), \eta(X))]$
    and $\E[W_{1,\phi}(f^*)] = \E_X[W_{1,\phi}(f^*(X), \eta(X))]$.
    Thus, the stationary condition is
    \begin{align*}
        -a_{0,+1}\phi'(f^*(\bar{x}))\eta(\bar{x}) - a_{0,-1}\phi'(-f^*(\bar{x}))(1 - \eta(\bar{x})) & \\
        -a_{1,+1}\phi'(f^*(\bar{x}))\eta(\bar{x})\calU_\phi^* + a_{1,-1}\phi'(-f^*(\bar{x}))(1 - \eta(\bar{x}))& = 0
        ,
    \end{align*}
    which is equivalent to
    \begin{align}
        \eta(\bar{x}) = \frac{(-a_{0,-1} + a_{1,-1}\calU_\phi^*)\phi'(-f^*(\bar{x}))}{(a_{0,+1} - a_{1,+1}\calU_\phi^*)\phi'(f^*(\bar{x})) + (-a_{0,-1} + a_{1,-1}\calU_\phi^*)\phi'(-f^*(\bar{x}))}
        \quad (\doteq \bar\eta_\mathsf{STA})
        \label{eq:supp:stationary-condition}
        .
    \end{align}
    From now on, we divide the cases to take care of the Bayes optimal condition
    $f^*(\bar{x})\{(\da_0 - \da_1\calU(f^*))\eta(\bar{x}) - (a_{1,-1}\calU(f^*) - a_{0,-1})\} \geq 0$.
    Since $\da_0 - \da_1\calU(f^*) > 0$ due to $\da_0 > 0$ and $\da_1 \leq 0$,
    the Bayes optimal condition can be rewritten as $f^*(\bar{x})\{\eta(\bar{x}) - \frac{a_{1,-1}\calU(f^*) - a_{0,-1}}{\da_0 - \da_1\calU(f^*)}\} \geq 0$.
    \begin{enumerate}
        \renewcommand{\labelenumi}{\arabic{enumi}) }

        \item \underline{If $f^*(\bar{x}) > 0$ and $\eta(\bar{x}) < \frac{a_{1,-1}\calU(f^*) - a_{0,-1}}{\da_0 - \da_1\calU(f^*)}$:}
        \par
        Let $\bar\eta_\mathsf{OPT} \doteq \frac{a_{1,-1}\calU(f^*) - a_{0,-1}}{\da_0 - \da_1\calU(f^*)}$.
        Note that $a_{0,-1} - a_{1,-1}\calU(f^*) < 0$ and $-a_{0,-1} + a_{1,-1}\calU_\phi^* > 0$
        since $a_{0,-1} \leq 0$, $a_{1,-1} \geq 0$, and either $a_{0,-1}$ or $a_{1,-1}$ is non-zero (condition~(3)).
        We show the contradiction $\bar\eta_\mathsf{OPT} \leq \bar\eta_\mathsf{STA}$, which can be transformed as follows.
        \begin{align}
            \bar\eta_{\mathsf{OPT}} \leq \bar\eta_{\mathsf{STA}}
            &\Longleftrightarrow
            \frac{1}{1 + \frac{a_{1,+1}\calU(f^*) - a_{0,+1}}{a_{0,-1} - a_{1,-1}\calU(f^*)}} \leq \frac{1}{1 + \frac{a_{0,+1} + a_{1,+1}\calU_\phi^*}{-a_{0,-1} + a_{1,-1}\calU_\phi^*} \cdot \frac{\phi'(f^*(\bar{x}))}{\phi'(-f^*(\bar{x}))}},
            \nonumber
            \\
            &\Longleftrightarrow
            \frac{a_{1,+1}\calU(f^*) - a_{0,+1}}{a_{0,-1} - a_{1,-1}\calU(f^*)} \geq \underbrace{\frac{a_{0,+1} + a_{1,+1}\calU_\phi^*}{-a_{0,-1} + a_{1,-1}\calU_\phi^*}}_{> 0} \cdot \frac{\phi'(f^*(\bar{x}))}{\phi'(-f^*(\bar{x}))}.
            \nonumber
            \\
            &\Longleftrightarrow
            \underbrace{\frac{a_{1,+1}\calU(f^*) - a_{0,+1}}{a_{0,-1} - a_{1,-1}\calU(f^*)}}_{\doteq H(\calU(f^*))} \cdot \frac{-a_{0,-1} + a_{1,-1}\calU_\phi^*}{a_{0,+1} + a_{1,+1}\calU_\phi^*} \geq \frac{\phi'(f^*(\bar{x}))}{\phi'(-f^*(\bar{x}))}.
            \label{eq:supp:calibration:contradiction:1}
        \end{align}

        If $a_{1,-1} \ne 0$, we have
        \begin{align*}
            H(t) = \frac{\frac{a_{0,+1}a_{1,-1} + a_{0,-1}a_{1,+1}}{a_{1,-1}^2}}{t + \frac{a_{0,-1}}{a_{1,-1}}} - \frac{a_{1,+1}}{a_{1,-1}}.
        \end{align*}
        Since $\frac{a_{0,+1}a_{1,-1} + a_{0,-1}a_{1,+1}}{a_{1,-1}^2} > 0$,
        $H$ is monotonically decreasing on $-\frac{a_{0,-1}}{a_{1,-1}} < t \leq 1$.
        Together with the assumption $\calU(f^*) > -\frac{a_{0,-1}}{a_{1,-1}}$,
        we have $H(\calU(f^*)) \geq H(1) = \frac{a_{0,+1} - a_{1,+1}}{a_{1,-1} - a_{0,-1}}$.

        If $a_{1,-1} = 0$, $H(t) = \frac{a_{1,+1}}{a_{0,-1}}t - \frac{a_{0,+1}}{a_{0,-1}}$,
        noting that either $a_{0,-1}$ or $a_{a,-1}$ is non-zero (condition~(3)).
        Here, we have $H(\calU(f^*)) \geq H(1)$ as well since $H$ is a decreasing linear function.

        Since $\phi$ is $\tau$-discrepant and $\tau$ satisfies the condition~(b),
        \begin{align*}
            \frac{\phi'(f^*(\bar{x}))}{\phi'(-f^*(\bar{x}))}
            &\leq \tau
            \\
            &\leq \underbrace{\frac{a_{0,+1} - a_{1,+1}}{a_{1,-1} + a_{0,-1}}}_{=H(1)} \cdot \frac{-a_{0,-1} + a_{1,-1}\calU_\phi^*}{a_{0,+1} + a_{1,+1}\calU_\phi^*} \qquad \text{(using (b))}
            \\
            &\leq \underbrace{\frac{a_{1,+1}\calU(f^*) - a_{0,+1}}{a_{0,-1} - a_{1,-1}\calU(f^*)}}_{=H(\calU(f^*))} \cdot \frac{-a_{0,-1} + a_{1,-1}\calU_\phi^*}{a_{0,+1} + a_{1,+1}\calU_\phi^*}
            ,
        \end{align*}
        which concludes Eq.~\eqref{eq:supp:calibration:contradiction:1} and $\bar\eta_{\mathsf{OPT}} \leq \bar\eta_{\mathsf{STA}}$ (contradiction).

        \item \underline{If $f^*(\bar{x}) \leq 0$ and $\eta(\bar{x}) > \frac{a_{1,-1}\calU(f^*) - a_{0,-1}}{\da_0 - \da_1\calU(f^*)}$:}
        \par
        We show the contradiction $\bar\eta_{\mathsf{OPT}} \geq \bar\eta_{\mathsf{STA}}$,
        which can be transformed in the same way as Eq.~\eqref{eq:supp:calibration:contradiction:1} as follows.
        \begin{align}
            \frac{a_{1,+1}\calU(f^*) - a_{0,+1}}{a_{0,-1} - a_{1,-1}\calU(f^*)} \cdot \frac{-a_{0,-1} + a_{1,-1}\calU_\phi^*}{a_{0,+1} + a_{1,+1}\calU_\phi^*} &\leq \frac{\phi'(f^*(\bar{x}))}{\phi'(-f^*(\bar{x}))}.
            \nonumber
            \\
            \Longleftrightarrow
            \frac{a_{0,-1} - a_{1,-1}\calU(f^*)}{a_{1,+1}\calU(f^*) - a_{0,+1}} \cdot \frac{a_{0,+1} + a_{1,+1}\calU_\phi^*}{-a_{0,-1} + a_{1,-1}\calU_\phi^*} &\geq \frac{\phi'(-f^*(\bar{x}))}{\phi'(f^*(\bar{x}))}.
            \label{eq:supp:calibration:contradiction:2}
        \end{align}
        Note that $a_{1,+1}\calU(f^*) - a_{0,+1} > 0$ and $-a_{0,-1} + a_{1,-1}\calU_\phi^* > 0$
        since $a_{0,+1} \geq 0$, $a_{0,-1} \leq 0$, $a_{1,+1} \geq 0$, and $a_{1,-1} \geq 0$.
        Since $\phi$ is $\tau$-discrepant and $\tau$ satisfies the condition~(c),
        \begin{align*}
            \frac{\phi'(-f^*(\bar{x}))}{\phi'(f^*(\bar{x}))}
            \leq \tau
            \leq \frac{a_{0,-1} - a_{1,-1}\calU(f^*)}{a_{1,+1}\calU(f^*) - a_{0,+1}} \cdot \frac{a_{0,+1} + a_{1,+1}\calU_\phi^*}{-a_{0,-1} + a_{1,-1}\calU_\phi^*}
            ,
            \qquad \text{(using (c))}
        \end{align*}
        which concludes Eq.~\eqref{eq:supp:calibration:contradiction:2} and $\bar\eta_{\mathsf{OPT}} \geq \bar\eta_{\mathsf{STA}}$ (contradiction).
    \end{enumerate}

    Combining the above cases, it follows that
    \begin{align*}
        \sup_{f \not\in \calB} \calU_\phi(f) < \sup_f \calU_\phi(f).
    \end{align*}
    Eventually, we claim that $\calU_\phi$ is $\calU$-calibrated using Proposition~\ref{prop:bayes-optimal-set}.
\end{proof}

\subsection{Non-negativity of Optimal Surrogate Utilities}
\label{sec:supp:non-negativity-of-surrogate-utilities}

Here, we briefly discuss that the optimal surrogate utilities are non-negative even though the numerator can be negative.
Let us focus on the F${}_\beta$ case:
\begin{align*}
    W_{0,\phi}(\xi, q) &= (1 + \beta^2)(1 - \phi(\xi))q, \\
    W_{1,\phi}(\xi, q) &= (1 + \phi(\xi))q + \phi(-\xi)(1 - q) + \beta^2 \pi, \\
    \calU_\phi(f) &= \frac{\E_XW_{0,\phi}(f(X), \eta(X))}{\E_XW_{1,\phi}(f(X), \eta(X))},
\end{align*}
and let $f^*$ and $\check{f}$ be suprema of $\calU_\phi$ and $\E_X[W_0(f(X), \eta(X))]$ in $f$ within all measurable functions, respectively.
Then,
\begin{align*}
    \calU_\phi(f^*)
    \geq \calU_\phi(\check{f})
    &= \frac{\sup_{f'}\E_XW_{0,\phi}(f'(X), \eta(X))}{\E_X[W_{1,\phi}(\check{f}(X), \eta(X))]}
    \\
    &\overset{(a)}{=} \frac{\E_X[H_{0,\phi}(\eta(X))]}{\E_X[W_{1,\phi}(\check{f}(X), \eta(X))]}
    \\
    &= \frac{(1 + \beta^2)\pi}{\E_X[W_{1,\phi}(\check{f}(X), \eta(X))]}
    \\
    & \geq 0,
\end{align*}
where $H_{0,\phi}(q) \doteq \sup_{\xi \in \R} W_{0,\phi}(\xi, q)$.
The equality~(a) holds under a certain regularity condition~\citep[Lemma~2.5]{Steinwart:2007}.
Hence, we confirm that the optimal value of $\calU_\phi$ is non-negative.

The same discussion holds for the Jaccard case.


\newpage

\section{Proof of Quasi-concavity of the Surrogate Utility}
\label{sec:supp:quasi-concavity}

\begin{proof}[Proof of Lemma~\ref{lem:quasi-concavity}]
    Define an $\alpha$-super-level set of $\hat\calU_\phi$ restricted in $\bar\calF$ as $\calA_\alpha \doteq \{f \in \bar\calF \mid \hat\calU_\phi(f) \geq \alpha\}$.
    It is enough to show $\calA_\alpha$ is a convex set for any $\alpha \geq 0$ owing to $f \in \bar\calF$.

    Fix any $\alpha \geq 0$.
    Then,
    \begin{align*}
        \hat\calU_\phi(f) \geq \alpha
        &\Longleftrightarrow \frac{\frac{1}{m}\sum_{i=1}^m \tilde{W}_{0,\phi}(f(x_i), y_i)}{\frac{1}{n-m}\sum_{j=m+1}^n \tilde{W}_{1,\phi}(f(x_j), y_j)} \geq \alpha
        \\
        &\Longleftrightarrow \underbrace{\frac{1}{m}\sum_{i=1}^m \tilde{W}_{0,\phi}(f(x_i), y_i) - \alpha \frac{1}{n-m}\sum_{j=m+1}^n \tilde{W}_{1,\phi}(f(x_j), y_j)}_{(\ast)} \geq 0
        .
    \end{align*}
    Here, $\frac{1}{m}\sum_{i=1}^m \tilde{W}_{0,\phi}(f(x_i), y_i)$ is concave in $f$ since it is a non-negative sum of concave functions.
    Note that $\tilde{W}_{0,\phi}(f(x_i), y_i)$ is concave in $f$ for any $(x_i, y_i)$ due to the definition of $\tilde{W}_{0,\phi}$ in Eq.~\eqref{eq:surrogate-scores} and the assumption $\phi$ is convex.
    Similarly, $\frac{1}{n-m}\sum_{j=m+1}^n \tilde{W}_{1,\phi}(f(x_j), y_j)$ is convex as well.
    Thus, $(\ast)$ is concave in $f$,
    which means that $\calA_\alpha$ is a convex set since any super-level sets of a concave function is convex.

    Hence, we confirm that $\calA_\alpha$ is convex for any $\alpha \geq 0$.
\end{proof}


\newpage

\section{Proof of Uniform Convergence}
\label{sec:supp:proof-uniform-convergence}

First, we need carefully analyze our \emph{non-smooth} surrogate loss to take handle of the Rademacher complexity~\citep{Bartlett:2002},
which is defined as follows.
\begin{definition}[Rademacher complexity]
    \label{def:rademacher-complexity}
    Let $\calS \doteq \{z_1, \dots, z_n\}$ be a sample with size $n$.
    Let $\calG \doteq \{ g \mid \calZ \to \R \}$ be a class of measurable functions,
    and $\sigma \doteq (\sigma_1, \dots, \sigma_n)$ be the Rademacher variables,
    that is, random variables taking $+1$ and $-1$ with even probabilities.
    Then, the Rademacher complexity of $\calG$ of the sample size $n$ is defined as
    \begin{align*}
        \mathfrak{R}_n(\calG) \doteq \E_{\calS} \E_{\sigma} \left[
            \sup_{g \in \calG} \left|\frac{1}{n} \sum_{i=1}^n \sigma_i g(Z_i)\right|
        \right]
        .
    \end{align*}
\end{definition}
Usually, we analyze the Rademacher complexity of the composite function class
$\phi \circ \calF \doteq \{ (x, y) \mapsto \phi(yf(x)) \mid f \in \calF \}$
by using the Ledoux-Talagrand's contraction inequality~\citep{Ledoux:1991}
when the surrogate $\phi$ is Lipschitz continuous:
$\mathfrak{R}_n(\phi \circ \calF) \leq 2 \rho_\phi \mathfrak{R}_n(\calF)$,
where $\rho_\phi$ is the Lipschitz norm of $\phi$.
On the other hand, we need to deal with the case of the uniform convergence of gradients,
which requires smoothness of the surrogate,
while $\tau$-discrepant loss is non-smooth surrogates.
Thus, we need an alternative analysis.

\begin{lemma}
    \label{lem:non-smooth-rademacher}
    Assume that $\phi$ is $\tau$-discrepant and can be decomposed as $\phi(m) = \phi_{+1}(m)\ind{m > 0} + \phi_{-1}(m)\ind{m \leq 0}$.
    For $k = 0, 1$, denote $\tilde{W}_{k,\phi}' \circ \calF \doteq \{ (x, y) \mapsto \tilde{W}_{k,\phi}'(f(x), y) \mid f \in \calF \}$.
    Then,
    \begin{align*}
        \mathfrak{R}_n(\tilde{W}_{k,\phi}' \circ \calF) \leq 2(\gamma_{+1} + \gamma_{-1}) \mathfrak{R}_n(\calF).
    \end{align*}
\end{lemma}

\begin{proof}
    First, we prove for $k = 0$.
    Note that $\tilde{W}_{0,\phi}'(f(x), +1) = (1 - \phi(f(x)))' = -\phi'(f(x))$,
    and that $\tilde{W}_{0, \phi}'(f(x), -1) = (-\phi(-f(x)))' = \phi'(-f(x))$,
    thus, $\tilde{W}_{0,\phi}'(f(x), y) = -y\phi'(yf(x))$.
    \begin{align*}
        &\hspace{-20pt}\mathfrak{R}_n(\tilde{W}_{0,\phi}' \circ \calF)
        \\
        &= \E_{\calS, \sigma} \left[
            \sup_{f \in \calF} \left|\frac{1}{n} \sum_{i=1}^n \sigma_i\tilde{W}_{0,\phi}'(f(x_i), y_i)\right|
        \right]
        \\
        &= \E_{\calS, \sigma} \left[
            \sup_{f \in \calF} \left|\frac{1}{n} \sum_{i=1}^n \sigma_i(-y_i\phi'(y_if(x_i)))\right|
        \right]
        \\
        &= \E_{\calS, \sigma} \left[
            \sup_{f \in \calF} \left|\frac{1}{n} \sum_{i=1}^n \sigma_i\phi'(y_if(x_i))\right|
        \right]
        \\ & \qquad \text{($\because$ $\sigma_i$ and $-\sigma_iy_i$ are distributed in the same way for a fixed $y_i$)}
        \\
        &= \E_{\calS, \sigma} \left[
            \sup_{f \in \calF} \left|\frac{1}{n} \sum_{i=1}^n \sigma_i\left\{\phi_{-1}'(y_if(x_i))\ind{y_if(x_i) \leq 0} + \phi_{+1}'(y_if(x_i))\ind{y_if(x_i) > 0}\right\} \right|
        \right]
        \\
        &\leq \underbrace{\E_{\calS, \sigma} \left[
            \sup_{f \in \calF} \left|\frac{1}{n} \sum_{i=1}^n \sigma_i\phi_{-1}'(y_if(x_i))\ind{y_if(x_i) \leq 0}\right|
        \right]}_\text{(A)} \\ & \qquad +
        \underbrace{\E_{\calS, \sigma} \left[
            \sup_{f \in \calF} \left|\frac{1}{n} \sum_{i=1}^n \sigma_i\phi_{+1}'(y_if(x_i))\ind{y_if(x_i) > 0}\right|
        \right]}_\text{(B)}
        ,
    \end{align*}
    where the last inequality is just the triangular inequality.
    For (A), let $\psi_{-1}(m) \doteq \phi_{-1}'(m)\frac{m}{|m|}$ if $m \ne 0$,
    and $\psi_{-1}(0) \doteq 0$.
    Since
    \begin{align*}
        \psi_{-1}'(m)
        &= \frac{(\psi_{-1}'(m)m)'|m| - \phi_{-1}'(m)m \cdot (|m|)'}{m^2}
        \\
        &= \frac{\phi_{-1}''(m)m|m| + \phi_{-1}'(m)|m| - \phi_{-1}'(m)m\cdot\frac{m}{|m|}}{m^2}
        \\
        &= \frac{\phi_{-1}''(m)m^3 + \phi_{-1}'(m)m^2 - \phi_{-1}'(m)m^2}{m^2 |m|}
        \\
        &= \phi_{-1}''(m)\frac{m}{|m|},
    \end{align*}
    the Lipschitz norm of $\psi_{-1}$ can be computed as
    \begin{align*}
    \sup_{f \in \calF, (x, y) \in \calX \times \calY}|\psi_{-1}'(f(x))|
    &= \sup_{f,x,y}\left|\phi_{-1}''(yf(x))\frac{yf(x)}{|yf(x)|}\right|
    \\
    &= \sup_{f,x,y}|\phi_{-1}''(yf(x))| \cdot \sup_{f,x,y}\left|\frac{yf(x)}{|yf(x)|}\right|
    \\
    &= \gamma_{-1}.
    \end{align*}
    Note that the Lipschitz norm of $\phi_{-1}'$ is $\gamma_{-1}$ because $\phi_{-1}$ is $\gamma_{-1}$-smooth.
    Then, we further bound (A) by using the fact $\ind{y_if(x_i) \leq 0} = \frac{1 - \frac{y_if(x_i)}{|y_if(x_i|}}{2}$.
    \begin{align*}
        \text{(A)}
        &= \E_{\calS, \sigma} \left[
            \sup_{f \in \calF} \left|\frac{1}{n} \sum_{i=1}^n \sigma_i\phi_{-1}'(y_if(x_i))\frac{1 - \frac{y_if(x_i)}{|y_if(x_i)|}}{2} \right|
        \right]
        \\
        &\leq \frac{1}{2} \E_{\calS, \sigma} \left[
            \sup_{f \in \calF} \left|\frac{1}{n} \sum_{i=1}^n \sigma_i\phi_{-1}'(y_if(x_i)) \right|
        \right] + \frac{1}{2} \E_{\calS, \sigma} \left[
            \sup_{f \in \calF} \left|\frac{1}{n} \sum_{i=1}^n \sigma_i\phi_{-1}'(y_if(x_i))\frac{y_if(x_i)}{|y_if(x_i)|} \right|
        \right]
        \\ & \hspace{40pt} \text{($\because$ triangular inequality)}
        \\
        &= \frac{1}{2} \E_{\calS, \sigma} \left[
            \sup_{f \in \calF} \left|\frac{1}{n} \sum_{i=1}^n \sigma_i\phi_{-1}'(y_if(x_i)) \right|
        \right] + \frac{1}{2} \E_{\calS, \sigma} \left[
            \sup_{f \in \calF} \left|\frac{1}{n} \sum_{i=1}^n \sigma_i\psi_{-1}(y_if(x_i))\frac{y_if(x_i)}{|y_if(x_i)|} \right|
        \right]
        \\
        &= \frac{1}{2}\mathfrak{R}_n(\phi_{-1}' \circ \calF) + \frac{1}{2}\mathfrak{R}_n(\psi_{-1} \circ \calF)
        \\
        &\leq \frac{1}{2}\cdot 2\gamma_{-1}\mathfrak{R}_n(\calF) + \frac{1}{2} \cdot 2\gamma_{-1}\mathfrak{R}_n(\calF)
        \\
        &= 2\gamma_{-1}\mathfrak{R}_n(\calF)
        ,
    \end{align*}
    where the inequality is the result of the Ledoux-Talagrand's contraction inequality~\citep[Theorem~4.12]{Ledoux:1991}.
    Note that both $\phi_{-1}'$ and $\psi_{-1}$ are $\gamma_{-1}$-Lipschitz.
    We can prove that (B) is bounded by $\gamma_{+1}\mathfrak{R}_n(\calF)$ from the above as well.
    Therefore, the claim is supported.
    We can prove the case $k = 1$ in the same manner.
\end{proof}

Now, we move on to the proof of Lemma~\ref{lem:uniform-convergence}.

\begin{proof}[Proof of Lemma~\ref{lem:uniform-convergence}]
    We write $\calV_\phi(f_\theta)$ as $\calV_\phi(\theta)$.
    If we explicit note for which sample we take the empirical average in $\hat\calV_\phi(\theta)$,
    let us write $\hat\calV_\phi(\theta; \calS)$.
    Let $\calE(\calS) \doteq \sup_{\theta \in \Theta} \|\hat\calV_\phi(\theta; \calS) - \calV_\phi(\theta)\|$.
    For simplicity, we write $z_i \doteq (x_i, y_i)$ and $\tilde{W}_{0,\phi}(\theta; z_i) \doteq \tilde{W}_{0,\phi}(f_\theta(x_i), y_i)$.
    First, we observe $\calE(\calS)$ admits the bounded difference property~\citep{McDiarmid:1989}.

    Denote that $\calS \doteq \{z_i\}_{i=1}^n$ and $\calS' \doteq \{z_1, \dots, z_k', \dots, z_n\}$.
    If $1 \leq k \leq m$,
    \begin{align*}
        \sup_{\calS \subset \calX \times \calY, z_k' \in \calX \times \calY} & |\calE(\calS) - \calE(\calS')|
        \\
        &\doteq \sup_{\calS, z_k'} \left|
            \sup_{\theta \in \Theta}\|\hat\calV_\phi(\theta;\calS) - \calV_\phi(\theta)\| -
            \sup_{\theta \in \Theta}\|\hat\calV_\phi(\theta;\calS') - \calV_\phi(\theta)\|
        \right|
        \\
        &\leq \sup_{\calS, z_k', \theta} \|\hat\calV_\phi(\theta;\calS) - \hat\calV_\phi(\theta;\calS')\|
        \qquad \text{($\because$ triangular inequality)}
        \\
        &= \frac{1}{m(n-m)} \sup_{\calS, z_k', \theta} \left\|
            \{\nabla\tilde{W}_{0,\phi}(\theta;z_k) - \nabla\tilde{W}_{0,\phi}(\theta;z_k')\}
            \sum_{j=m+1}^n \tilde{W}_{1,\phi}(\theta;z_j)
            \right. \\ & \hspace{90pt} \left.
            - \{\tilde{W}_{0,\phi}(\theta;z_k) - \tilde{W}_{0,\phi}(\theta;z_k')\} \sum_{j=m+1}^n \nabla\tilde{W}_{1,\phi}(\theta;z_j)
        \right\|
        \\
        &\leq \frac{1}{m(n-m)} \sup_{\calS, z_k', \theta} \left\{
            \left(\|\nabla\tilde{W}_{0,\phi}(\theta;z_k)\| + \|\nabla\tilde{W}_{0,\phi}(\theta;z_k')\|\right) \sum_{j=m+1}^n |\tilde{W}_{1,\phi}(\theta;z_j)|
            \right. \\ & \hspace{90pt} \left.
            + \left(|\tilde{W}_{0,\phi}(\theta;z_k)| + |\tilde{W}_{0,\phi}(\theta;z_k')|\right) \sum_{j=m+1}^n \|\nabla\tilde{W}_{1,\phi}(\theta;z_j)\|
        \right\}
        \\
        &\leq \frac{
            2\rho_0 c_\calX \cdot (n-m)c_1 + 2c_0 \cdot (n-m) \rho_1 c_\calX
        }{m(n-m)}
        \\
        &= \frac{4c_\calX(\rho_1 c_0 + \rho_0 c_1)}{n},
    \end{align*}
    where the second inequality also holds due to the triangular inequality,
    and the last inequality follows from the fact that $\tilde{W}_{0,\phi}$ and $\tilde{W}_{1,\phi}$ are $\rho_0$-/$\rho_1$-Lipschitz and bounded by $c_0$ and $c_1$, respectively.
    The same holds for the case $m + 1 \leq k \leq n$.

    Thus, $\calE$ is the bounded difference with a constant $(4c_\calX(\rho_1 c_0 + \rho_0 c_1)) / n$ for each index,
    and we can obtain the following inequality by McDiarmid's inequality~\citep{McDiarmid:1989}:
    \begin{align*}
        \P[\calE(\calS) - \E_\calS[\calE(\calS)] > \epsilon]
        \leq 2\exp\left(-\frac{n\epsilon^2}{8c_\calX^2(\rho_1 c_0 + \rho_0 c_1)^2}\right),
    \end{align*}
    which is equivalent to
    \begin{align*}
        \calE(\calS) - \E_\calS[\calE(\calS)] \leq \sqrt{\frac{8c_\calX^2(\rho_1 c_0 + \rho_0 c_1)^2\log\frac{2}{\delta}}{n}},
    \end{align*}
    with probability at least $1 - \delta$.

    Next, we bound $\E_\calS[\calE(\calS)]$ by the \emph{symmetrization device}~\citep[Lemma~6.3]{Ledoux:1991}.
    \begin{align}
        \E_\calS[\calE(\calS)]
        &= \E_\calS\left[
            \sup_{\theta \in \Theta} \|\hat\calV_\phi(\theta;\calS) - \calV_\phi(\theta)\|
        \right]
        \nonumber \\
        &\leq \E_\calS
            \underbrace{\sup_\theta \left\|\frac{1}{m(n-m)} \sum_{i=1}^m \sum_{j=m+1}^n \tilde{W}_{1,\phi}(\theta;z_j) \nabla\tilde{W}_{0,\phi}(\theta;z_i) - \E[W_{1,\phi} \nabla W_{0,\phi}] \right\|}_{\text{(A)}}
            \nonumber \\ & \hspace{10pt}
             + \E_\calS \underbrace{\sup_\theta \left\|\frac{1}{m(n-m)} \sum_{i=1}^m \sum_{j=m+1}^n \tilde{W}_{0,\phi}(\theta;z_i) \nabla\tilde{W}_{1,\phi}(\theta;z_j) - \E[W_{0,\phi} \nabla W_{1,\phi}] \right\|}_{\text{(B)}}
        ,
        \label{eq:supp:expectation-difference}
    \end{align}
    where the second line is the result of the triangular inequality,
    and
    \begin{align*}
        &\E_\calS[\text{(A)}] \\
        &= \E_\calS \sup_\theta \left\|
            \frac{1}{m(n-m)} \sum_{i=1}^m \sum_{j=m+1}^n \tilde{W}_{1,\phi}(\theta;z_j) \left( \nabla\tilde{W}_{0,\phi}(\theta;z_i) - \E[\nabla W_{0,\phi}] \right)
            \right. \\ & \hspace{45pt} \left.
            + \frac{1}{m(n-m)} \sum_{i=1}^m \sum_{j=m+1}^n \E[\nabla W_{0,\phi}] \left(\tilde{W}_{1,\phi}(\theta;z_j) - \E[W_{1,\phi}]\right)
        \right\|
        \\
        &\leq \E_\calS \sup_\theta \left\{
            \frac{1}{m(n-m)} \sum_{j=m+1}^n |\tilde{W}_{1,\phi}(\theta;z_j)| \cdot \left\| \sum_{i=1}^m \nabla\tilde{W}_{0,\phi}(\theta;z_i) - \E[\nabla W_{0,\phi}] \right\|
            \right. \\ & \hspace{45pt} \left.
            + \frac{1}{m(n-m)} \sum_{i=1}^m \|\E[\nabla W_{0,\phi}]\| \cdot \left| \sum_{j=m+1}^n \tilde{W}_{1,\phi}(\theta;z_j) - \E[W_{1,\phi}]\right|
        \right\}
        \\
        &\leq \E_\calS \left[ \sup_\theta \left\{
            c_1 \left\| \frac{1}{m} \sum_{i=1}^m \nabla\tilde{W}_{0,\phi}(\theta;z_i) - \E[\nabla W_{0,\phi}]\right\|
            + \rho_0 c_\calX \left| \frac{1}{n-m}\sum_{j=m+1}^n \tilde{W}_{1,\phi}(\theta;z_j) - \E[W_{1,\phi}]\right|
        \right\} \right]
        \\
        &= c_1 \underbrace{ \E_\calS \left[ \sup_\theta \left\| \frac{1}{m} \sum_{i=1}^m \nabla\tilde{W}_{0,\phi}(\theta;z_i) - \E[\nabla W_{0,\phi}]\right\| \right] }_{\text{(A')}}
        + \rho_0 c_\calX \underbrace{ \E_\calS \left[ \sup_\theta \left| \frac{1}{n-m}\sum_{j=m+1}^n \tilde{W}_{1,\phi}(\theta;z_j) - \E[W_{1,\phi}]\right| \right] }_{\text{(A'')}}
        ,
    \end{align*}
    where the first inequality is the triangular inequality.
    Now we introduce the Rademacher random variables $\sigma_{1:n} \doteq \{\sigma_1, \dots, \sigma_n\}$ that are independently and uniformly distributed on $\{+1, -1\}$.
    \begin{itemize}
        \item
        For (A'), we can bound it from the above by the symmetrization device and the fact that $\|\cdot\|_2 \leq \|\cdot\|_1$.
        \begin{align*}
            \text{(A')}
            &= \E_\calS \left[ \sup_\theta \left\| \frac{1}{m} \sum_{i=1}^m \nabla\tilde{W}_{0,\phi}(\theta;z_i) - \E[\nabla W_{0,\phi}]\right\| \right]
            \\
            &\leq \E_{\calS, \sigma_{1:m}} \left[
                \sup_\theta \sum_{l=1}^d \left| \frac{1}{n} \sum_{i=1}^m \nabla_{\theta_l} \tilde{W}_{0,\phi}(\theta;z_i) - \E[\nabla_{\theta_l} W_{0,\phi}]\right|
            \right]
            && \text{($\|\cdot\|_2 \leq \|\cdot\|_1$)}
            \\
            &\leq \sum_{l=1}^d \E_{\calS, \sigma_{1:m}} \left[\sup_\theta \left|\frac{2}{m} \sum_{i=1}^m\sigma_i\nabla_{\theta_l}\tilde{W}_{0,\phi}(\theta;z_i) \right| \right]
            && \text{(symmetrization device)}
            \\
            &= \sum_{l=1}^d 2 \E_{\calS, \sigma_{1:m}} \left[ \sup_\theta \left|\frac{1}{m} \sum_{i=1}^m \sigma_i\tilde{W}_{0,\phi}'(\theta;z_i) \cdot x_l\right| \right]
            && \text{($f_\theta(x) = \theta^\top x$)}
            \\
            &\leq \sum_{l=1}^d 2 \E_{\calS, \sigma_{1:m}} \left[ \sup_\theta \left|\frac{1}{m} \sum_{i=1}^m \sigma_i\tilde{W}_{0,\phi}'(\theta;z_i)\right| \cdot c_\calX \right]
            && \text{($|x_l| \leq \|x\| \leq c_\calX$ $\forall x \in \calX$)}
            \\
            &\leq 4 d c_\calX (\gamma_{+1} + \gamma_{-1}) \mathfrak{R}_m(\calF)
            \\
            &= 4 d c_\calX (\gamma_{+1} + \gamma_{-1}) \mathfrak{R}_{n/2}(\calF),
        \end{align*}
        where the last inequality uses Lemma~\ref{lem:non-smooth-rademacher}.

        \item
        For (A''), we can bound it from the above by the symmetrization device.
        \begin{align*}
            \text{(A'')}
            &= \E_\calS \left[ \sup_\theta \left| \frac{1}{n-m}\sum_{j=m+1}^n \tilde{W}_{1,\phi}(\theta;z_j) - \E[W_{1,\phi}]\right| \right]
            \\
            &\leq \E_{\calS, \sigma_{m:n-m}} \left[ \sup_\theta \left|\frac{2}{n-m} \sum_{j=m+1}^n \sigma_j\tilde{W}_{1,\phi}(\theta;z_i) \right| \right]
            && \text{(symmetrization device)}
            \\
            &\leq 4\rho_1 \mathfrak{R}_{n-m}(\calF)
            && \text{(contraction inequality)}
            \\
            &= 4\rho_1 \mathfrak{R}_{n/2}(\calF),
        \end{align*}
        where the second inequality uses the Ledoux-Talagrand's contraction inequality~\citep[Theorem~4.12]{Ledoux:1991},
        together with the fact that $\tilde{W}_{1,\phi}$ is $\rho_1$-Lipschitz continuous.
    \end{itemize}

    Thus, Eq.~\eqref{eq:supp:expectation-difference} can be bounded as follows.
    \begin{align*}
        &\E_\calS[\calE(\calS)]
        \\
        &\leq c_1 \text{(A')} + \rho_0 c_\calX \text{(A'')} + \E_\calS[\text{(B)}]
        \\
        &\leq 4dc_\calX c_1(\gamma_{+1} + \gamma_{-1}) \mathfrak{R}_{n/2}(\calF) + 4\rho_0 \rho_1 c_\calX \mathfrak{R}_{n/2}(\calF) + \underbrace{
            4dc_\calX c_0(\gamma_{+1} + \gamma_{-1}) \mathfrak{R}_{n/2}(\calF) + 4\rho_1 \rho_0 c_\calX \mathfrak{R}_{n/2}(\calF)
        }_\text{can be proven in the same manner as (A)}
        \\
        &= (4c_\calX c_0 d \gamma + 4c_\calX c_1 d \gamma + 8\rho_0\rho_1c_\calX)\mathfrak{R}_{n/2}(\calF)
        \hspace{40pt} (\gamma \doteq \gamma_{+1} + \gamma_{-1})
        \\
        &\leq (4c_\calX c_0 d \gamma + 4c_\calX c_1 d \gamma + 8\rho_0\rho_1c_\calX)\frac{\sqrt{2}c_\calX c_\Theta}{\sqrt{n}}
        ,
    \end{align*}
    where the last inequality comes from \citet[Theorem~4.3]{Mohri:2012},
    which results in $\mathfrak{R}_n(\calF) = \mathfrak{R}_n(\{x \mapsto \theta^\top x \mid \theta \in \Theta\}) \leq c_\calX c_\Theta / \sqrt{n}$.

    After all, we obtain the desired uniform bound:
    with probability at least $1 - \delta$,
    \begin{align*}
        \sup_{\theta \in \Theta} & \|\hat\calV_\phi(\theta;\calS) - \calV_\phi(\theta)\|
        = \calE(\calS)
        \\
        &\leq \E_\calS[\calE(\calS)] + \frac{\sqrt{8}c_\calX(\rho_1 c_0 + \rho_0 c_1)\sqrt{\log\frac{2}{\delta}}}{\sqrt{n}}
        \\
        &\leq \frac{(4c_\calX c_0 d \gamma + 4c_\calX c_1 d \gamma + 8\rho_0\rho_1c_\calX) + \sqrt{8}c_\calX(\rho_1 c_0 + \rho_0 c_1)\sqrt{\log\frac{2}{\delta}}}{\sqrt{n}}
        .
    \end{align*}
\end{proof}


\newpage

\section{Experimental Results}
\label{sec:supp:experiment}

\subsection{Details of Datasets}

Datasets that we use throughout this section are obtained from the \emph{UCI Machine Learning Repository}~\citep{Lichman:2013} and the \emph{LIBSVM}~\citep{CC01a}.
For those which have independent training data, validation data, and test data, all of them are merged into one dataset.
We randomly split the original data with the ratio $8:2$,
and the former is used for training while the latter is used for evaluation.
Each feature value is scaled between zero and one.

\begin{table*}[h]
    \centering
    \caption{
        Details of datasets.
    }
    \label{tab:supp:datasets}
    \scalebox{0.8}{
    \begin{tabular}{cccc} \hline
        Dataset & dimension & sample size & class-prior
        \\ \hline
        adult & 123 & 48842 & 0.239 \\
        australian & 14 & 690 & 0.445 \\
        breast-cancer & 10 & 683 & 0.350 \\
        cod-rna & 8 & 331152 & 0.333 \\
        diabetes & 8 & 768 & 0.651 \\
        german.numer & 24 & 1000 & 0.300 \\
        heart & 13 & 270 & 0.444 \\
        ionosphere & 34 & 351 & 0.641 \\
        mushrooms & 112 & 8124 & 0.482 \\
        phishing & 68 & 11055 & 0.557 \\
        phoneme & 5 & 5404 & 0.293 \\
        skin\_nonskin & 3 & 245057 & 0.208 \\
        sonar & 60 & 208 & 0.394 \\
        spambase & 57 & 4601 & 0.394 \\
        splice & 60 & 1000 & 0.517 \\
        w8a & 300 & 64700 & 0.030 \\
        \hline
    \end{tabular}
    }
\end{table*}

\subsection{Details of Baseline Methods}
\label{sec:supp:baseline}

We describe the details of baseline methods.
Baselines~2 and~3 are also mentioned in Sec.~\ref{sec:related}.

\textbf{Baseline 1 (ERM):}
The first baseline is the vanilla empirical risk minimization, which does not optimize the metric of our interest but accuracy.
The hinge loss and $\ell_2$-regularization are employed with the regularization parameter $10^{-2}$.

\textbf{Baseline 2 (W-ERM):}
Weighted empirical risk minimization, or cost-sensitive empirical risk minimization,
is often used to optimize non-linear performance metrics~\citep{Koyejo:2014,Narasimhan:2014,Parambath:2014}.
Here, we applied a simple approach:
Find a cost parameter from a given cost parameter space,
which gives the maximum validation performance of a classifier trained by the cost-sensitive empirical risk minimization~\citep{Scott:2012}.
The training dataset is split to 4 to 1 at random, and the latter is saved for validation of a regularization parameter.
The former set is further split to 9 to 1 at random,
and the former 90\% is used for training the base classifier,
while the latter 10\% is used for the validation.
As the base cost-sensitive learner, we use the hinge loss minimizer with $\ell_2$-regularization (a regularization parameter is chosen from $\{10^{-1}, 10^{-3}, 10^{-5}\}$ by cross validation).
The cost parameter is chosen from the range $[10^{-3}, 1 - 10^{-3}]$ evenly split to 20 ranges,
that is, $\left\{10^{-3} + \frac{1 - 2 \cdot 10^{-3}}{20}i \mid i = 1, \dots, 20\right\}$.

\textbf{Baseline 3 (Plug-in):}
Plug-in estimator is one of the other common methods to optimize the non-linear performance metrics~\citep{Koyejo:2014,Yan:2018},
which is the two-step method:
To estimate the class posterior probability $\hat\eta(x) = p(y=+1|x)$ first,
and then to decide the optimal threshold $\hat\delta$.
The classifier is constructed as $x \mapsto \sign(\hat\eta(x) - \hat\delta)$.
The training dataset is split to 4 to 1 at random, and the latter is saved for validation of a regularization parameter.
The former set is further split to 9 to 1 at random,
and they are independently used for the first and second step.
For estimating $p(y=+1|x)$ (the first step), the logistic regression is used~\citep{Reid:2009},
with $\ell_2$-regularization (a regularization parameter is chosen from $\{10^{-1}, 10^{-3}, 10^{-5}\}$ by cross validation).
For deciding $\hat\delta$, we pick a threshold with the highest validation metric
from $\left\{10^{-3} + \frac{1 - 2 \cdot 10^{-3}}{20}i \mid i = 1, \dots, 20\right\}$.

\subsection{Convergence Comparison}
\label{sec:supp:convergence-comparison}

Figures~\ref{fig:supp:f1-itertest} and~\ref{fig:supp:jac-itertest} are the full version of the convergence comparison of U-GD and U-BFGS.
Figure~\ref{fig:supp:f1-itertest} shows the result of F${}_1$-measure,
and Figure~\ref{fig:supp:jac-itertest} shows the result of Jaccard index.
The vertical axes show test metric values, where the higher the better.
Note that both F${}_1$-measure and Jaccard index ranges over zero to one.
The horizontal axes show the number of iterations.
For each dataset, metric, and method, we ran 300 iterations to see their convergence behaviors.

Overall, U-BFGS shows faster convergence than U-GD in terms of the number of iterations.
In almost all cases, U-BFGS converges within 30 iterations, except german.numer and mushrooms in Jaccard case.
Moreover, it usually achieves higher performance than U-GD.
U-GD convergences require at least around 100 iterations (mushrooms and phishing in F${}_1$ case),
and sometimes it does not converge even within 300 iterations such as heart and ionosphere in F${}_1$ and Jaccard cases.

\begin{figure}[h]
    \centering
    \begin{minipage}{0.32\columnwidth}
        \includegraphics[width=\columnwidth]{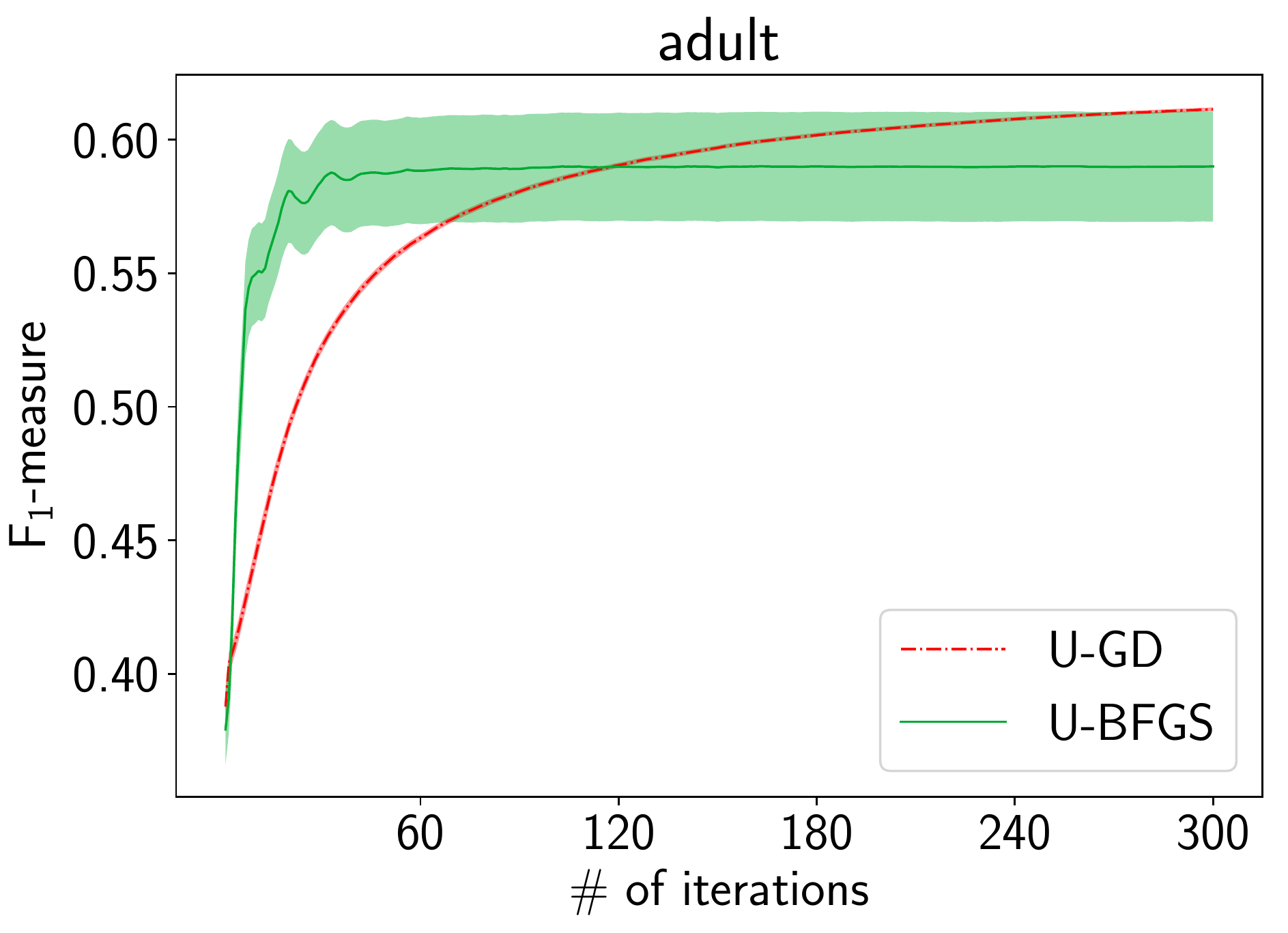}
    \end{minipage}
    \begin{minipage}{0.32\columnwidth}
        \includegraphics[width=\columnwidth]{figures/itertest/itertest_f1_australian.pdf}
    \end{minipage}
    \begin{minipage}{0.32\columnwidth}
        \includegraphics[width=\columnwidth]{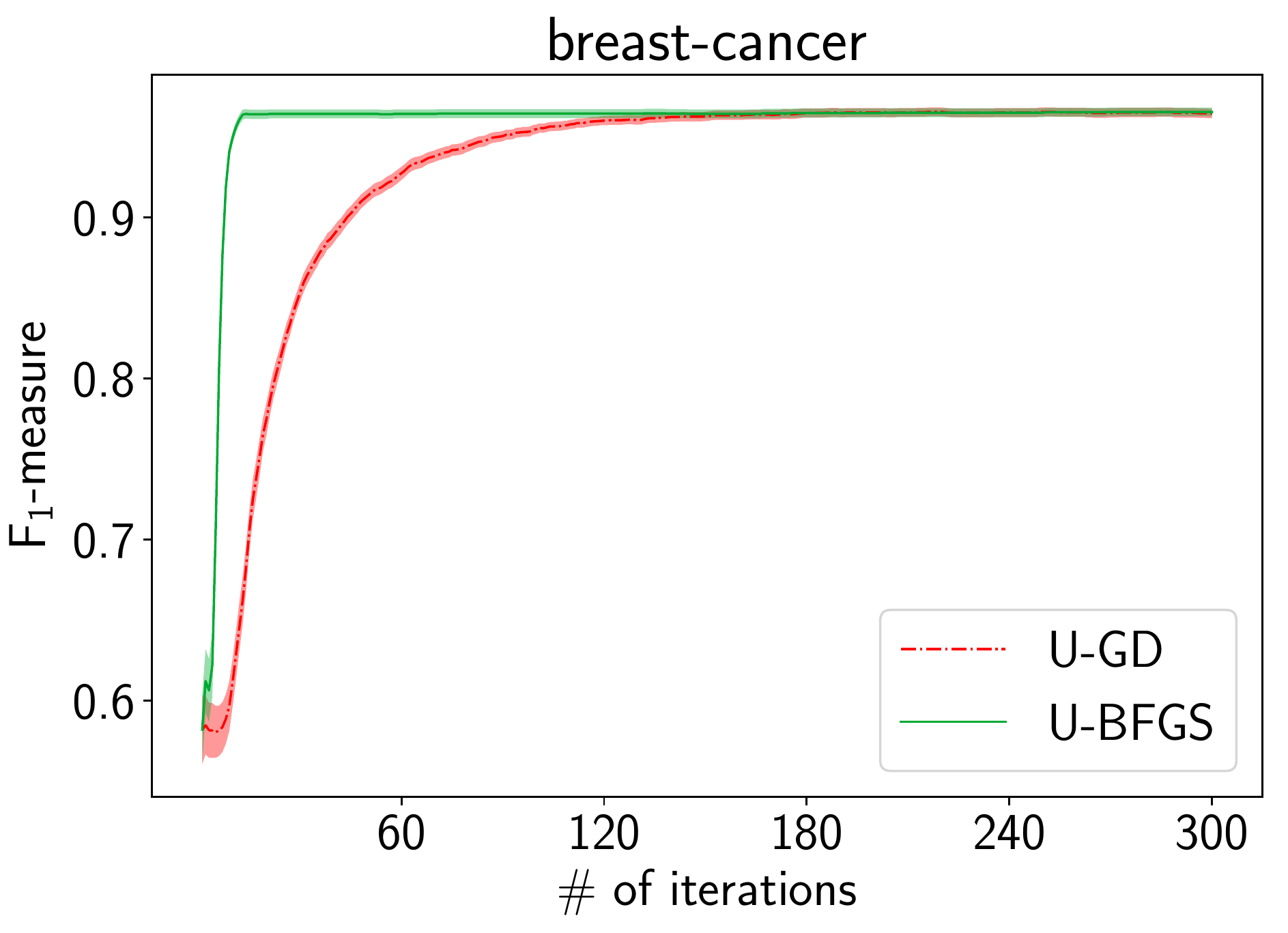}
    \end{minipage}
    \begin{minipage}{0.32\columnwidth}
        \includegraphics[width=\columnwidth]{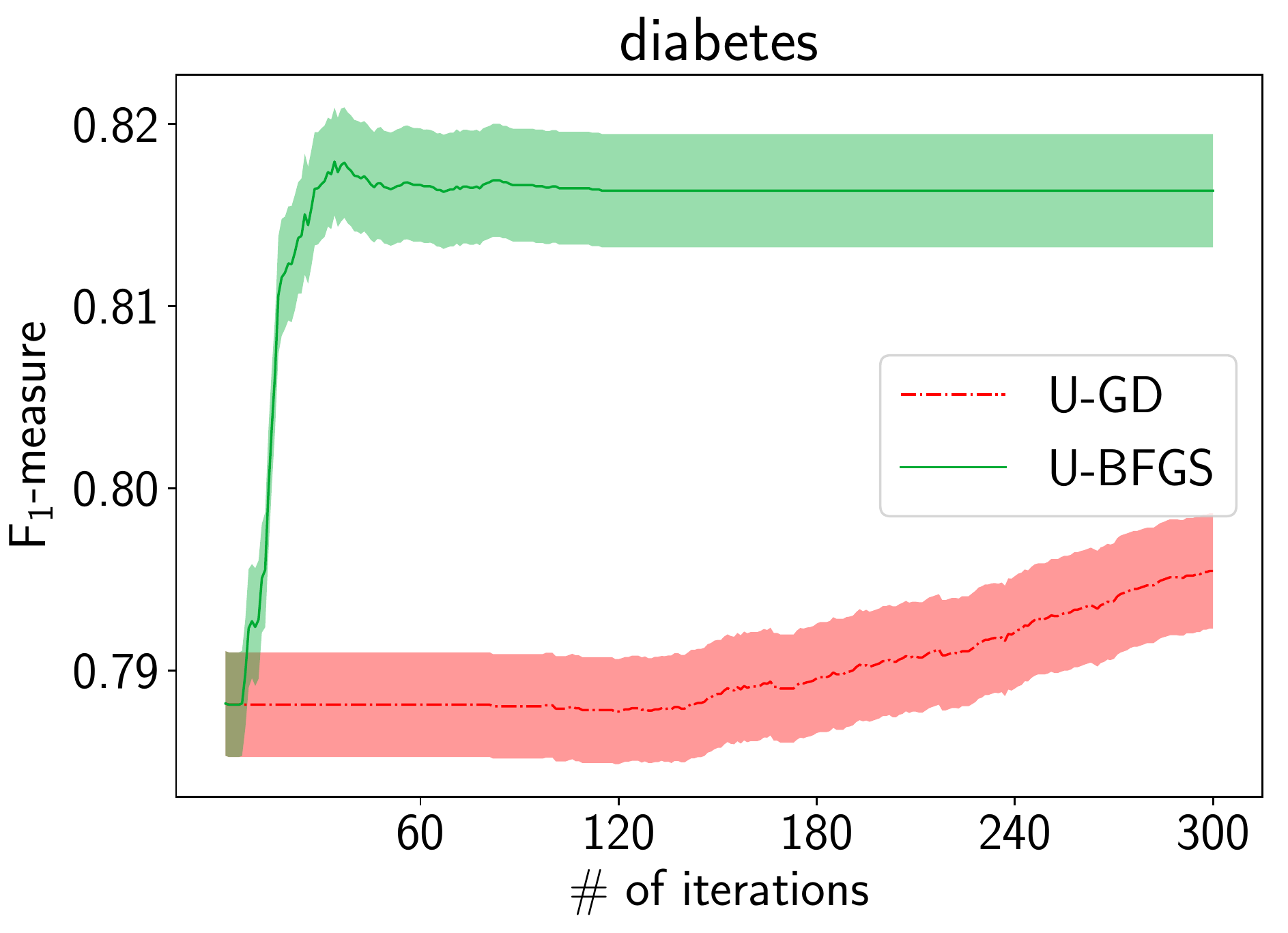}
    \end{minipage}
    \begin{minipage}{0.32\columnwidth}
        \includegraphics[width=\columnwidth]{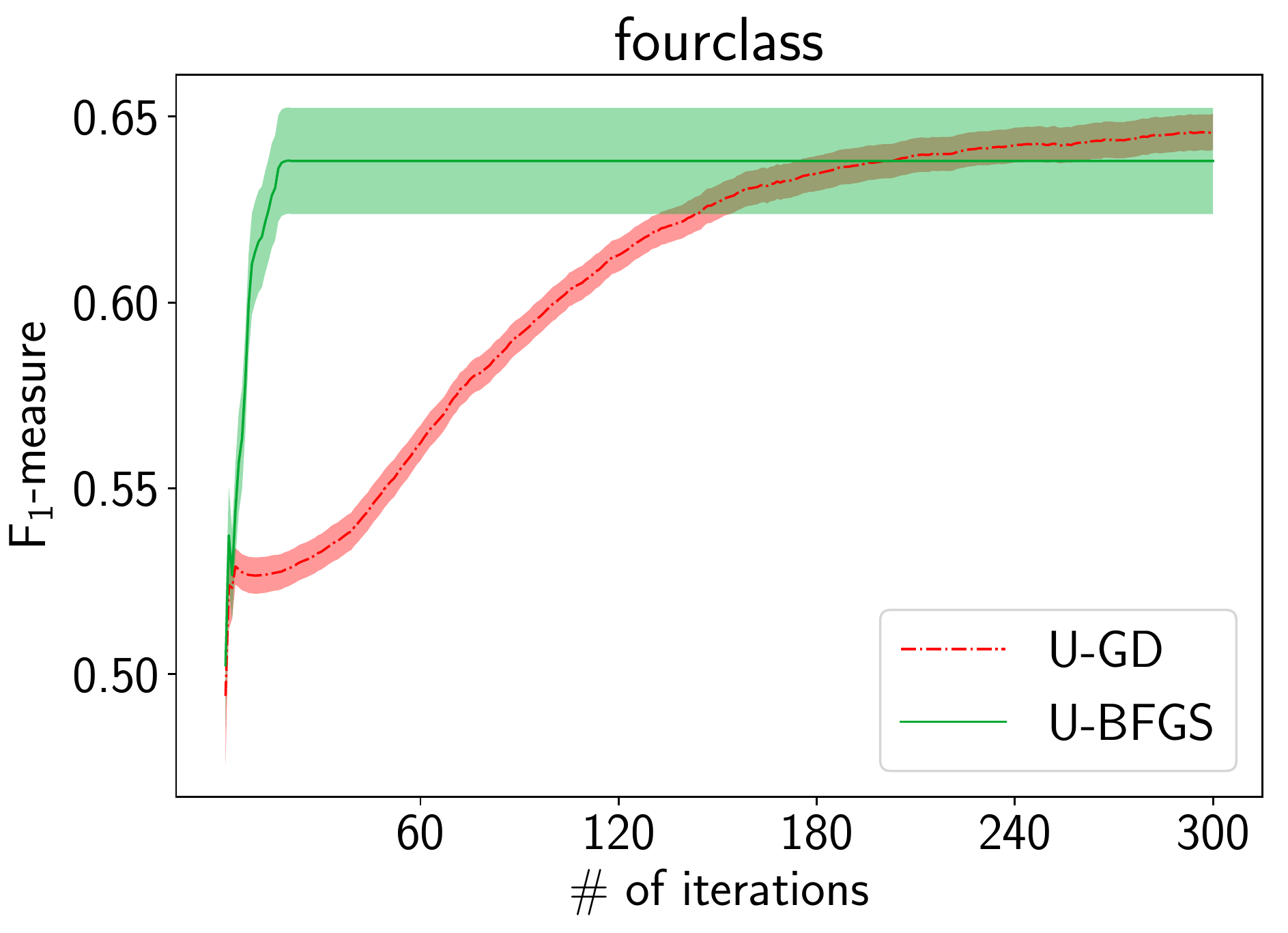}
    \end{minipage}
    \begin{minipage}{0.32\columnwidth}
        \includegraphics[width=\columnwidth]{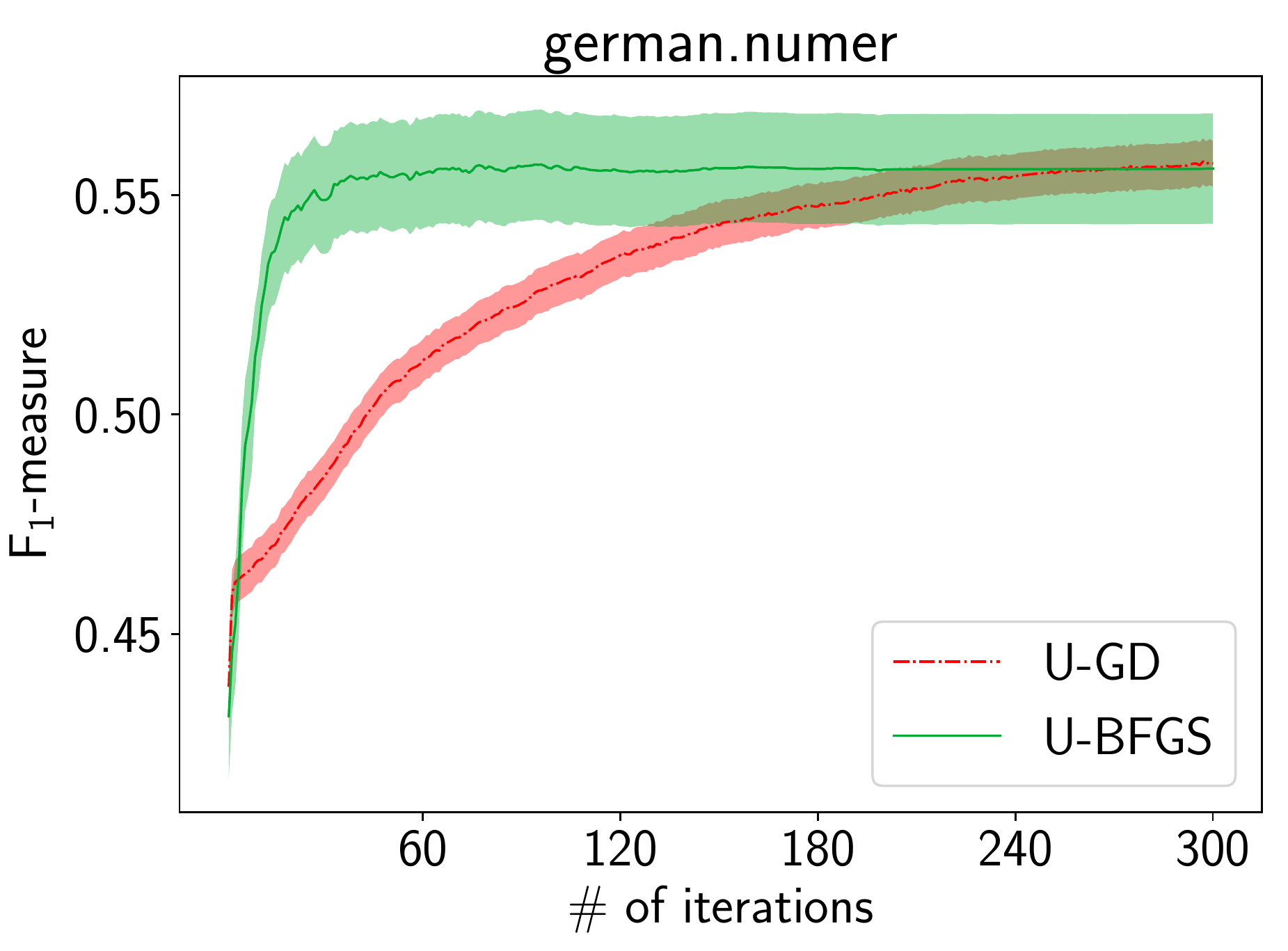}
    \end{minipage}
    \begin{minipage}{0.32\columnwidth}
        \includegraphics[width=\columnwidth]{figures/itertest/itertest_f1_heart.pdf}
    \end{minipage}
    \begin{minipage}{0.32\columnwidth}
        \includegraphics[width=\columnwidth]{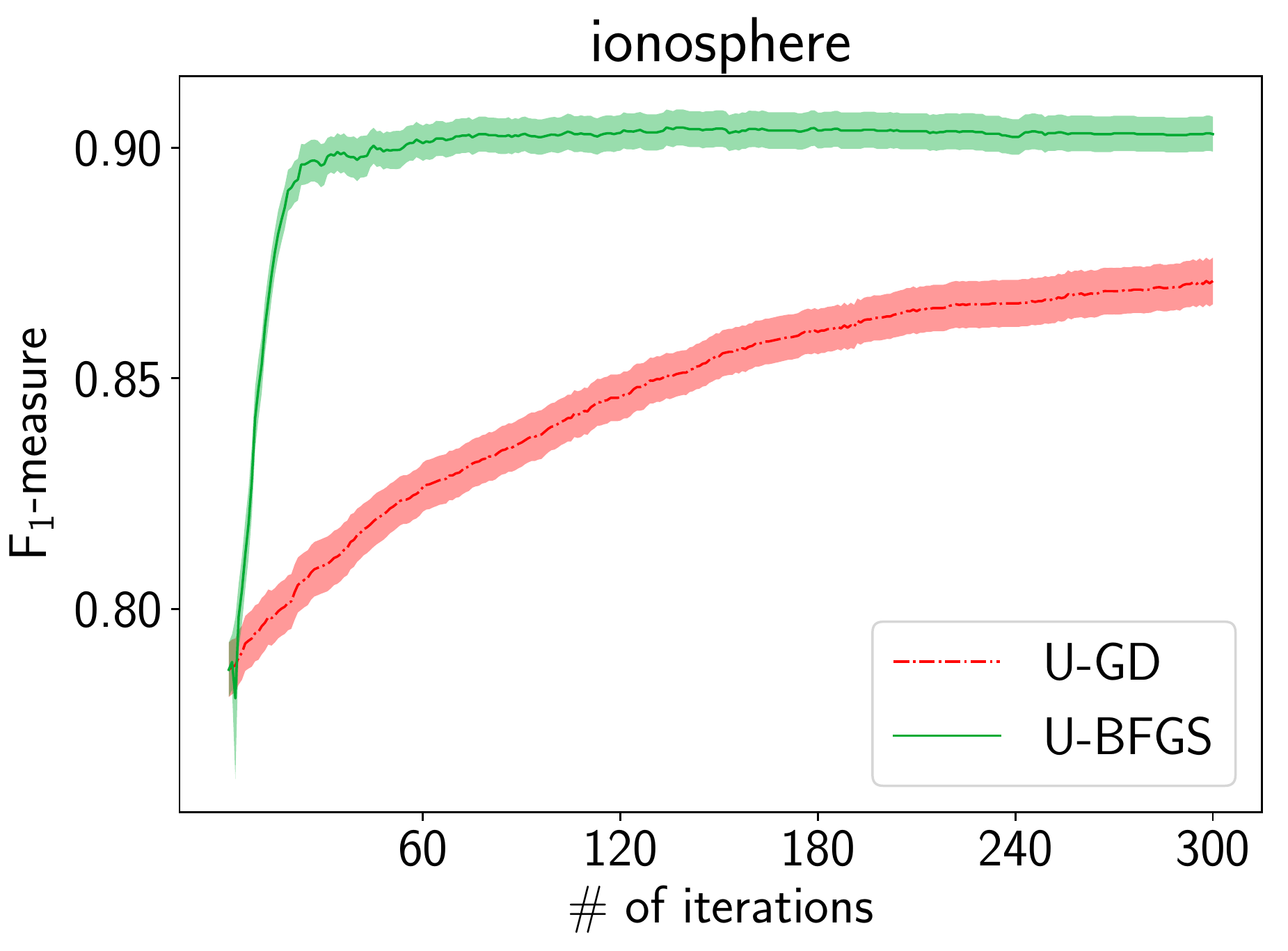}
    \end{minipage}
    \begin{minipage}{0.32\columnwidth}
        \includegraphics[width=\columnwidth]{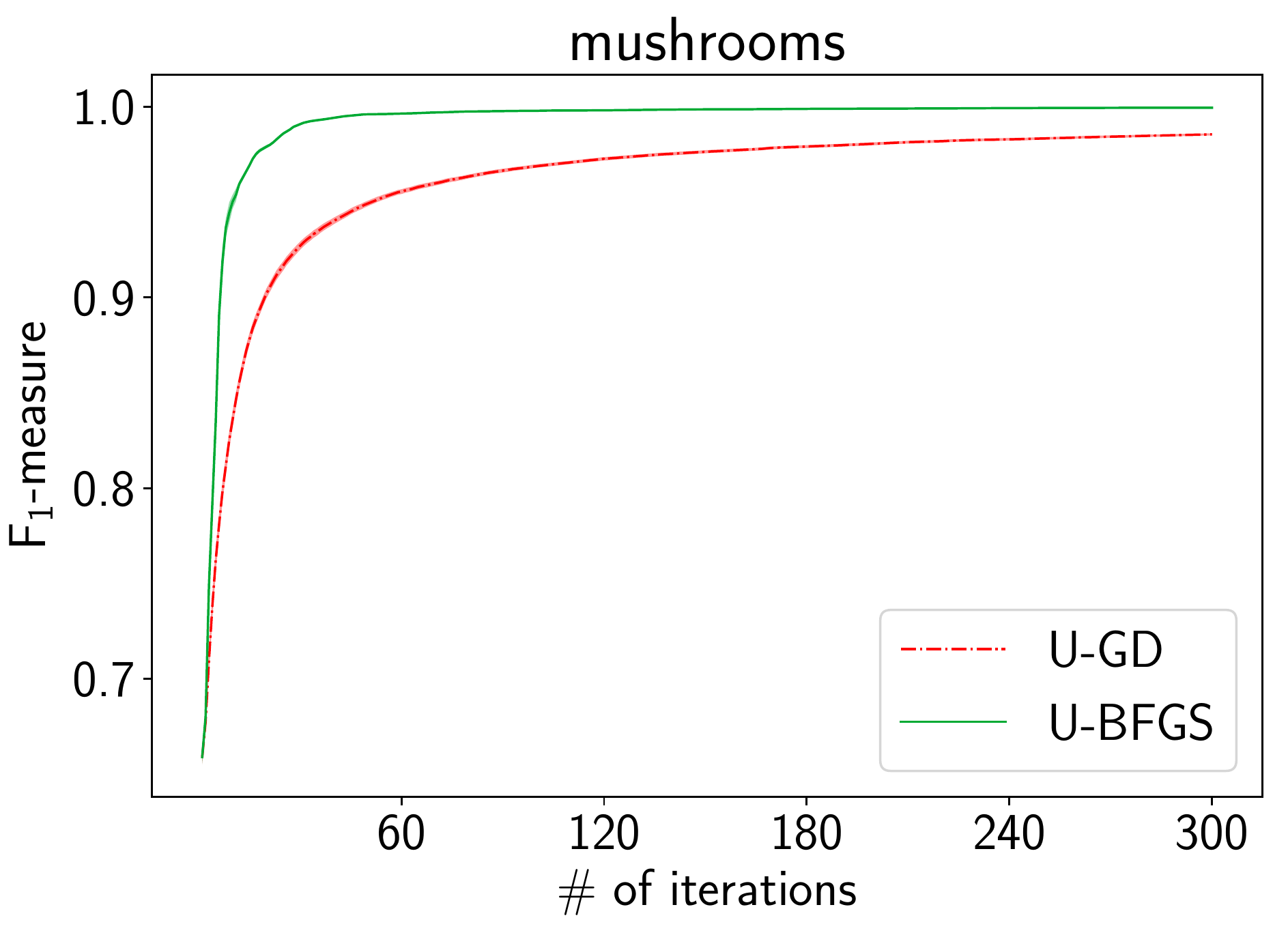}
    \end{minipage}
    \begin{minipage}{0.32\columnwidth}
        \includegraphics[width=\columnwidth]{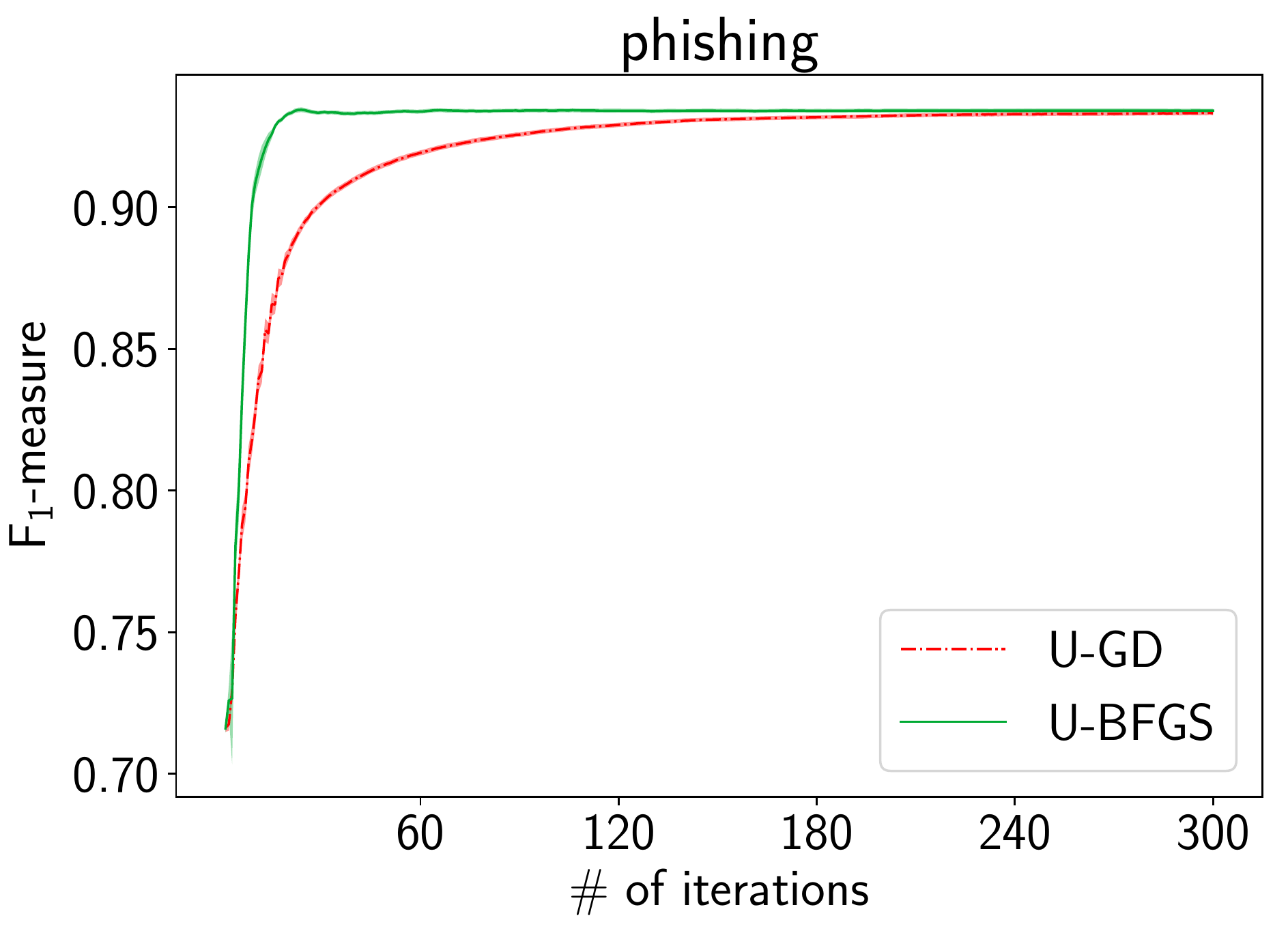}
    \end{minipage}
    \begin{minipage}{0.32\columnwidth}
        \includegraphics[width=\columnwidth]{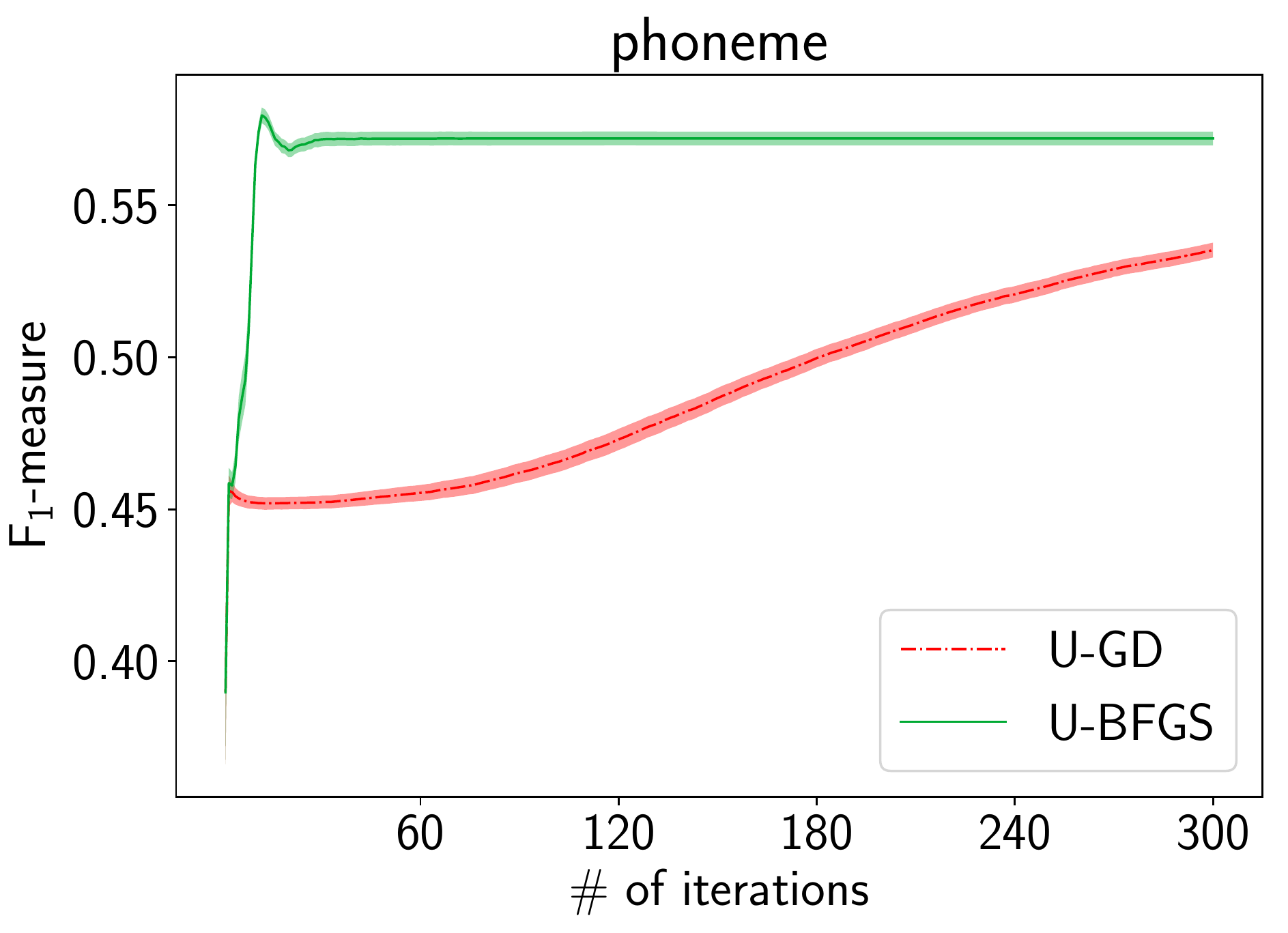}
    \end{minipage}
    \begin{minipage}{0.32\columnwidth}
        \includegraphics[width=\columnwidth]{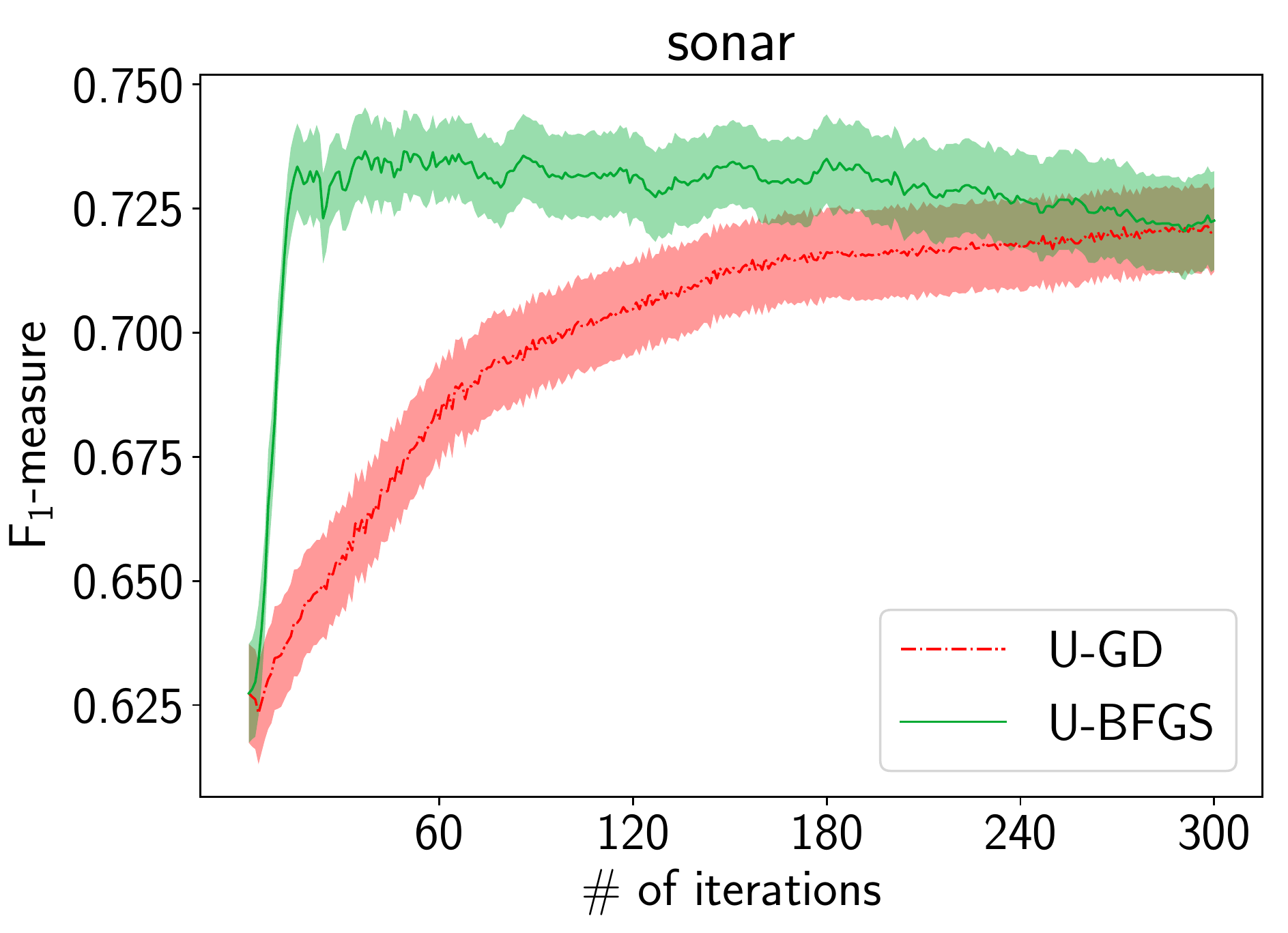}
    \end{minipage}
    \begin{minipage}{0.32\columnwidth}
        \includegraphics[width=\columnwidth]{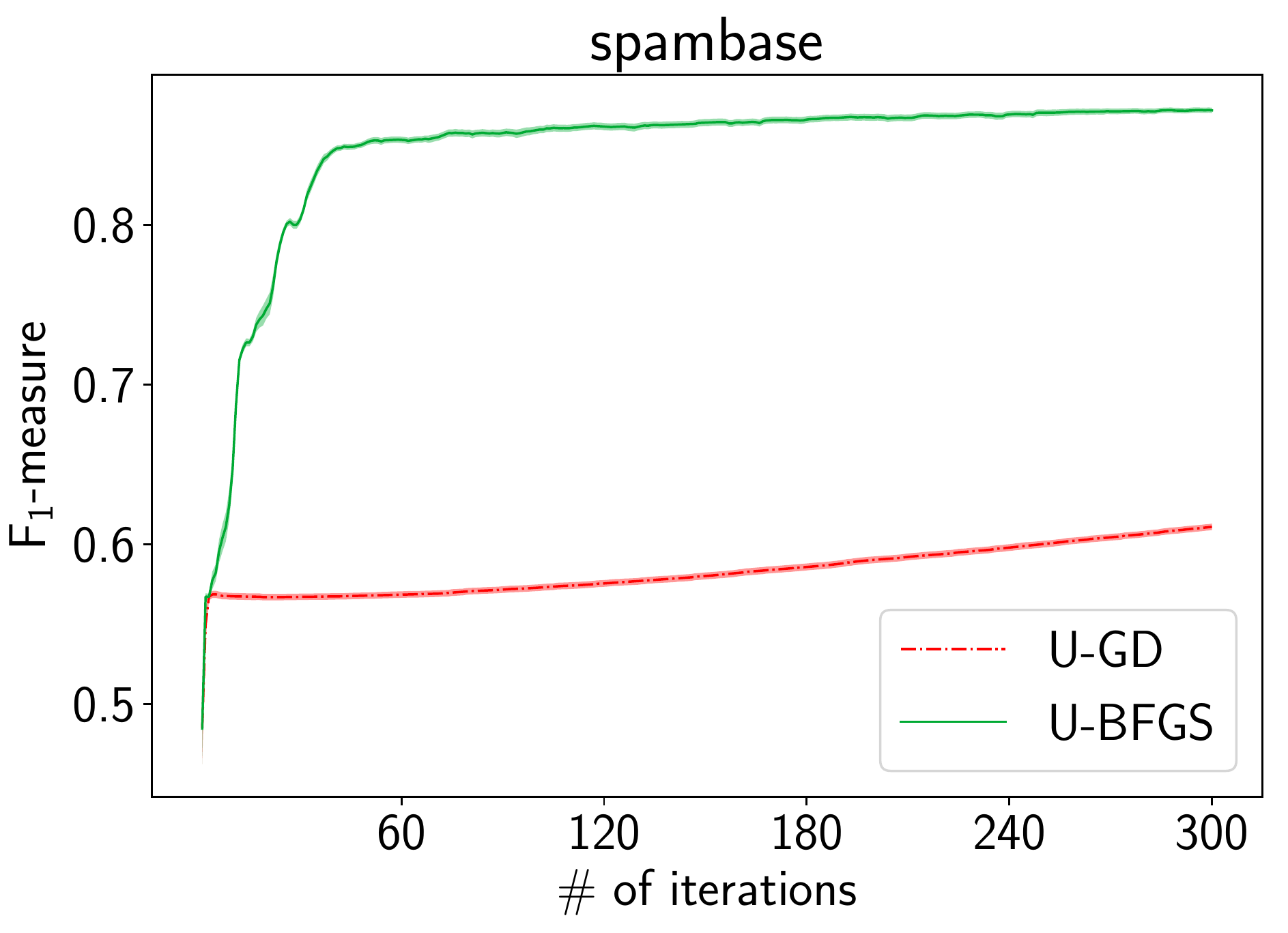}
    \end{minipage}
    \begin{minipage}{0.32\columnwidth}
        \includegraphics[width=\columnwidth]{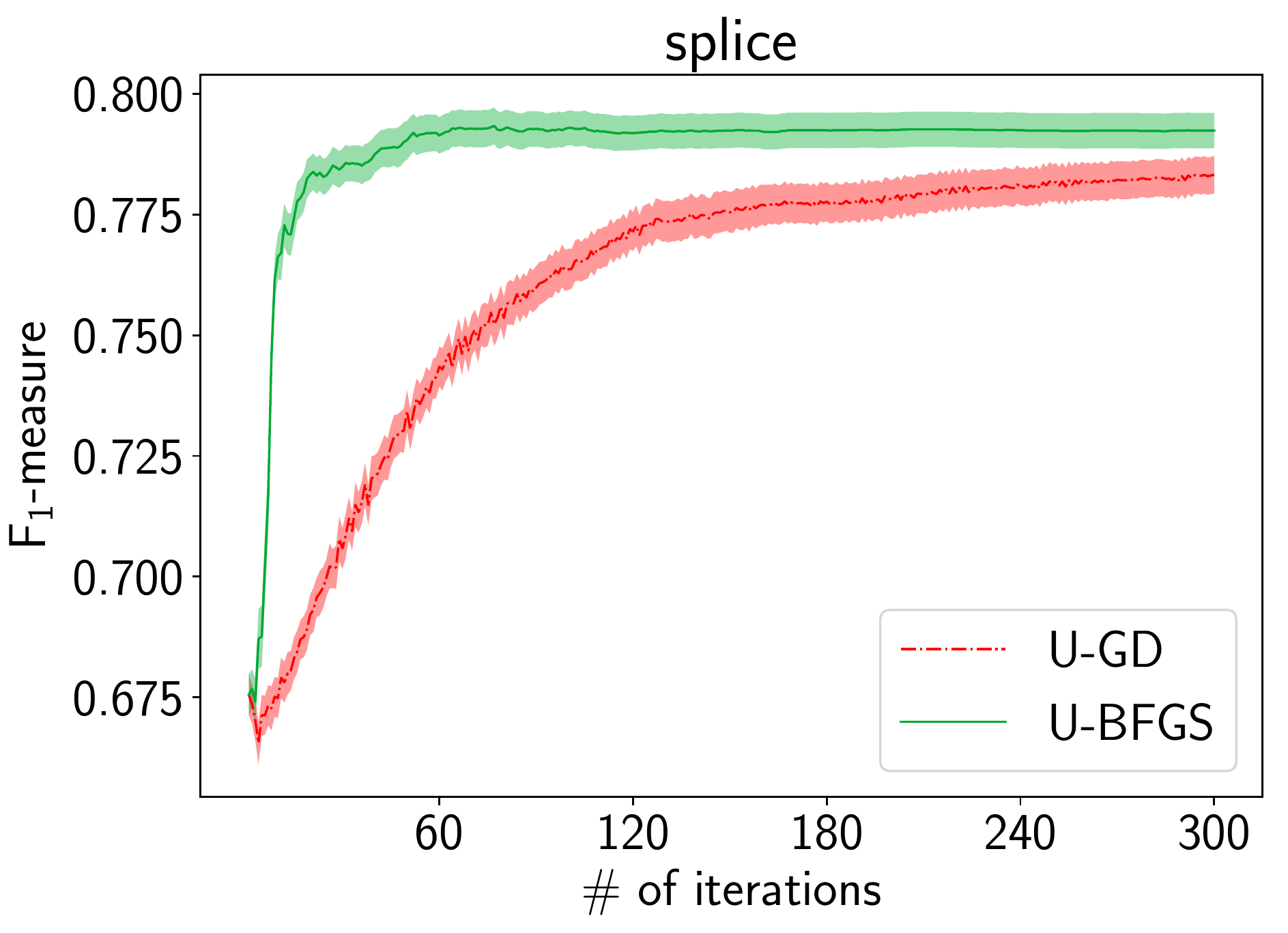}
    \end{minipage}
    \caption{
        Convergence comparison of the F${}_1$-measure (vertical axes).
        Standard errors of 50 trials are shown as shaded areas.
    }
    \label{fig:supp:f1-itertest}
\end{figure}

\begin{figure}[h]
    \centering
    \begin{minipage}{0.32\columnwidth}
        \includegraphics[width=\columnwidth]{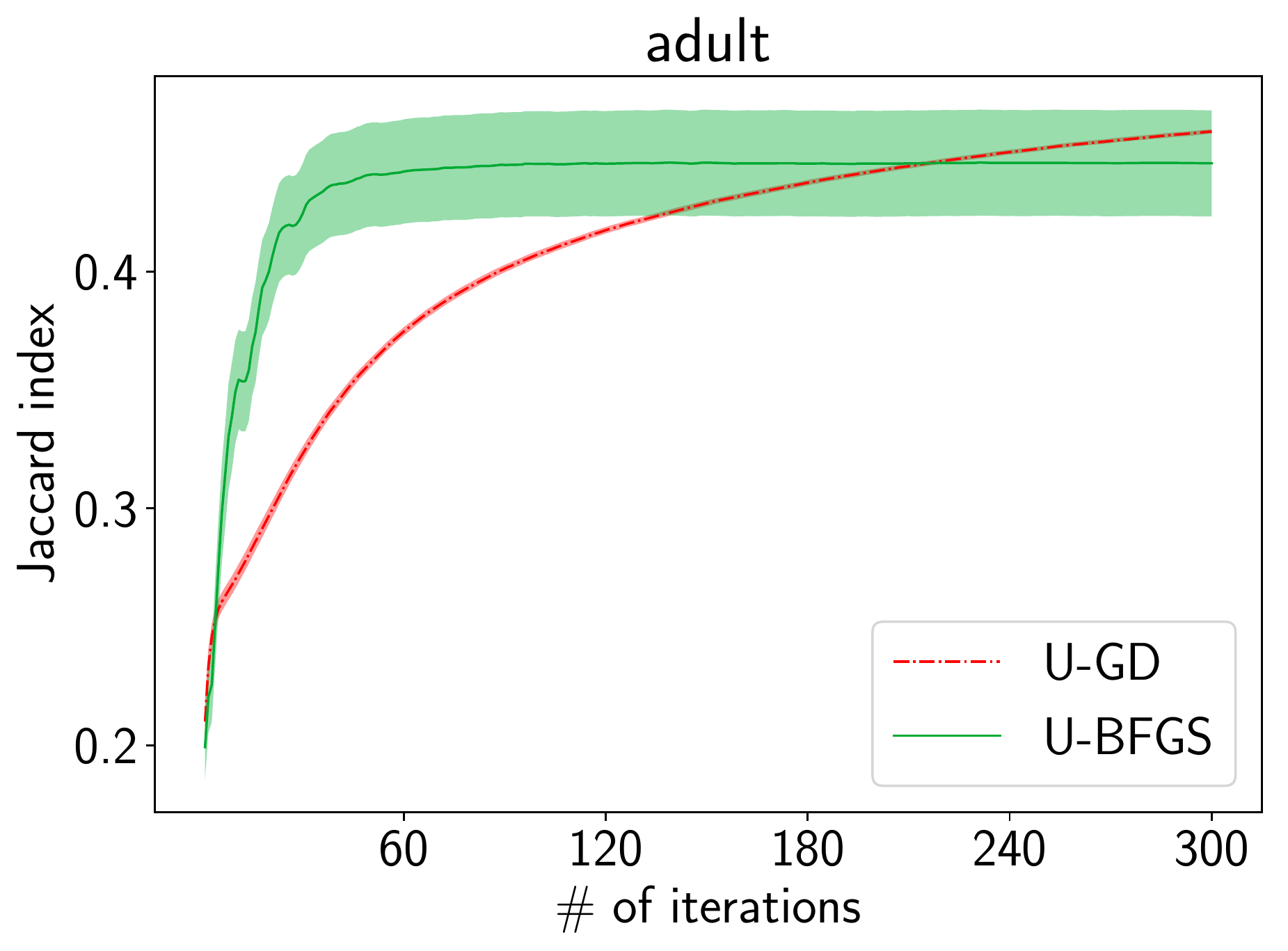}
    \end{minipage}
    \begin{minipage}{0.32\columnwidth}
        \includegraphics[width=\columnwidth]{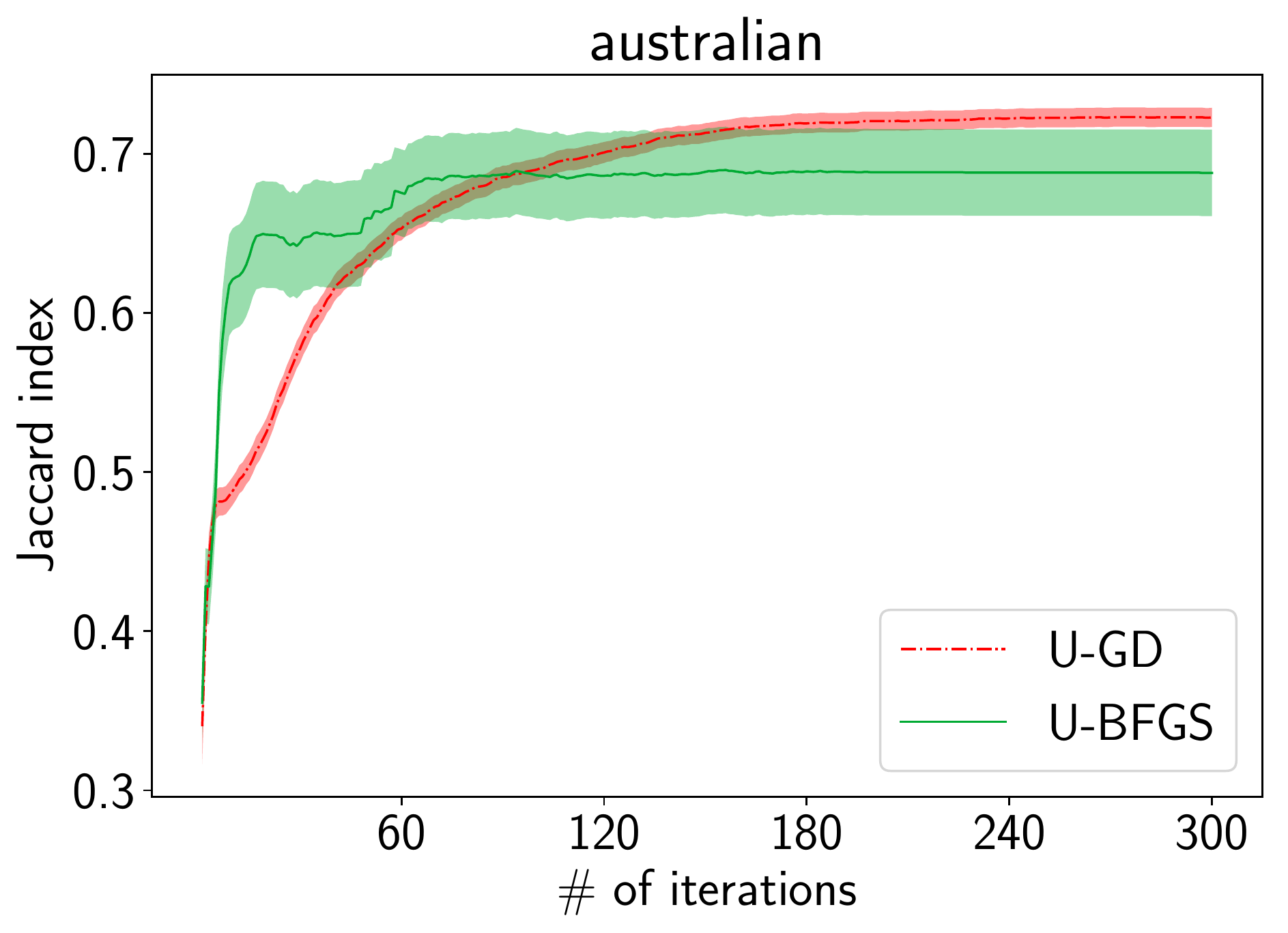}
    \end{minipage}
    \begin{minipage}{0.32\columnwidth}
        \includegraphics[width=\columnwidth]{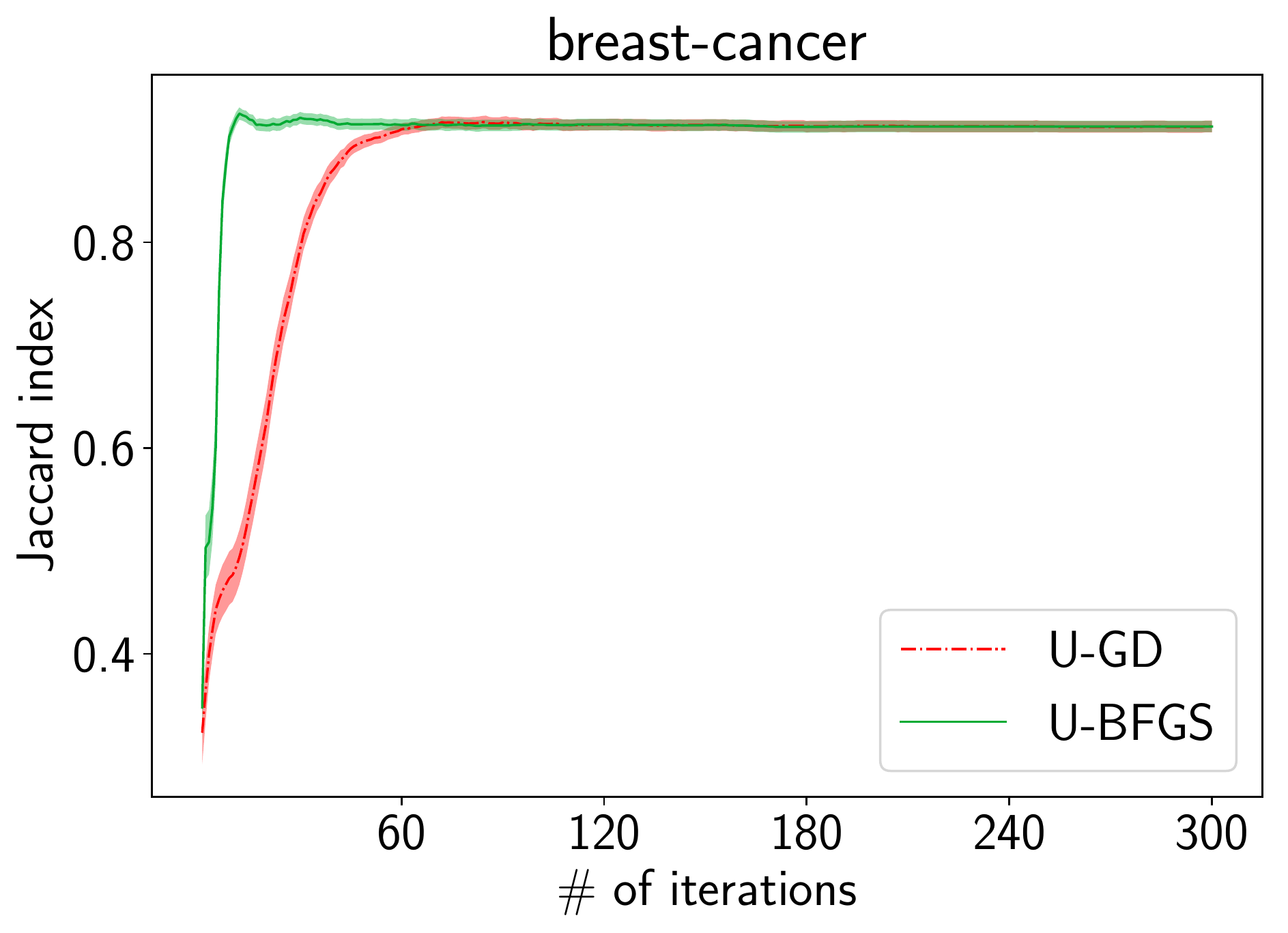}
    \end{minipage}
    \begin{minipage}{0.32\columnwidth}
        \includegraphics[width=\columnwidth]{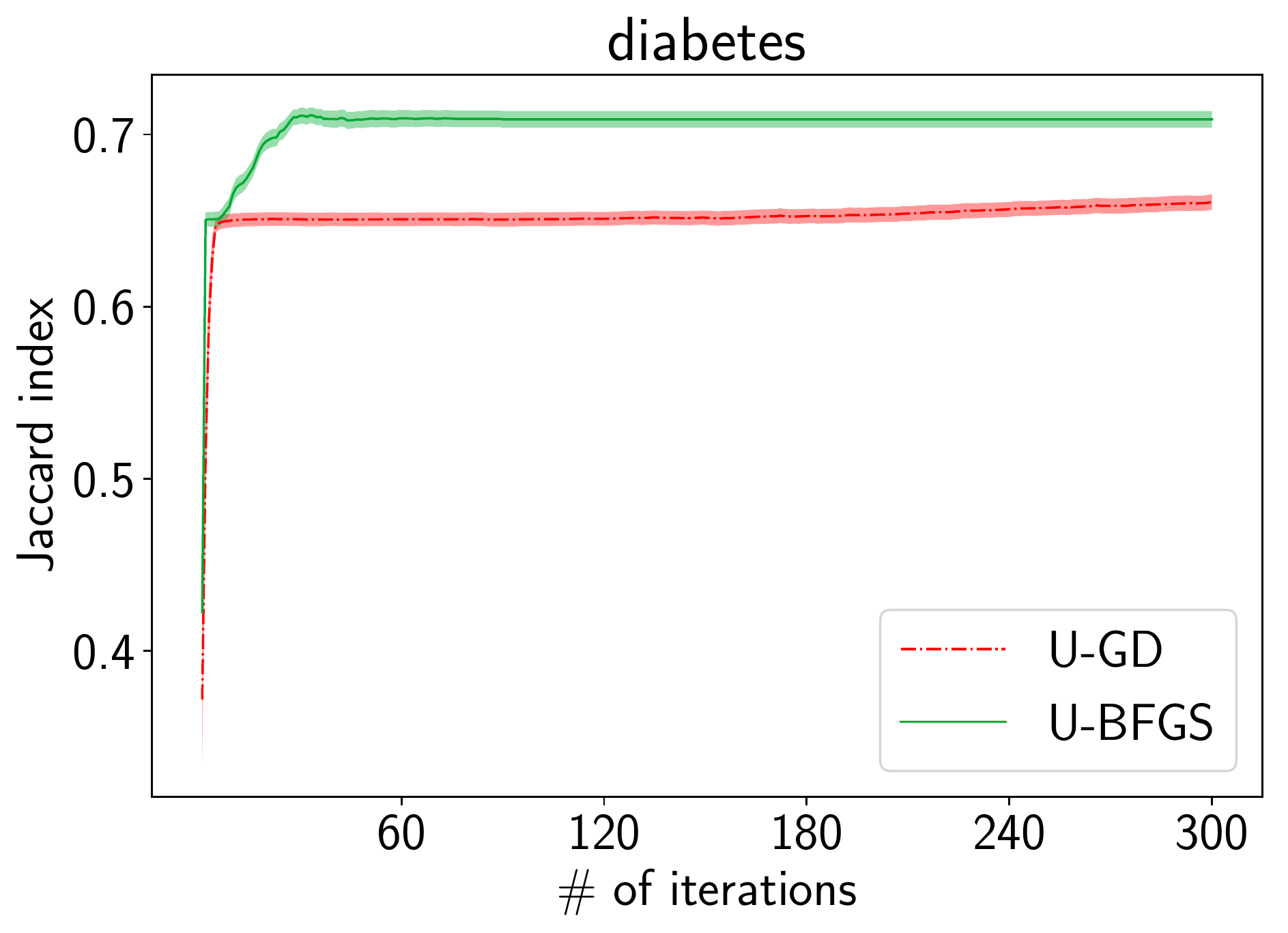}
    \end{minipage}
    \begin{minipage}{0.32\columnwidth}
        \includegraphics[width=\columnwidth]{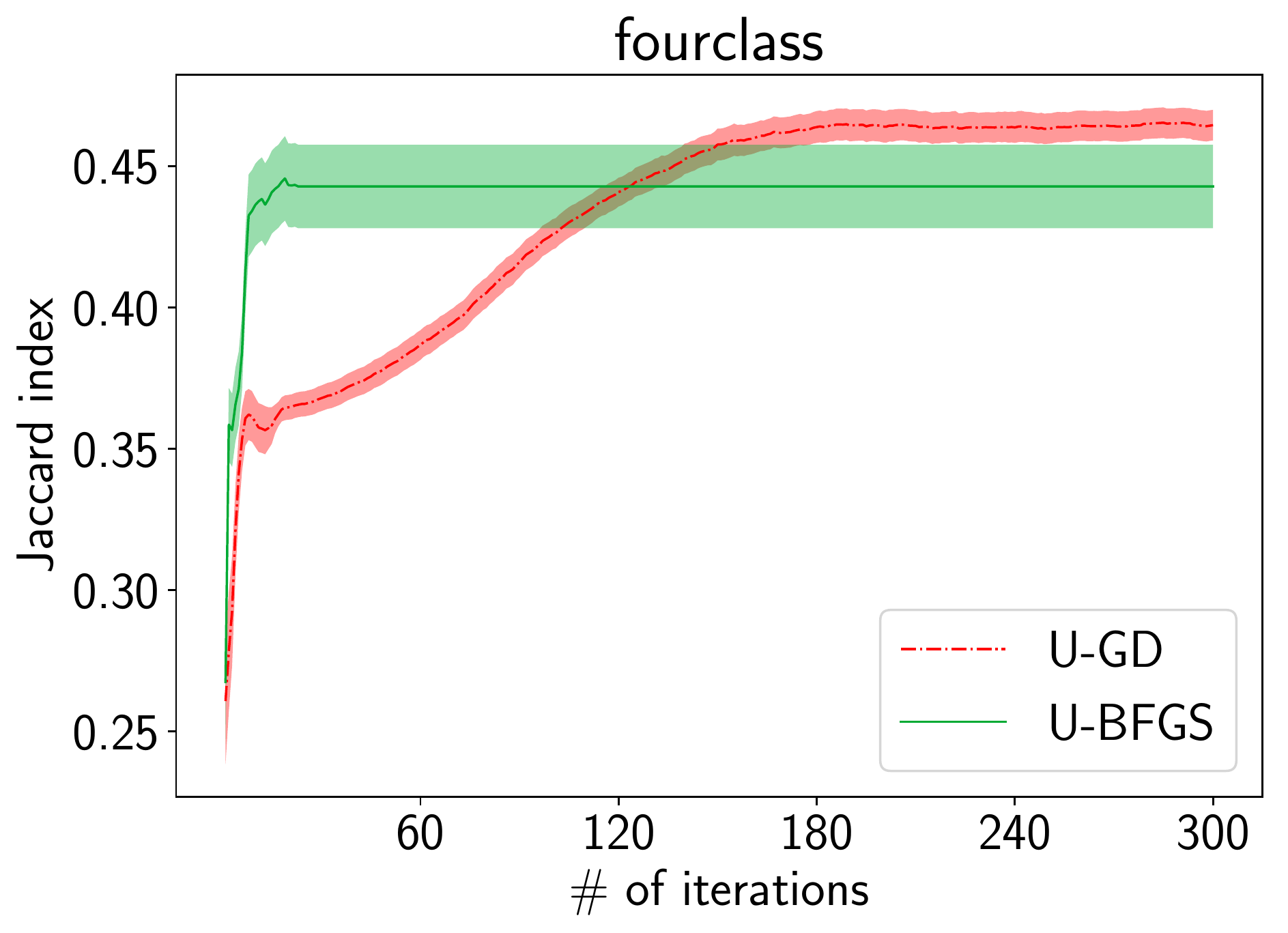}
    \end{minipage}
    \begin{minipage}{0.32\columnwidth}
        \includegraphics[width=\columnwidth]{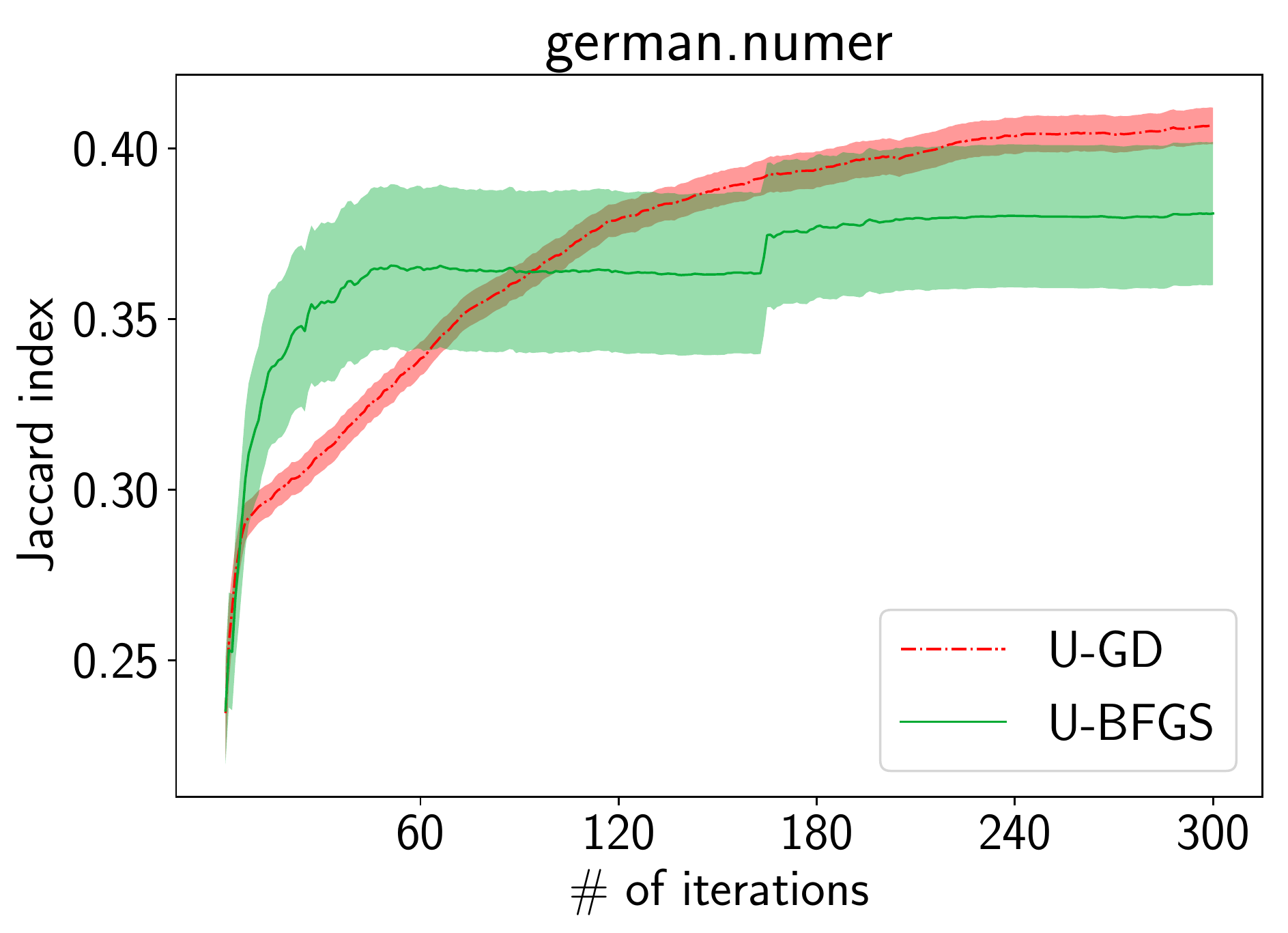}
    \end{minipage}
    \begin{minipage}{0.32\columnwidth}
        \includegraphics[width=\columnwidth]{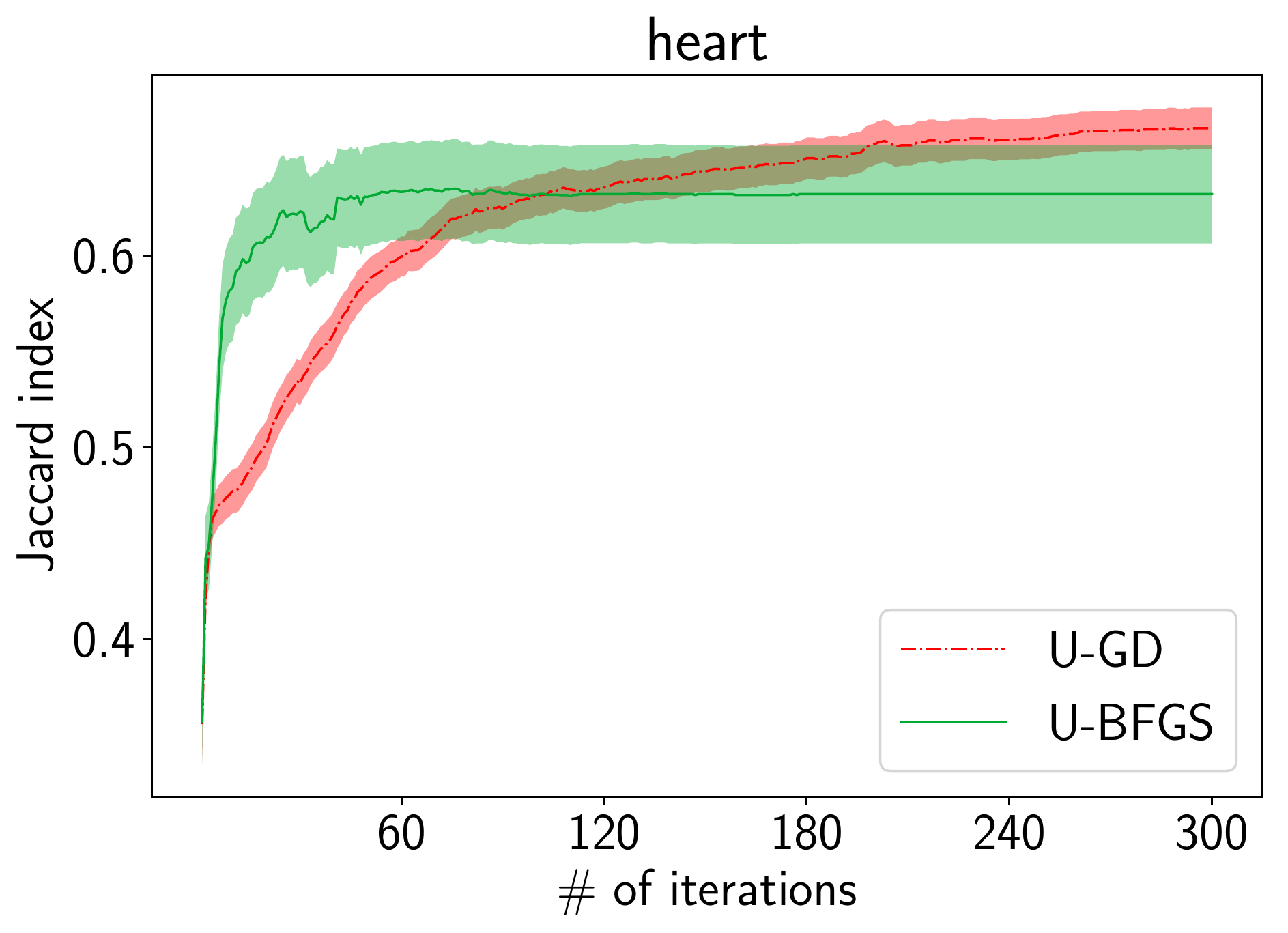}
    \end{minipage}
    \begin{minipage}{0.32\columnwidth}
        \includegraphics[width=\columnwidth]{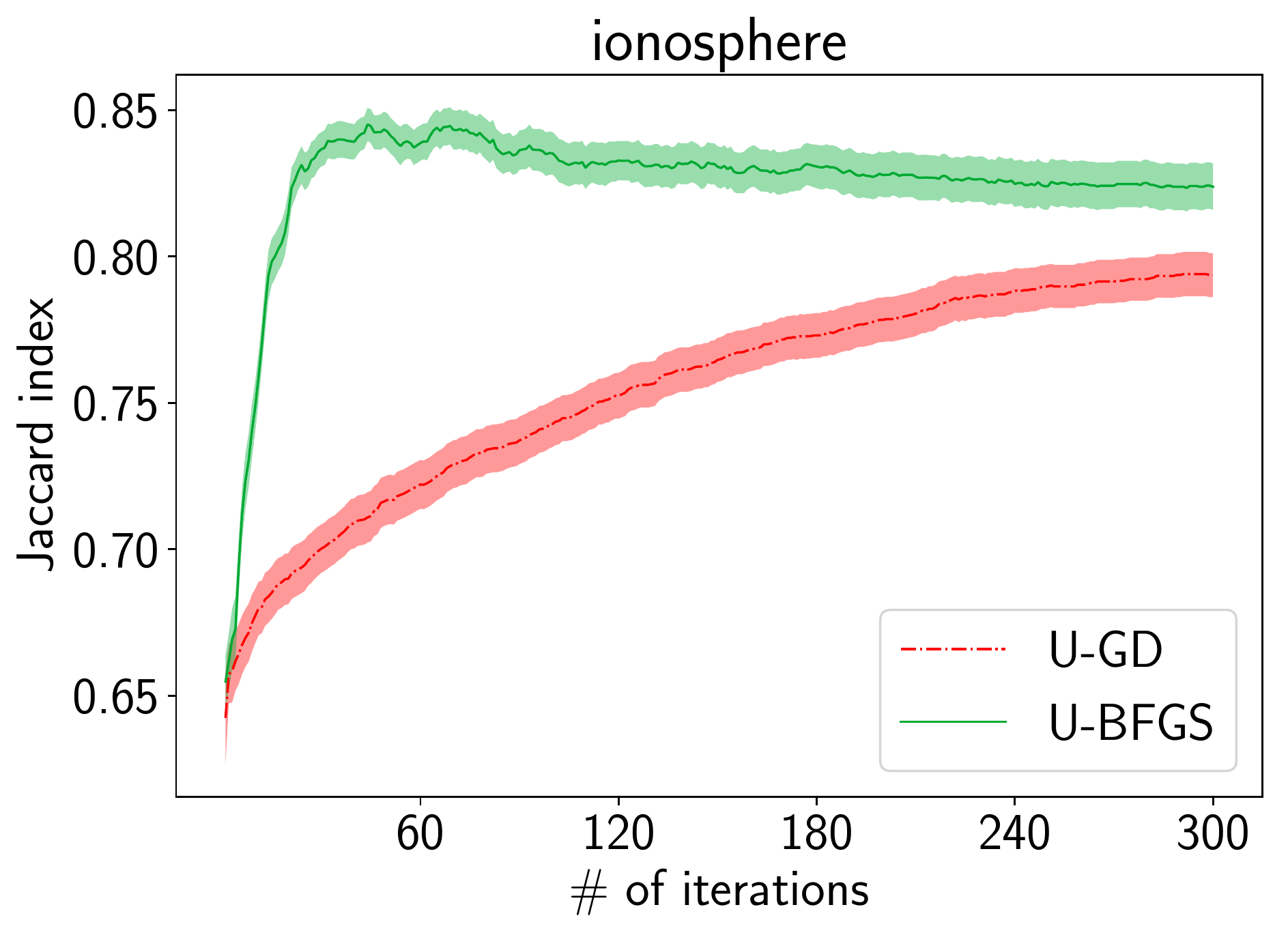}
    \end{minipage}
    \begin{minipage}{0.32\columnwidth}
        \includegraphics[width=\columnwidth]{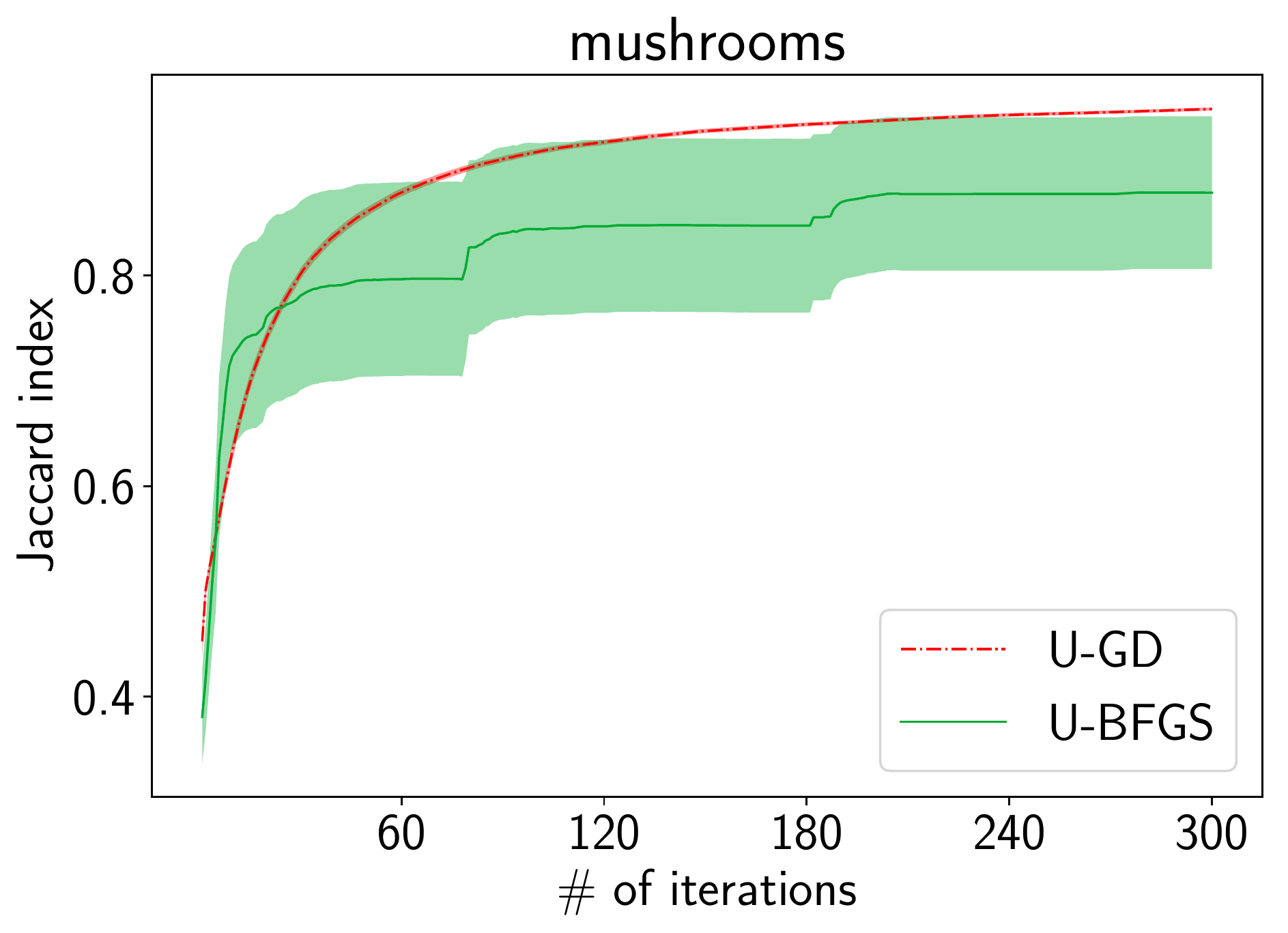}
    \end{minipage}
    \begin{minipage}{0.32\columnwidth}
        \includegraphics[width=\columnwidth]{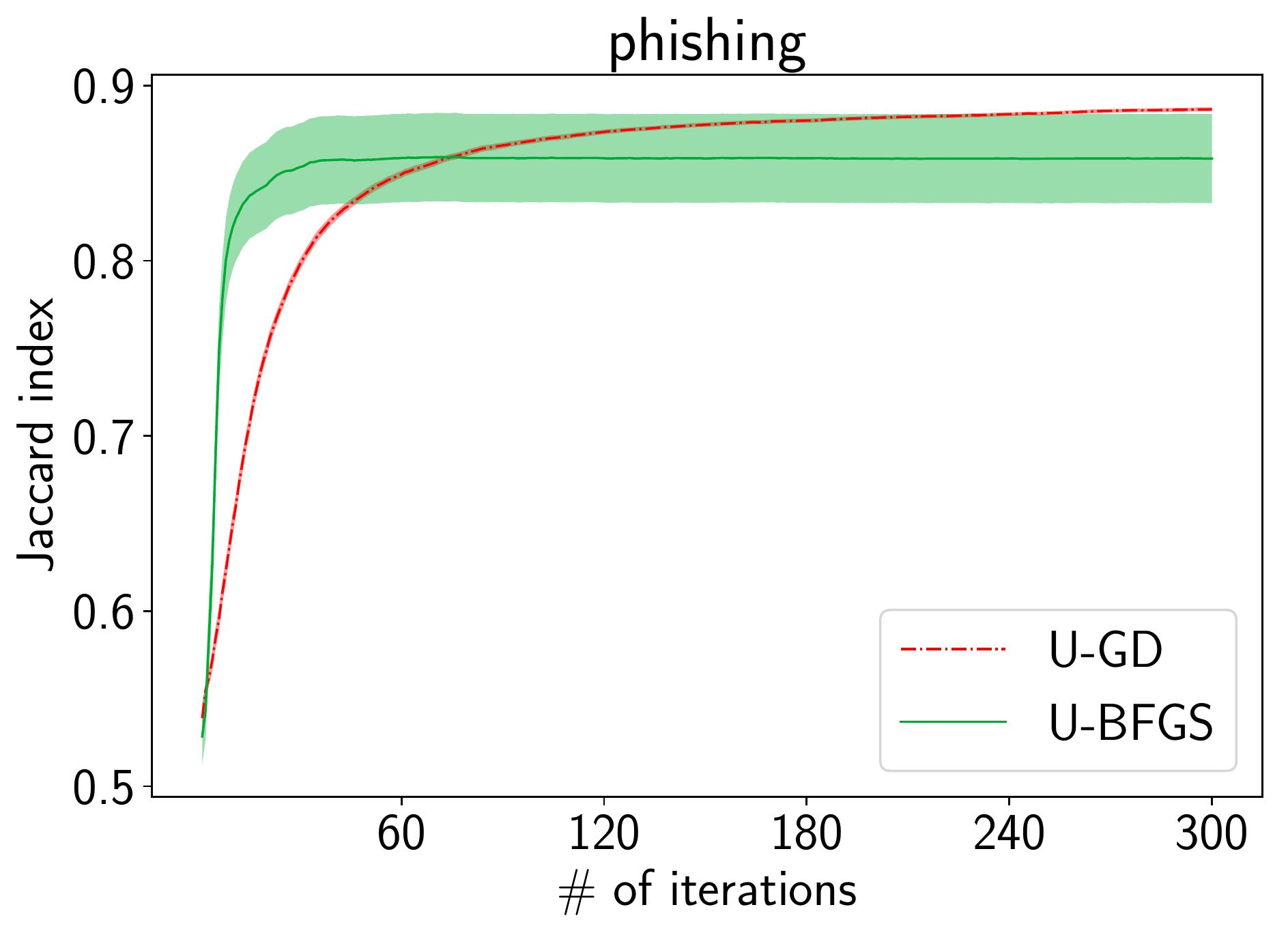}
    \end{minipage}
    \begin{minipage}{0.32\columnwidth}
        \includegraphics[width=\columnwidth]{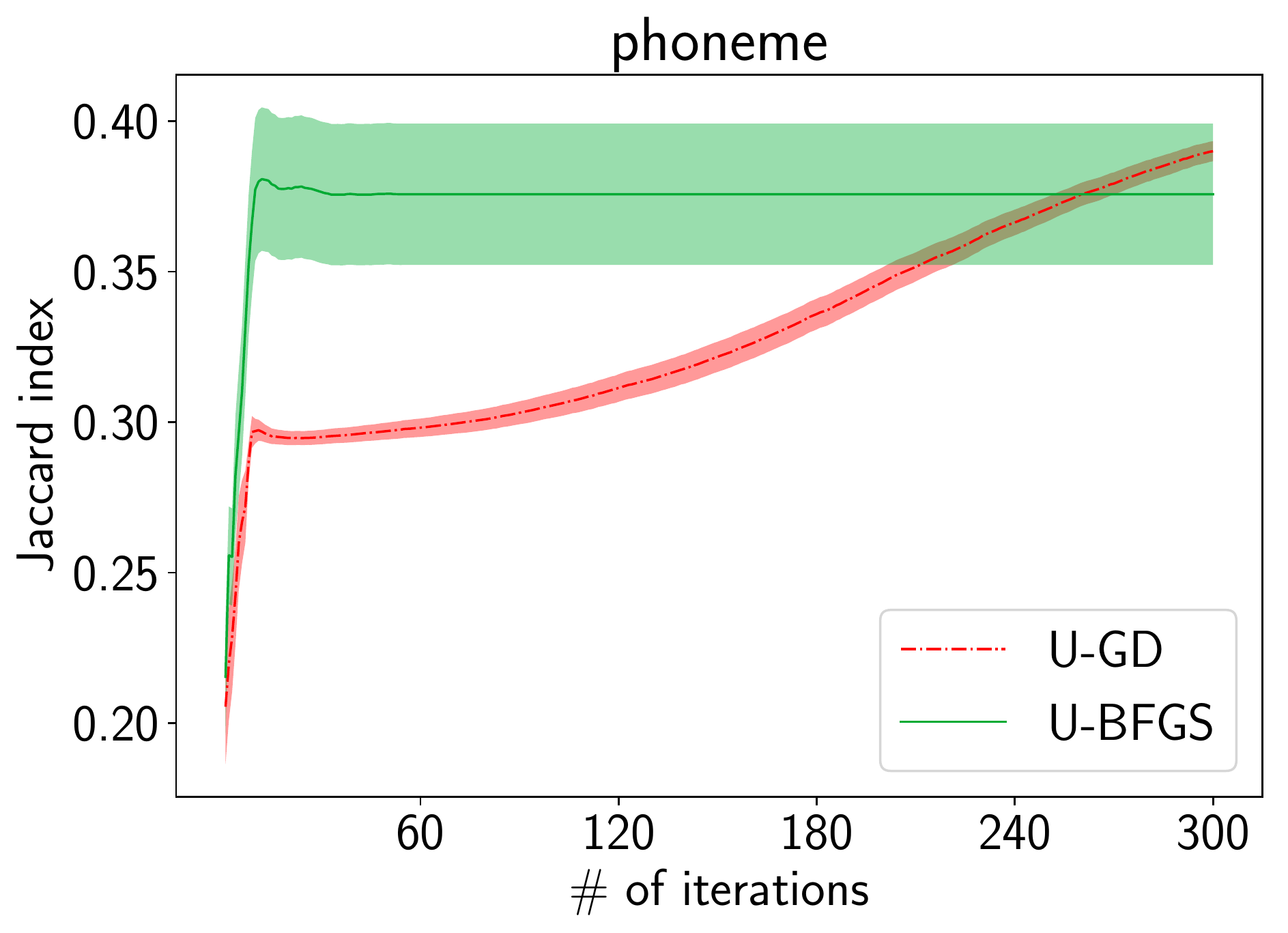}
    \end{minipage}
    \begin{minipage}{0.32\columnwidth}
        \includegraphics[width=\columnwidth]{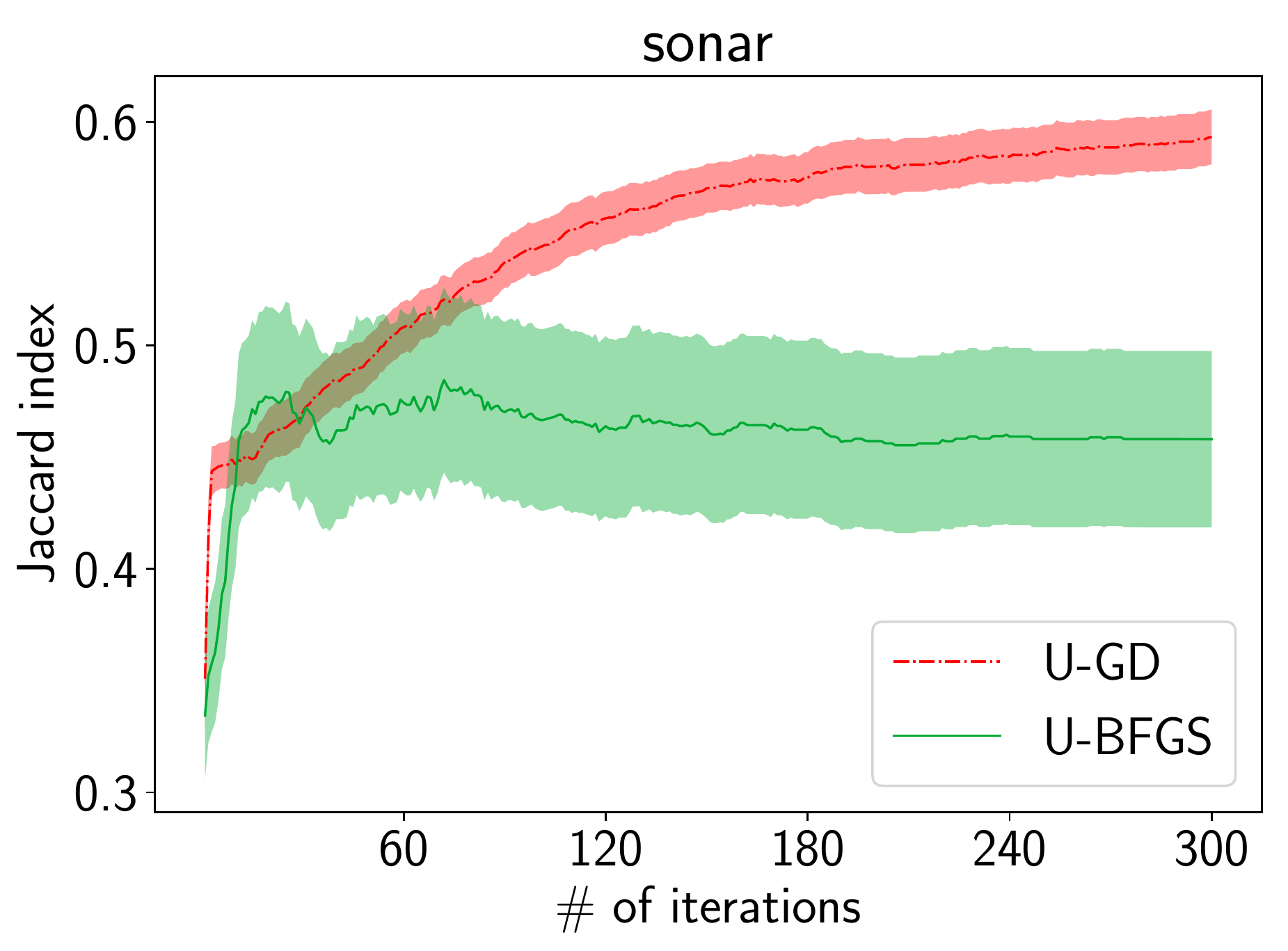}
    \end{minipage}
    \begin{minipage}{0.32\columnwidth}
        \includegraphics[width=\columnwidth]{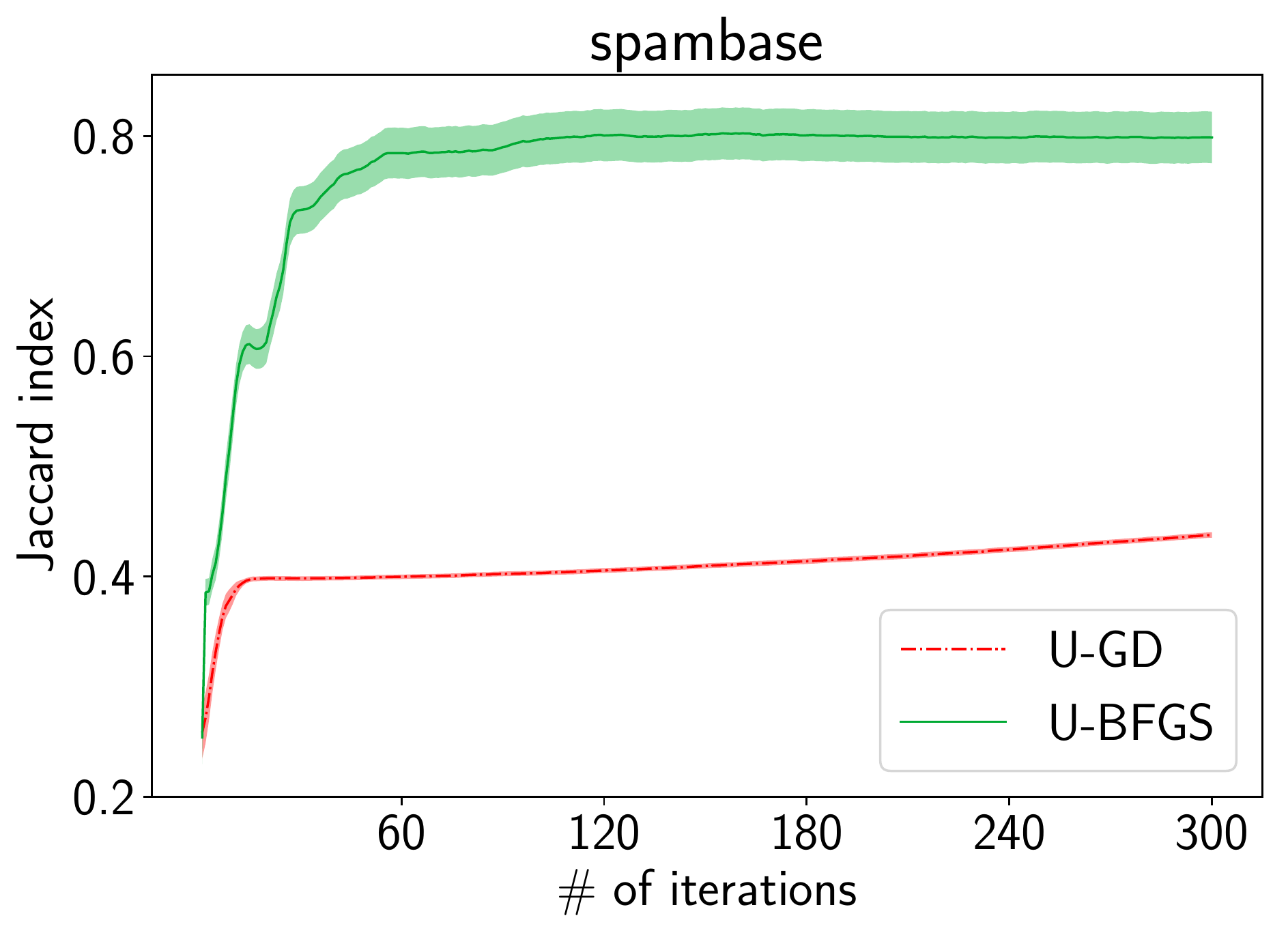}
    \end{minipage}
    \begin{minipage}{0.32\columnwidth}
        \includegraphics[width=\columnwidth]{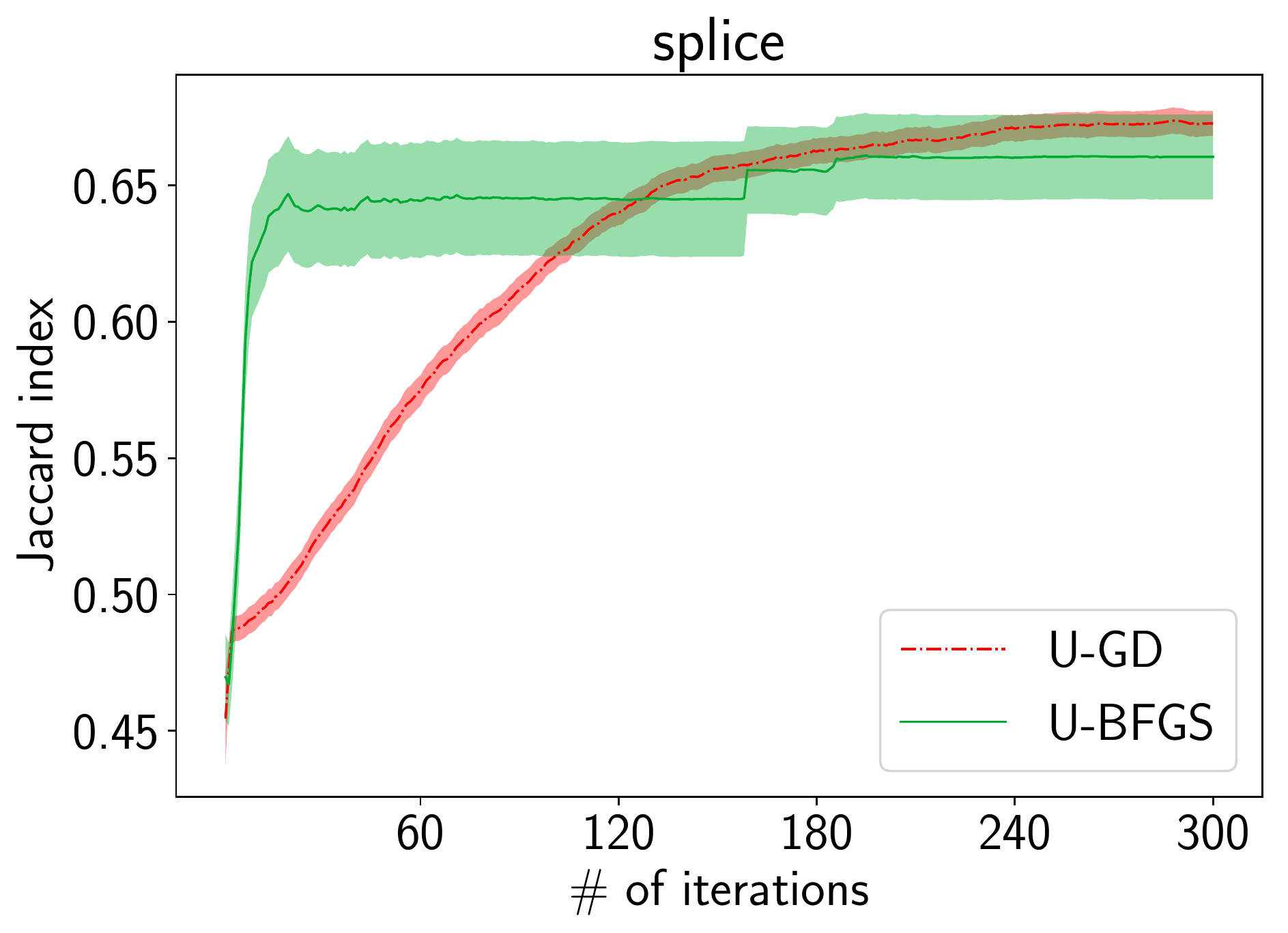}
    \end{minipage}
    \caption{
        Convergence comparison of the Jaccard index (vertical axes).
        Standard errors of 50 trials are shown as shaded areas.
    }
    \label{fig:supp:jac-itertest}
\end{figure}

\subsection{Performance Comparison with Benchmark Data}
\label{sec:supp:benchmark}

Benchmark results are shown in Tabs.~\ref{tab:supp:result-f1} and~\ref{tab:supp:result-jac}.
Each entry shows its final metric value for either F${}_1$-measure or Jaccard index.
For each dataset, we first picked the method with the highest test performance as a outperforming method within that dataset,
then conducted one-sided t-test with the significant level 5\%,
and they are also regarded as outperforming methods if the performance differences are not significant as a result of hypothesis tests.
Outperforming methods are indicated in bold-faces.

As general tendencies, we observe that U-BFGS and Plug-in work well for both F${}_1$-measure and Jaccard index.
As for F${}_1$-measure, their performances are competitive, while U-BFGS is better as for Jaccard index.
In practice, both U-BFGS and Plug-in are worth being tried.

As for other methods:
ERM does not work good as we expect, because it does not optimize the metrics of our interests, F${}_1$-measure and Jaccard index, at all.
W-ERM does not work as well as Plug-in even though both of them are known to be consistent to the linear-fractional utilities.
We may need more finer split of the threshold search space, or try a binary-search-type algorithm provided by recent work~\citep{Yan:2018}.
U-GD does not work as well as U-BFGS contrary to our expectation.
We may need more iterations to make U-GD converge, as we see in Figures~\ref{fig:supp:f1-itertest} and~\ref{fig:supp:jac-itertest}.
Note that we ran 100 iterations for both U-GD and U-BFGS for the results shown in Tabs.~\ref{tab:supp:result-f1} and~\ref{tab:supp:result-jac}.

\begin{table*}[h]
    \centering
    \caption{
        Results of the F$_1$-measure:
        50 trials are conducted for each pair of a method and dataset.
        Standard errors (multiplied by $10^4$) are shown in parentheses.
        Bold-faces indicate outperforming methods, chosen by one-sided t-test with the significant level 5\%.
    }
    \label{tab:supp:result-f1}
    \scalebox{0.8}{
    \begin{tabular}{cccccc} \hline
        (F${}_1$-measure) & \multicolumn{2}{c}{Proposed} & \multicolumn{3}{c}{Baselines} \\ \cmidrule(lr){2-3} \cmidrule(lr){4-6}
        Dataset & U-GD & U-BFGS & ERM & W-ERM & Plug-in
        \\ \hline
        adult & 0.617 (101)  & 0.660 (11)  & 0.639 (51)  & 0.676 (18)  & \textbf{0.681 (9)} \\
        australian & \textbf{0.843 (41)} & \textbf{0.844 (45)} & 0.820 (123)  & 0.814 (116)  & 0.827 (51)  \\
        breast-cancer & \textbf{0.963 (31)} & \textbf{0.960 (32)} & 0.950 (37)  & 0.948 (44)  & 0.953 (40)  \\
        cod-rna & 0.802 (231)  & 0.594 (4)  & 0.927 (7)  & 0.927 (6)  & \textbf{0.930 (2)} \\
        diabetes & \textbf{0.834 (32)} & \textbf{0.828 (31)} & 0.817 (50)  & 0.821 (40)  & 0.820 (42)  \\
        fourclass & \textbf{0.638 (70)} & \textbf{0.638 (64)} & 0.601 (124)  & 0.591 (212)  & 0.618 (64)  \\
        german.numer & 0.561 (102)  & \textbf{0.580 (74)} & 0.492 (188)  & 0.560 (107)  & \textbf{0.589 (73)} \\
        heart & \textbf{0.796 (101)} & \textbf{0.802 (99)} & \textbf{0.792 (80)} & 0.764 (151)  & 0.764 (137)  \\
        ionosphere & \textbf{0.908 (49)} & \textbf{0.901 (43)} & 0.883 (104)  & 0.842 (217)  & \textbf{0.897 (54)} \\
        madelon & \textbf{0.666 (19)} & 0.632 (67)  & 0.491 (293)  & 0.639 (110)  & \textbf{0.663 (24)} \\
        mushrooms & 1.000 (1)  & 0.997 (7)  & \textbf{1.000 (1)} & 1.000 (2)  & 0.999 (4)  \\
        phishing & 0.937 (29)  & \textbf{0.943 (7)} & \textbf{0.944 (8)} & 0.940 (12)  & \textbf{0.944 (8)} \\
        phoneme & \textbf{0.648 (27)} & 0.559 (22)  & 0.530 (201)  & 0.616 (135)  & 0.633 (35)  \\
        skin\_nonskin & 0.870 (3)  & 0.856 (4)  & 0.854 (7)  & \textbf{0.877 (8)} & 0.838 (5)  \\
        sonar & \textbf{0.735 (95)} & \textbf{0.740 (91)} & 0.706 (121)  & 0.655 (189)  & \textbf{0.721 (113)} \\
        spambase & 0.876 (27)  & 0.756 (61)  & 0.887 (42)  & 0.881 (58)  & \textbf{0.903 (18)} \\
        splice & 0.785 (49)  & \textbf{0.799 (46)} & 0.785 (55)  & 0.771 (67)  & \textbf{0.801 (45)} \\
        w8a & 0.297 (80)  & 0.284 (96)  & 0.735 (35)  & \textbf{0.742 (29)} & \textbf{0.745 (26)} \\
        \hline
    \end{tabular}
    }
\end{table*}

\begin{table*}[h]
    \centering
    \caption{
        Results of the Jaccard index:
        50 trials are conducted for each pair of a method and dataset.
        Standard errors (multiplied by $10^4$) are shown in parentheses.
        Bold-faces indicate outperforming methods, chosen by one-sided t-test with the significant level 5\%.
    }
    \label{tab:supp:result-jac}
    \scalebox{0.8}{
    \begin{tabular}{cccccc} \hline
        (Jaccard index) & \multicolumn{2}{c}{Proposed} & \multicolumn{3}{c}{Baselines} \\ \cmidrule(lr){2-3} \cmidrule(lr){4-6}
        Dataset & U-GD & U-BFGS & ERM & W-ERM & Plug-in
        \\ \hline
        adult & 0.499 (44)  & 0.498 (11)  & 0.471 (51)  & 0.510 (20)  & \textbf{0.516 (10)} \\
        australian & \textbf{0.735 (63)} & \textbf{0.733 (59)} & 0.702 (144)  & 0.693 (143)  & 0.707 (76)  \\
        breast-cancer & \textbf{0.921 (54)} & \textbf{0.918 (55)} & 0.905 (66)  & 0.903 (78)  & \textbf{0.913 (69)} \\
        cod-rna & 0.854 (3)  & 0.785 (8)  & 0.864 (11)  & 0.865 (9)  & \textbf{0.869 (3)} \\
        diabetes & \textbf{0.714 (44)} & 0.702 (50)  & 0.692 (70)  & 0.698 (56)  & 0.695 (60)  \\
        fourclass & \textbf{0.469 (69)} & \textbf{0.457 (68)} & 0.436 (112)  & 0.434 (171)  & 0.449 (66)  \\
        german.numer & \textbf{0.433 (64)} & \textbf{0.429 (69)} & 0.335 (153)  & 0.391 (98)  & \textbf{0.418 (71)} \\
        heart & \textbf{0.665 (135)} & \textbf{0.675 (135)} & \textbf{0.664 (102)} & 0.629 (178)  & 0.626 (163)  \\
        ionosphere & \textbf{0.826 (76)} & \textbf{0.829 (65)} & 0.796 (134)  & 0.749 (245)  & \textbf{0.815 (87)} \\
        madelon & \textbf{0.495 (31)} & 0.459 (69)  & 0.346 (225)  & 0.474 (100)  & \textbf{0.496 (27)} \\
        mushrooms & 0.999 (2)  & 0.995 (4)  & \textbf{1.000 (1)} & 0.999 (4)  & 0.997 (7)  \\
        phishing & 0.883 (43)  & \textbf{0.893 (11)} & \textbf{0.894 (14)} & 0.888 (22)  & \textbf{0.894 (15)} \\
        phoneme & 0.435 (51)  & 0.436 (24)  & 0.371 (160)  & \textbf{0.450 (104)} & \textbf{0.461 (34)} \\
        skin\_nonskin & 0.744 (5)  & 0.751 (5)  & 0.746 (10)  & \textbf{0.780 (13)} & 0.722 (7)  \\
        sonar & \textbf{0.600 (125)} & \textbf{0.600 (111)} & 0.552 (147)  & 0.495 (202)  & \textbf{0.572 (134)} \\
        spambase & \textbf{0.827 (22)} & 0.708 (22)  & 0.798 (67)  & 0.790 (86)  & \textbf{0.824 (31)} \\
        splice & \textbf{0.670 (60)} & \textbf{0.672 (56)} & 0.646 (71)  & 0.629 (84)  & \textbf{0.672 (57)} \\
        w8a & 0.496 (151)  & 0.452 (28)  & 0.580 (44)  & \textbf{0.590 (35)} & \textbf{0.595 (33)} \\
        \hline
    \end{tabular}
    }
\end{table*}

\subsection{Sample Complexity}
\label{sec:supp:sample-complexity}

It is interesting to study the relationship between the metric performances and the size of samples,
because we expect Plug-in, which requires to estimate probabilities accurately, does not work well when the size of samples is quite small.
Figures~\ref{fig:supp:f1-sample-complexity} and~\ref{fig:supp:jac-sample-complexity} show the sample complexity results.
Even though learning is not stable for small samples (e.g., heart and w8a),
we can observe clear differences in some cases such as
cod-rna, diabetes, german.numer, ionosphere, sonar, and splice in F${}_1$-measure, and
australian, cod-rna, diabetes, ionosphere, phishing, sonar, and spambase in Jaccard index,
where either U-GD or U-BFGS works better than Plug-in even if sample sizes are quite small around 20 to 40.
In addition, Plug-in seldom works significantly better than the gradient-based methods in the cases where sample sizes range around 100 to 400 as investigated in this section.
This is contrary to the behavior shown in Tabs.~\ref{tab:supp:result-f1} and~\ref{tab:supp:result-jac},
where the full-size datasets are used to train classifiers.

As a conclusion, it can be a good option to consider using the gradient-based methods where sample sizes are very small.

\begin{figure}[h]
    \centering
    \begin{minipage}{0.32\columnwidth}
        \includegraphics[width=\columnwidth]{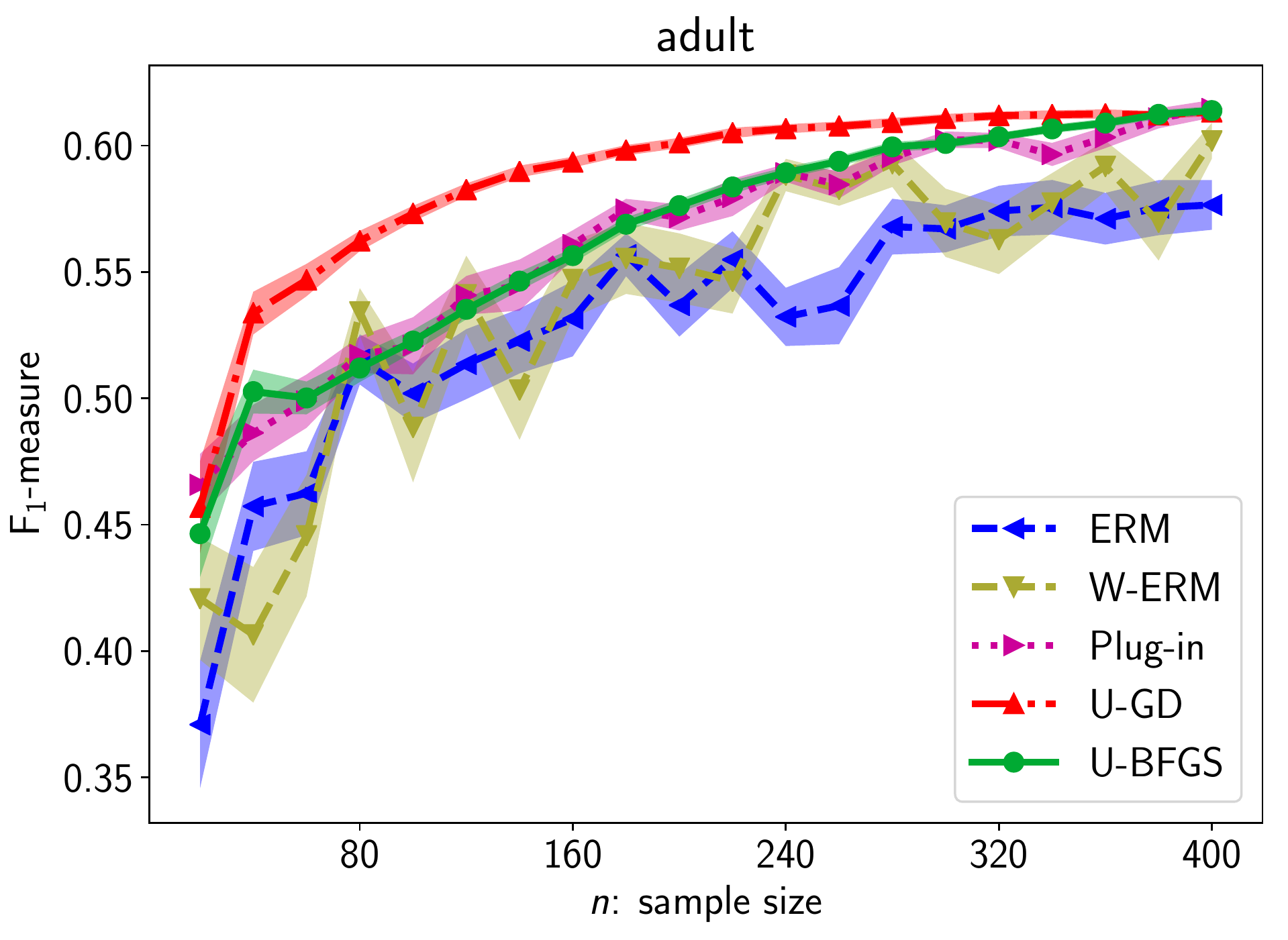}
    \end{minipage}
    \begin{minipage}{0.32\columnwidth}
        \includegraphics[width=\columnwidth]{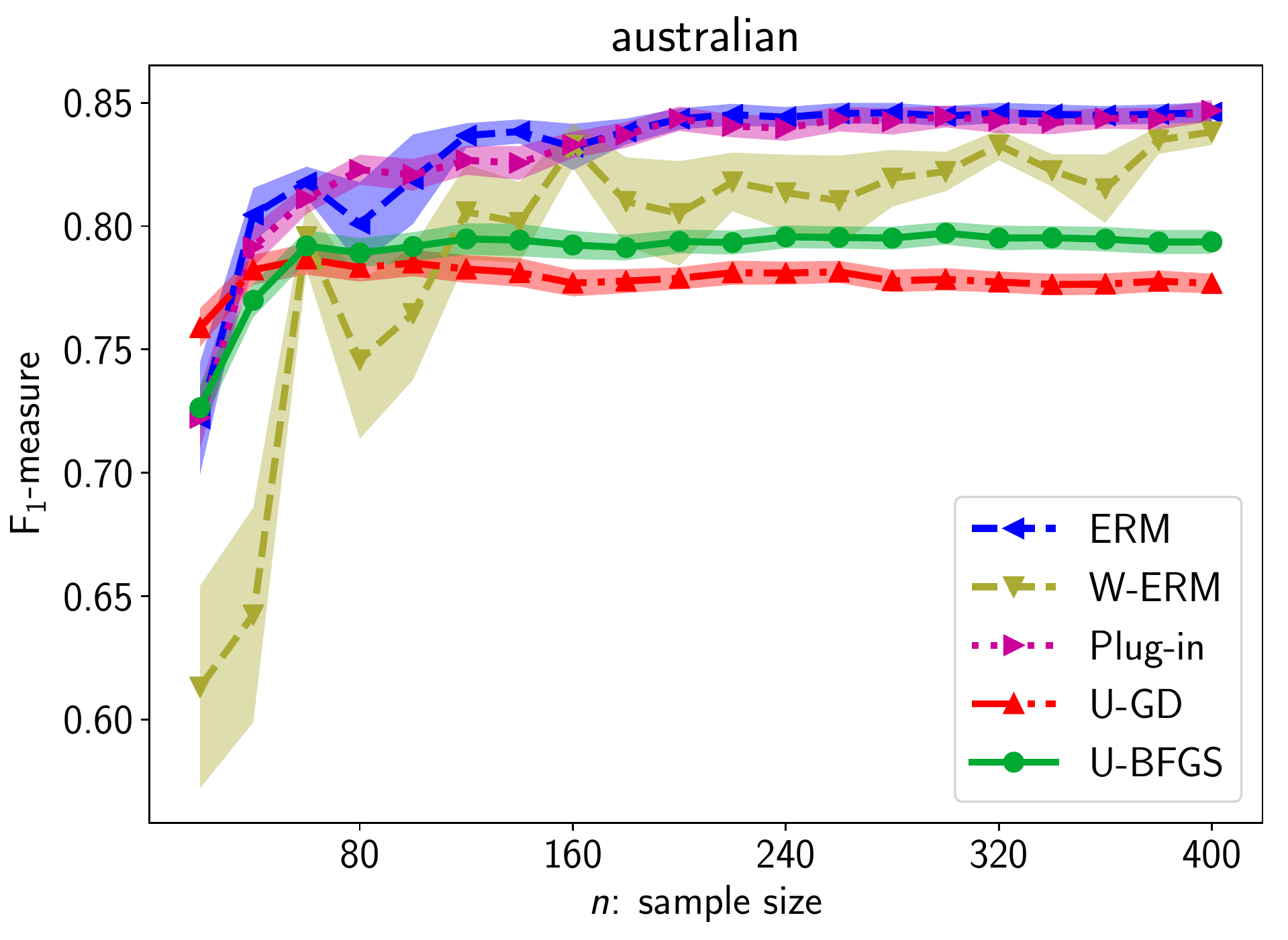}
    \end{minipage}
    \begin{minipage}{0.32\columnwidth}
        \includegraphics[width=\columnwidth]{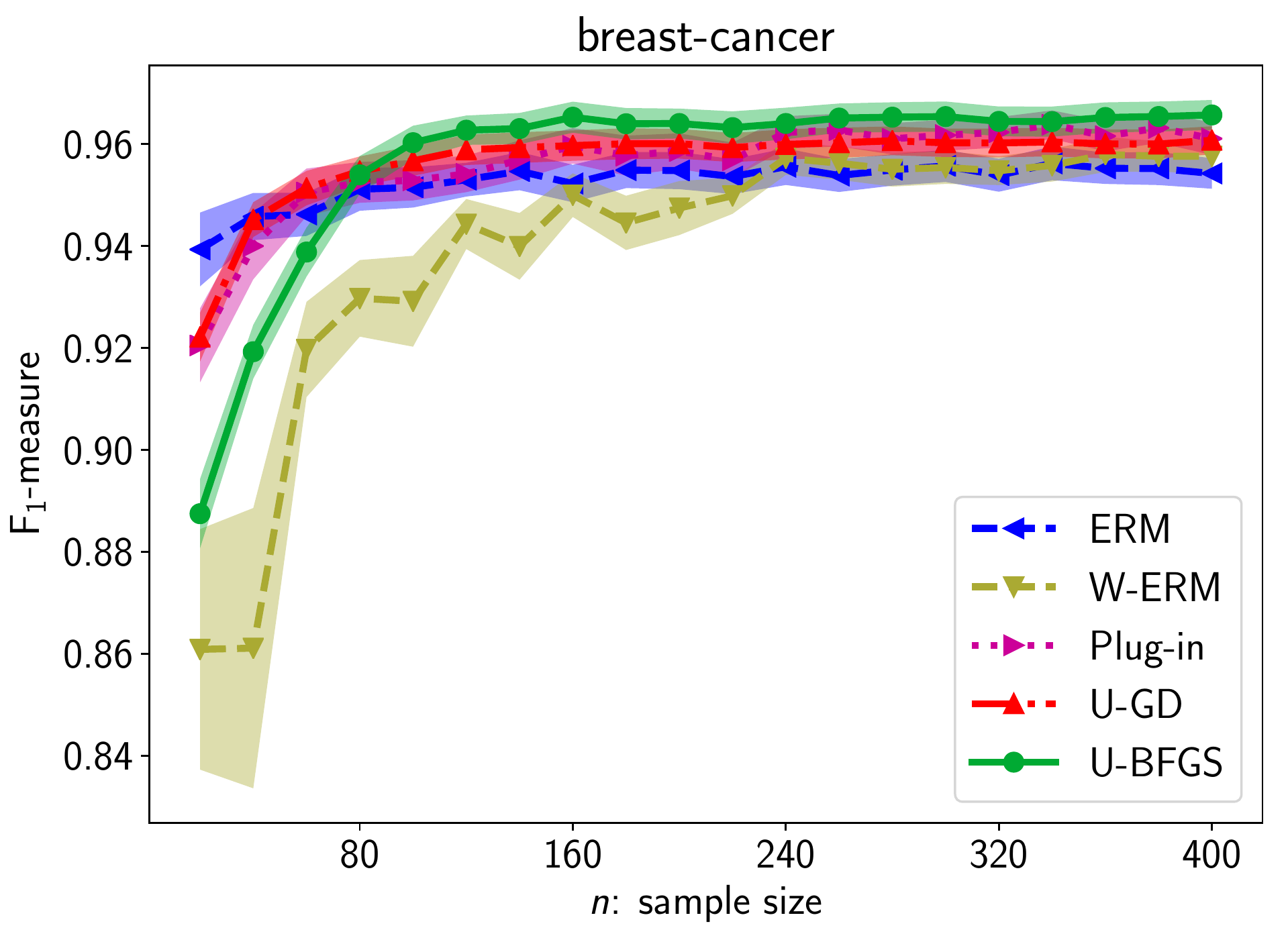}
    \end{minipage}
    \begin{minipage}{0.32\columnwidth}
        \includegraphics[width=\columnwidth]{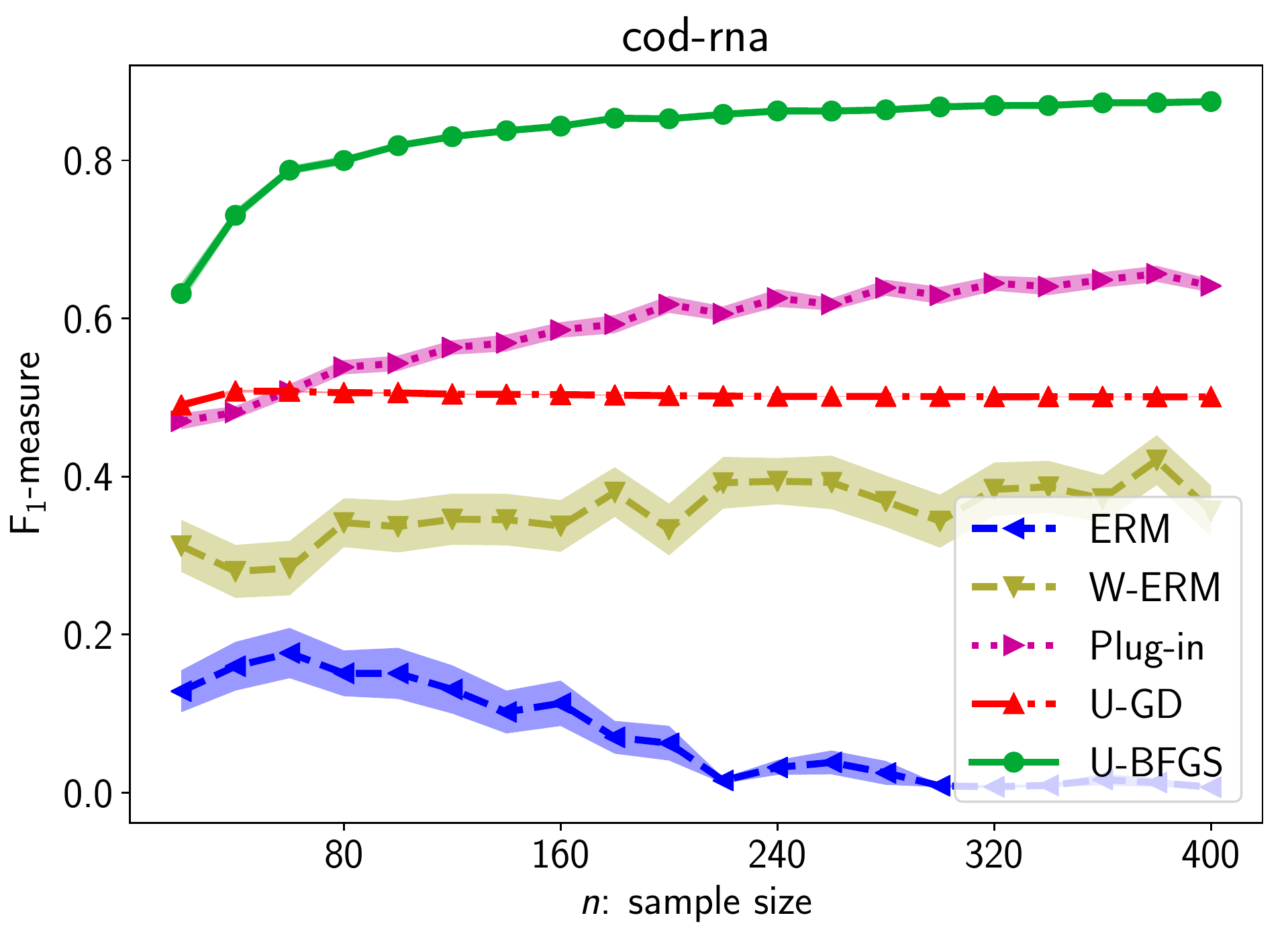}
    \end{minipage}
    \begin{minipage}{0.32\columnwidth}
        \includegraphics[width=\columnwidth]{figures/sample_complexity/sample_comp_f1_diabetes.pdf}
    \end{minipage}
    \begin{minipage}{0.32\columnwidth}
        \includegraphics[width=\columnwidth]{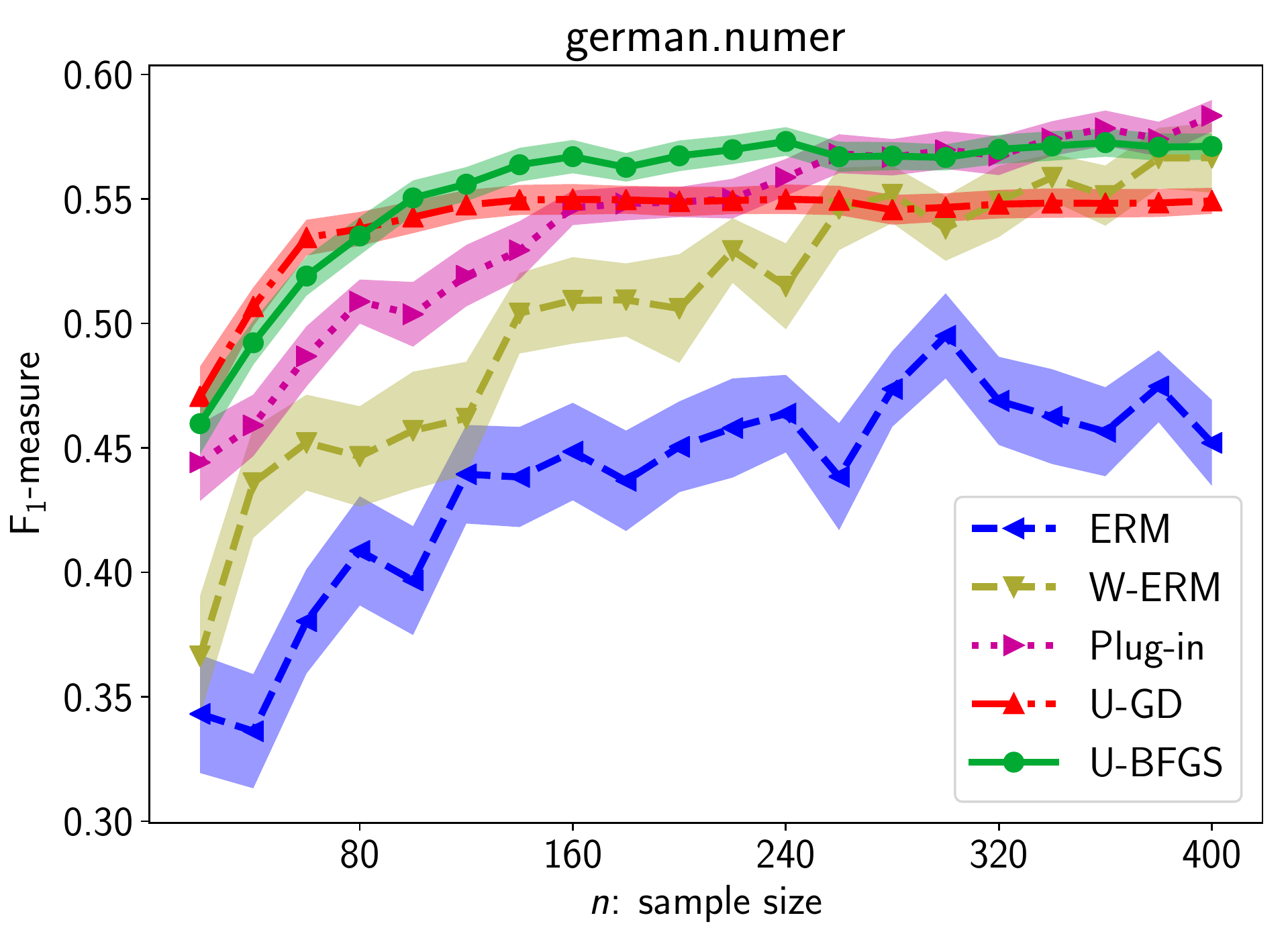}
    \end{minipage}
    \begin{minipage}{0.32\columnwidth}
        \includegraphics[width=\columnwidth]{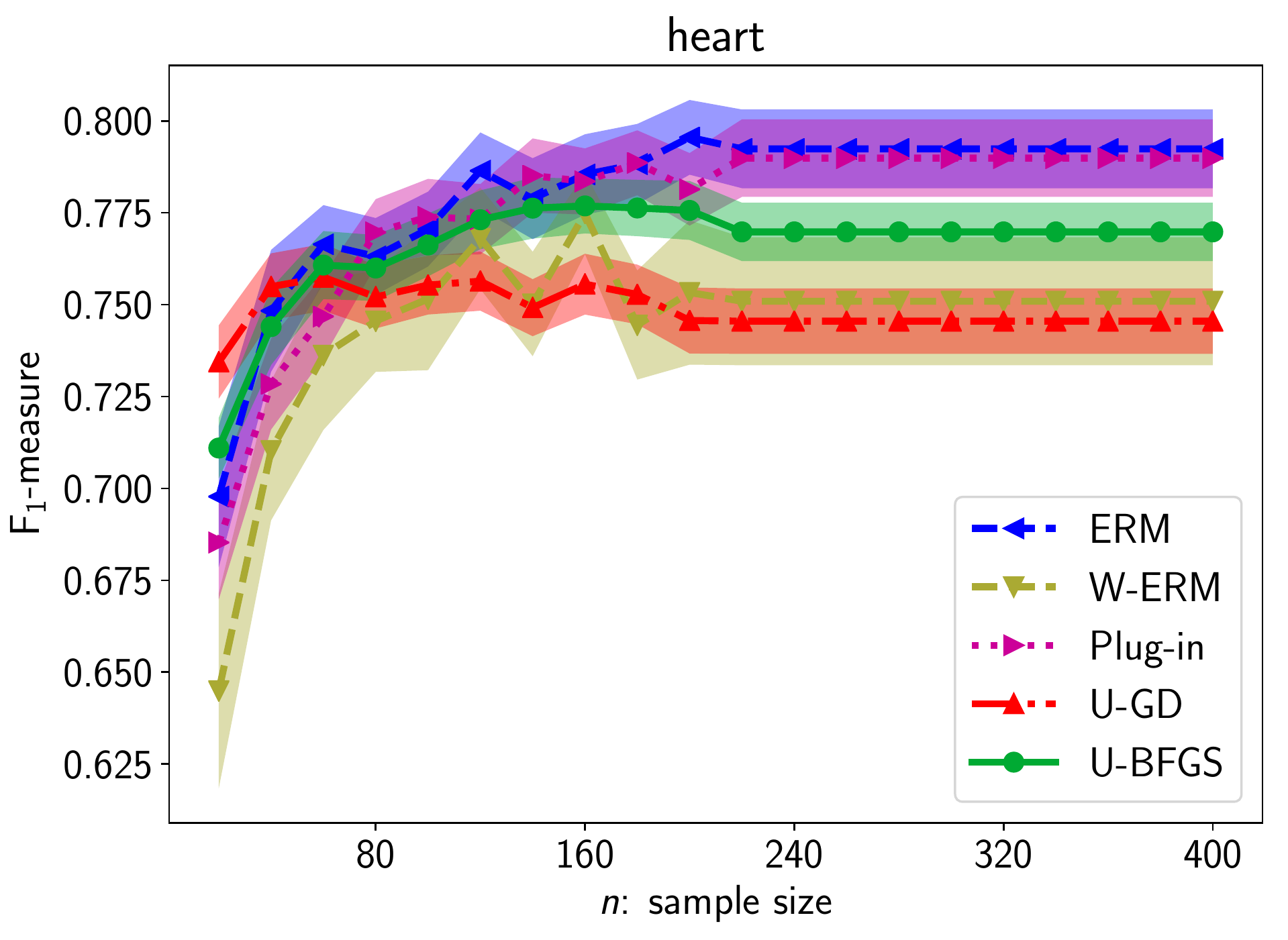}
    \end{minipage}
    \begin{minipage}{0.32\columnwidth}
        \includegraphics[width=\columnwidth]{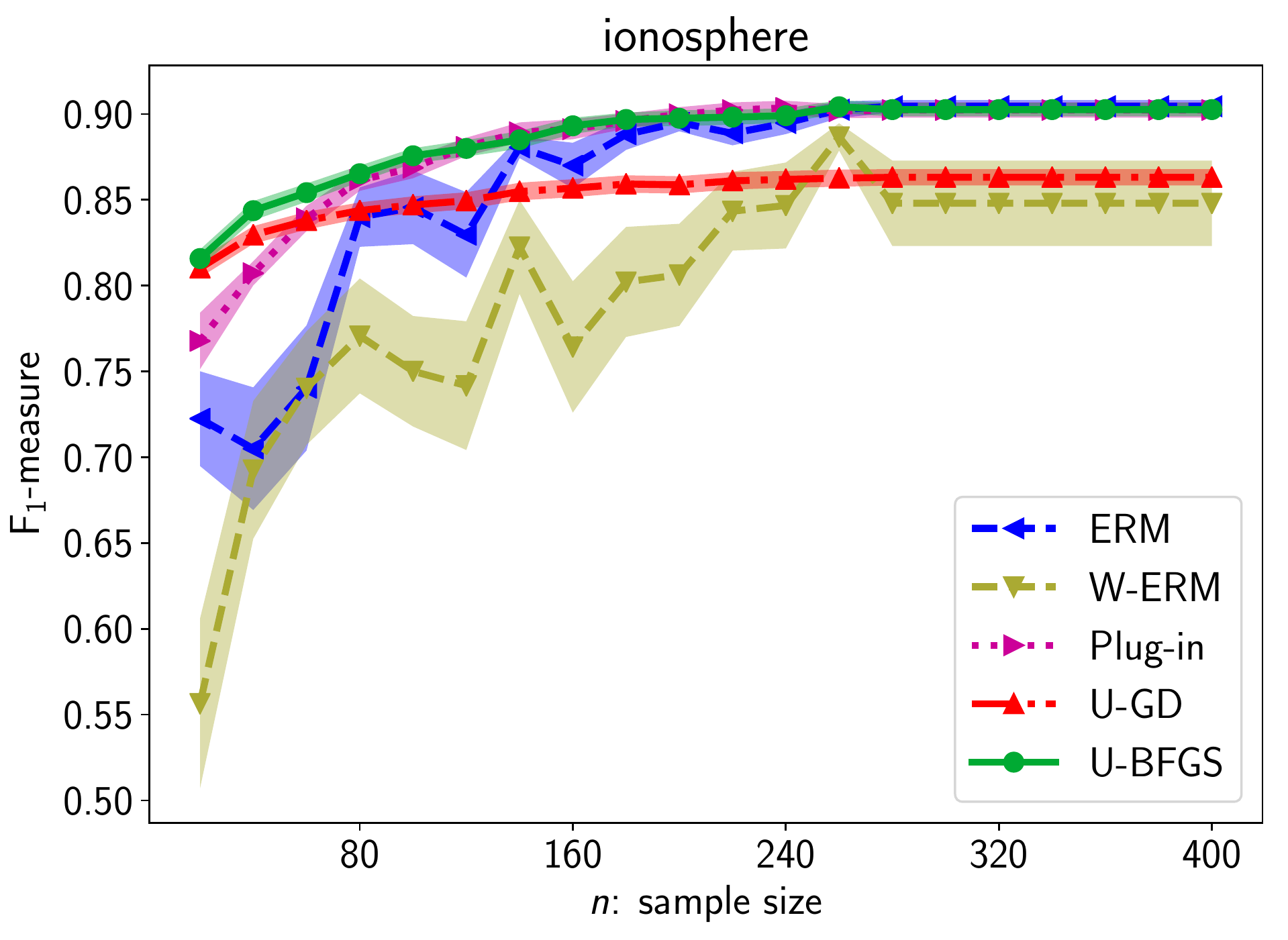}
    \end{minipage}
    \begin{minipage}{0.32\columnwidth}
        \includegraphics[width=\columnwidth]{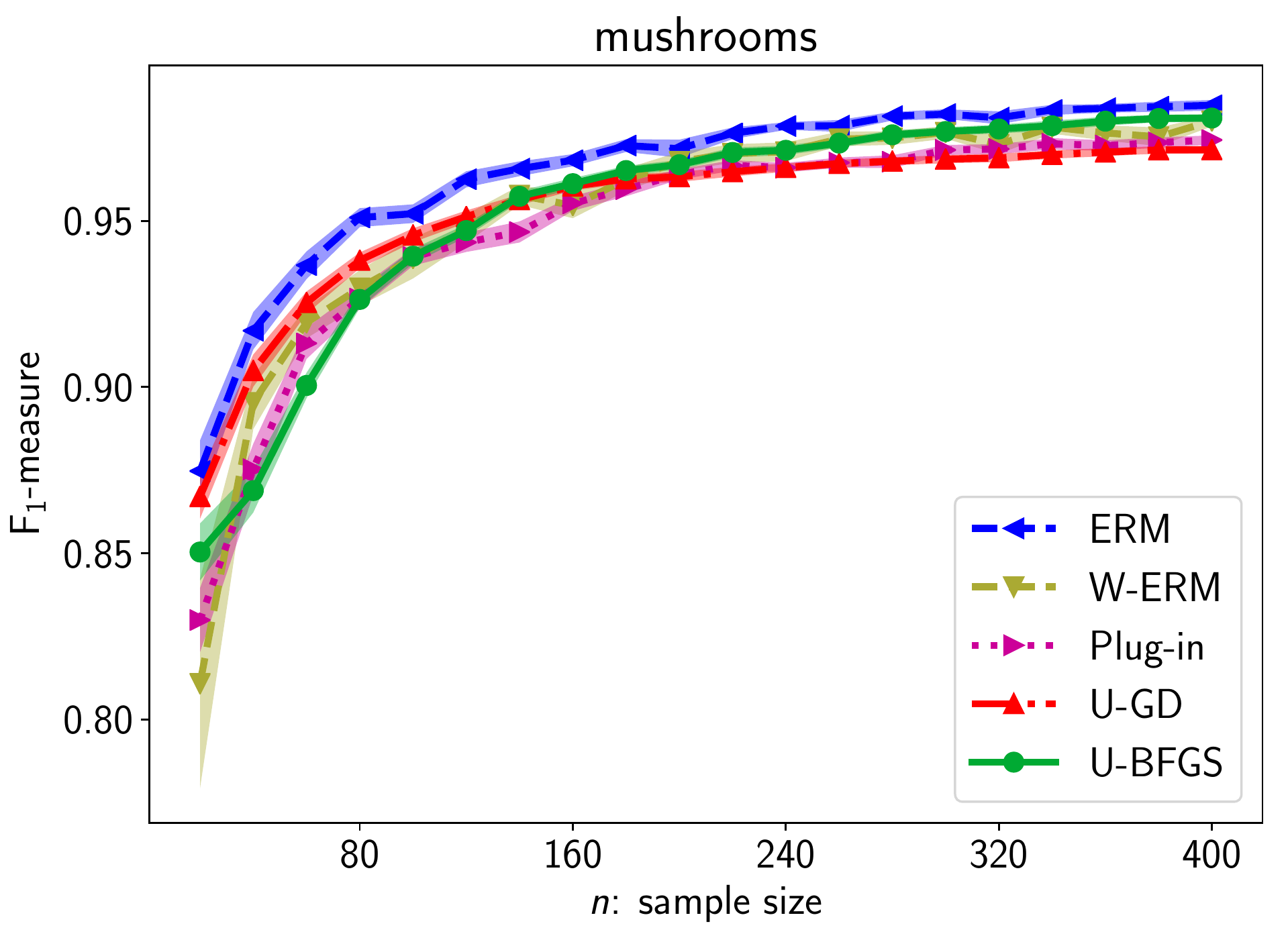}
    \end{minipage}
    \begin{minipage}{0.32\columnwidth}
        \includegraphics[width=\columnwidth]{figures/sample_complexity/sample_comp_f1_phishing.pdf}
    \end{minipage}
    \begin{minipage}{0.32\columnwidth}
        \includegraphics[width=\columnwidth]{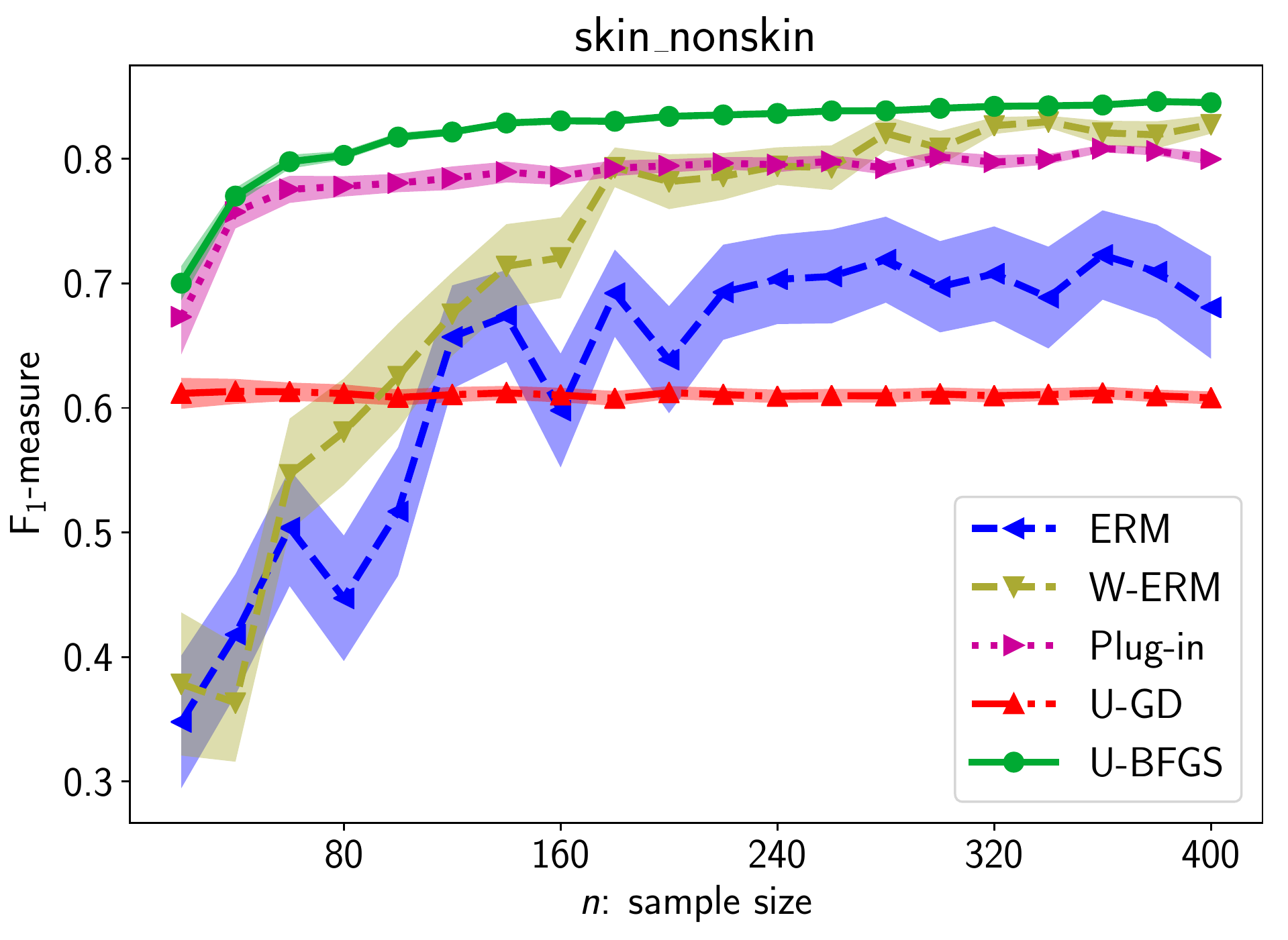}
    \end{minipage}
    \begin{minipage}{0.32\columnwidth}
        \includegraphics[width=\columnwidth]{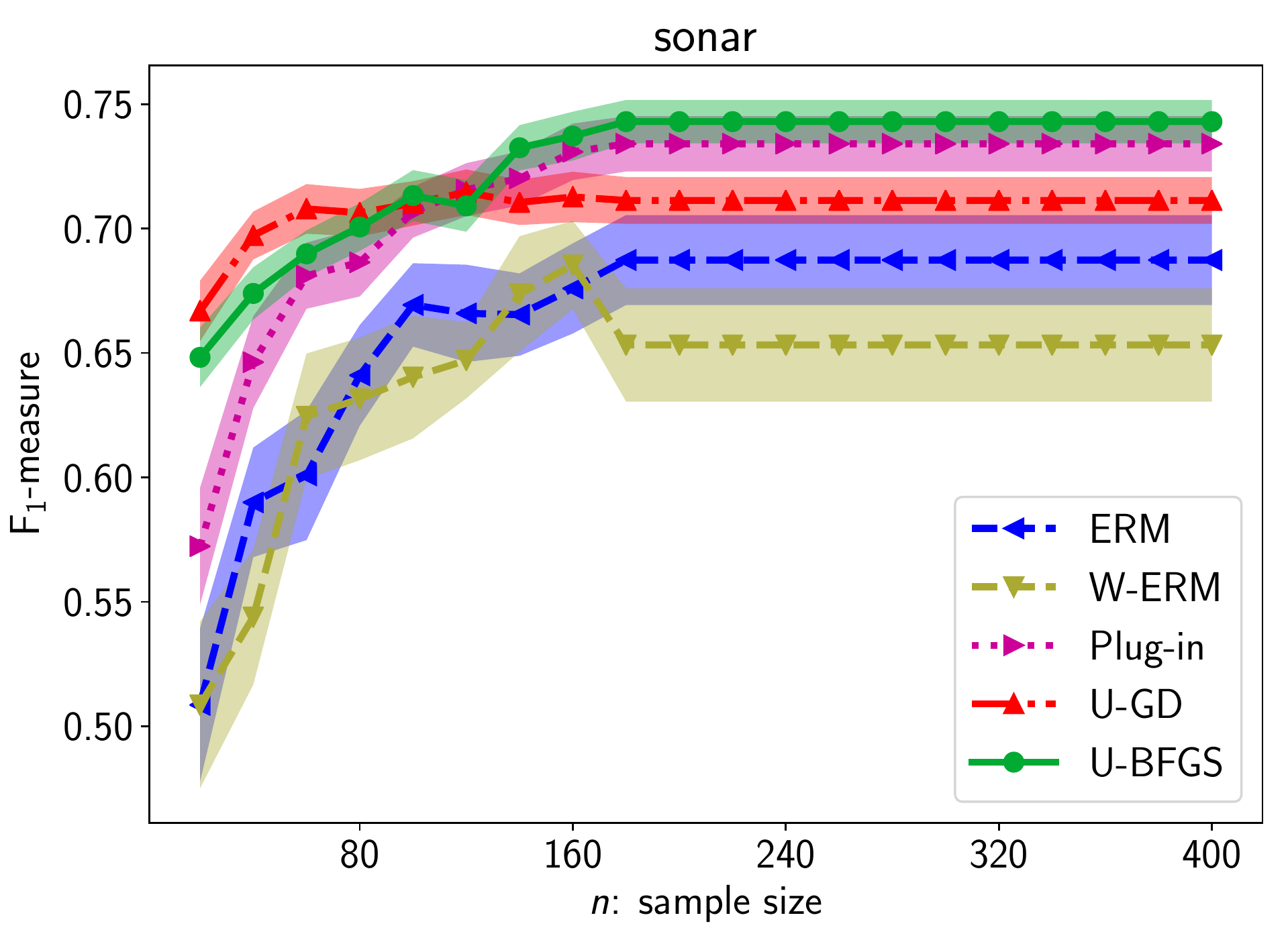}
    \end{minipage}
    \begin{minipage}{0.32\columnwidth}
        \includegraphics[width=\columnwidth]{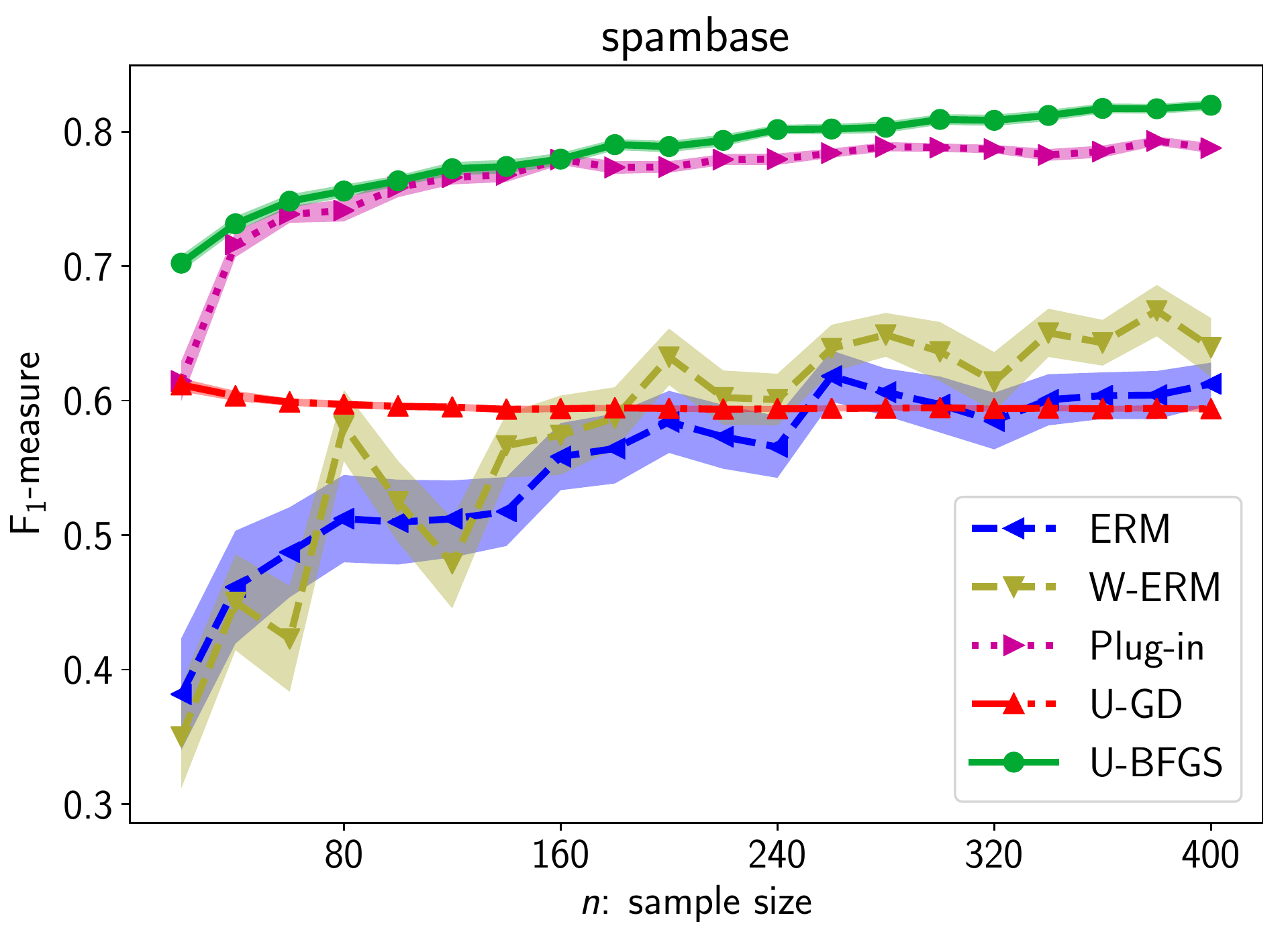}
    \end{minipage}
    \begin{minipage}{0.32\columnwidth}
        \includegraphics[width=\columnwidth]{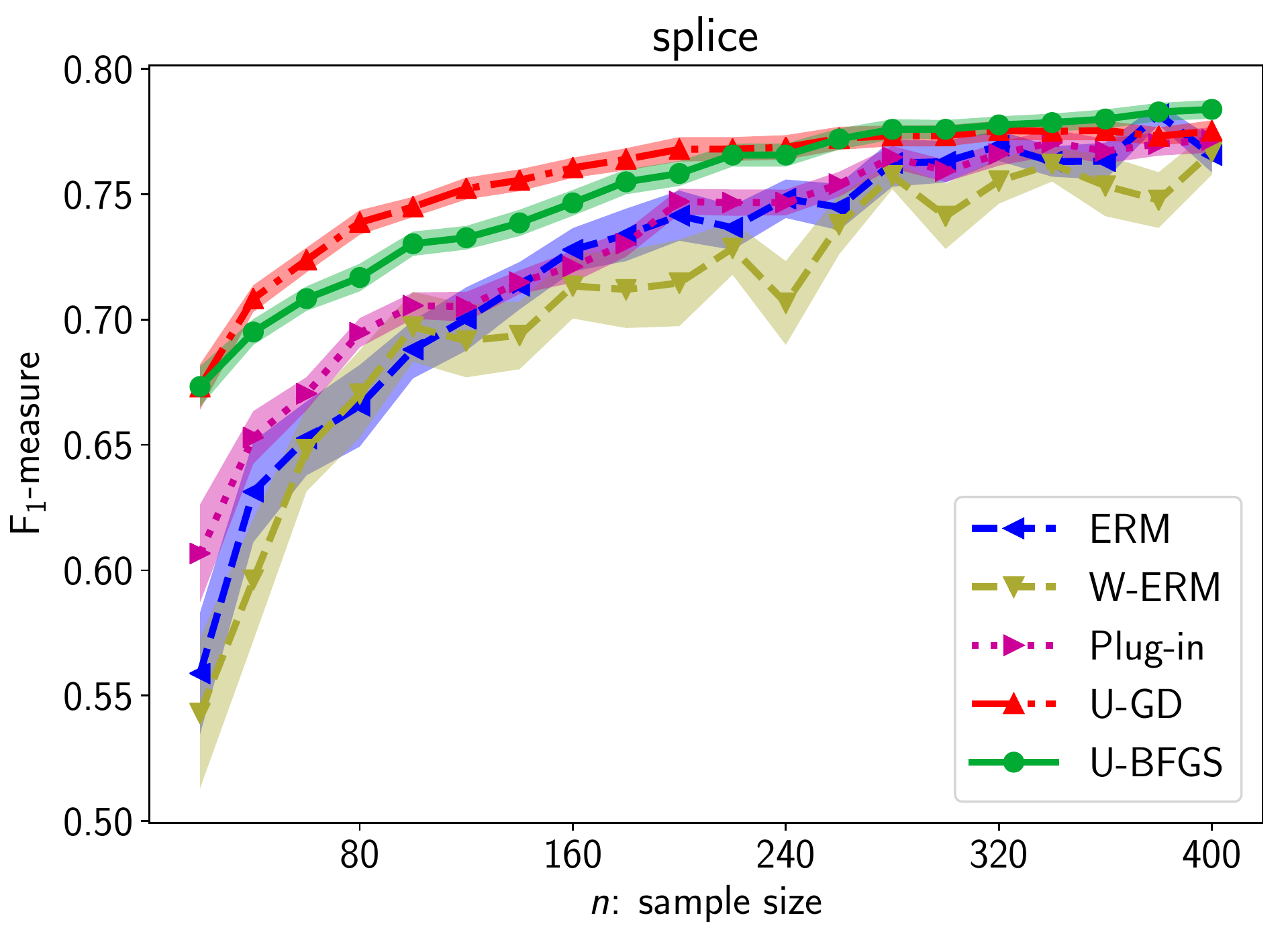}
    \end{minipage}
    \begin{minipage}{0.32\columnwidth}
        \includegraphics[width=\columnwidth]{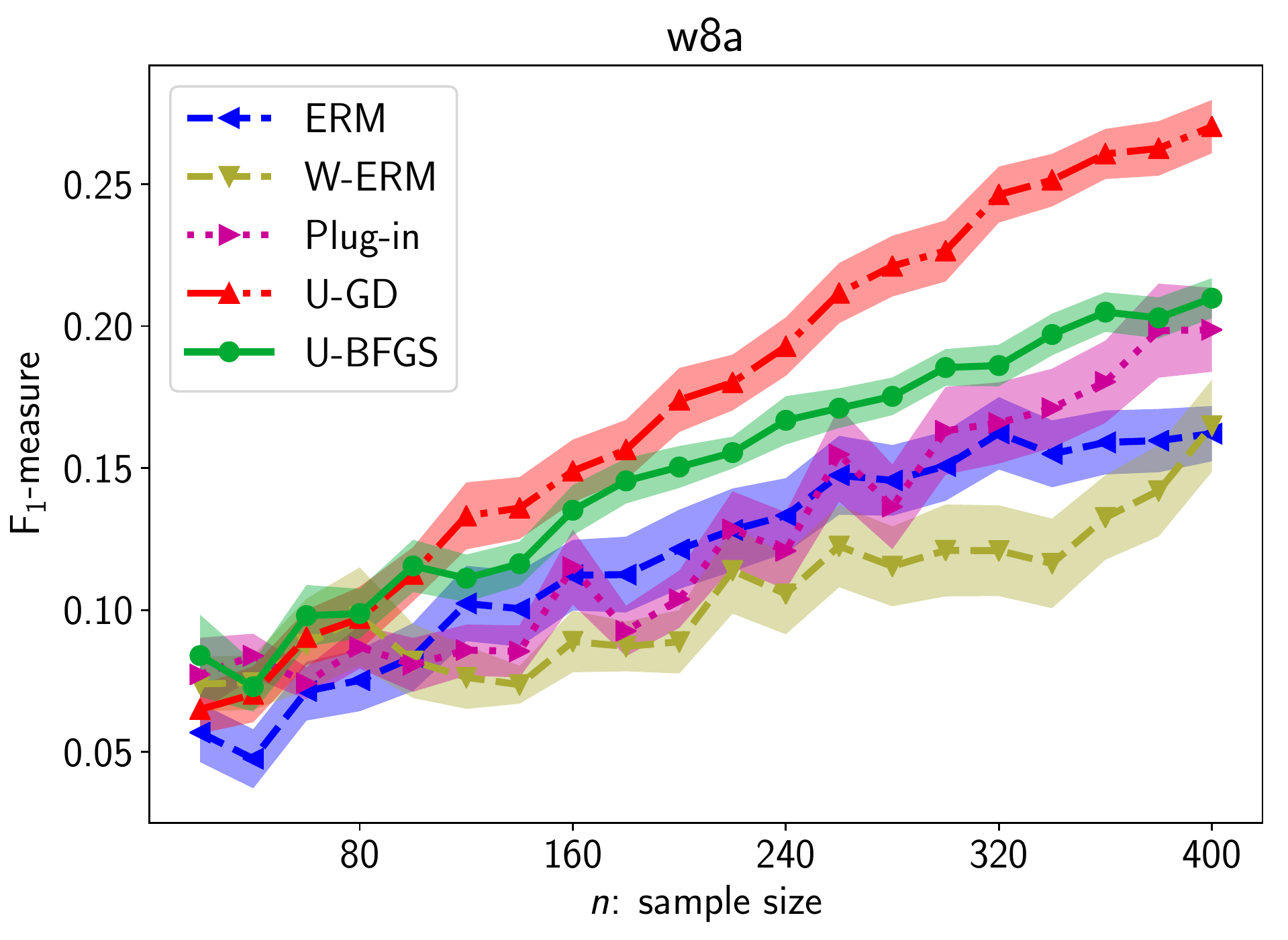}
    \end{minipage}
    \caption{
        The relationship of the test F${}_1$-measure (vertical axes) and sample size (horizontal axes).
        Standard errors of 50 trials are shown as shaded areas.
    }
    \label{fig:supp:f1-sample-complexity}
\end{figure}

\begin{figure}[h]
    \centering
    \begin{minipage}{0.32\columnwidth}
        \includegraphics[width=\columnwidth]{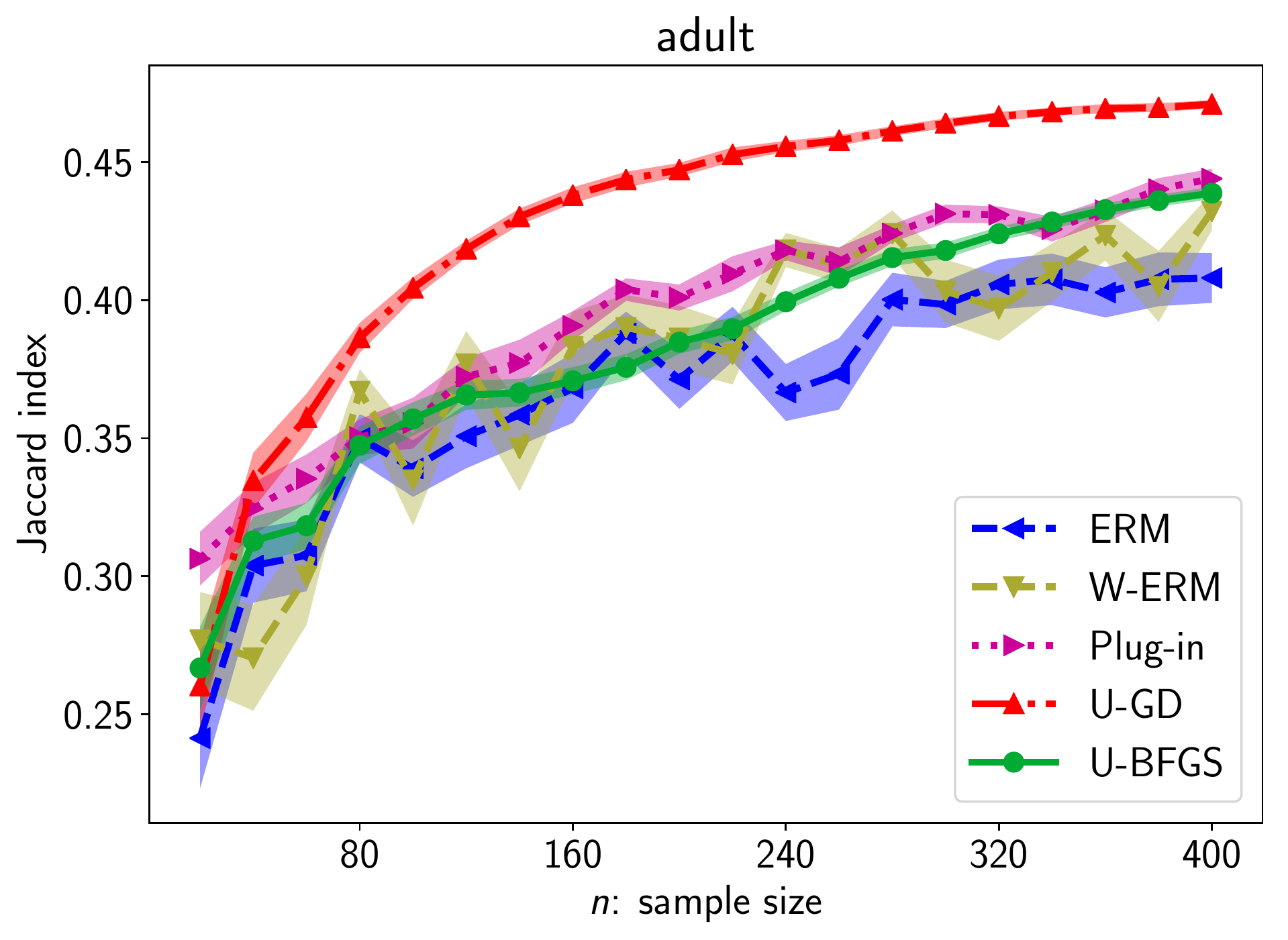}
    \end{minipage}
    \begin{minipage}{0.32\columnwidth}
        \includegraphics[width=\columnwidth]{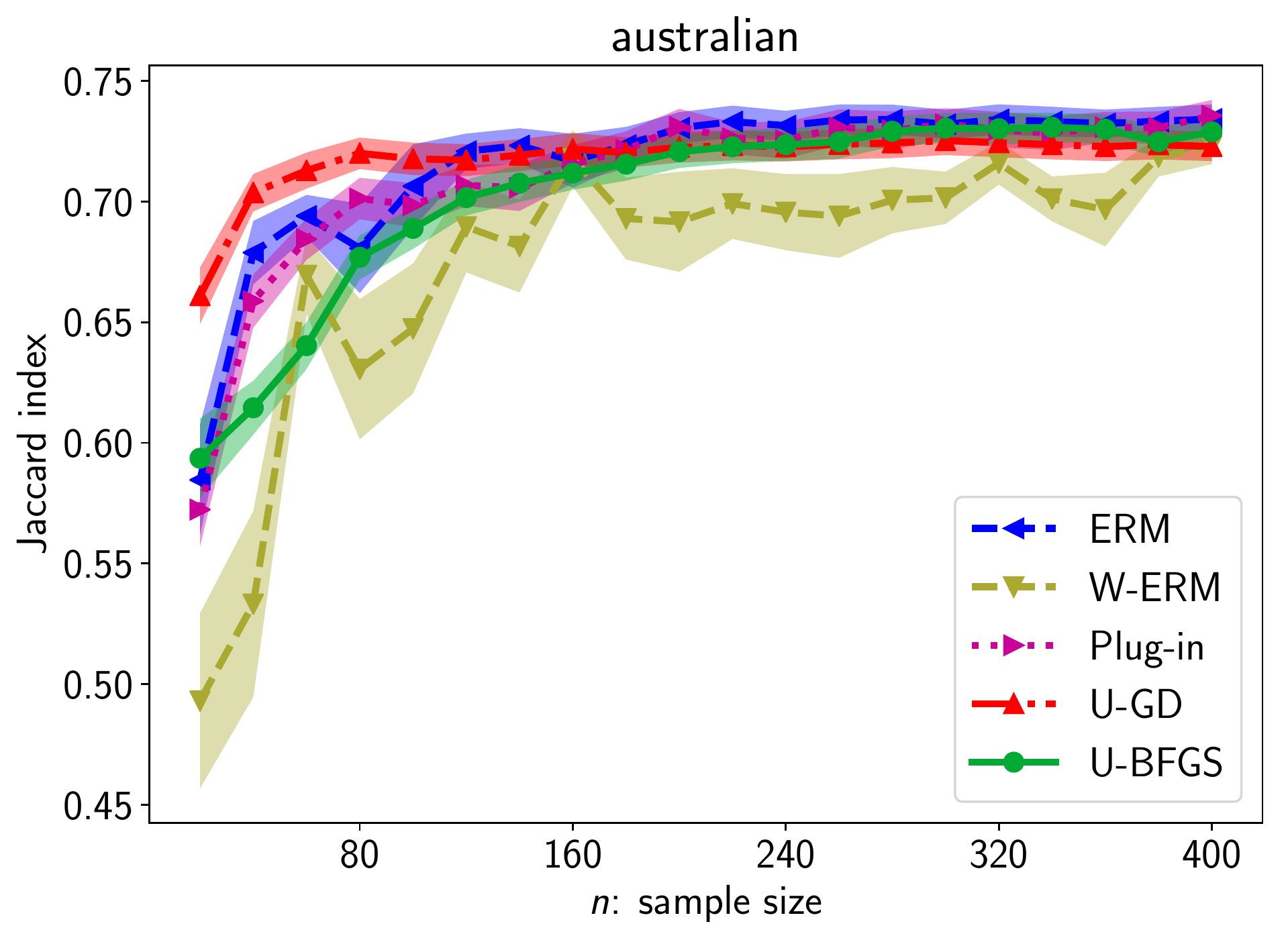}
    \end{minipage}
    \begin{minipage}{0.32\columnwidth}
        \includegraphics[width=\columnwidth]{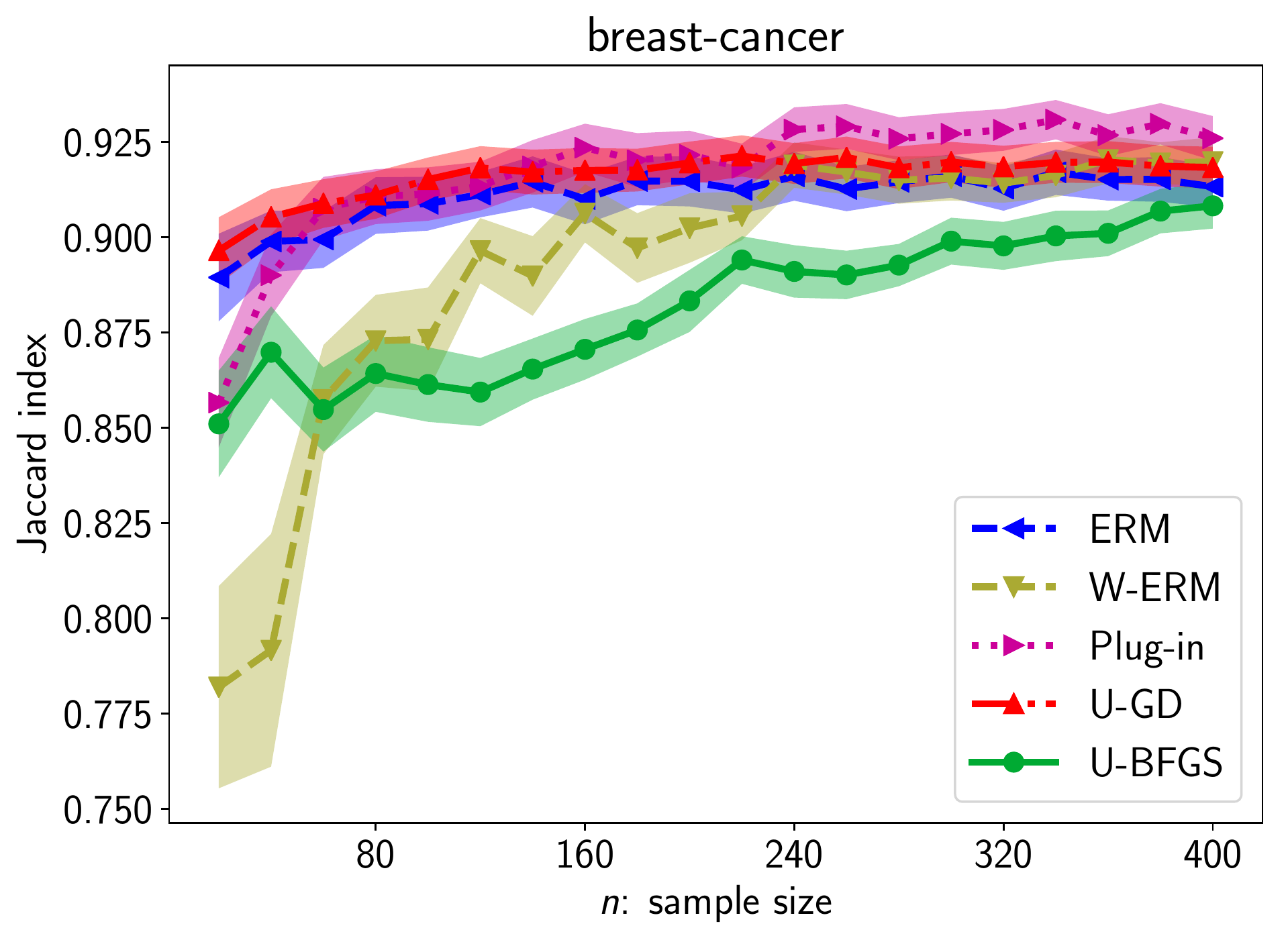}
    \end{minipage}
    \begin{minipage}{0.32\columnwidth}
        \includegraphics[width=\columnwidth]{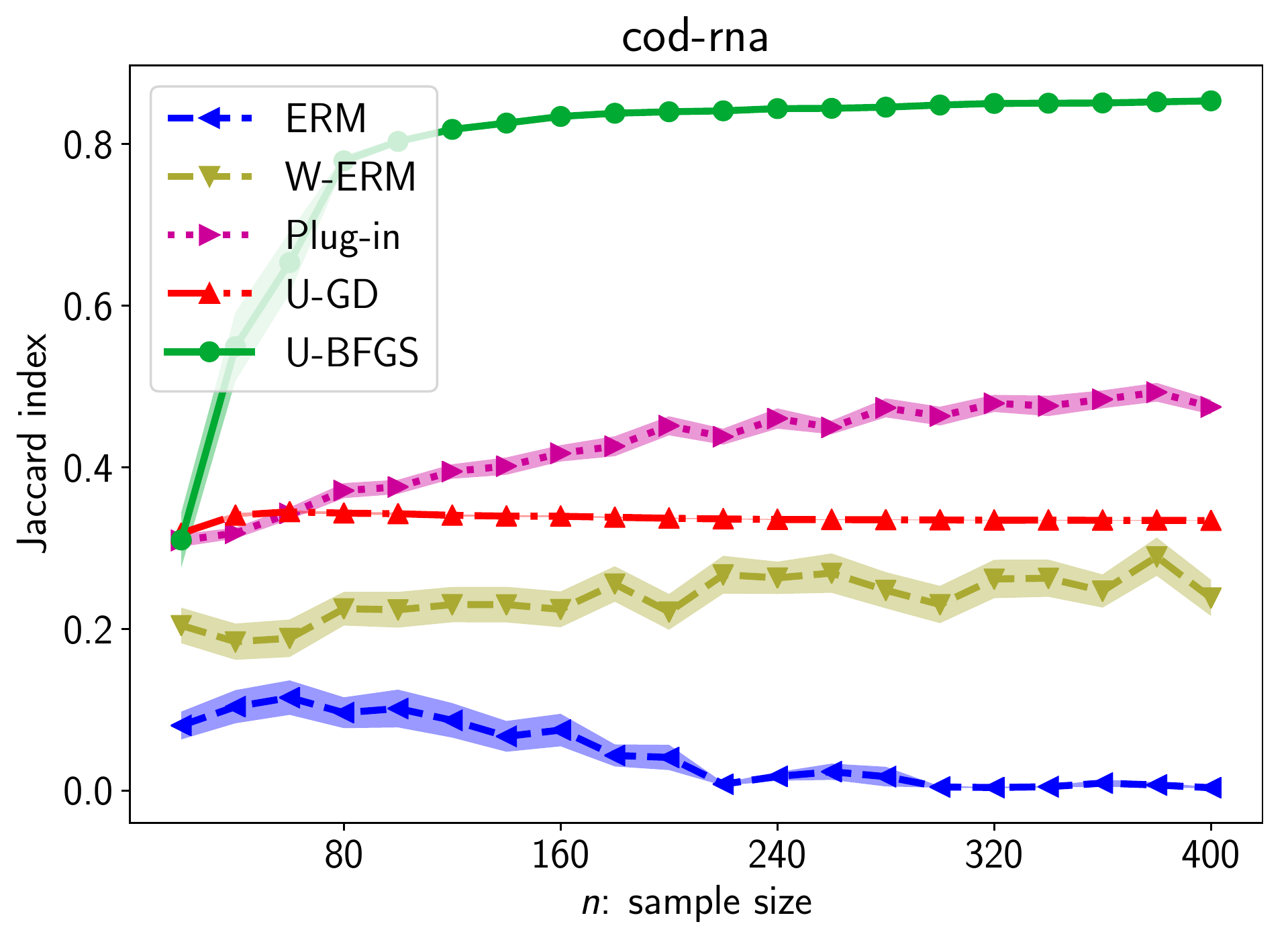}
    \end{minipage}
    \begin{minipage}{0.32\columnwidth}
        \includegraphics[width=\columnwidth]{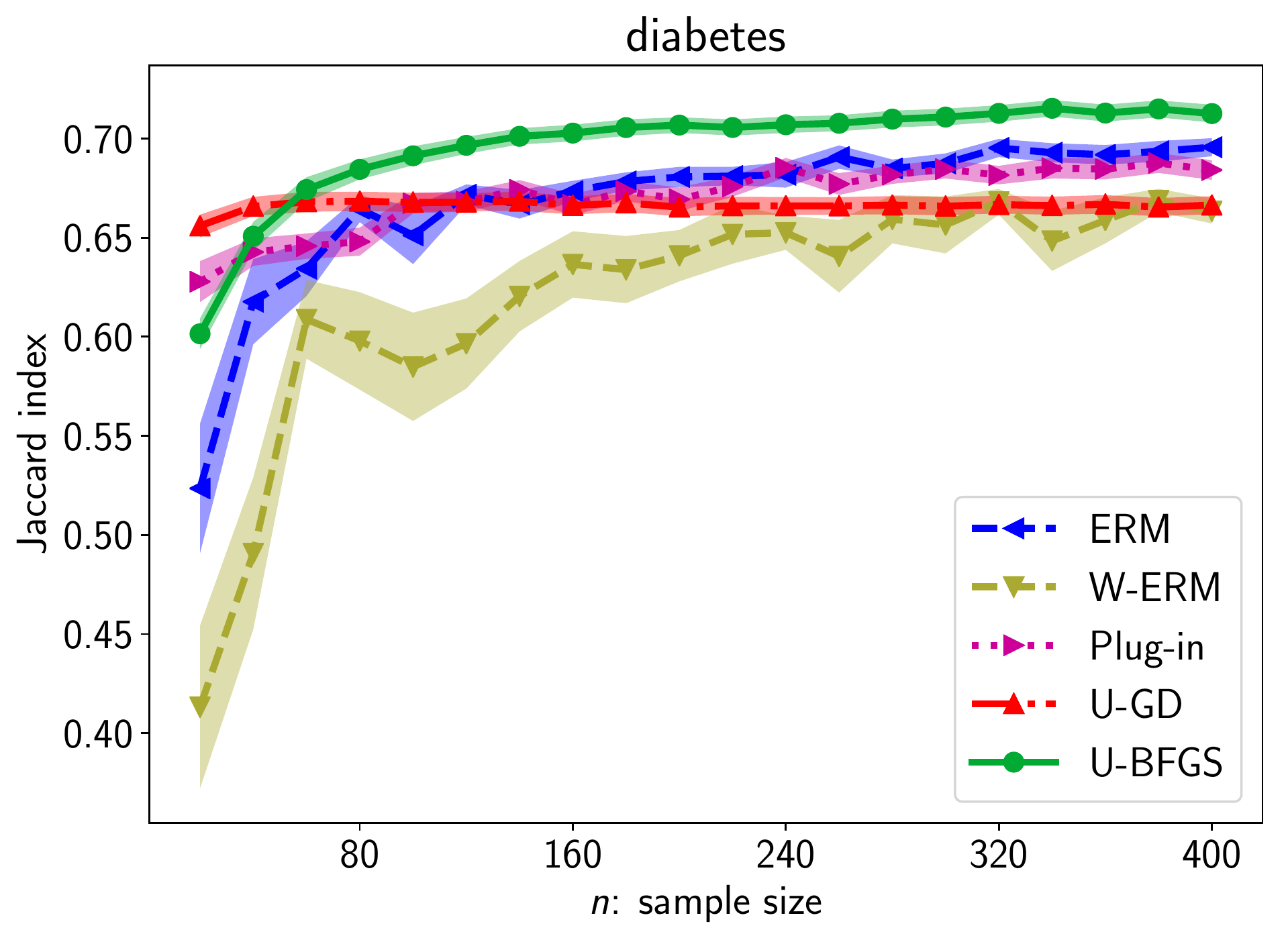}
    \end{minipage}
    \begin{minipage}{0.32\columnwidth}
        \includegraphics[width=\columnwidth]{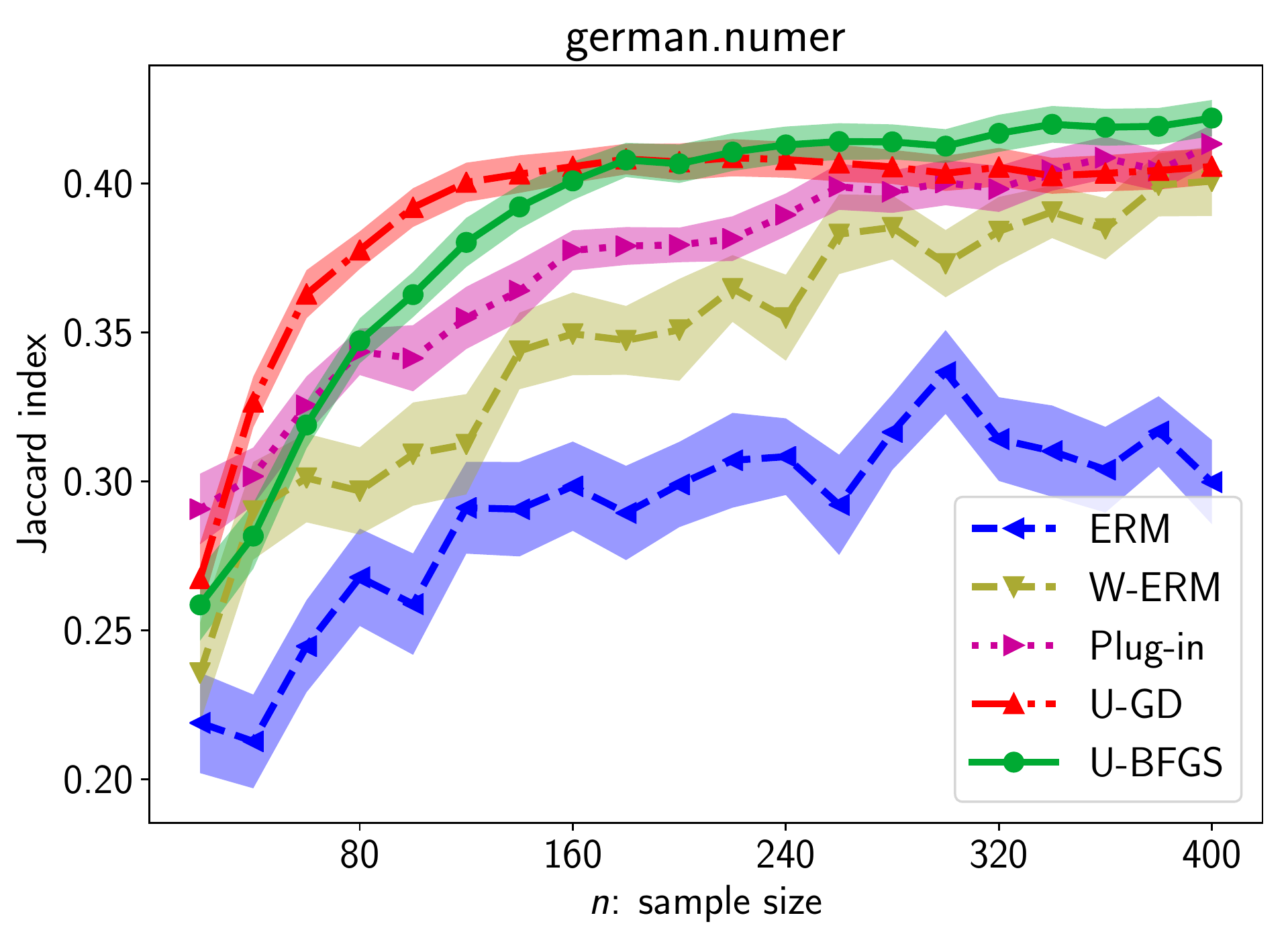}
    \end{minipage}
    \begin{minipage}{0.32\columnwidth}
        \includegraphics[width=\columnwidth]{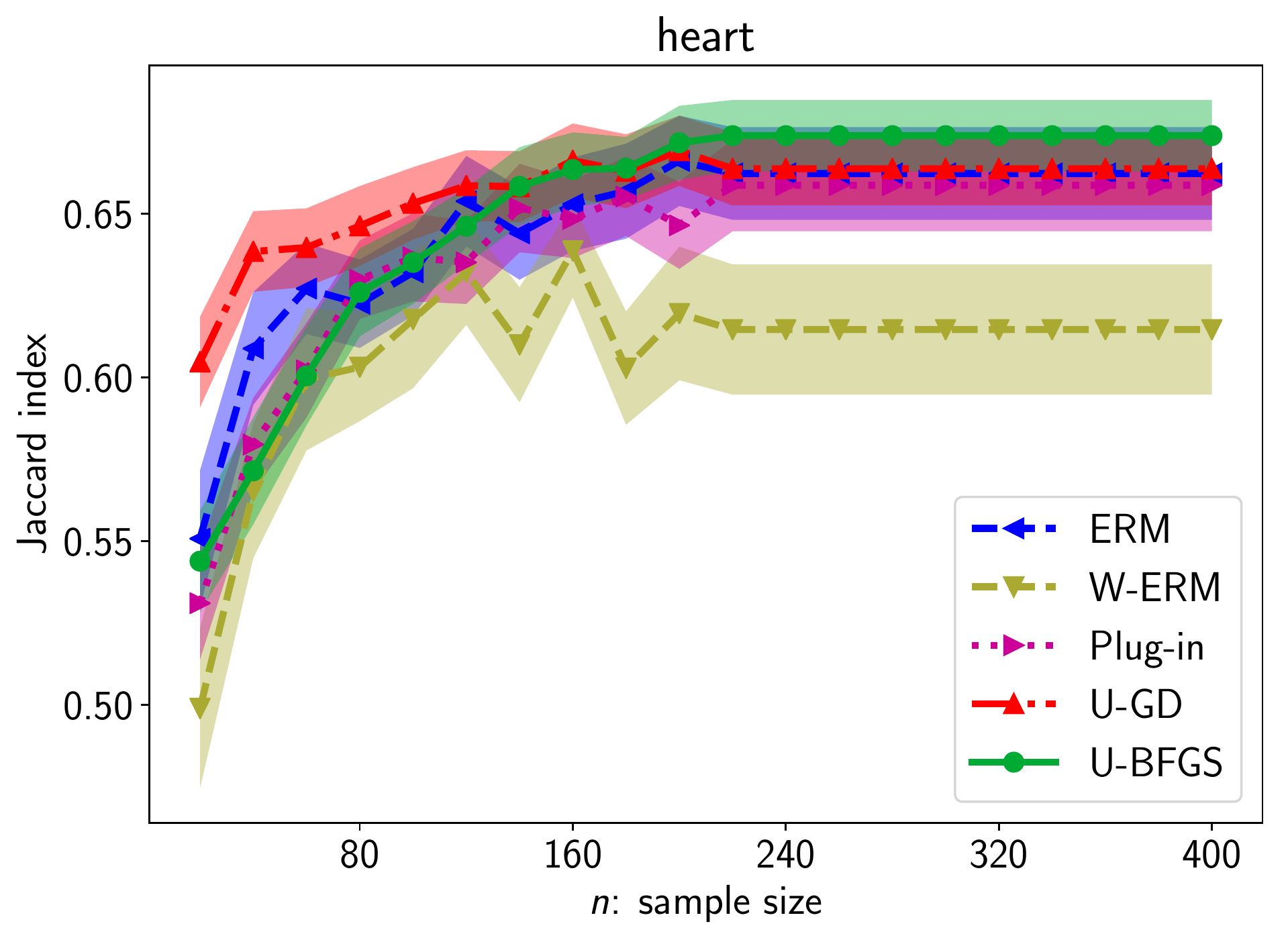}
    \end{minipage}
    \begin{minipage}{0.32\columnwidth}
        \includegraphics[width=\columnwidth]{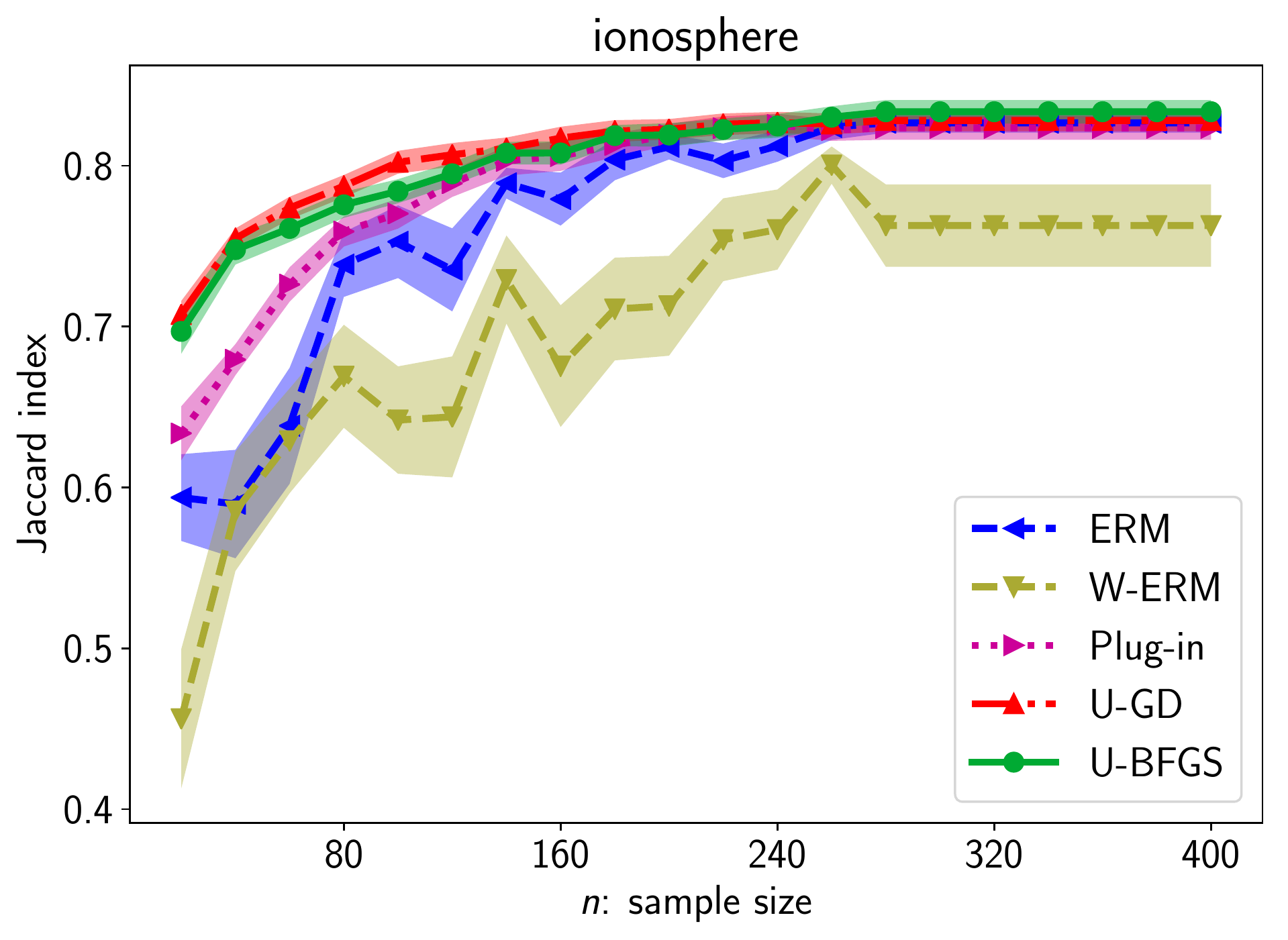}
    \end{minipage}
    \begin{minipage}{0.32\columnwidth}
        \includegraphics[width=\columnwidth]{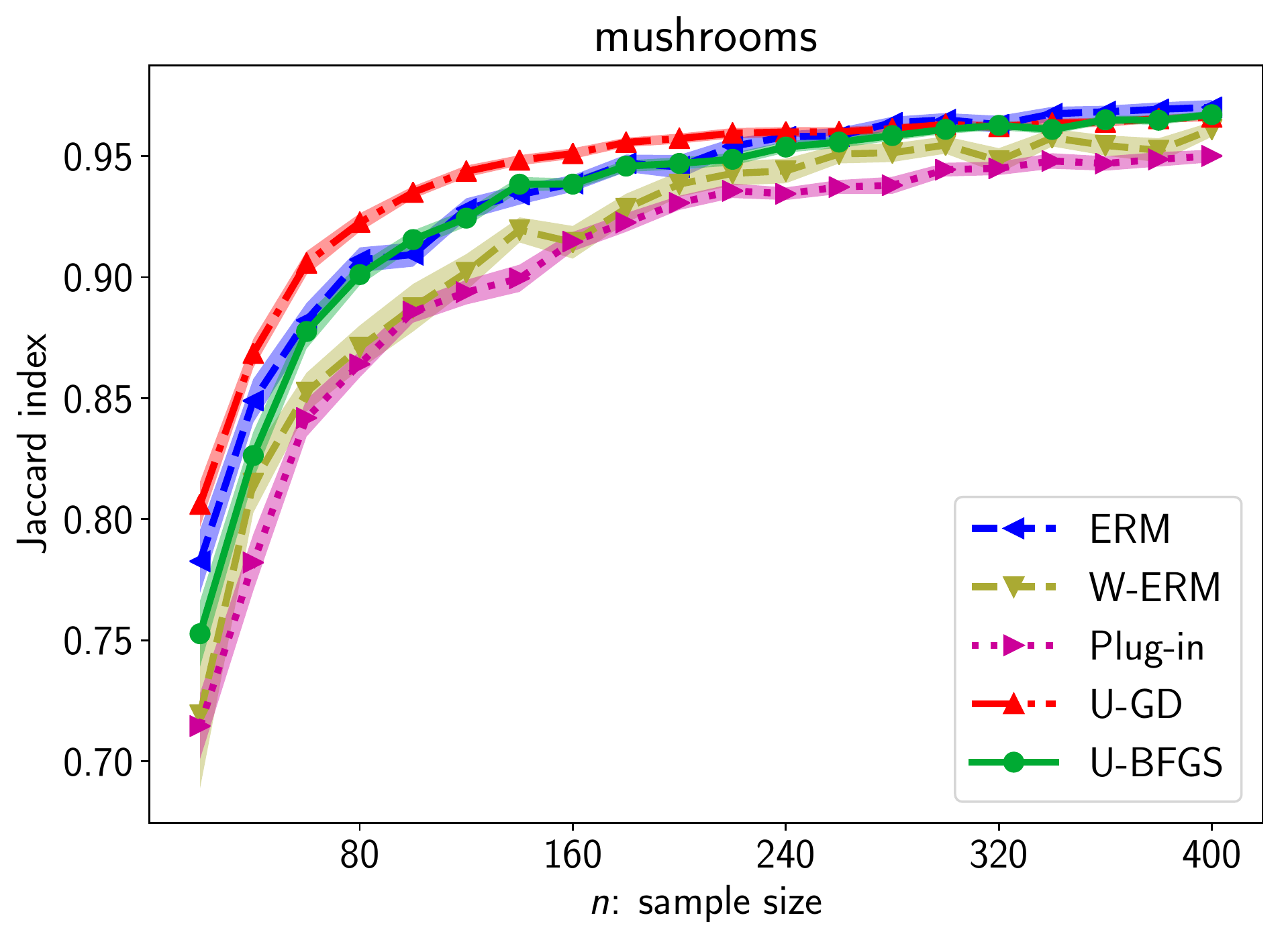}
    \end{minipage}
    \begin{minipage}{0.32\columnwidth}
        \includegraphics[width=\columnwidth]{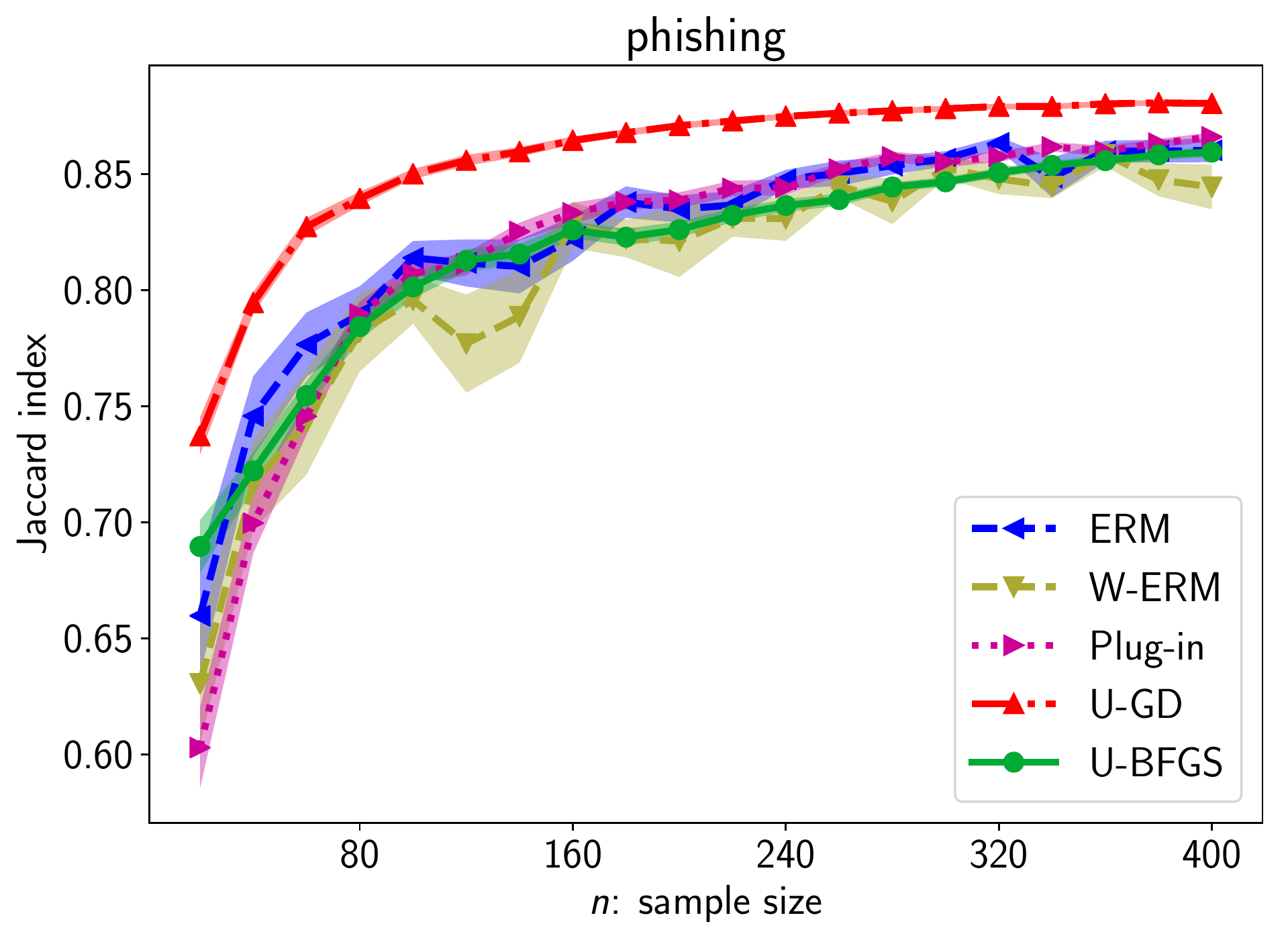}
    \end{minipage}
    \begin{minipage}{0.32\columnwidth}
        \includegraphics[width=\columnwidth]{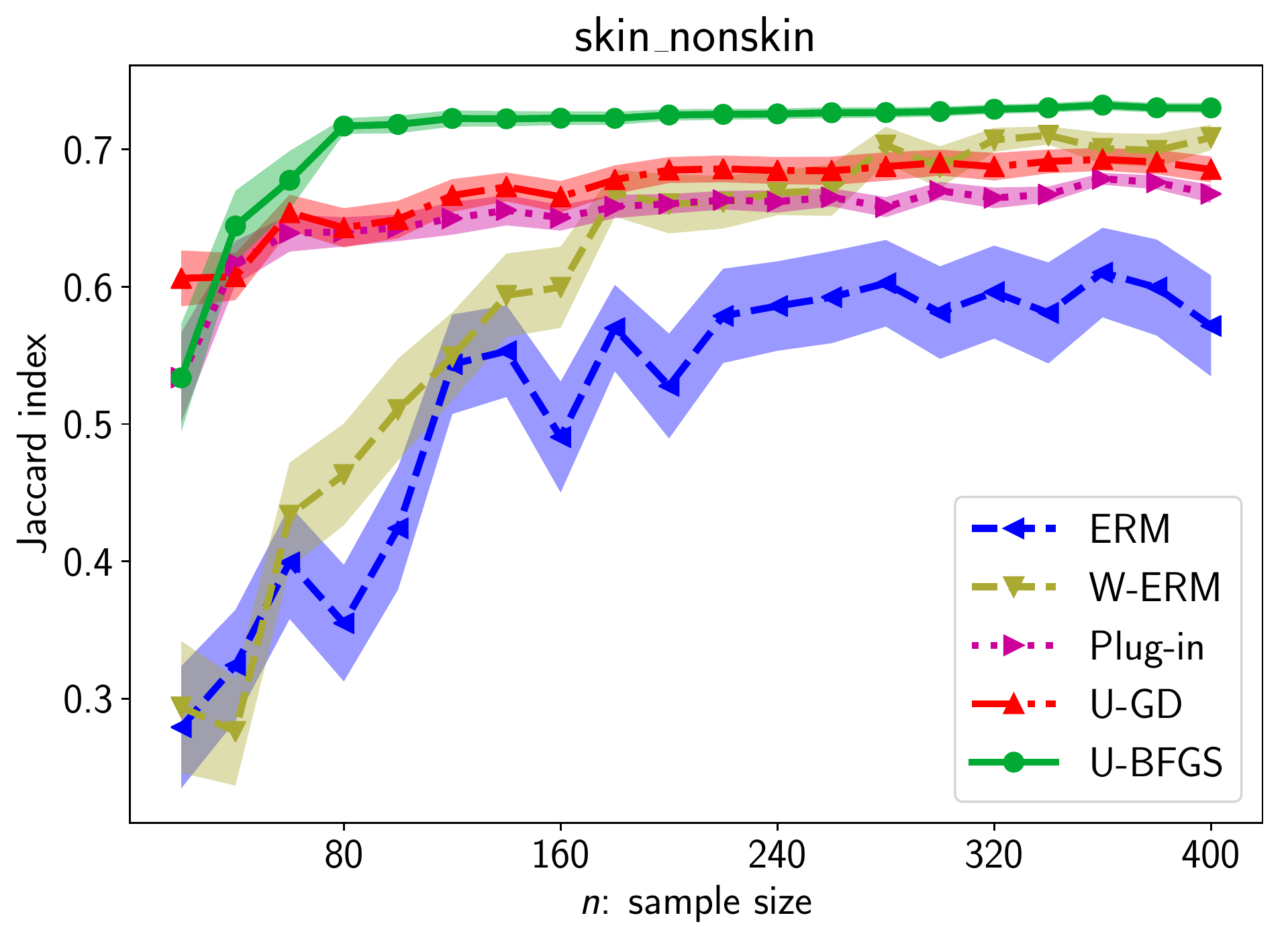}
    \end{minipage}
    \begin{minipage}{0.32\columnwidth}
        \includegraphics[width=\columnwidth]{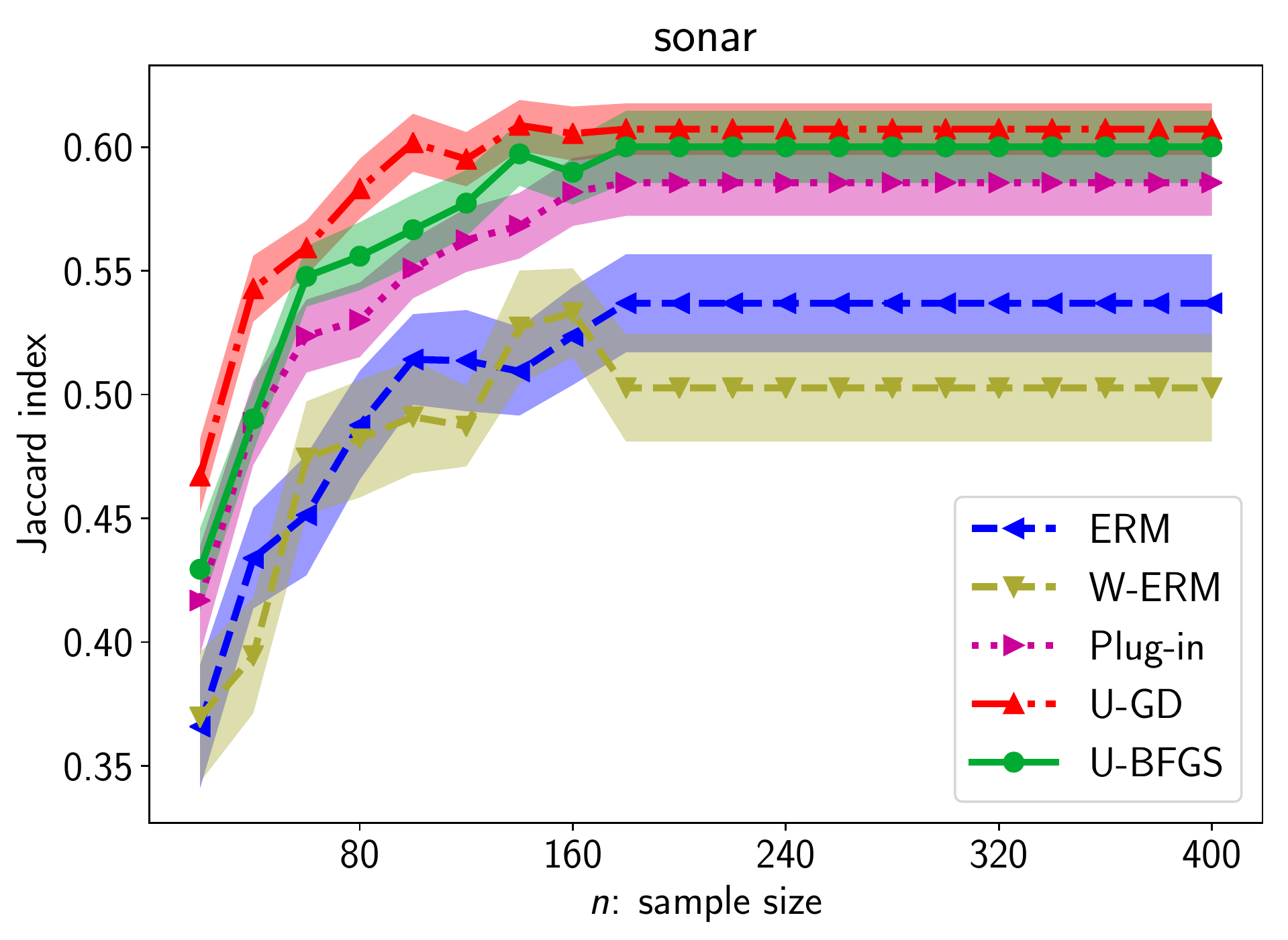}
    \end{minipage}
    \begin{minipage}{0.32\columnwidth}
        \includegraphics[width=\columnwidth]{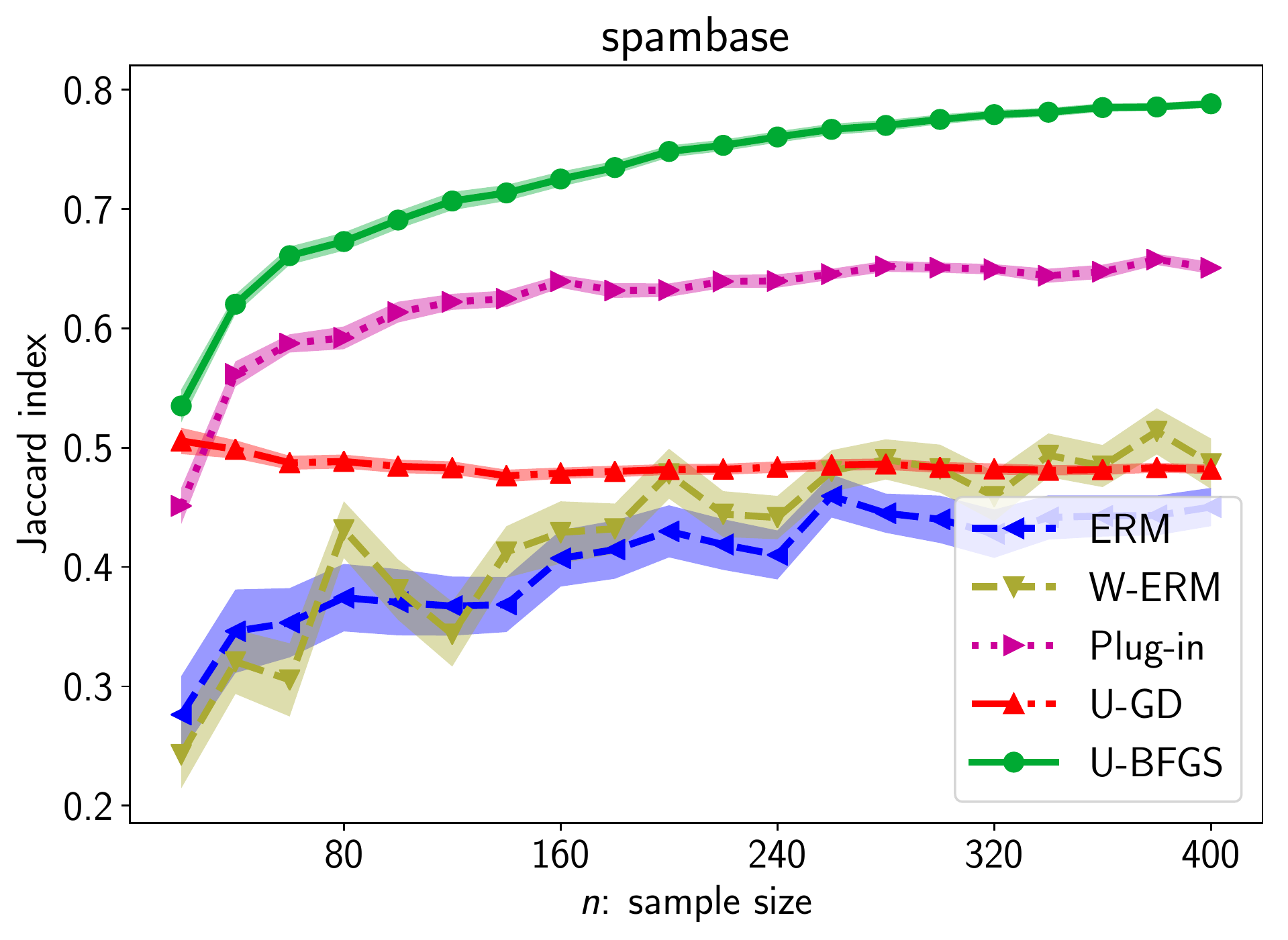}
    \end{minipage}
    \begin{minipage}{0.32\columnwidth}
        \includegraphics[width=\columnwidth]{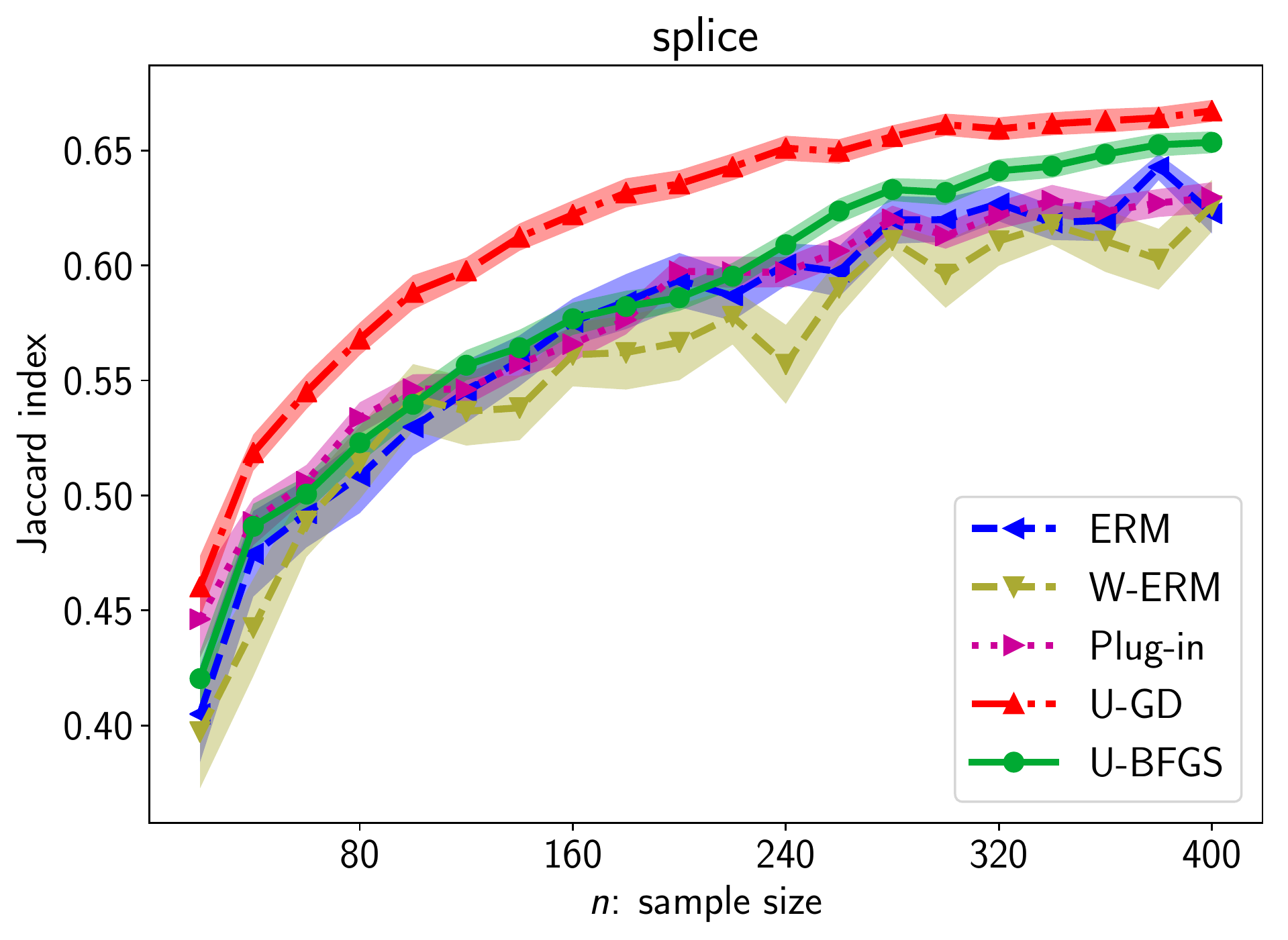}
    \end{minipage}
    \begin{minipage}{0.32\columnwidth}
        \includegraphics[width=\columnwidth]{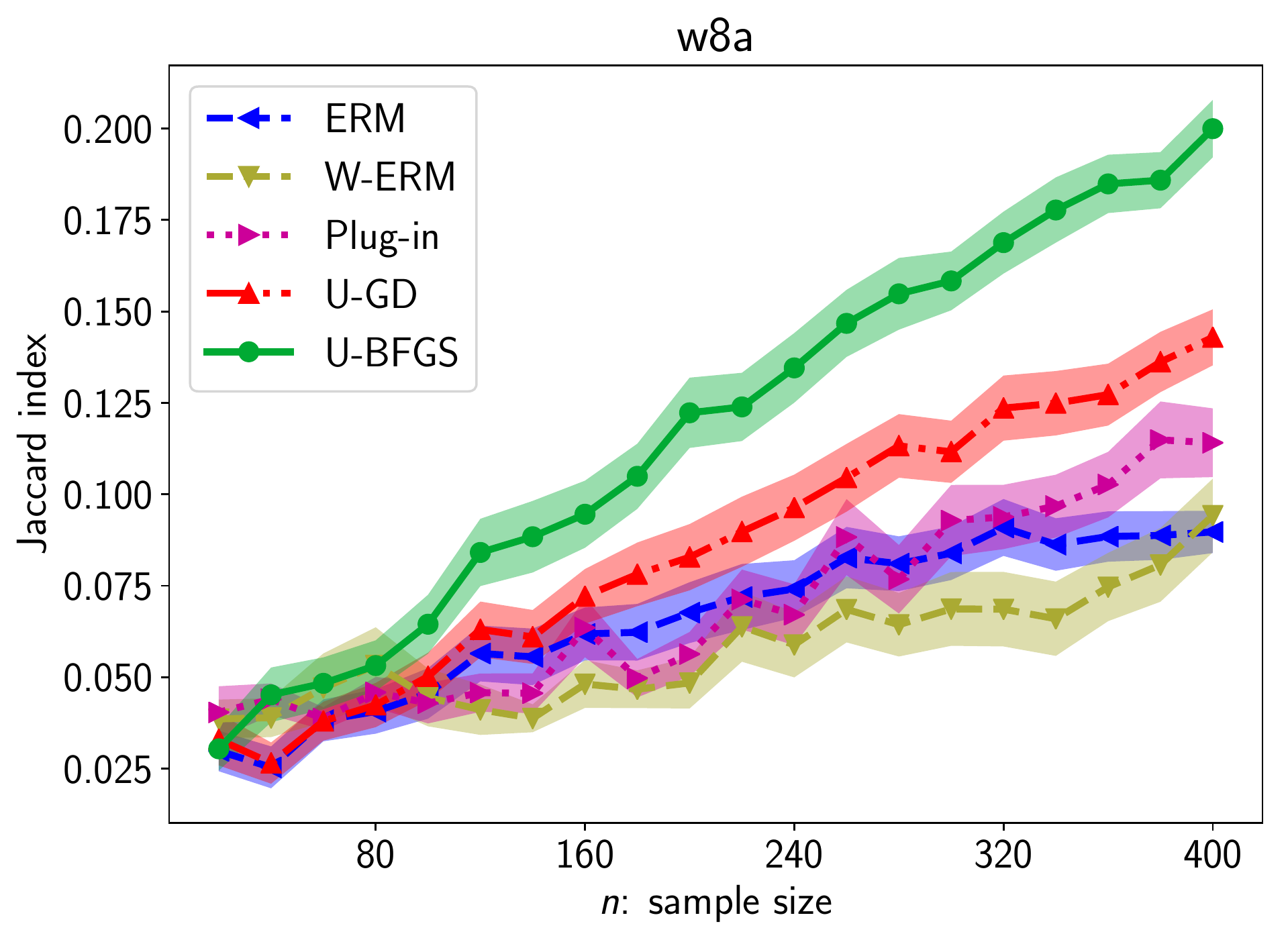}
    \end{minipage}
    \caption{
        The relationship of the test Jaccard (vertical axes) and sample size (horizontal axes).
        Standard errors of 50 trials are shown as shaded areas.
    }
    \label{fig:supp:jac-sample-complexity}
\end{figure}

\subsection{Performance Sensitivity on \texorpdfstring{$\tau$}{tau}}
\label{sec:supp:tau-sensitivity}

Lastly, we see the performance sensitivity on the choices of $\tau$.
We change $\tau \in \{0.1, 0.2, \dots, 0.9\}$ and run U-GD and U-BFGS for both the F${}_1$-measure and Jaccard index.
The results are summarized in Figures~\ref{fig:supp:f1-sensitivity} and~\ref{fig:supp:jac-sensitivity}.
From these figure, we can say there is a tendency that the performance becomes better as $\tau$ becomes closer to $1$.
For example, the below combinations of the datasets and metrics have such a tendency.
\begin{itemize}
    \item australian, breast-cancer, german.numer, heart, ionosphere, mushrooms, phishing, and splice in the F${}_\beta$-measure,
    \item australian, mushrooms, phishing, and splice in the Jaccard index.
\end{itemize}
However, there are also other cases where there exist extrema of the performance with respect to the choices of $\tau$.
For example, the below combinations of the datasets and metrics have such a tendency.
\begin{itemize}
    \item german.numer and sonar in the F${}_\beta$-measure,
    \item breast-cancer, heart, ionosphere and sonar in the Jaccard index.
\end{itemize}
From our theoretical results in Theorems~\ref{thm:f-measure-calibration} and~\ref{thm:jaccard-calibration},
we cannot determine whether the surrogate utility is calibrated or not
if $\tau$ exceeds about $0.33$ for the F${}_\beta$-measure, and becomes closer to $1.0$ for the Jaccard index.
These thresholds are not so clear in Figures~\ref{fig:supp:f1-sensitivity} and~\ref{fig:supp:jac-sensitivity} since the conditions on $\tau$ is merely sufficient conditions,
as we explain in Sec.~\ref{sec:calibration}.
Further analyses on the discrepancy parameter are left for future work.

\begin{figure}[h]
    \centering
    \begin{minipage}{0.32\columnwidth}
        \includegraphics[width=\columnwidth]{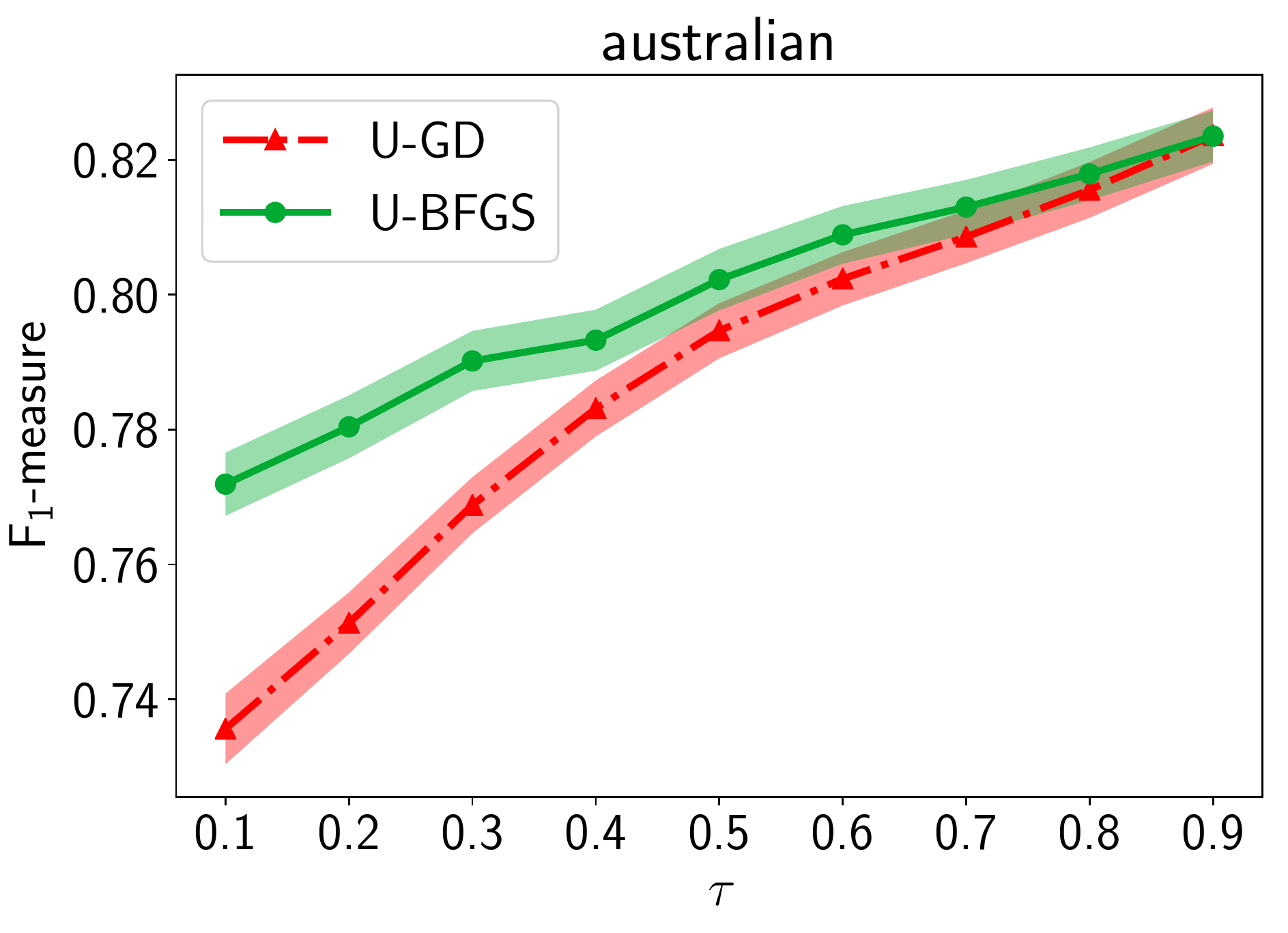}
    \end{minipage}
    \begin{minipage}{0.32\columnwidth}
        \includegraphics[width=\columnwidth]{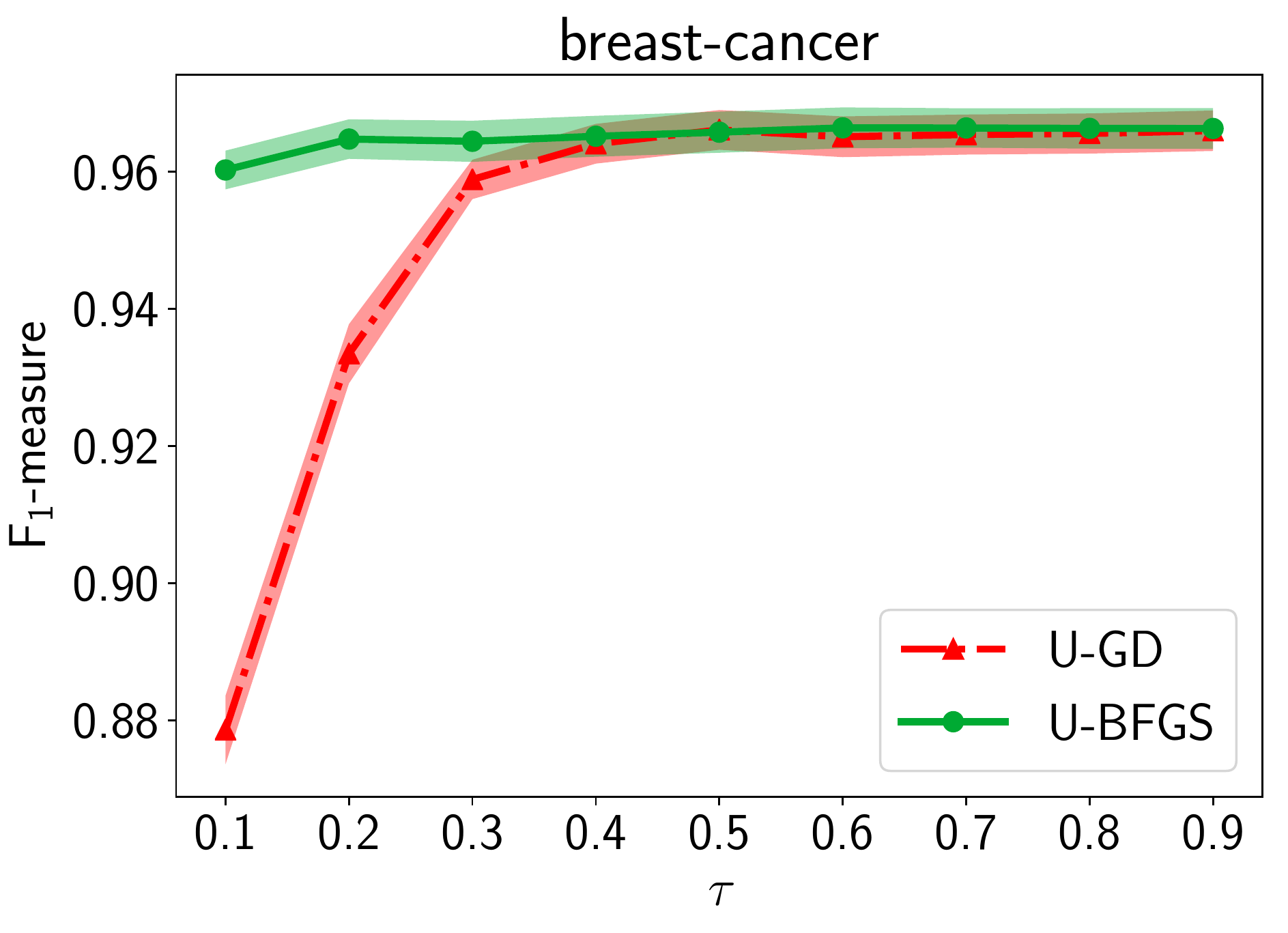}
    \end{minipage}
    \begin{minipage}{0.32\columnwidth}
        \includegraphics[width=\columnwidth]{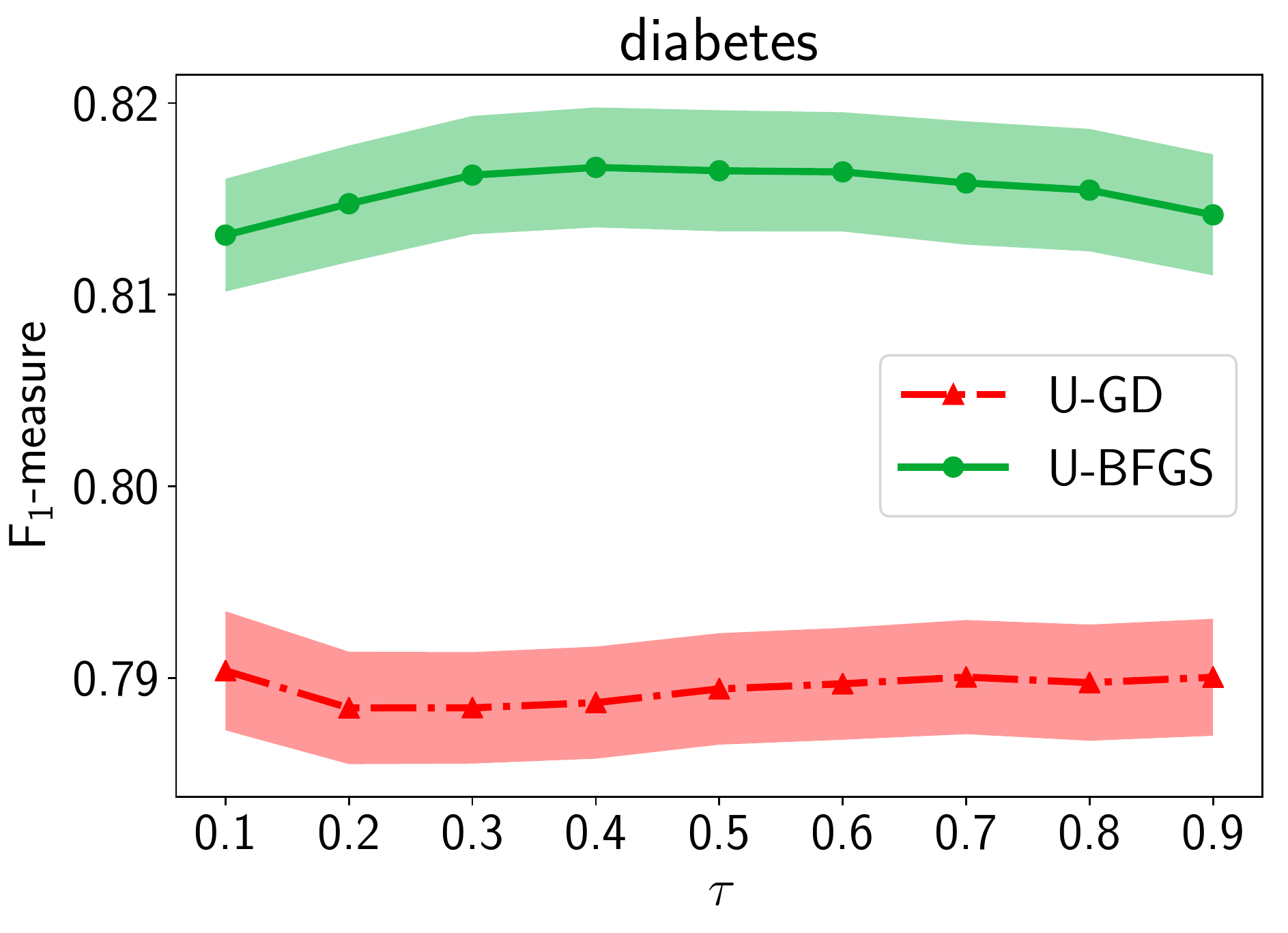}
    \end{minipage}
    \begin{minipage}{0.32\columnwidth}
        \includegraphics[width=\columnwidth]{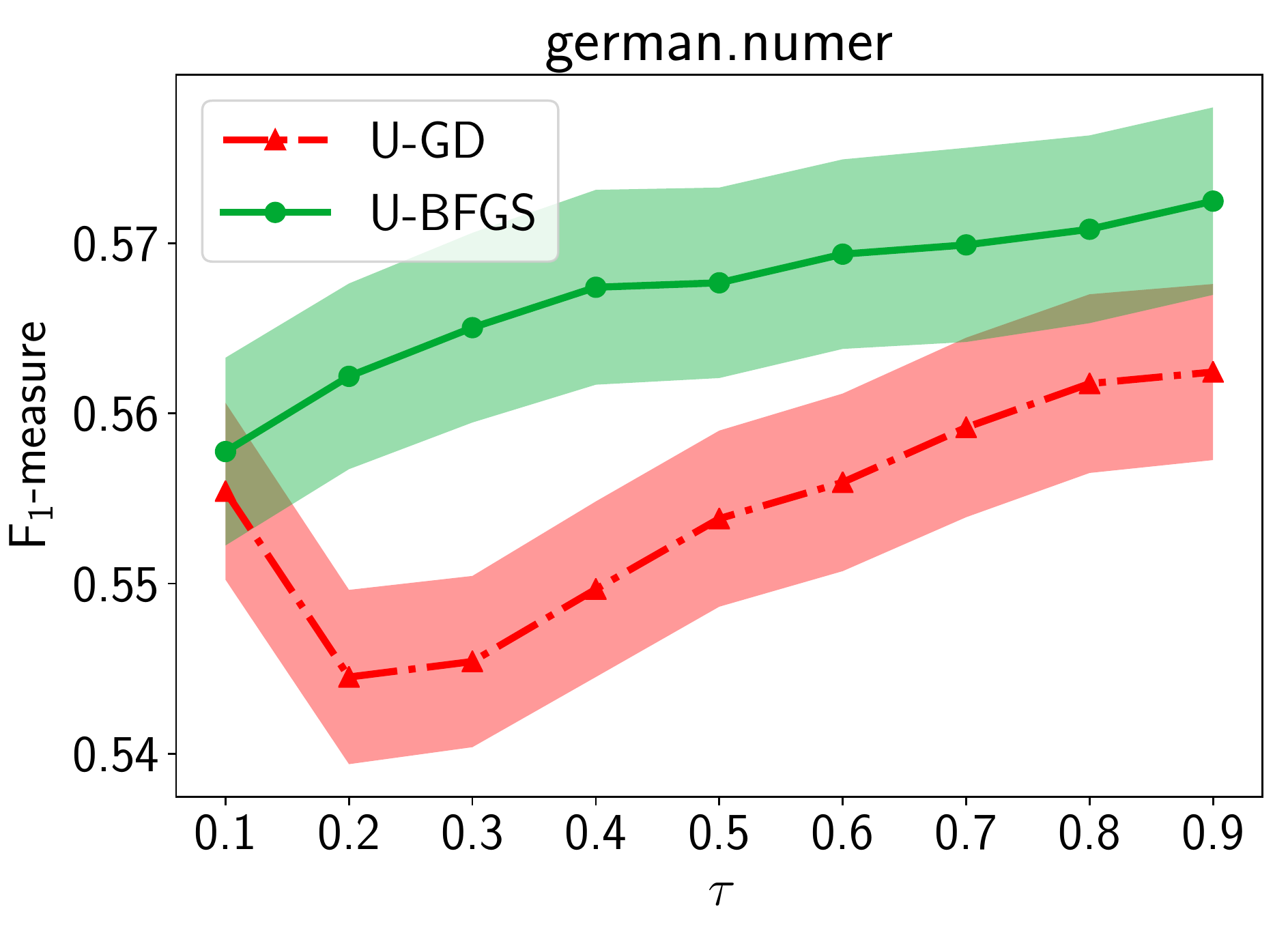}
    \end{minipage}
    \begin{minipage}{0.32\columnwidth}
        \includegraphics[width=\columnwidth]{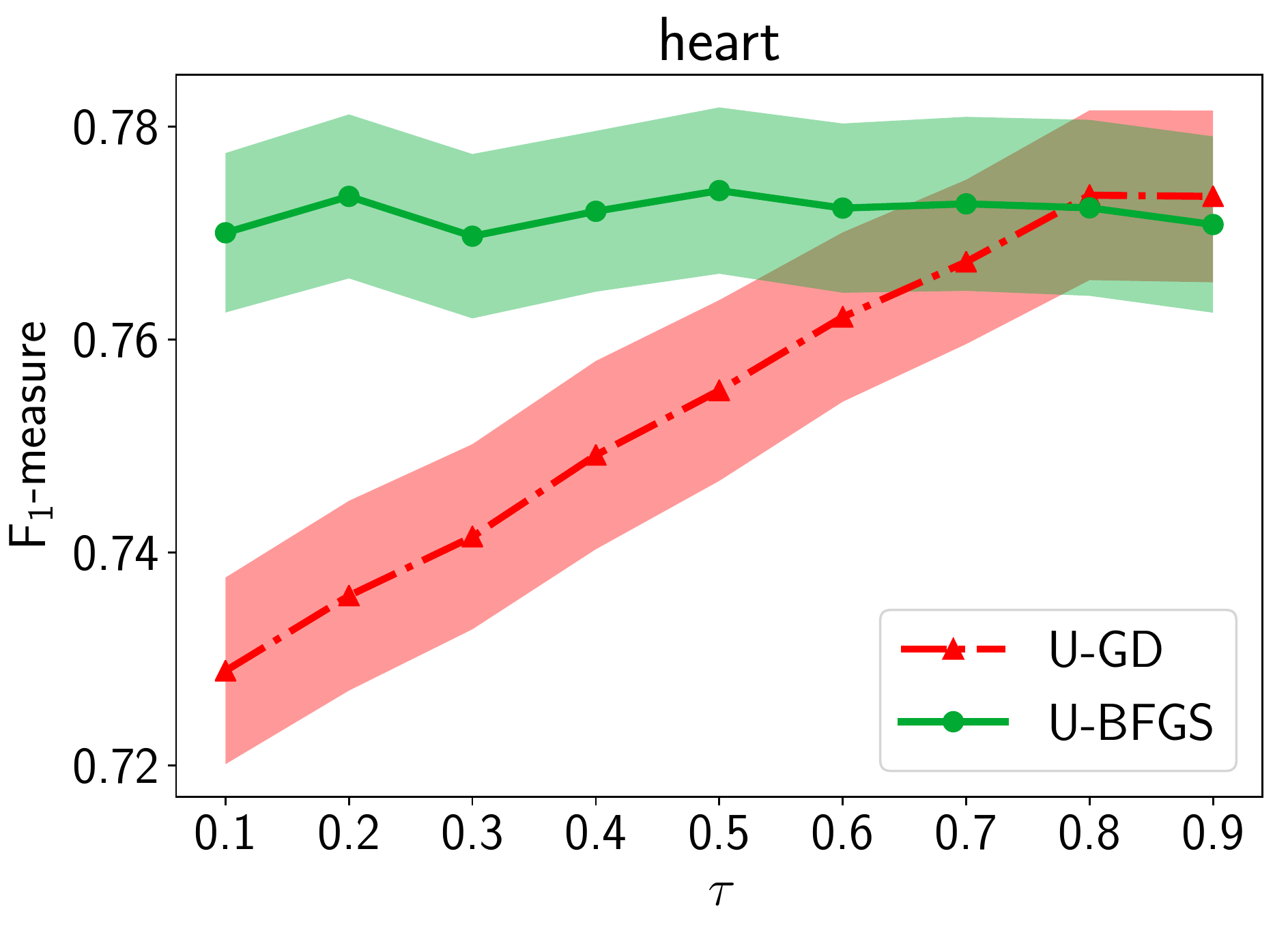}
    \end{minipage}
    \begin{minipage}{0.32\columnwidth}
        \includegraphics[width=\columnwidth]{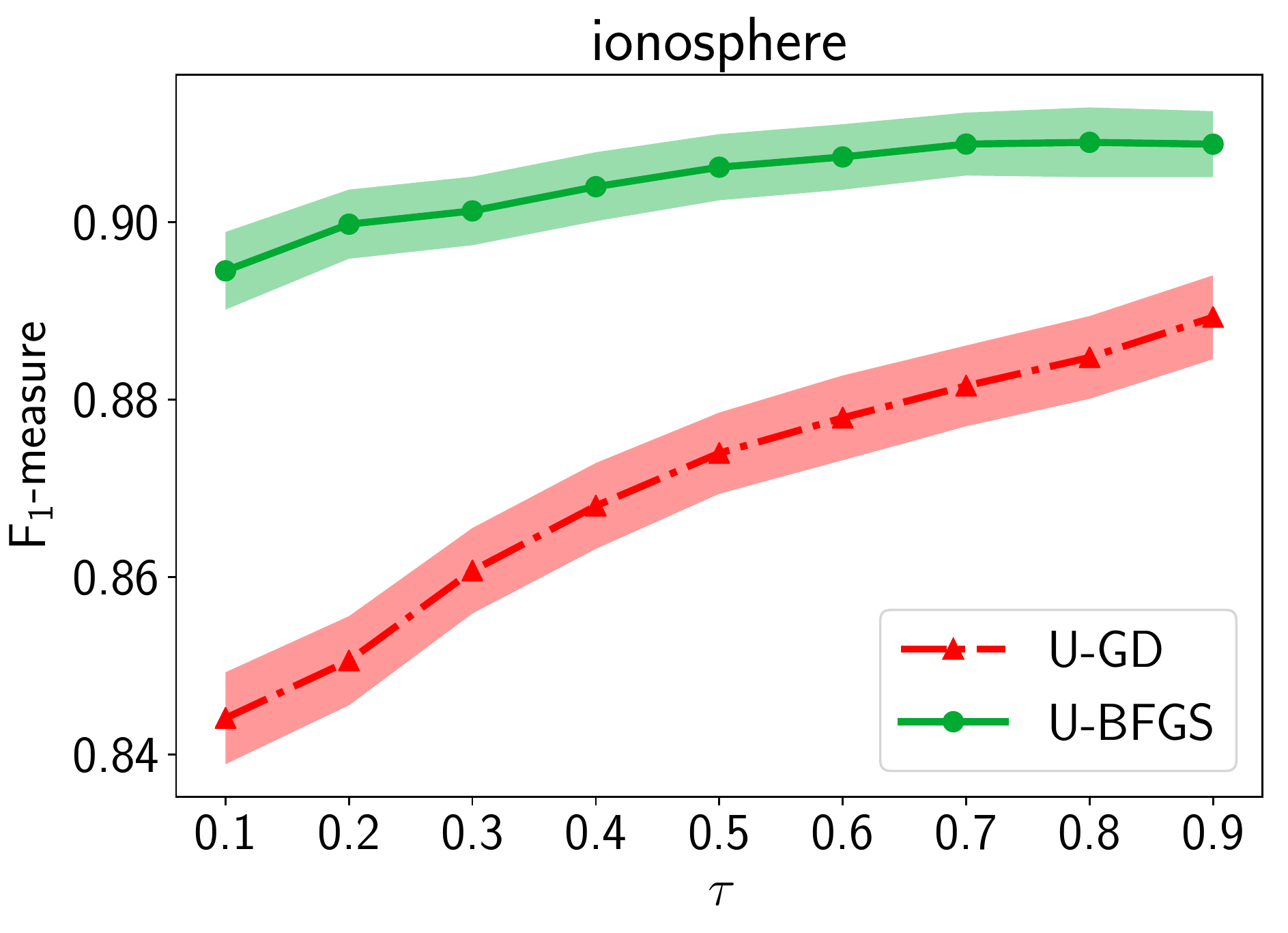}
    \end{minipage}
    \begin{minipage}{0.32\columnwidth}
        \includegraphics[width=\columnwidth]{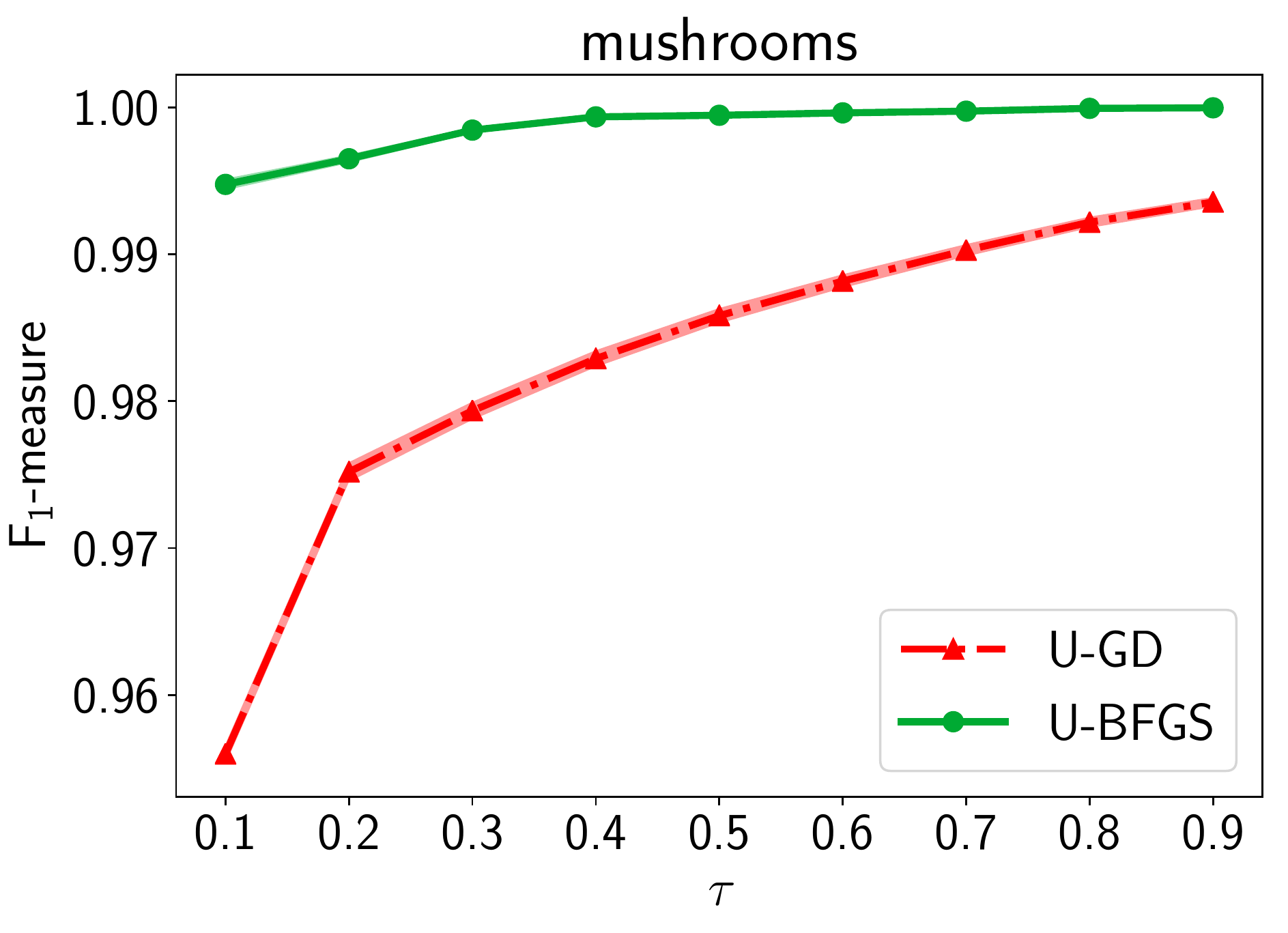}
    \end{minipage}
    \begin{minipage}{0.32\columnwidth}
        \includegraphics[width=\columnwidth]{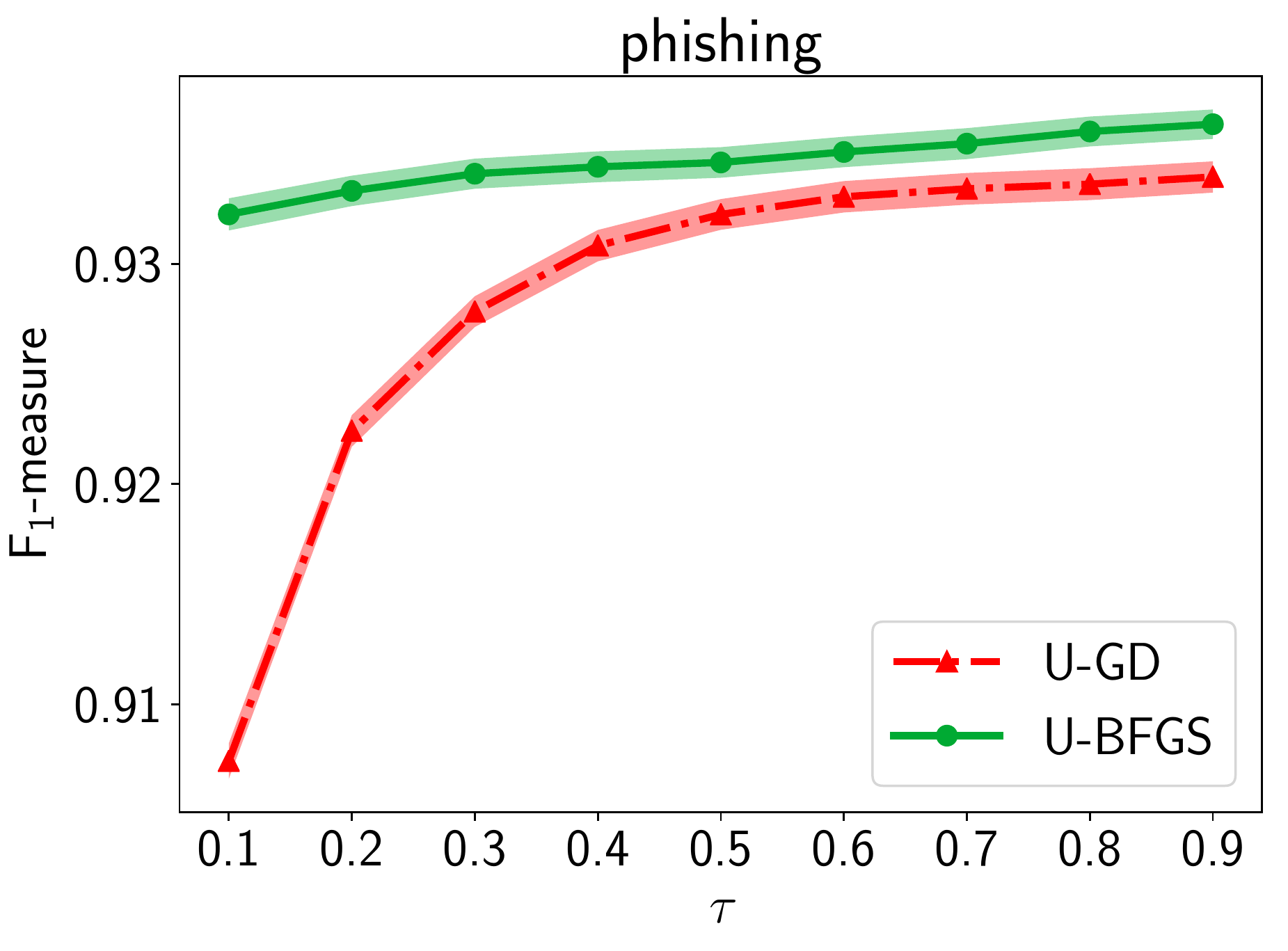}
    \end{minipage}
    \begin{minipage}{0.32\columnwidth}
        \includegraphics[width=\columnwidth]{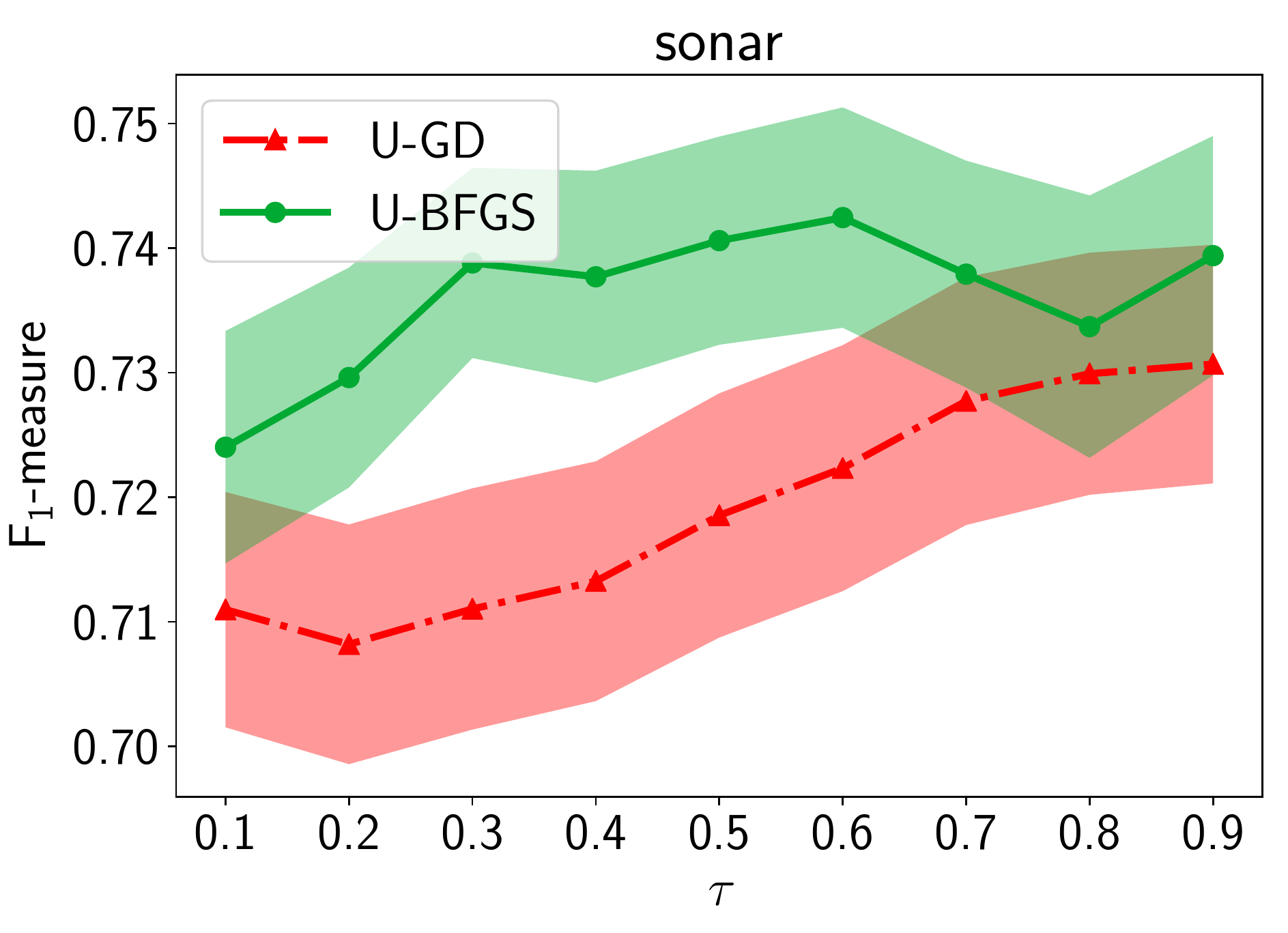}
    \end{minipage}
    \begin{minipage}{0.32\columnwidth}
        \includegraphics[width=\columnwidth]{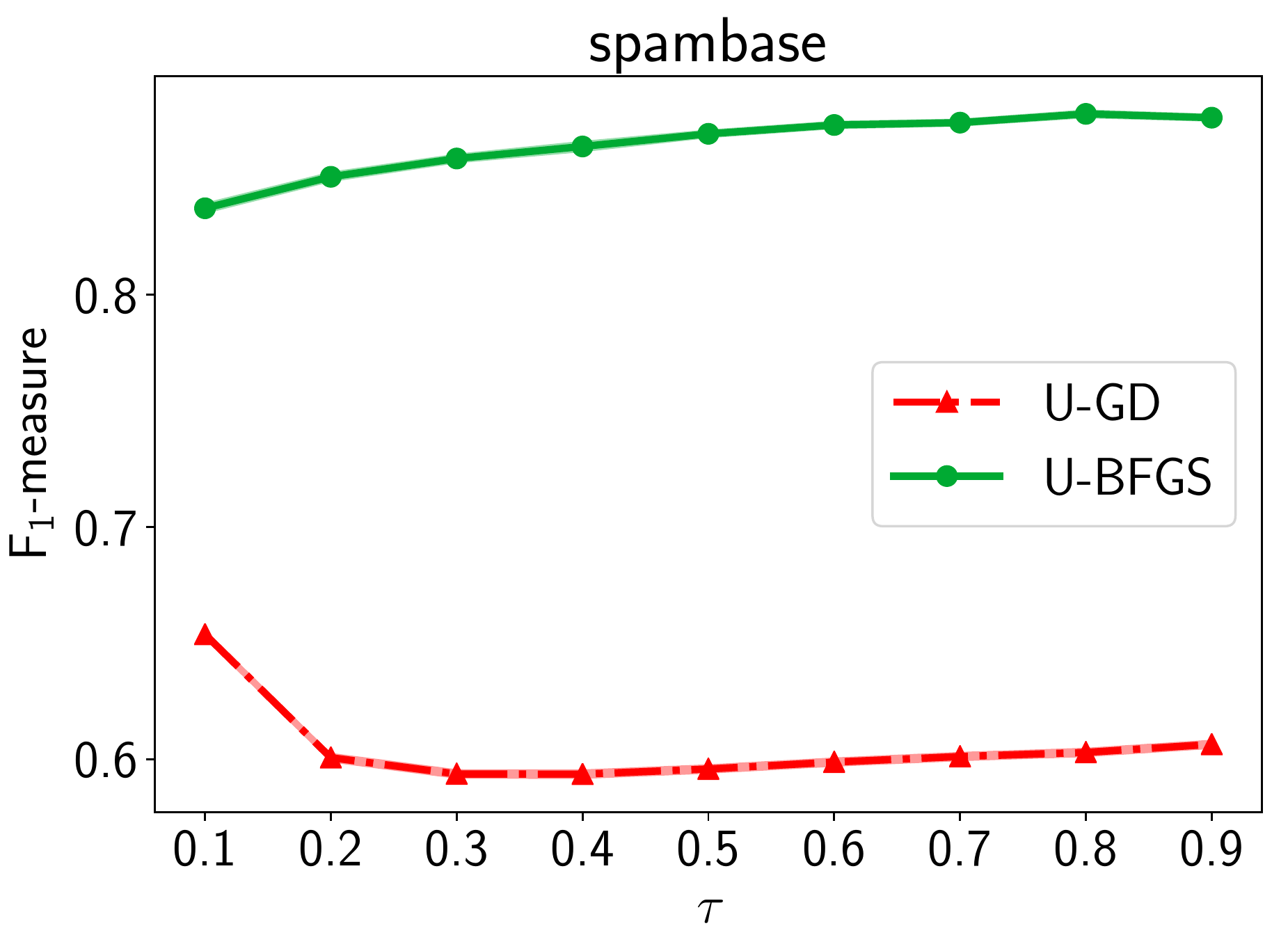}
    \end{minipage}
    \begin{minipage}{0.32\columnwidth}
        \includegraphics[width=\columnwidth]{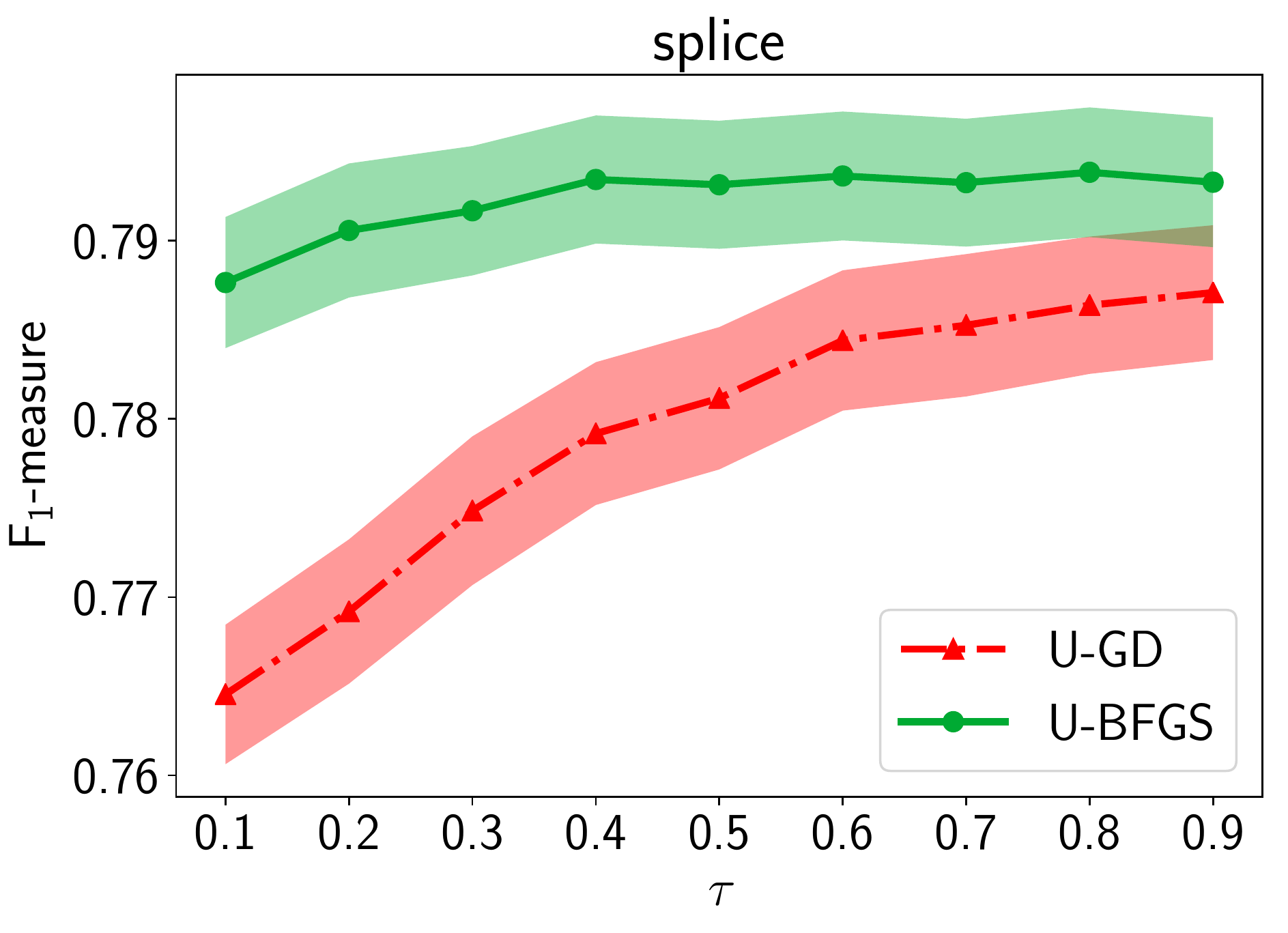}
    \end{minipage}
    \caption{
        The relationship of the test F${}_1$-measure (vertical axes) and the choices of $\tau$ (horizontal axes).
        Standard errors of 50 trials are shown as shaded areas.
    }
    \label{fig:supp:f1-sensitivity}
\end{figure}

\begin{figure}[h]
    \centering
    \begin{minipage}{0.32\columnwidth}
        \includegraphics[width=\columnwidth]{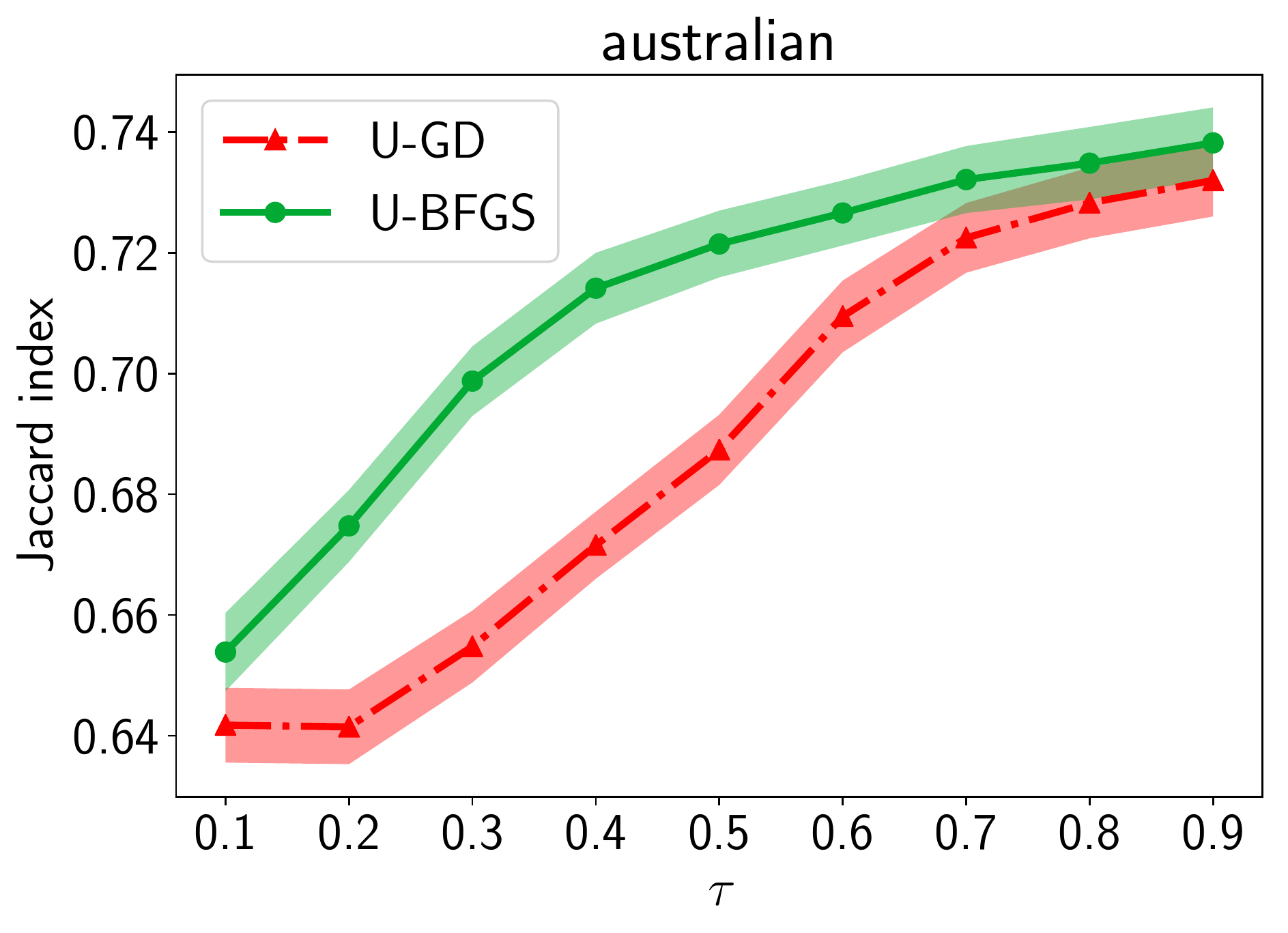}
    \end{minipage}
    \begin{minipage}{0.32\columnwidth}
        \includegraphics[width=\columnwidth]{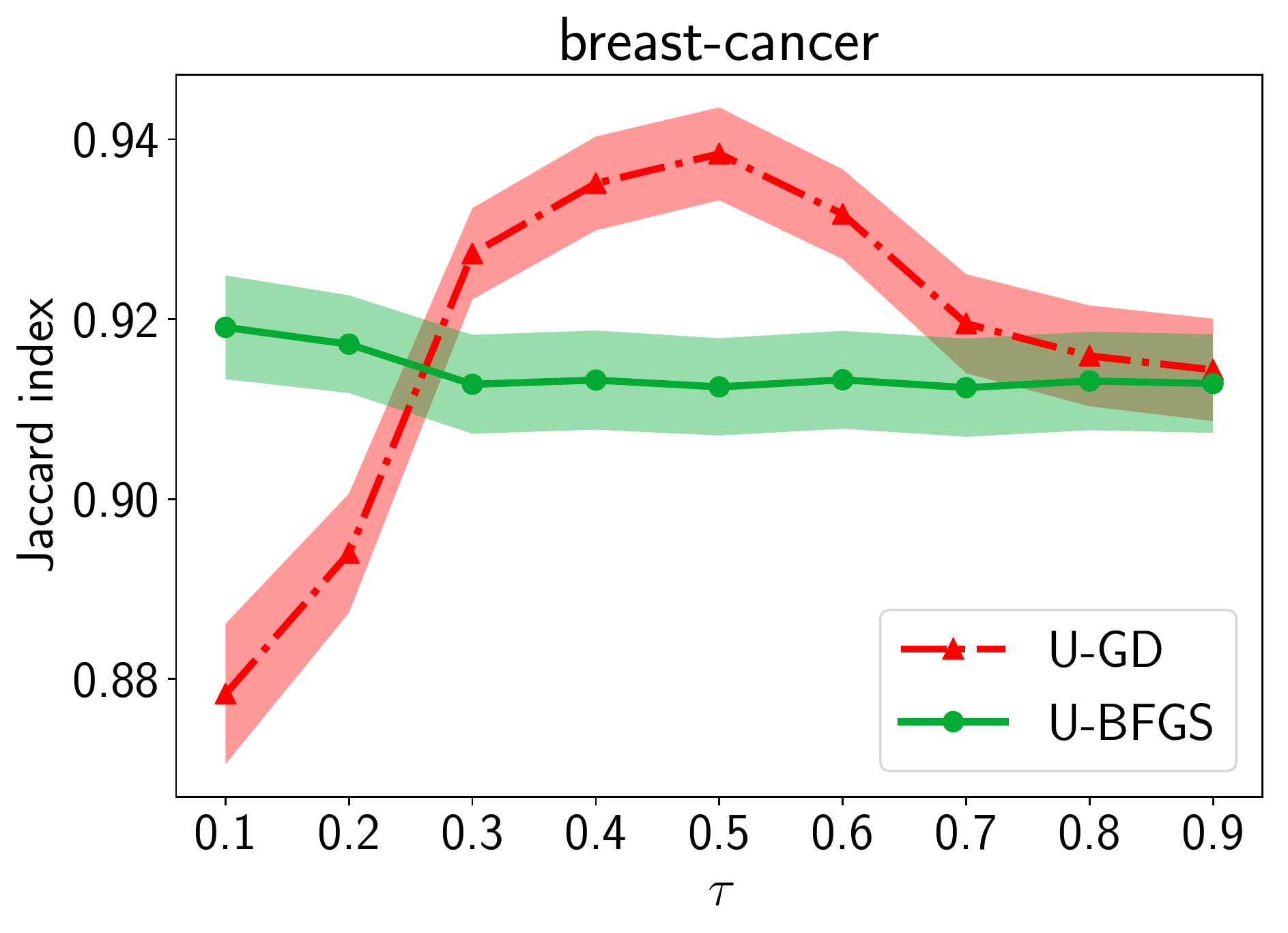}
    \end{minipage}
    \begin{minipage}{0.32\columnwidth}
        \includegraphics[width=\columnwidth]{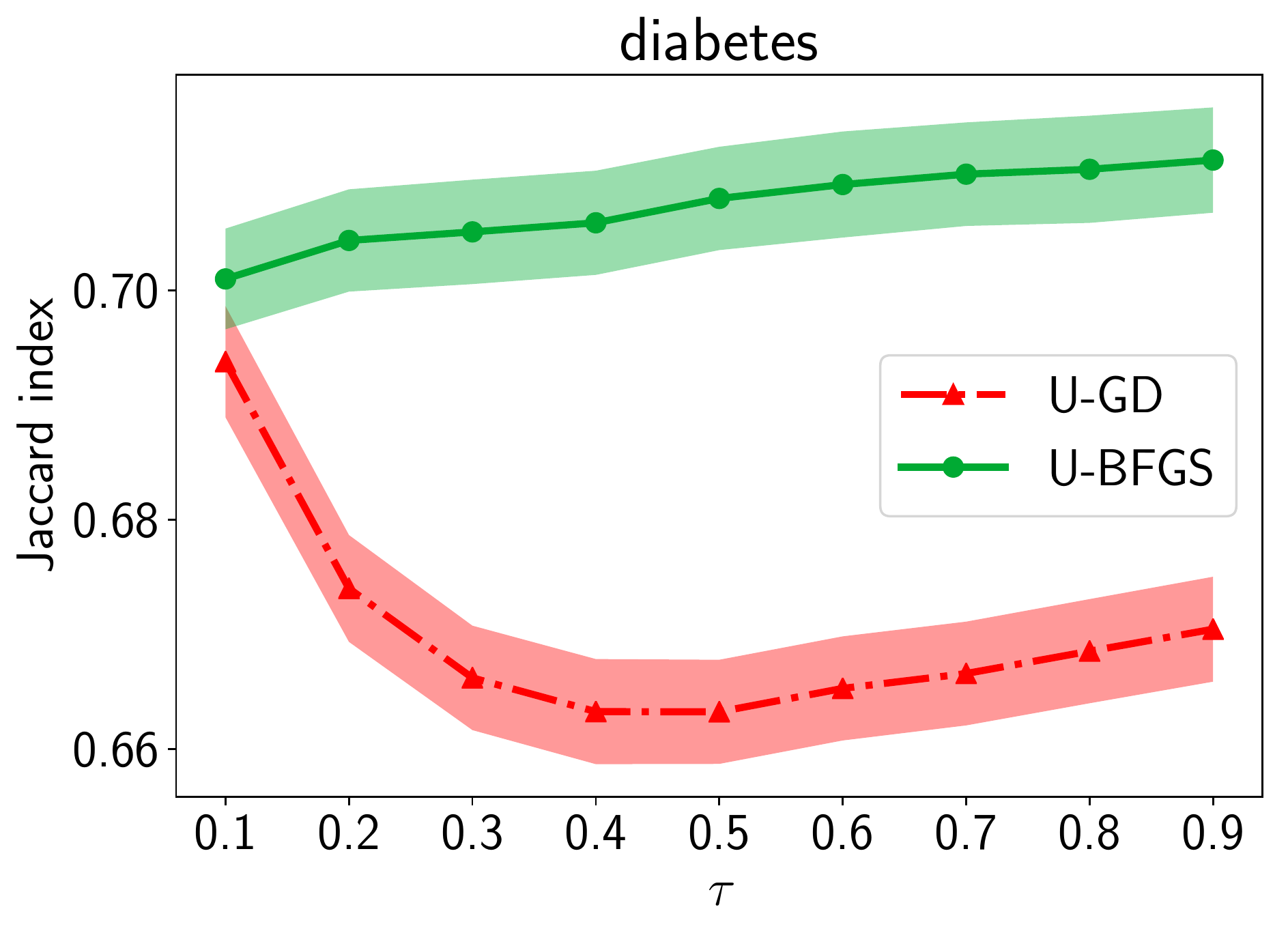}
    \end{minipage}
    \begin{minipage}{0.32\columnwidth}
        \includegraphics[width=\columnwidth]{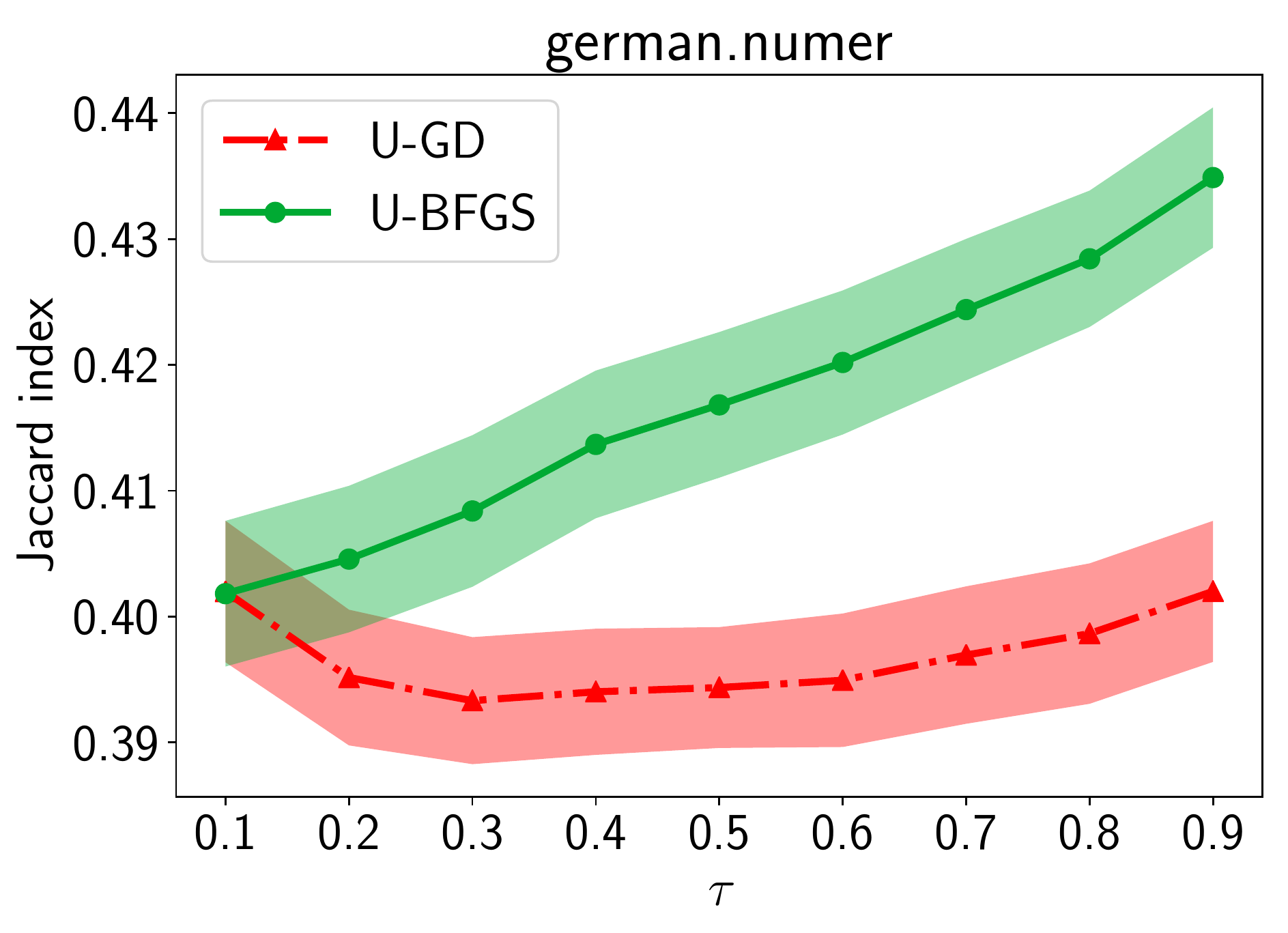}
    \end{minipage}
    \begin{minipage}{0.32\columnwidth}
        \includegraphics[width=\columnwidth]{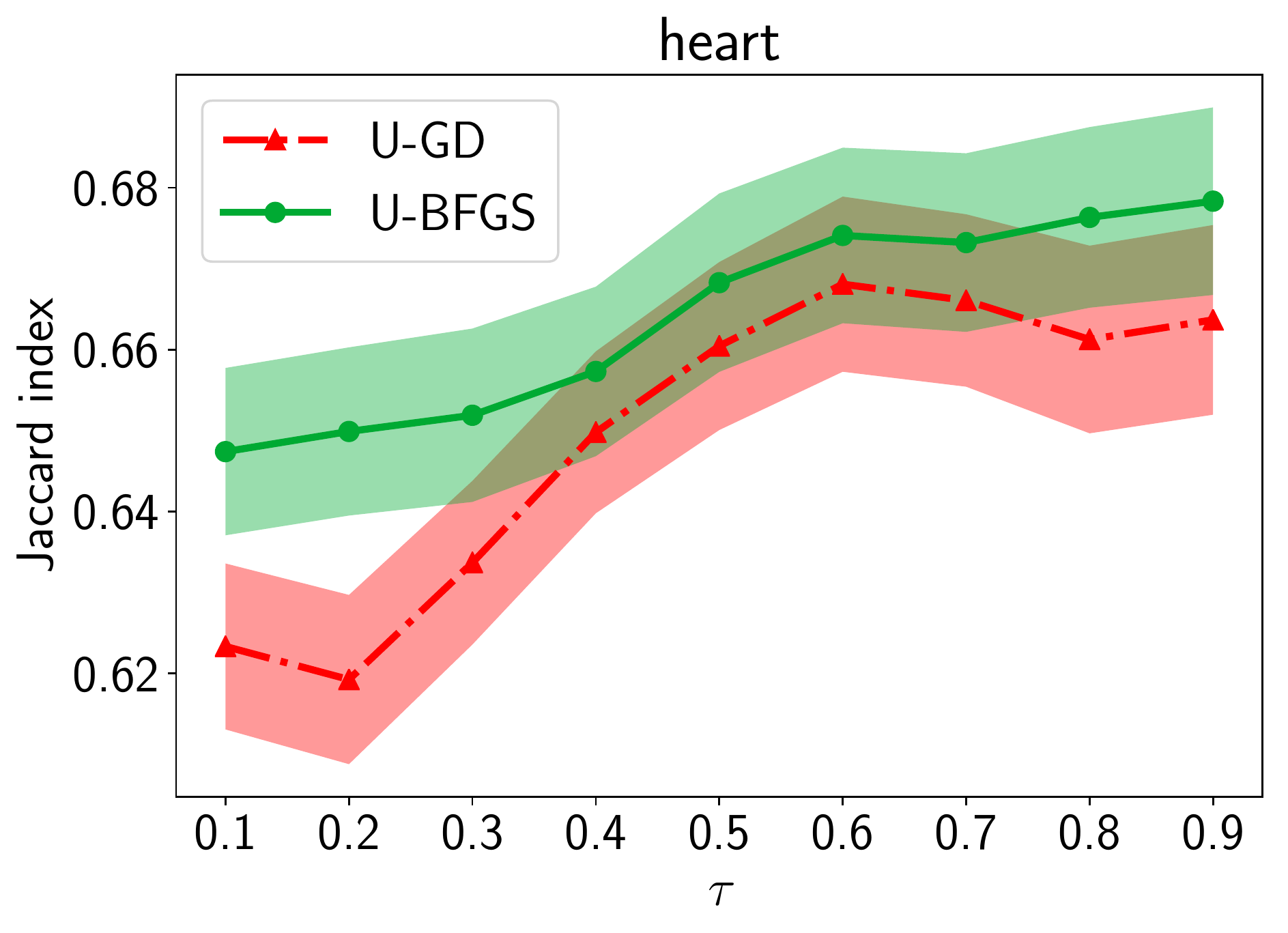}
    \end{minipage}
    \begin{minipage}{0.32\columnwidth}
        \includegraphics[width=\columnwidth]{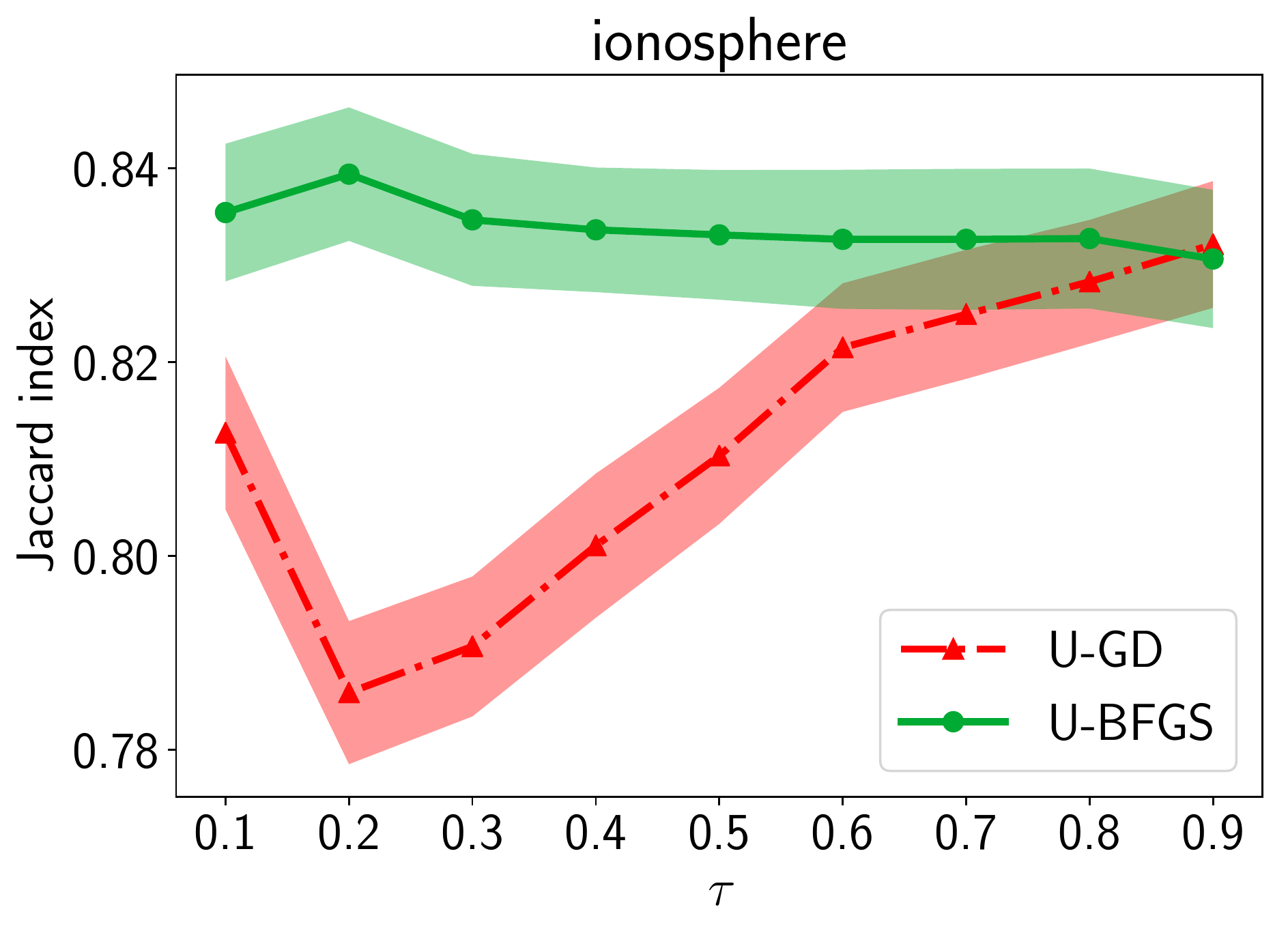}
    \end{minipage}
    \begin{minipage}{0.32\columnwidth}
        \includegraphics[width=\columnwidth]{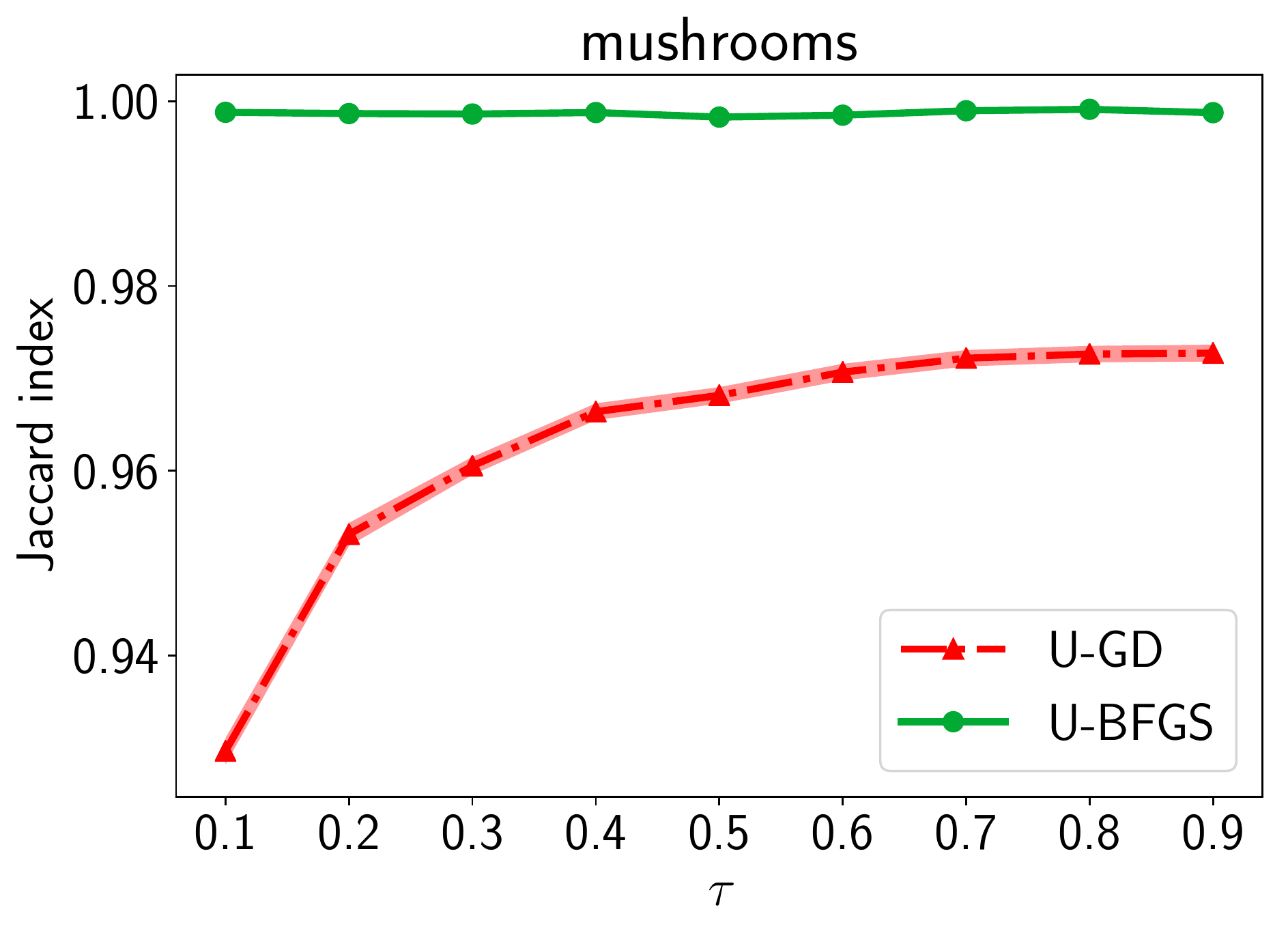}
    \end{minipage}
    \begin{minipage}{0.32\columnwidth}
        \includegraphics[width=\columnwidth]{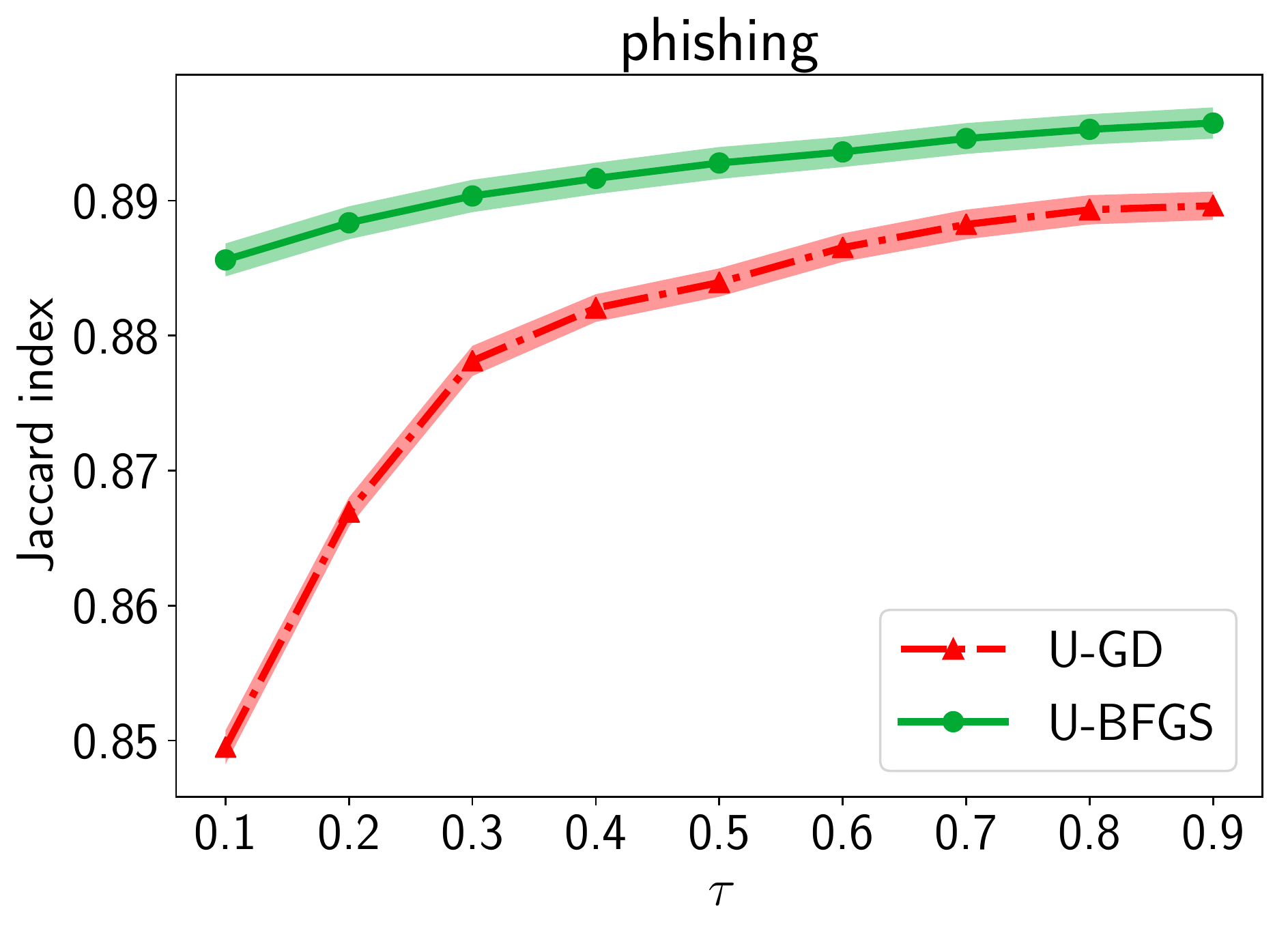}
    \end{minipage}
    \begin{minipage}{0.32\columnwidth}
        \includegraphics[width=\columnwidth]{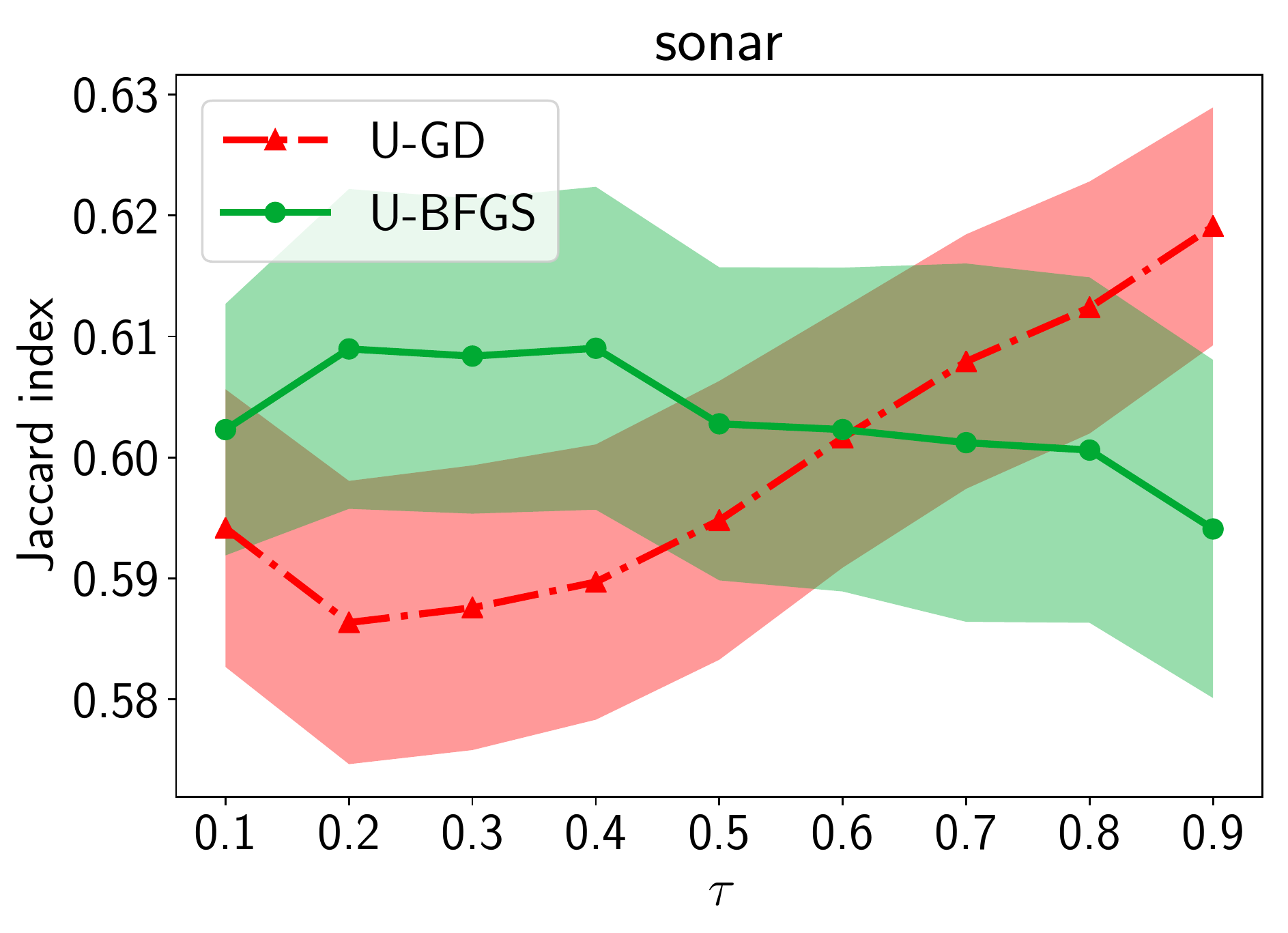}
    \end{minipage}
    \begin{minipage}{0.32\columnwidth}
        \includegraphics[width=\columnwidth]{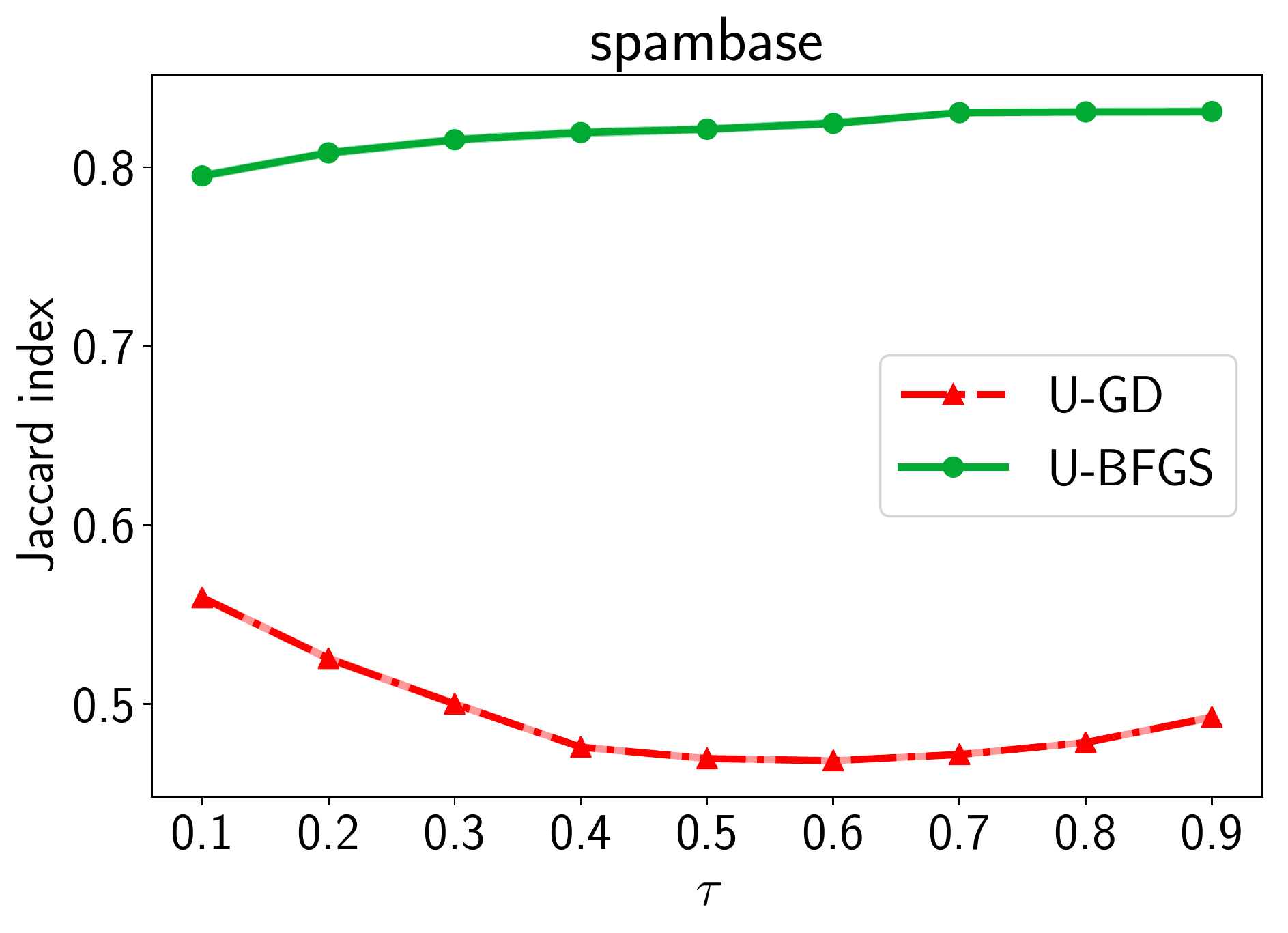}
    \end{minipage}
    \begin{minipage}{0.32\columnwidth}
        \includegraphics[width=\columnwidth]{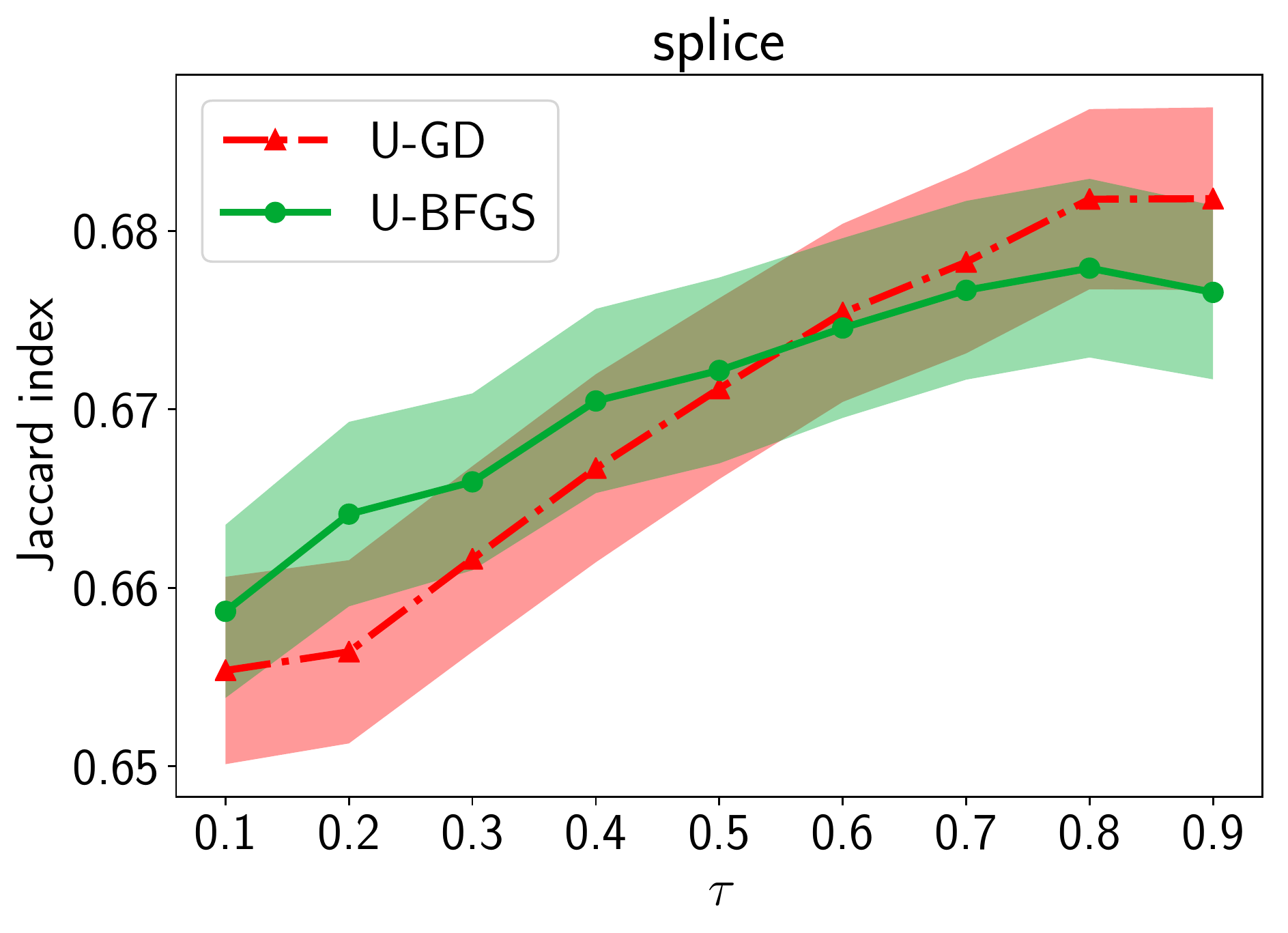}
    \end{minipage}
    \caption{
        The relationship of the test Jaccard (vertical axes) and the choices of $\tau$ (horizontal axes).
        Standard errors of 50 trials are shown as shaded areas.
    }
    \label{fig:supp:jac-sensitivity}
\end{figure}

\end{document}